%% file: main_kernel_alternative_proxy.tex
\documentclass[twoside]{article}

\usepackage[accepted]{aistats2025}
\input math_commands.tex

\usepackage{xcolor}
\usepackage{amsthm}
\usepackage{subfig}
\usepackage{hyperref}
\usepackage{url}
\usepackage{graphicx}
\usepackage{enumitem}

\newcommand{\dimitri}[1]{%
\ifmmode
\text{\textcolor{olive}{[Dim: #1]}}
\else
\textcolor{olive}{[Dim: #1]}
\fi
}

\newtheorem{theorem}{Theorem}[section]
\newtheorem{lemma}[theorem]{Lemma}
\newtheorem{definition}[theorem]{Definition}
\newtheorem{proposition}[theorem]{Proposition}
\newtheorem{corollary}[theorem]{Corollary}
\newtheorem{remark}[theorem]{Remark}

\newtheorem{assumption}[theorem]{Assumption}
\newtheorem{algorithm_state}[theorem]{Algorithm}
\newtheorem*{algorithm_state*}{Algorithm}
\allowdisplaybreaks
\setitemize{itemsep=2pt,topsep=0pt,parsep=0pt,partopsep=0pt,leftmargin=*}
\usepackage[round]{natbib}

\begin{document}

%

%

\twocolumn[

\aistatstitle{Density Ratio-based Proxy Causal Learning Without Density Ratios}


\aistatsauthor{ Bariscan Bozkurt \And Ben Deaner \And  Dimitri Meunier}

\aistatsaddress{ Gatsby Computational Neuroscience Unit\And  University College London\And Gatsby Computational Neuroscience Unit } 

\aistatsauthor{  Liyuan Xu \And Arthur Gretton}

\aistatsaddress{ Gatsby Computational Neuroscience Unit\And Gatsby Computational Neuroscience Unit } 

\runningauthor{Bariscan Bozkurt, Ben Deaner, Dimitri Meunier, Liyuan Xu, Arthur Gretton}
]

\begin{abstract}
We address the setting of Proxy Causal Learning (PCL), which has the goal of estimating causal effects from observed data in the presence of hidden confounding. Proxy methods accomplish this task using two proxy variables related to the latent confounder: a treatment proxy (related to the treatment) and an outcome proxy (related to the outcome). Two approaches have been proposed to perform causal effect estimation given proxy variables; however only one of these has found mainstream acceptance, since the other was understood to require density ratio estimation - a challenging task in high dimensions. In the present work, we propose a practical and effective implementation of the second approach, which bypasses explicit density ratio estimation and is suitable for continuous and high-dimensional treatments. We employ kernel ridge regression to derive estimators, resulting in simple closed-form solutions for dose-response and conditional dose-response curves, along with consistency guarantees. Our methods empirically demonstrate superior or comparable performance to existing frameworks on synthetic and real-world datasets.  \looseness=-1 
\end{abstract}

\section{INTRODUCTION}

\begin{figure}[ht!]
\vspace{.3in}
\centering
\includegraphics[trim = {0cm 0cm 0cm 0cm},clip,width=0.37\textwidth]{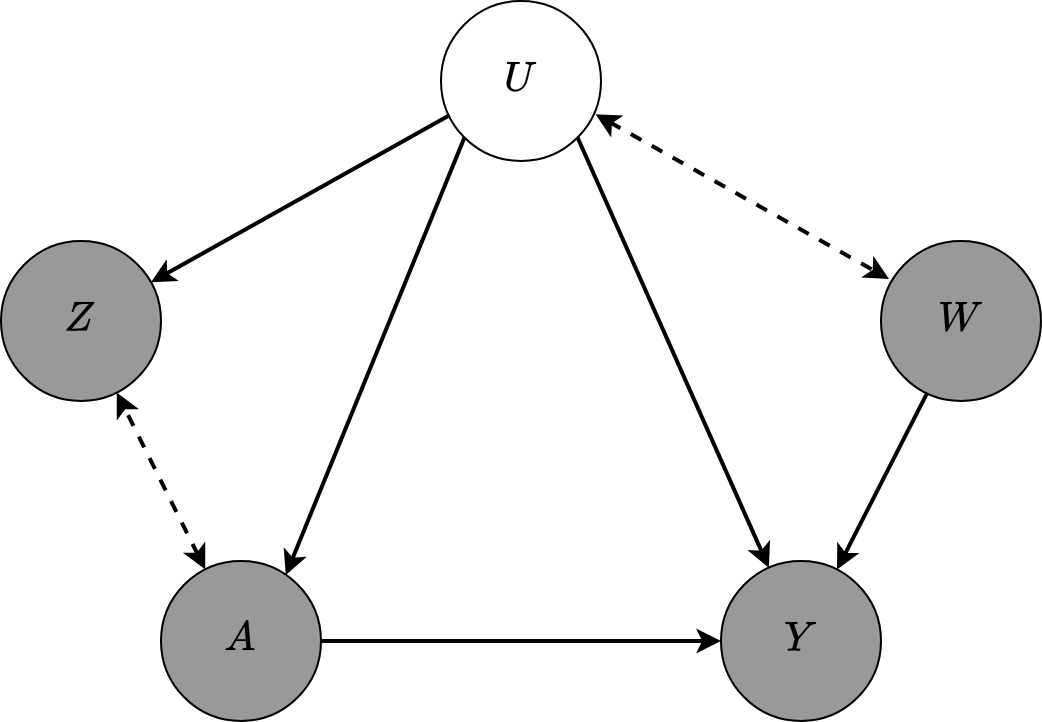}
\vspace{.3in}
\caption{An instance of a Directed Acyclic Graph (DAG) for the PCL setting, which satisfies the required Assumption (\ref{assumption:proxy}). In this graph, the gray circles denote the observed variables: $A$ denotes the treatment, $Y$ denotes the outcome, $Z$ denotes the treatment proxy, and $W$ denotes the outcome proxy. The white circle denotes the unobserved confounding variable $U$. 
Bi-directional dotted arrows indicate that either direction in the DAG is possible, or that both variables may share a common ancestor.}
\label{fig:ProxyCausalDAG}
\vspace{-0.5cm}
\end{figure}

Causal inference aims to measure the impact of interventions on real-world outcomes, a crucial task across various scientific disciplines.  Examples include assessing how changes in flight ticket prices will alter consumer demand \citep{Blundell_measuringPriceResponsiveness}, the consequences of grade retention on students' cognitive development \citep{Fruehwirth_GradeRetentions}, the impact of medical treatments on patient health \citep{The_effectiveness_of_right_heart_catheterization, CHOI20021173}, and evaluating policies such as Job Corps \citep{Schochet_JobCorps}. In this context, the intervention is referred to as a \emph{treatment}, which affects the \emph{outcome}. However, estimating the causal relationship between treatment and outcome is a challenging task due to confounding variables - factors that influence both the treatment and the outcome - potentially leading to spurious correlations.

One widely used assumption is that no unobserved confounding variable exists \citep{peaRob95}, which allows causal effect to be estimated via regression or backdoor adjustment \citep{Pearl_2009}. Although there are numerous methods relying on this assumption, e.g. \citep{Hill_BayesianNonparametric, pmlr_v48_johansson16, NEURIPS2018_a50abba8, RahulKernelCausalFunctions}, it is often restrictive since it requires measuring all the covariates that account for the confounding variables. Recently, a growing literature has explored a milder assumption: the availability of proxy variables to the latent confounding variables. \citet{Miao2018Identifying} demonstrated in the \emph{Proxy Causal Learning} (PCL) framework that two proxy variables - a treatment proxy (possibly directly causally linked to the treatment) and an outcome proxy (possibly directly causally linked to the outcome) - are sufficient for recovering the underlying causal relation by utilizing an \emph{outcome bridge function}, without the need to explicitly recover the confounder \citep[unlike][who explicitly recover the confounder in the discrete-valued categorical setting]{Kuroki2014Mesurement}. The corresponding causal graph is illustrated in Figure \ref{fig:ProxyCausalDAG}. An alternative line of research has proposed methods for causal effect estimation by leveraging a \emph{treatment bridge function} for discrete treatment \citep{semiparametricProximalCausalInference} and continuous treatment \citep{deaner2023proxycontrolspaneldata,wu2024doubly}, and our approach follows this setting. Specifically, \citet{semiparametricProximalCausalInference,deaner2023proxycontrolspaneldata} introduced a bridge function $\varphi_0(Z, a)$ that satisfies $\E[\varphi_0(Z, a) | W, A = a] = 1/p(A = a | W)$, where $Z$ denotes the treatment proxy, $A$ denotes the treatment, and $W$ denotes the outcome proxy, as illustrated in Figure \ref{fig:ProxyCausalDAG}. They showed that the dose-response can be identified through the expectation $\E[Y \varphi_0(Z, A) \mathbf{1}[A = a]]$. 
While the method of \citet{semiparametricProximalCausalInference} is limited to discrete treatments,
 \cite{wu2024doubly} address the case of continuous treatments by replacing the indicator function with a kernel function $K(A - a)$, yielding a dose-response estimator of the form $\E[Y \varphi_0(Z, a) K(A-a)]$; and by using a kernel density estimation {or conditional normalizing flows} to approximate $p(A = a | W)$. Kernel density estimates can converge slowly when the treatment and proxy variables are high-dimensional, however \citep[see e.g.][]{Wasserman06}. 
 
 In the present work, we propose a treatment proxy approach that eliminates the need for explicit density estimation. By leveraging a slightly different bridge function, we simplify the loss function {by decomposing it into terms that depend on distinct distributions, akin to the approach of \cite{Kanamori_LeastSquareImportance}}, allowing us to express all quantities of interest in terms of inner products in \emph{reproducing kernel Hilbert Spaces} (RKHS), removing the need for density ratio estimation. We further extend our approach to the conditional dose-response curve. In summary, our main contributions are as follows:
\begin{itemize}
    \item We propose a novel family of kernel-based algorithms to estimate causal functions in the PCL setting, using a treatment bridge function;
    \item Our RKHS formulation  allows us to provide closed-form expressions for causal effect estimation, including for continuous and high-dimensional treatments;
    \item We prove the consistency of proposed estimators;
    \item We demonstrate that our treatment proxy approach matches or outperforms existing PCL algorithms.
\end{itemize}

The paper is organized as follows: Section (\ref{sec:RelatedWorks}) reviews related work, Section (\ref{sec:AlternativeProxyIdentification}) presents the problem definition and identification results, and Section (\ref{sec:AlternativeProxyAlgorithms}) outlines our estimation algorithms. Consistency results are in Section (\ref{sec:Consistency}), followed by experiments in Section (\ref{sec:NumericalExperiments}).

\section{RELATED WORK}
\label{sec:RelatedWorks}

Proxy causal learning  \citep{Miao2018Identifying,deaner2023proxycontrolspaneldata}, building on the seminal work of \citet{Kuroki2014Mesurement}, tackles the problem of unmeasured confounding  by using two types of proxy variables, namely \emph{the treatment proxy} and \emph{the outcome proxy}. Given these proxies, there are two approaches to obtain the causal effect. One approach is to estimate an \emph{outcome bridge function}. This is a function of the outcome proxy, and we can obtain the causal effect by integrating the outcome bridge over the (observed) outcome proxy. \citet{Miao2018Identifying} show that the outcome bridge function is the solution of an inverse problem known as a Fredholm integral equation of the first kind \citep{linearIntegralEquationskress2013}. Although it is ill-posed in general, a number of methods have been proposed to solve it by limiting the functional space, including sieve bases \citep{deaner2023proxycontrolspaneldata}, RKHSs \citep{Mastouri2021ProximalCL,singh2023kernelmethodsunobservedconfounding} and neural networks \citep{xu2021deep,kompa2022deep,Kallus2021Causal}. The other approach, known as \emph{alternative PCL}, considers a \emph{treatment bridge function}. This is a function of the treatment proxy, which can be used as the adjustment weights in estimating causal effects, similar to the inverse propensity score  \citep{rosRub83}. Such a function can be obtained by solving another Fredholm integral equation \citep{deaner2023proxycontrolspaneldata,semiparametricProximalCausalInference}, but it is more challenging since it involves the conditional density function.  Recently, \citet{wu2024doubly} proposed a ``plug-in'' approach, which explicitly performs density estimation in obtaining a treatment bridge. 
However, conditional density estimation is costly and suffers from slow convergence when the treatment is high-dimensional.
Instead, we propose to bypass the need for density estimation using the conditional kernel mean embedding \citep{HilbertSpaceEmbeddingforConditionalDistributions,grunewalder2012conditional,park2020measure,klebanov2020rigorous,li2022optimal}.

Proxy methods are also used in domain adaptation under distribution shifts \citep{pmlr-v206-alabdulmohsin23a, pmlr-v238-tsai24b}, where the underlying DAG resembles ours. Domain adaptation focuses on transferring models across domains with shifting unobserved confounders (e.g., patients in different hospitals). \cite{pmlr-v238-tsai24b} use kernel-based outcome-bridge function similar to \citep{Mastouri2021ProximalCL}, whereas we employ treatment-bridge functions for causal effect estimations under unmeasured confounding. While both they and we use kernel methods, their objectives and estimands differ, highlighting distinct but complementary approaches. \looseness=-1

\section{PROBLEM SETTING AND IDENTIFICATION}
\label{sec:AlternativeProxyIdentification}
\subsection{Problem Setting for Treatment Effects}
In this section, we establish the problem setting for learning causal effects, which are statements about the counterfactual outcomes that arise from hypothetical interventions. Consider the treatment $A \in \gA$ and its associated outcome $Y \in \R$ that we observe. We assume the existence of an unobserved confounding variable $U \in \gU$ that affects both $A$ and $Y$. Our aim is to estimate the causal effects presented in the following definition.

\begin{definition} The treatment effects are: 
    \vspace{-0.3cm}
    \begin{enumerate}
        \item[i-)] {Dose-response:} $f_{\text{ATE}}({a}) = \E[\E[Y | A = a, U]]$ quantifies the counterfactual mean outcome across the entire population given the intervention where everyone receives the treatment $a$. ATE signifies that its semiparametric analogue is the average treatment effect. \looseness=-1
        \vspace{-0.2cm}
        \item[ii-)] {Conditional dose-response:} $f_{\text{ATT}}(a, a') = \E[\E[Y | A = a, U] | A = a']$ quantifies the counterfactual mean outcome for individuals who actually received treatment $A = a'$ given the intervention where they receive treatment $a$. ATT signifies that its semiparametric analogue is the average treatment effect on the treated. \looseness=-1
    \end{enumerate}
    \label{def:structuralFunctions}
\end{definition}
The primary challenge in estimating these functions is that the confounder $U$ is not directly observable. To address this issue, we assume access to two proxy variables: $Z$, the treatment proxy, and $W$, the outcome proxy. We list the assumptions below that will be used throughout the development of our method. The following conditional independence assumption is implied by the graphical model in Figure (\ref{fig:ProxyCausalDAG}). \looseness=-1

\begin{assumption}\label{assumption:proxy}
    We assume the following conditional independence statements: i-) $Y \perp Z  | U, A$ (Conditional Independence for $Y$), ii-) $ W \perp Z  | U, A$ and $W \perp A | U$ (Conditional Independence for $W$).
    \label{assum:ProxyCausalAssumptions1}
\end{assumption}
 We also make the following completeness assumption.
 \begin{assumption}[Completeness]
    For any square integrable function $\ell : \gU \rightarrow \R$, for all $a \in \gA$: $\E[\ell(U) | W, A = a] = 0$  $p(W) - $almost everywhere (a.e.) if and only if $\ell(U) = 0$ $p(U) - $a.e.\looseness=-1
\label{assum:AlternativeProxyAssumptionCompleteness1}
\end{assumption}

Both Assumptions (\ref{assum:ProxyCausalAssumptions1}) and (\ref{assum:AlternativeProxyAssumptionCompleteness1}) are used in the alternative PCL frameworks \citep{deaner2023proxycontrolspaneldata,semiparametricProximalCausalInference, wu2024doubly}. The completeness condition ensures that the proxy variable $W$ has sufficient variability relative to the unobserved confounder $U$, thereby making it possible to identify the treatment effects.

\subsection{Identification of the Structural Functions}
In this section, we establish the identifiability  of the structural functions presented in Definition (\ref{def:structuralFunctions}). 
To obtain the ATE function given in Definition (\ref{def:structuralFunctions}-i), we consider the \emph{bridge function} $\varphi_0^{\text{ATE}}(z, a)$, defined as a solution of the functional equation
\begin{align}
    \E[\varphi_0^{\text{ATE}}(Z, a) | W, A = a] = \frac{p(W) p(a)}{p(W, a)}.
\label{eq:AlternativeProxyATEBridgeFunction}
\end{align}
The densities $p(W)$, $p(A)$, and $p(W, A)$ denote the marginal distributions of $W$ and $A$, and the joint distribution of $(W,A)$, respectively. Similar to ATE, we consider a bridge function $\varphi_0^{\text{ATT}}(z, a, a')$ to identify $f_{ATT}$ that satisfies the functional equation
 \begin{align}
    \E[\varphi_0^{\text{ATT}}(Z, a, a') | W, A = a] = \frac{p(W, a') p(a)}{p(W, a) p(a')}.\label{eq:AlternativeProxyATTBridgeFunction}
\end{align}
In the following theorem, we establish the identifiability of the structural functions $f_{\text{ATE}}$ and $f_{\text{ATT}}$.
\begin{theorem}
Let Assumptions (\ref{assum:ProxyCausalAssumptions1}) and (\ref{assum:AlternativeProxyAssumptionCompleteness1}) hold. Furthermore, suppose that there exist square integrable functions $\varphi_0^{\text{ATE}}$ and $\varphi_0^{\text{ATT}}$ that satisfy Equations (\ref{eq:AlternativeProxyATEBridgeFunction}) and (\ref{eq:AlternativeProxyATTBridgeFunction}), respectively. Then,
\begin{enumerate}
    \item The dose-response curve is given by $f_{\text{ATE}}(a) = \E[Y \varphi_0^{\text{ATE}}(Z, a) | A = a]$.\looseness=-1
    \item The conditional dose-response curve is given by $f_{\text{ATT}}(a, a') = \E[Y \varphi_0^{\text{ATT}}(Z, a, a') | A = a]$.
\end{enumerate}
\label{thm:AlternativeProxyATEandATTIdentification}
\end{theorem}
Theorem (\ref{thm:AlternativeProxyATEandATTIdentification}) is proved in the supplementary material S.M. (Section \ref{sec:Identification_Appendix}). The extension to settings with additional observable confounders is in S.M. (Sec. \ref{sec:KernelAlternativeProxyWithAdditionalCovariates}).
\begin{remark}
Our ATE identification result differs from previous works \citep{wu2024doubly, semiparametricProximalCausalInference}. Specifically, \citep{wu2024doubly} extends \citep{semiparametricProximalCausalInference} by replacing the indicator function with a kernel for continuous treatments, identifying ATE via expectations over $p(Y, Z, A)$. By contrast, our approach uses conditional expectations over $p(Y, Z | A = a)$, enabled by a modified bridge function definition with an additional $p(A = a)$ term. This distinction highlights the novelty of our approach to ATE identification. Furthermore, unlike \citep{wu2024doubly}, our approach does not require density ratio estimation for bridge function estimation, as outlined in Section (\ref{sec:AlternativeProxyAlgorithms}).
\end{remark}
\begin{remark}
For the simpler case where $\{\gA, \gY, \gW, \gZ, \gU\}$ are discrete, we present the identification result in Theorem (\ref{thm:AlternativeProxyATEIdentificationDiscrete}). Specifically, in this setting, the dose-response can be identified through matrix-vector multiplication of probability matrices, which can be estimated from the variables $(A, Y, W, Z)$, allowing for the estimation of the dose-response curve. \looseness=-1
\end{remark}

In order to ensure that solutions exist for both ATE and ATT bridge functions, we require the following completeness assumption:
\begin{assumption}
    For any square integrable function $\ell: \gU \rightarrow \R$, for all $a \in \gA$: $\E[\ell(U) | Z, A = a] = 0$ $p(Z)-$a.e. if and only if $\ell(U) = 0 \quad p(U)$-a.e.
\label{assum:ExistenceCompletenessAssumption}
\end{assumption}
The assumption above, along with mild regularity and integrability conditions, suffices for the existence of solutions to Equations (\ref{eq:AlternativeProxyATEBridgeFunction}) and (\ref{eq:AlternativeProxyATTBridgeFunction}). Further discussions about the existence are provided in S.M. (Sec. \ref{sec:existenceOfBridgeFunction}).\looseness=-1

\section{METHODS}
\label{sec:AlternativeProxyAlgorithms}
With the identification results of ATE and ATT functions in hand, we are prepared to develop algorithms to estimate these structural functions. To achieve this, we must solve Equations (\ref{eq:AlternativeProxyATEBridgeFunction}) and (\ref{eq:AlternativeProxyATTBridgeFunction}). We assume that the bridge functions reside within RKHSs. We then use a two-stage regression approach to solve for the bridge functions. Ultimately, we derive closed-form solutions to estimate the structural functions.

\subsection{Reproducing Kernel Hilbert Space}
Consider any space $\gF \in \{\gA, \gW, \gZ\}$. We denote the positive semi-definite kernel on $\gF$ as $k_\gF: \gF \times \gF \rightarrow \R$. The corresponding canonical feature map is  $\phi_\gF(f)$, where $\phi_\gF(f) = k_\gF(f, \cdot) \in \gH_\gF$, and $\gH_\gF$ refers to the RKHS of real-valued functions defined on  $\gF$. The inner product and the norm of the RKHS are denoted by $\langle .,. \rangle_{\gH_\gF}$ and $\|.\|_{\gH_\gF}$, respectively. When it is clear from the context, we drop the subscript $\gH_\gF$ from the inner product notation. For notational convenience, we use $\gH_{\gF\gG}$ to denote the tensor product space $\gH_\gF \otimes \gH_\gG$ that is isometrically isomorphic to $S_2(\gH_\gG, \gH_\gF)$ the Hilbert space of Hilbert-Schmidt operators from $\gH_\gG$ to $\gH_\gF$ \citep{aubin2011applied}. Correspondingly, $\phi_{\gF\gG}(f, g)$ denotes the tensor product feature map $\phi_\gF(f) \otimes \phi_\gG(g)$. Given any distribution $p(F)$ on $\gF$ and a kernel $k_\gF$ for which $\E[k_\gF(F, F)] < \infty$, $\mu_F = \int \phi_\gF(f) p(f) d f \in \gH_\gF$ is known as the kernel mean embedding of $p$ \citep{HilbertSpaceEmbeddingforDistributions,gretton2013introduction}. 
Similarly, for a conditional distribution $p(F | g)$ for each $g \in \gG$, the operator $ \mu_{F | G}(g) = \int \phi_\gF(f) p(f | g) d f \in \gH_\gF$ is called the conditional mean embedding (CME) of $p(F| g)$ \citep{HilbertSpaceEmbeddingforConditionalDistributions, grunewalder2012conditional,park2020measure, klebanov2020rigorous,li2022optimal}. 
\looseness=-1

\subsection{Algorithms to Estimate Causal Functions} \label{sec:main_algo}
\subsubsection{Dose-Response Estimation} 
To estimate the dose-response curve, we first approximate the bridge function $\varphi_0^{\text{ATE}}$, which is defined as the solution of Equation (\ref{eq:AlternativeProxyATEBridgeFunction}).  Let $r(W, A)$ denote $p(W) p(A) / p(W, A)$. Our objective is to find the optimal solution to the least-squares loss $\E[(r(W, A) - \E[\varphi(Z, A) | W, A])^2]$. Minimizing this loss is challenging for two reasons: (i) it requires knowledge of the density ratios; 
(ii) it involves a conditional expectation. 
For the first challenge, one could perform density ratio estimation, but this is particularly difficult in high-dimensional settings. We avoid the need for density ratio estimation by simplifying the least-squares objective as follows:\looseness=-1
\begin{align}
    &\E\big[\big(r(W, A) - \E[{\varphi}(Z, A) | W, A]\big)^2\big]\nonumber\\
    &=\E\big[ \E[{\varphi}(Z, A) | W, A]^2 \big]\nonumber \\&-\int \frac{2p(w) p(a)}{p(w, a)} \E[{\varphi}(Z, a) | w, a] p(w, a) d w d a + \text{const.}\nonumber\\
    &=\E\big[ \E[{\varphi}(Z, A) | W, A]^2 \big] \nonumber\\&-\enspace2 \E_{W} \E_{A} \big[ \E[{\varphi}(Z, A) | W, A]\big] + \text{const.}
\label{eq:MainTextSimplifiedLossATEPopulation}
\end{align}
Here, $\E_W\E_A[.]$ denotes decoupled expectations with respect to $p(w)$ and $p(a)$, i.e., for any function function $\ell$, $\E_W \E_A[\ell(W, A)] = \int \ell(w, a) p(w) p(a) dw da$, and $'\text{const.}'$ represents terms independent of $\varphi$. The above simplification allows us to avoid the need to estimate a density ratio, since minimizing Equation (\ref{eq:MainTextSimplifiedLossATEPopulation}) does not require knowing $r(W, A)$. To overcome the second challenge, we assume that $\varphi_0^{\text{ATE}}$ resides in the RKHS $\gH_{\gZ\gA}$. Observe that for any $\varphi \in \gH_{\gZ\gA}$,
\begin{align*}
    &\E[\varphi(Z, a) | W = w, A = a] \\&= \E[\langle \varphi, \phi_\gZ(Z) \otimes \phi_\gA(a) \rangle_{\gH_{\gZ\gA}} | W = w, A = a]\\
    &= \langle \varphi, \E[\phi_\gZ(Z) | W = w, A = a] \otimes \phi_\gA(a) \rangle_{\gH_{\gZ\gA}} \\&= \langle \varphi, \mu_{Z|W, A} (w, a)\otimes \phi_\gA(a) \rangle_{\gH_{\gZ\gA}},
\end{align*}
where $\mu_{Z|W, A} (w, a)$ denotes the CME $\E[\phi_\gZ(Z) | W = w, A = a]$. With this further simplification, we can write Equation (\ref{eq:MainTextSimplifiedLossATEPopulation}) as
\begin{align*}
    &\E\big[ \langle {\varphi}, \mu_{Z|W, A} (W, A)\otimes \phi_\gA(A) \rangle_{\gH_{\gZ\gA}}^2 \big] \\&- 2 \E_{W} \E_{A} \big[ \langle \varphi, \mu_{Z|W, A} (W, A)\otimes \phi_\gA(A) \rangle_{\gH_{\gZ\gA}}\big] + \text{const.}
\end{align*}
As a result, our approach to estimate the bridge function will be a \emph{two-stage procedure}: (i) The first stage estimates the CME $\mu_{Z|W, A} (W, A)$; (ii) the second stage minimizes the modified regression loss using the approximated CME. Since we identify the dose-response through a conditional expectation, our approach requires learning an additional conditional mean embedding, as outlined below, which we consider as a third-stage regression. 

\textbf{First Stage Regression:} Under the regularity condition that $\E[\ell(Z) | W = \cdot, A = \cdot ] \in \gH_{\gW\gA}$ for all $\ell\in \gH_\gZ$, there exist an operator $C_{Z | W, A} \in S_2(\gH_{\gW\gA}, \gH_\gZ)$ such that $\mu_{Z|W, A} (w, a) = C_{Z | W, A} \phi_{\gW\gA}(w, a)$ \citep{HilbertSpaceEmbeddingforConditionalDistributions, li2024towards}. $C_{Z | W, A}$ is called the CME operator and can be estimated by vector-valued regression \citep{grunewalder2012conditional,mollenhauer2020nonparametric,li2022optimal,li2024towards}. The estimation of the CME operator is referred to as the \emph{first-stage regression}. Given first-stage samples $\{w_i, a_i, z_i\}_{i=1}^n$, $C_{Z | W, A}$ is learned by minimizing the regularized vector-valued least-squares cost
$$
    \hat{\gL}^{c}(C) = \frac{1}{n} \sum_{i = 1}^n \|\phi_\gZ(z_i) - C\phi_{\gW \gA}(w_i, a_i) \|_{\gH_\gZ}^2 + \lambda_1 \|C\|_{S_2}^2
$$
i.e., $\hat{C}_{Z | W, A} = \argmin_{C \in S_2(\gH_{\gW\gA}, \gH_\gZ)} \hat{\gL}^{c}(C)$. The solution to this problem is given by $\hat{C}_{Z | W, A} = \hat{C}_{Z, (W, A)} (\hat{C}_{W, A} + \lambda_1 I)^{-1}$ \citep{GruLevBalPonetal12},
where $\hat{C}_{Z, (W, A)} = \frac{1}{n} \sum_{i = 1}^n \phi_\gZ(z_i) \otimes \phi_{\gW\gA}(w_i, a_i)$ and $\hat{C}_{W, A} = \frac{1}{n} \sum_{i = 1}^n \phi_{\gW\gA}(w_i, a_i) \otimes  \phi_{\gW\gA}(w_i, a_i)$. 

\textbf{Second Stage Regression:} Given the CME estimation $\hat{\mu}_{Z|W, A} (w, a) = \hat{C}_{Z | W, A} \phi_{\gW\gA}(w, a)$ from the first-stage, we aim to minimize the modified least-squares loss. We minimize the empirical counterpart of this loss with Tikhonov regularization, using the CME estimate $\hat{\mu}_{Z|W, A}$ and the second-stage data $\{\Tilde{z}_i, \Tilde{w}_i, \Tilde{a}_i\}_{i = 1}^m$. The empirical loss $\hat{\gL}^{2SR}_m(\varphi)$ is expressed as
\begin{align*}
    &\hat{\gL}^{2SR}_m(\varphi)=\frac{1}{m} \sum_{i = 1}^m \langle \varphi, \hat{\mu}_{Z|W, A} (\Tilde{w}_i, \Tilde{a}_i) \otimes \phi_\gA(\Tilde{a}_i) \rangle_{\gH_{\gZ\gA}}^2 -\\& \frac{2}{m (m-1)} \sum_{\substack{i,j = 1 \\ j \ne i}}^m \Big\langle \varphi, \hat{\mu}_{Z|W, A} (\Tilde{w}_j, \Tilde{a}_i) \otimes \phi_\gA(\Tilde{a}_i) \Big\rangle_{\gH_{\gZ\gA}} \\&+ \lambda_2 \|\varphi\|^2_{\gH_{\gZ\gA}}
\end{align*}
We denote its minimizer on $\gH_{\gZ\gA}$ as $\hat{\varphi}_{\lambda_2, m}$. This is referred to as the \emph{second-stage regression}. 

\textbf{Third Stage Regression:} We can then estimate the ATE function as $f_{\text{ATE}}(a) \approx \E[Y \hat{\varphi}_{\lambda_2, m}(Z, a) | A = a]$ in closed-form via kernel matrices.
We note that:
\begin{align*}
    \E[Y &\hat{\varphi}_{\lambda_2, m}(Z, a) | A = a] \\&= \E[Y \langle \hat{\varphi}_{\lambda_2, m}, \phi_\gZ(Z) \otimes \phi_\gA(a) \rangle | A = a]\\
    &=\langle \hat{\varphi}_{\lambda_2, m}, \E[Y\phi_\gZ(Z) | A = a] \otimes \phi_\gA(a) \rangle.
\end{align*}
The evaluation of the above inner product requires the estimation of $\E[Y\phi_\gZ(Z) | A = a]$. We obtain this via kernel ridge regression using data from either first-stage, second-stage, or a combination of both. This can be considered as a \emph{third-stage regression}. Similar to the first-stage regression, we estimate the conditional mean operator $C_{YZ|A}$ such that $C_{YZ|A}\phi_\gA(a) = \E[Y\phi_\gZ(Z) | A = a]$. For simplicity in our algorithm derivations, we use first-stage data to estimate this conditional mean. Replacing this conditional mean with its estimate, we can estimate the dose-response with the inner product $\hat{f}_{\text{ATE}}(a) = \langle \hat{\varphi}_{\lambda_2, m}, \hat{\E}[Y\phi_\gZ(Z) | A = a] \otimes \phi_\gA(a) \rangle$.
The algorithm below provides the closed-form solution for ATE estimation as a result of the three stages, and its derivation can be found in S.M. (Sec. \ref{sec:ATE_algorithm_proof}).\looseness=-1

\begin{algorithm_state}
Let the first and second-stage data be denoted by $\{z_i, w_i, a_i\}_{i = 1}^n$ and $\{\Tilde{w}_i, \Tilde{a}_i\}_{i = 1}^m$, respectively, and $(\lambda_1, \lambda_2, \lambda_3)$ be the regularization parameters. For $F \in \{A, W, Z\}$ with domain $\gF$, the first-stage kernel matrices are denoted as $\mK_{FF} = [k_\gF(f_i, f_j)]_{ij} \in \R^{n \times n}$, $\mK_{Ff} = [k_\gF(f_i, f)]_{i} \in \R^{n}$, where $\{f_i\}_{i = 1}^n$ denotes the first-stage data samples. For the second-stage variables $\tilde{F} \in \{\Tilde{A}, \Tilde{W}\}$, the kernel matrices are denoted as: $\mK_{\Tilde{F} \Tilde{F}} = [k_\gF(\Tilde{f}_i, \Tilde{f}_j)]_{ij} \in \R^{m \times m}$, $\mK_{F \Tilde{F}} = [k_\gF(f_i, \Tilde{f}_j)]_{ij} \in \R^{n \times m}$, $\mK_{F \Tilde{f}} = [k_\gF(f_i, \Tilde{f})]_i \in \R^n$, and $\mK_{\Tilde{F} f} = [k_\gF(\Tilde{f}_j, f)]_j \in \R^m$. Define the following matrices: i-) $\mB = (\mK_{W W} \odot \mK_{A A} + n \lambda_1 \mI)^{-1}(\mK_{W \Tilde{W}} \odot \mK_{A \Tilde{A}})\in \R^{n \times m}$, ii-) $\bar{\mB} \in \R^{n \times m}$ is the matrix, where $j$-th column is given by $\Bar{\mB}_{:, j} = \frac{1}{m} \sum_{\substack{l = 1 \\ l \ne j}}^m (\mK_{W W} \odot \mK_{A A} + n \lambda_1 \mI)^{-1}(\mK_{W \Tilde{w}_l} \odot \mK_{A \Tilde{a}_j})$,
where $\mI \in \R^{n \times n}$ is the identity matrix.
Furthermore, let $\{\alpha_i\}_{i = 1}^{m + 1}$ be the minimizer of the cost function $\hat{\gL}^{2SR}_m(\malpha) = \frac{1}{m} \malpha^T \mL^T \mL \malpha -2 \malpha^T \mM + \lambda_2 \malpha^T \mN \malpha$ where\looseness=-1
\begin{align*}
    \mL &= \begin{bmatrix}
        \mB^T \mK_{Z Z}\mB \odot \mK_{\Tilde{A}\Tilde{A}} \nonumber\\ (\frac{\vone}{m})^T \big[{\mB}^T \mK_{Z Z} \bar{\mB} \odot \mK_{\Tilde{A}\Tilde{A}}\big]^T
    \end{bmatrix}^T \in \R^{m \times (m+1)},\\
    \mM &= \begin{bmatrix}
     [ \mB^T \mK_{Z Z} \Bar{\mB} \odot \mK_{\Tilde{A} \Tilde{A}} ] \frac{\vone}{m}\\
     (\frac{\vone}{m})^T \Big[ \Bar{\mB}^T \mK_{Z Z} \Bar{\mB} \odot \mK_{\Tilde{A} \Tilde{A}}\Big] \frac{\vone}{m}
    \end{bmatrix} \in \R^{(m+1)},
\end{align*}
$\mN=\begin{bmatrix}\mL & \mM \end{bmatrix} \in \R^{(m+1)\times (m + 1)}$, $\vone \in \R^m$ is vector of ones, and $\malpha = \begin{bmatrix} \alpha_1 & \alpha_2 & \ldots & \alpha_m & \alpha_{m+1}\end{bmatrix}^T \in \R^{m + 1}$.
Then the dose-response estimation can be written in  closed-form as $\hat{f}_{\text{ATE}}(a) = \malpha^T \mE$, where
\begin{align*}
    &\mE = \begin{bmatrix}
        \mB^T \mD \mK_{A a} \odot \mK_{\Tilde{A} a}\\
        \big( \Bar{\mB}^T \mD \mK_{A a} \odot \mK_{\Tilde{A} a}\big) \frac{\vone}{m}
    \end{bmatrix} \in \R^{m + 1},
\end{align*}
$\mD = \mK_{Z Z} \text{diag}(\mY) [\mK_{A A} + n \lambda_3 \mI]^{-1} \in \R^{n \times n}$ and $\mY = \begin{bmatrix}
    y_1 & y_2 & \ldots & y_n
\end{bmatrix}^T \in \R^n$.
\label{algo:ATE_algorithm_with_confounders}
\end{algorithm_state}

\subsubsection{Conditional Dose-Response Estimation} As with the dose-response, we aim to minimize the least-squares loss $\E_{W, A}[(r(W, A, a') - \E[\varphi(Z, A, a') | W, A])^2]$ for conditional dose-response. Our method allows for similar simplifications in the second-stage regression as before, thereby bypassing explicit density estimation in minimizing the loss function $\gL^{\text{2SR}}$. Let $r(W, A, a') = \frac{p(W, a') p(A) }{p(W, A) p (a')}$ and suppose that the bridge function $\varphi$ lies in the RKHS $\gH_{\gZ\gA\gA}$ (which denotes $\gH_\gZ \otimes \gH_\gA \otimes \gH_\gA$). Specifically, the loss function can be simplified as (see S.M., Sec. \ref{sec:ATT_algorithm_proof}):
 \begin{align*}
     &\E[(r(W, A, a') - \E[\varphi(Z, A, a') | W, A])^2]\\
     &=\E[\langle \varphi, \mu_{Z|W, A} (W, A) \otimes \phi_\gA(A) \otimes \phi_\gA(a') \rangle^2]\\
     &-2 \E_{A}[\langle \varphi, C_{Z|W, A}(\E_{W | A = a'}[\phi_\gW(W)] \otimes \phi_\gA(A)) \\&\otimes \phi_\gA(A) \otimes \phi_\gA(a') \rangle] + \text{const.}
 \end{align*}
 The notation $\E_{W|A = a'}$ denotes the expectation with respect to the conditional distribution $p(W|A = a')$, i.e., for any function $\ell$, $\E_{W|A = a'}[\ell(W)] = \int \ell(w) p (w|a') dw$. We need to estimate the CME $\E[\phi_\gW(W) | A = a']$, which can be obtained using kernel ridge regression on second-stage data, and expressed in  closed form as $\hat{\E}[\phi_\gW(W) | A = a'] = \sum_{i = 1}^m \theta_i \phi_\gW(\Tilde{w}_i) = \Phi_{\gW} \vtheta$,
where $\vtheta = (\mK_{\Tilde{A} \Tilde{A}} + m \zeta \mI)^{-1} \mK_{\Tilde{A} a'}$, $\Phi_{\gW} = \begin{bmatrix}
    \phi_\gW(\Tilde{w}_1) & \ldots & \phi_\gW(\Tilde{w}_m)
\end{bmatrix}$, and $\zeta$ is the regularization parameter for this CME estimate. Hence, we can write the sample-based loss, $\hat{\mathcal{L}}^{\text{2SR}}_m(\varphi)$, of the second-stage regression with Tikhonov regularization,
\begin{align*}
    &\frac{1}{m} \sum_{i = 1}^m \langle \varphi, \hat{\mu}_{Z|W, A} (\Tilde{w}_i, \Tilde{a}_i) \otimes \phi_\gA(\Tilde{a}_i) \otimes \phi_\gA(a') \rangle^2 \nonumber\\ &- \frac{2}{m} \sum_{\substack{i,j = 1 \\ i \ne j}}^m \langle \varphi,  \theta_i \hat{\mu}_{Z | W, A}(\Tilde{w}_i, \Tilde{a}_j) \otimes \phi_\gA(\Tilde{a}_j) \otimes \phi_\gA(a') \rangle \\&+ \lambda_2 \|\varphi\|_{\gH_{\gZ\gA\gA}}^2.
\end{align*}
Using the estimate $\hat{\varphi}_{\lambda_2, m}$ for $\varphi_0^{\text{ATT}}$ from the  second stage regression, the conditional dose-response curve is given by the inner product $\hat{f}_{\text{ATT}}(a, a') = \langle \hat{\varphi}_{\lambda_2, m}, \hat{\E}[Y\phi_\gZ(Z) | A = a] \otimes \phi_\gA(a) \otimes \phi_\gA(a')\rangle$, where we again need the conditional mean $ \hat{\E}[Y\phi_\gZ(Z) | A = a]$ as for the dose-response algorithm.
The algorithm below provides a closed-form solution for the conditional dose-response estimate,  with proof in S.M. (Sec. \ref{sec:ATT_algorithm_proof}).

\begin{algorithm_state}
Denote the first- and second-stage data as $\{z_i, w_i, a_i\}_{i = 1}^n$ and $\{\Tilde{w}_i, \Tilde{a}_i\}_{i = 1}^m$, respectively, and let $(\lambda_1, \lambda_2, \lambda_3, \zeta)$ be the regularization parameters. Define the kernel matrices, the matrix $\mD \in \R^{n \times n}$, and the matrix $\mB \in \R^{n \times m}$ as in Algorithm (\ref{algo:ATE_algorithm_with_confounders}). Furthermore, define $\Tilde{\mB}$, where $j$-th column is given by $\Tilde{\mB}_{:, j} = \sum_{\substack{l = 1 \\ l \ne j}}^m (\mK_{W W} \odot \mK_{A A} + n \lambda_1 \mI)^{-1}(\theta_l \mK_{W \Tilde{w}_l} \odot \mK_{A \Tilde{a}_j})$ with $\theta_i = [(\mK_{\Tilde{A} \Tilde{A}} + m \zeta \mI)^{-1} \mK_{\Tilde{A} a'}]_i$. For a given $a'$, let $\{\alpha_i\}_{i = 1}^{m + 1}$ be the minimizer of $\hat{\gL}^{2SR}_m(\malpha) = k_\gA(a', a')\big(\frac{k_\gA(a', a')}{m} \malpha^T \mL^T \mL \malpha -2 \malpha^T \mM + \lambda_2 \malpha^T \mN \malpha \big)$ where\looseness=-1
\begin{align*}
    \mL &= \begin{bmatrix}
        \mB^T \mK_{Z Z}\mB \odot \mK_{\Tilde{A}\Tilde{A}} \nonumber\\ (\frac{\vone}{m})^T \big[{\mB}^T \mK_{Z Z} \Tilde{\mB} \odot \mK_{\Tilde{A}\Tilde{A}}\big]^T
    \end{bmatrix}^T \in \R^{m \times (m+1)},\\
    \mM &= \begin{bmatrix}
     [ \mB^T \mK_{Z Z} \Tilde{\mB} \odot \mK_{\Tilde{A} \Tilde{A}} ] \frac{\vone}{m}\\
     (\frac{\vone}{m})^T \Big[ \Tilde{\mB}^T \mK_{Z Z} \Tilde{\mB} \odot \mK_{\Tilde{A} \Tilde{A}}\Big] \frac{\vone}{m}
    \end{bmatrix} \in \R^{(m+1)},
\end{align*}
and $\mN=\begin{bmatrix}\mL & \mM \end{bmatrix} \in \R^{(m+1)\times (m + 1)}$. Then, the conditional dose-response estimate can be written in  closed-form as $\hat{f}_{\text{ATT}}(a) = k_\gA(a', a') \malpha^T \mE$, where
\begin{align*}
    \mE = \begin{bmatrix}
        \mB^T \mD \mK_{A a} \odot \mK_{\Tilde{A} a}\\
        \big( \Tilde{\mB}^T \mD \mK_{A a} \odot \mK_{\Tilde{A} a}\big) \frac{\vone}{m}
    \end{bmatrix}\in \R^{m + 1}.
\end{align*}
\label{algo:ATT_algorithm_with_confounders}
\end{algorithm_state}

\section{CONSISTENCY}
\label{sec:Consistency}
We present non-asymptotic uniform consistency guarantees for the dose-response curve; similar guarantees for the conditional dose-response curve are provided in S.M. (Sec. \ref{sec:consistency_appendix}). We recall that for the third-stage regression, we can re-use data from the first and second stages, and we denote by $t$ the number of samples used for that stage; thus $n, m, t$ are the number of samples used in stages 1, 2, and 3, respectively. Likewise  $\lambda_1, \lambda_2, \lambda_3$ are the regularization parameters for their respective stages. 
We assume that each regression stage is well specified, as follows:

\begin{assumption} \label{asst:well_specifiedness}
    (1) There exists $C_{Z | W, A} \in S_2(\gH_{\gW\gA}, \gH_\gZ)$ such that $\mu_{Z|W, A}(W, A) = C_{Z | W, A} \phi_{\gW\gA}(W, A)$; (2) There exists a solution $\varphi_0 \in \gH_{\gZ\gA}$ of Equation (\ref{eq:AlternativeProxyATEBridgeFunction}); (3)  There exists $C_{YZ | A} \in S_2(\gH_{\gA}, \gH_\gZ)$ such that $\E[Y\phi_\gZ(Z) | A] = C_{YZ | A} \phi_{\gA}(A)$.
\end{assumption}

We impose the following additional conditions. 
\begin{assumption} 
For $\gF \in \{\gA, \gW, \gZ\}$, we assume
    \begin{itemize}
        \item[i.]  $\gF$ is a Polish space;
        \item[ii.] $k_\gF(f, .)$, is continuous for almost every $f \in \gF $ and is also bounded by $\kappa$ for almost every $f \in \gF $, i.e., $\sup_{f \in \gF} \|k_\gF(f, .)\|_{\gH_\gF} \le \kappa$;
        \item[iii.] There exists $R,\sigma > 0$ such that $\forall \enspace q \geq 2$, $P_A-$almost surely, $\E[(Y - \E[Y \mid A])^q \mid A] \leq \frac{1}{2}q!\sigma^2R^{q-2}.$ 
    \end{itemize}
\label{assum:Stage1ConsistencyKernelAssumptions_main}
\end{assumption}

\begin{assumption} Let $\bar{\varphi}_0$ be the minimum RKHS norm bridge function solution from Definition (\ref{def:min_norm_bridge}), and let $\Sigma_1$, $\Sigma_2$, and $\Sigma_3$ be covariance operators associated with first, second, and third-stage regressions, respectively, as defined in Definition (\ref{definition:CovarianceOperators_ATE}). We assume that the following conditions hold:\label{asst:src_all_stages_main}
\begin{itemize}
    \item[i.] There exists a constant $B_1 < \infty$ such that for a given $\beta_1 \in (1, 3]$,
    $
    \|{C}_{Z|W,A}\Sigma_1^{-\frac{\beta_1-1}{2}}\|_{S_2(\gH_{\gW\gA}, \gH_\gZ)} \le B_1
    $   
    \item[ii.] There exists a constant $B_2 < \infty$ such that for a given $\beta_2 \in (1, 3]$, 
    $
        \|\Sigma_2^{-\frac{\beta_2-1}{2}}\bar{\varphi}_0\|_{\gH_{\gZ\gA}} \le B_2.
    $
    \item[iii.] There exists a constant $B_3 < \infty$ such that for a given $\beta_3 \in (1,3]$,
    $
    \|C_{YZ \mid A}\Sigma_3^{-\frac{\beta_3-1}{2}}\|_{S_2(\gH_{\gA}, \gH_{\gZ})} \leq B_3.
    $
\end{itemize}
\end{assumption}


Note that Assumption~(\ref{assum:ExistenceCompletenessAssumption}) implies that there exists a solution of Equation (\ref{eq:AlternativeProxyATEBridgeFunction}), Assumption~(\ref{asst:well_specifiedness}-2) further requires that at least one solution lies in the RKHS. The Bernstein condition in Assumption (\ref{assum:Stage1ConsistencyKernelAssumptions_main}-iii) regulates observation noise, while the source condition in Assumption (\ref{asst:src_all_stages_main}) links regression smoothness to covariance operators \citep{caponnetto2007optimal,fischer2020sobolev}. Theorem~(\ref{th:final_rate_ate_main}) also assumes an eigenvalue decay condition in Assumption (\ref{asst:evd_all_stages}) to characterize RKHSs' effective dimension.

\begin{theorem} \label{th:final_rate_ate_main}
Let Assumptions~(\ref{assum:ProxyCausalAssumptions1}), (\ref{assum:AlternativeProxyAssumptionCompleteness1}), (\ref{assum:ExistenceCompletenessAssumption}), (\ref{asst:well_specifiedness}), (\ref{assum:Stage1ConsistencyKernelAssumptions_main}), (\ref{asst:src_all_stages_main}) and (\ref{asst:evd_all_stages}) hold with parameters $\beta_1, \beta_2, \beta_3 \in (1,3]$ and $p_1, p_2, p_3 \in (0,1]$. Set $\lambda_1 = n^{-\frac{1}{\beta_1 + p_1}}$ and $\lambda_3 = t^{-\frac{1}{\beta_3 + p_3}}$. Fix $\iota > 0 $ and $n = m^{\iota\frac{\beta_1+p_1}{\beta_1 - 1}}$.
    \begin{itemize}
        \item[i.]If $\iota \le \frac{\beta_2+1}{\beta_2 + p_2}$, let $\lambda_2 = m^{-\frac{\iota}{\beta_2+1}}$, then \\
        $
        \|\hat{f}_{ATE} - f_{ATE}\|_\infty = O_p \left(t^{-\frac{1}{2}\frac{\beta_3-1}{\beta_3 + p_3}} + m^{-\frac{\iota}{2}\frac{\beta_2-1}{\beta_2+1}}\right)
        $
        \item[ii.] If $\iota \geq \frac{\beta_2+1}{\beta_2 + p_2}$, let $\lambda_2 = m^{-\frac{1}{\beta_2 + p_2}}$, then \\
        $
        \|\hat{f}_{ATE} - f_{ATE}\|_\infty = O_p \left(t^{-\frac{1}{2}\frac{\beta_3-1}{\beta_3 + p_3}} + m^{-\frac{1}{2}\frac{\beta_2-1}{\beta_2 + p_2}} \right)
        $
    \end{itemize}
\end{theorem}

The proof and details on the assumptions are given in S.M. (Sec. \ref{sec:consistency_appendix}). Parameters $\{p_i\}_{i=1}^3$ control the effective dimension of the RKHSs used in the three stages \citep{caponnetto2007optimal}. A value $p_i \to 0$ corresponds to a finite dimensional RKHS, while larger $p_i$ means slower decay of eigenvalues of the covariance operator and hence a larger effective dimension. Parameters $\{\beta_i\}_{i=1}^3$ control the smoothness of $C_{Z \mid W,A}$, $\varphi_0$ and $C_{YZ \mid A}$ respectively. Larger $\beta_i$ corresponds to a smoother operator or function. $\iota$ controls the ratio between stage 1 and stage 2 samples to achieve a fast rate in the setting (ii). Indeed, at $\iota = (\beta_2+1)/(\beta_2 + p_2)$, the convergence rate ($m^{-\frac{1}{2}\frac{\beta_2-1}{\beta_2 + p_2}}$) is minimax optimal in $m$ while requiring the fewest observations from stage 1 \citep{caponnetto2007optimal}.  

\textbf{Comparison to outcome bridge function.} Convergence guarantees for ATE when the outcome bridge technique is used have  rate $t^{-1/2}$ instead of our rate $t^{-\frac{1}{2}\frac{\beta_3-1}{\beta_3 + p_3}}$ (\citet[Proposition 1][]{Mastouri2021ProximalCL}, \citet[Theorem 4][]{singh2023kernelmethodsunobservedconfounding}). This is a consequence of the fact that our treatment bridge algorithm requires a third regression for stage 3, while the outcome bridge approach only requires an averaging for stage 3. 
While this poses a disadvantage in principle, it may be less important in practice than the ease of estimation of the treatment bridge vs the outcome bridge, in the same way that IPW and direct estimates may each be advantageous in different regimes \citep{BanRob05}. This can be observed in our experiments. 
Moreover, in \citet{Mastouri2021ProximalCL,singh2023kernelmethodsunobservedconfounding}, the rates stop improving at  smoothness  $\beta_i = 2$ while our rates improve up to $\beta_i=3$. This improvement is obtained from a tighter control of the approximation error in RKHS norm, as observed by \citet[][Remark 7]{meunier2023nonlinear}. This improvement can also be applied to previous works with an outcome bridge. Last, we emphasize that for consistency in the well-specified case, as treated by \citet{Mastouri2021ProximalCL,singh2023kernelmethodsunobservedconfounding} and in our work, the kernels are not required to be characteristic (contrary to assumptions made in the earlier works), nor is $Y$ required to be bounded ($Y$ need only be sub-exponential).

\textbf{Saturation effect.} Benefits from high smoothness beyond the saturation point at $\beta_i = 3$ can be obtained by using alternative spectral regularization techniques \citep{engl1996regularization}. Results were recently obtained by \cite{meunier2025optimal} for conditional mean embedding learning. The application  to the proxy setting is an interesting topic of future study. 


\section{NUMERICAL EXPERIMENTS}
\label{sec:NumericalExperiments}
\begin{figure*}[ht!]
\centering
\subfloat[]
{\includegraphics[trim = {0cm 0cm 0cm 0.0cm},clip,width=0.245\textwidth]{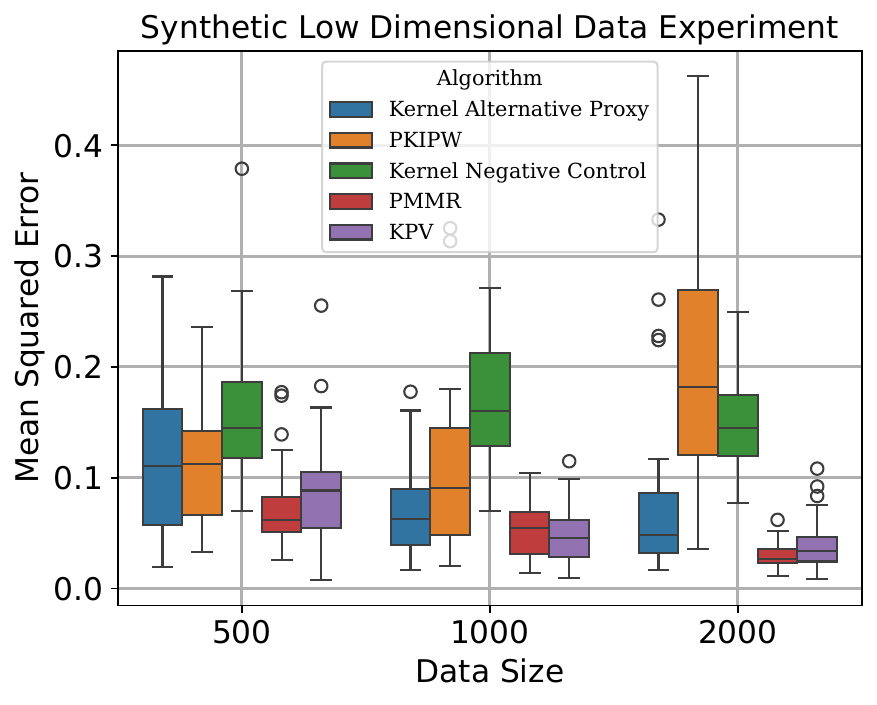} \label{fig:alternativeProxyMethodATEComparisonLowDim}}
\subfloat[]
{\includegraphics[trim = {0cm 0cm 0cm 0.0cm},clip,width=0.245\textwidth]{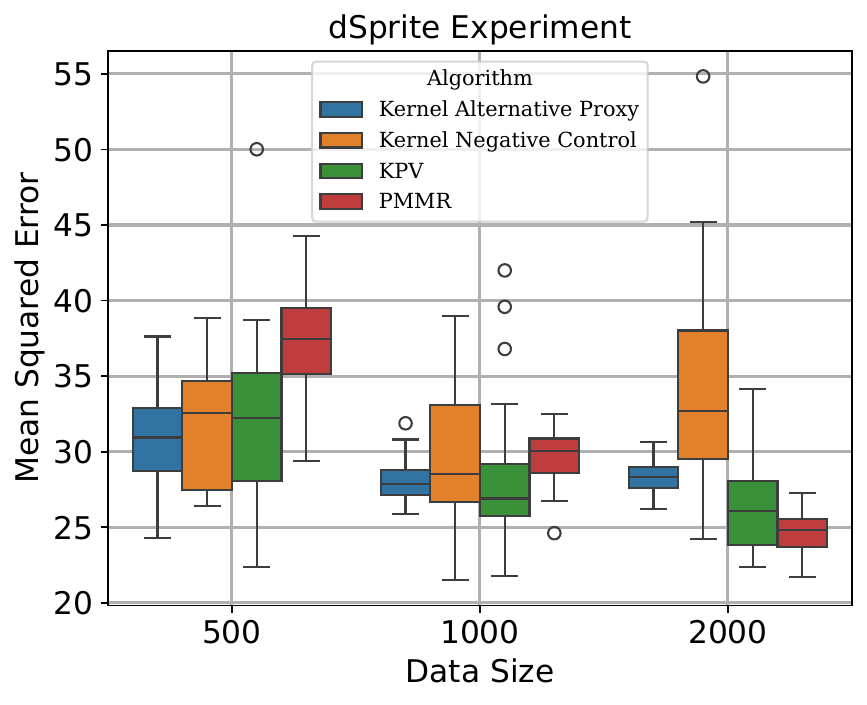}
\label{fig:alternativeProxyMethodATEComparisonDSprite}}
\subfloat[]
{\includegraphics[trim = {0cm 0cm 0cm 0.0cm},clip,width=0.245\textwidth]{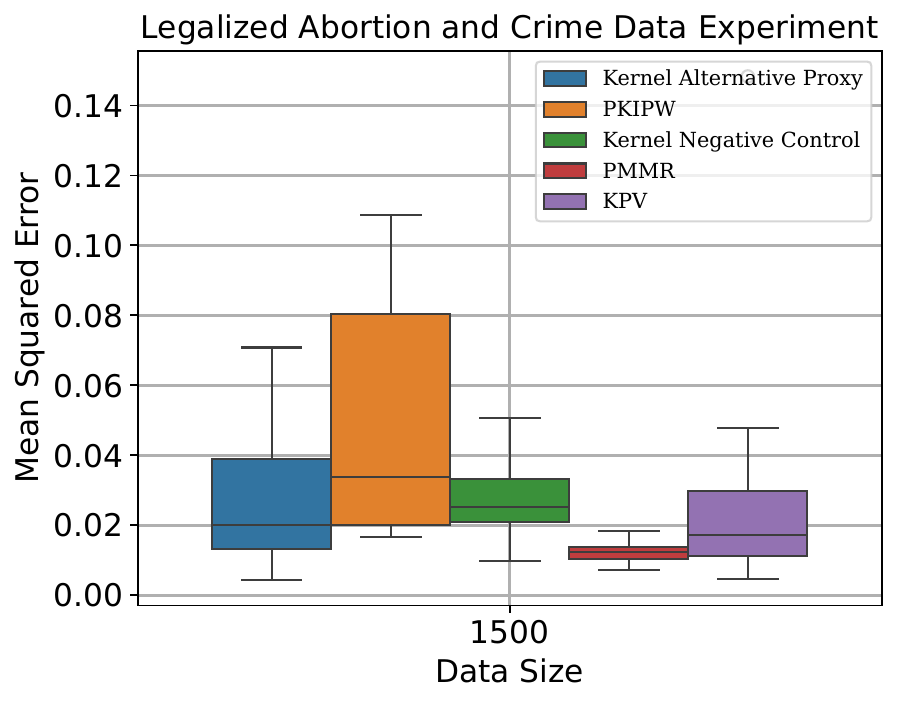}
\label{fig:alternativeProxyMethodATEComparisonAbortion}}
\subfloat[]
{\includegraphics[trim = {0cm 0cm 0cm 0.0cm},clip,width=0.253\textwidth]{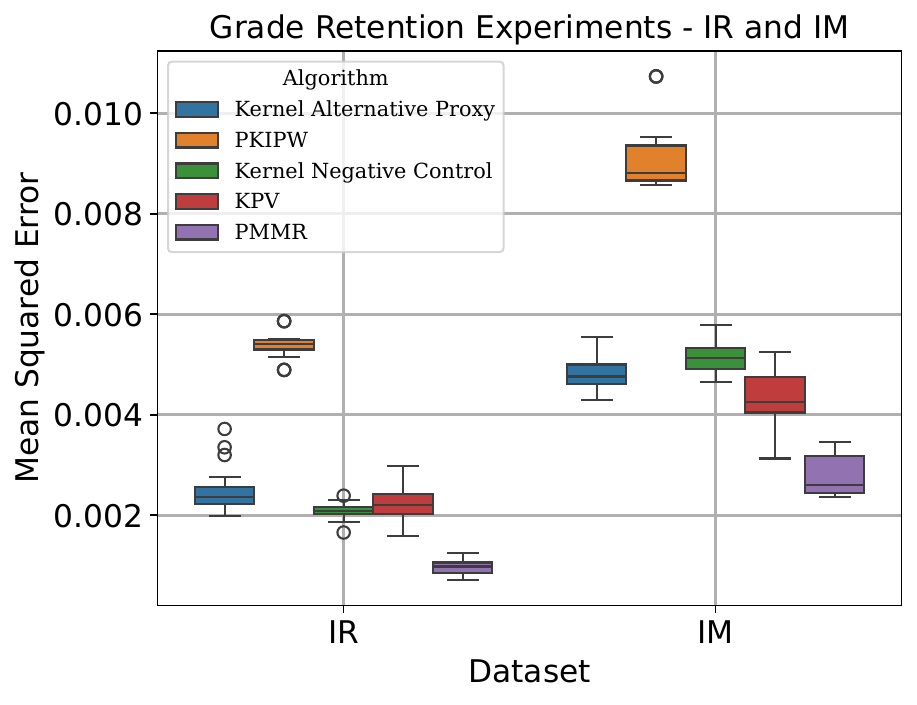}
\label{fig:alternativeProxyMethodATEComparisonDeaner}}
\newline\caption{Dose-response curve estimation across various datasets and algorithms: \emph{Kernel Alternative Proxy} (Ours), PKIPW \citep{wu2024doubly}, Kernel Negative Control \citep{singh2023kernelmethodsunobservedconfounding}, KPV \citep{Mastouri2021ProximalCL}, and PMMR \citep{Mastouri2021ProximalCL}. (a) Synthetic low-dimensional setting, (b) dSprite dataset, (c) legalized abortion and crime dataset, and (d) grade retention and cognitive outcome datasets.}
\label{fig:alternativeProxyMethodATEComparison}
\end{figure*}
In this section, we assess the empirical performance of our proposed framework for estimating causal structural functions using both synthetic data and real-world tasks. We compare our method to other PCL frameworks, including Proximal Kernel Inverse Probability Weighted (PKIPW) \citep{wu2024doubly}, Kernel Negative Control \citep{singh2023kernelmethodsunobservedconfounding}, Proxy Maximum Moment Restriction (PMMR) \citep{Mastouri2021ProximalCL} and Kernel Proxy Variable (KPV) \citep{Mastouri2021ProximalCL}. Additionally, for an ATT experiment, we compare our method to the Kernel-ATT algorithm proposed in \citep{RahulKernelCausalFunctions}, which assumes access to the confounding variable $U$. For each experiment (except those involving PKIPW), we employed a Gaussian kernel $k_\gF(f_i, f_j) = \exp(- \|f_i - f_j\|^2_2/(2 l^2))$ unless otherwise stated, where $l$ is the length-scale of the kernel. The Gaussian kernel's length scale was selected using the median interpoint distance heuristic, if not specified otherwise. In the PKIPW experiments, we used the Epanechnikov kernel, consistent with the original implementation of \citet{wu2024doubly}. To select the regularization parameters $\lambda_1$ and $\lambda_3$ (and $\zeta$ for the ATT) in our proposed methods, we employed leave-one-out cross-validation (LOOCV), which has a closed-form expression in the case of kernel ridge regression. For tuning the second-stage regularization parameter $\lambda_2$, we used the first-stage data as a held-out set to measure the validation loss, with added complexity regularization to avoid overfitting \citep{pmlr-v151-meanti22a}. Additional details, including ablation studies, hyperparameter selection procedures, and a GitHub link to our implementation can be found in S.M. (Sec. \ref{sec:SupplementaryNumericalExperiments}). \looseness=-1
\subsection{Dose-Response Experiments}
\label{sec:DoseResponseExperiments}
We assess the performance of our proposed ATE algorithm on four datasets that are described below.\looseness=-1

\textbf{{Low-Dimensional Setting:}} We use the data generation process outlined by \citep{wu2024doubly}, 
which incorporates a non-linear relationship between treatment and outcome:
\begin{align*}
    &U_1 \sim \mathcal{U}[-1, 2], \enspace U_2 \sim \mathcal{U}[0, 1] - \mathbf{1}[0 \le U_1 \le 1],\\ &W = [U_2 + \mathcal{U}[-1, 1], U_1 + \mathcal{N}(0, 1)]\\
    &Z = [U_2 + \mathcal{N}(0, 1), U_1 + \mathcal{U}[-1, 1]], A := U_1 + \mathcal{N}(0, 1)\\
    &Y := 3 \cos(2(0.3U_2 + 0.3U_1 + 0.2) + 1.5A) + \mathcal{N}(0, 1),
\end{align*}
where $\mathcal{U}[a, b]$ denotes the uniform distribution over the interval $[a, b]$, and $\mathcal{N}(\mu, \sigma^2)$ denotes Gaussian distribution with mean $\mu$ and variance $\sigma^2$. We use training sets of size $500$, $1000$, and $2000$ in our experiments. Figure (\ref{fig:alternativeProxyMethodATEComparisonLowDim}) illustrates the mean squared error results, averaged over 30 realizations, comparing our method to other approaches. Our proposed method outperforms PKIPW and Kernel Negative Control, particularly as the size of the training data increases. 

\textbf{dSprite:} 
The \emph{Disentanglement testing Sprite dataset (dSprite)} contains images of size $64 \times 64$, described by latent parameters: \emph{scale}, \emph{rotation}, \emph{posX}, and \emph{posY} \citep{dsprites17}. This dataset has been used to test the disentenglament properties of unsupervised models \citep{higgins2017betavae}. \cite{xu_dSprite} introduced benchmark dataset for PCL setting, where they treat the \emph{dSprite} images as the high-dimensional treatments. Specifically, they consider the flattened image that is corrupted by Gaussian noise as the treatment, and the structural function of interest is defined as $f_{\text{ATE}}(A) = (\|B A\|_2^2 - 3000) / 500$, where $A \in \R^{4096}$ and $B \in \R^{10 \times 4096}$, with entries of $B$ given by $B_{ij} = |32 - j| / 32$. The outcome is defined by the expression $Y = 12 (\text{\emph{posY}} - 0.5)^2f_{\text{ATE}}(A) + \epsilon$ where $\epsilon \sim \mathcal{N}(0, 0.5)$. The treatment inducing proxy $Z \in \R^3$ includes three variables: \emph{scale}, \emph{rotation}, \emph{posX}. The outcome-inducing proxy is another \emph{dSprite} image that shares the same \emph{posY} variable as the treatment while the remaining variables are fixed as $\text{\emph{scale}} = 0.8$. $\text{\emph{rotation}} = 0$, and $\text{\emph{posX}} = 0.5$.  We use training sets of size $500$, $1000$, and $2000$ in our experiments. Figure (\ref{fig:alternativeProxyMethodATEComparisonDSprite}) shows the mean squared error results averaged over 30 realizations, comparing our method to other approaches. As the code of \citep{wu2024doubly} does not support the high dimensional treatment, we do not report results for PKIPW in this experiment. Our method outperforms Kernel Negative Control and achieves comparable performance to KPV and PMMR. 

\textbf{Legalized Abortion and Crime Dataset:} We analyze the data from \citep{LegalizedAbortionData}, as preprocessed by \citep{woody2020estimatingheterogeneouseffectscontinuous}, following a similar approach to \citep{Mastouri2021ProximalCL, wu2024doubly}. The data is sourced from the GitHub repository of the code published by \citep{Mastouri2021ProximalCL}\footnotemark{}. The key variables in the causal graph are summarized as follows: (i) treatment variable ($A$): effective abortion rate, (ii) outcome varible ($Y$): murder rate, (iii) treatment proxy variable ($Z$): generosity to aid families with dependent children, and (iv) outcome proxy variable ($W$): beer consumption per capita, log-prisoner population per capita, and concealed weapons law. 
The remaining variables are captured by the unobserved confounding variable set $U$. Figure (\ref{fig:alternativeProxyMethodATEComparisonAbortion}) demonstrates the mean squared error results averaged over $30$ realizations and compares our method to other approaches (we tested each of the 10 data files\footnotemark[\value{footnote}] with three different data splits). Our method outperforms PKIPW and delivers comparable results to other kernel-based methods. Compared to Kernel Negative Control, our method achieves a lower mean squared error but exhibits a higher variance.\footnotetext{\url{https://github.com/yuchen-zhu/kernel_proxies}}

\textbf{Grade Retention and Cognitive Outcome:}
We apply our proposed method to study the effect of grade retention on the long-term cognitive outcome using data from the ECLS-K panel study \citep{Fruehwirth_GradeRetentions, deaner2023proxycontrolspaneldata}. We obtain the data from \citep{Mastouri2021ProximalCL},\footnotemark[\value{footnote}] where the key variables are as follows: (i) treatment variable ($A$): grade retention, (ii) outcome variable ($Y$): cognitive test scores in Maths and Readings at age $11$, (iii) treatment proxy variable ($Z$): the average of $1$st/$2$nd and $3$rd/$4$th year elementary scores, and (iv) outcome proxy variable ($W$): the cognitive and behavioral test scores from kindergarten.
Figure (\ref{fig:alternativeProxyMethodATEComparisonDeaner}) presents the mean squared error results averaged over $30$ realizations ($3$ realizations for each of the $10$ data files), along with a comparison to other methods. In both of the datasets  (IR: Reading grade retention; IM: Math grade retention), our proposed method performs better than PKIPW. 


\textbf{Further Comparison of Our Approach With Outcome Bridge-Based Methods:}
A key question as a result of our experiments is whether our method outperforms outcome bridge-based methods under certain conditions. To address this, we conduct an ablation study, detailed in S.M. (Sec. \ref{sec:ComparisonOfTreatmentVsOutcomeBridge_Appendix}), comparing both approaches across six synthetic settings that vary the informativeness of two proxy variables, $Z$ and $W$, relative to the confounders. The results, summarized in Table (\ref{tab:ATE_TreatmentProxy_vs_Outcome_Proxy_Table1}), indicate that our method performs better in settings where $W$ is more informative, while outcome bridge-based methods excel when $Z$ is more informative. Our analysis suggests that our method is more robust to violations of the completeness Assumption (\ref{assum:ExistenceCompletenessAssumption})—which ensures the existence of treatment bridge functions—while outcome bridge-based methods depend on this assumption for identifiability. We hypothesize that our method is more sensitive to violations of the identifiability condition (as defined by Assumption \ref{assum:AlternativeProxyAssumptionCompleteness1}) than to violations of the bridge function existence condition. Further experiments with the Job Corps dataset in S.M. (Sec. \ref{sec:Appendix_JobCorpsExperiments}) support these findings and highlight the complementary strengths of treatment and outcome bridge-based methods. Future work will explore these trade-offs in greater depth.
\subsection{Synthetic Data Experiment for ATT}
To demonstrate the effectiveness of our proposed method in conditional dose-response curve estimation, we use the low-dimensional synthetic data setting from the previous section. In particular, we train  our method to estimate $f_{\text{ATT}}(a, a')$ for different values of $a'$. Figure (\ref{fig:alternativeProxyMethodATTComparison}) shows dose-response estimates for $a' \in \{-1, -0.5, 0.25, 0.5\}$ using data size of $2000$. We compare with Kernel-ATT \citep{RahulKernelCausalFunctions} and Kernel Negative Control \citep{singh2023kernelmethodsunobservedconfounding}. The Kernel-ATT algorithm assumes access to the confounding variables $U$ so we used the variables $(A, Y, U)$ for this algorithm, making it an oracle method. Notably, our method produces results closer to Kernel-ATT than achieved by Kernel Negative Control. S.M. (Sec. \ref{sec:Appendix_JobCorpsExperiments}) provides additional experimental results on the Job Corps dataset for conditional dose-response.
\begin{figure}[ht!]
\centering
\subfloat[]
{\includegraphics[trim = {0cm 0cm 0cm 0.0cm},clip,width=0.225\textwidth]{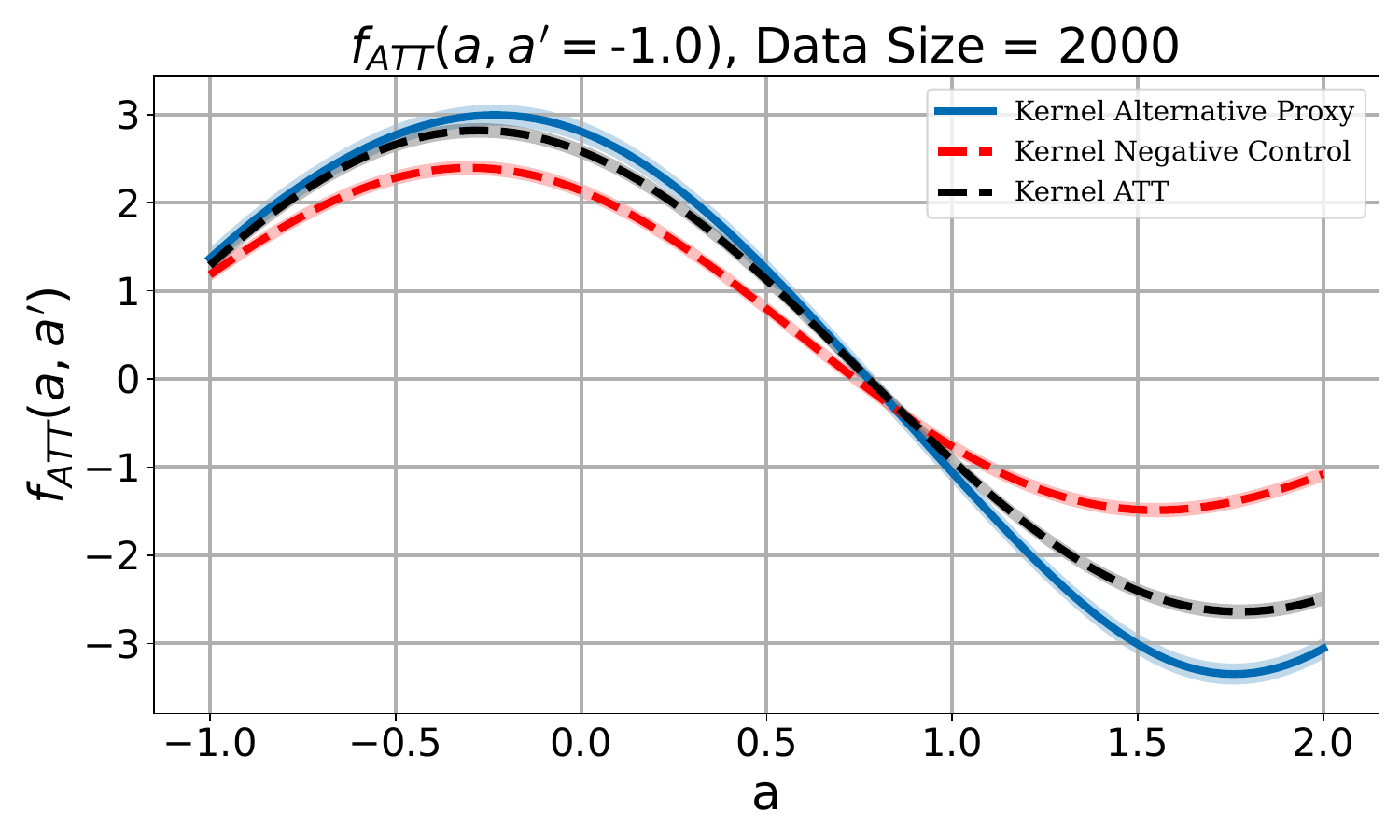}}
\subfloat[]
{\includegraphics[trim = {0cm 0cm 0cm 0.0cm},clip,width=0.225\textwidth]{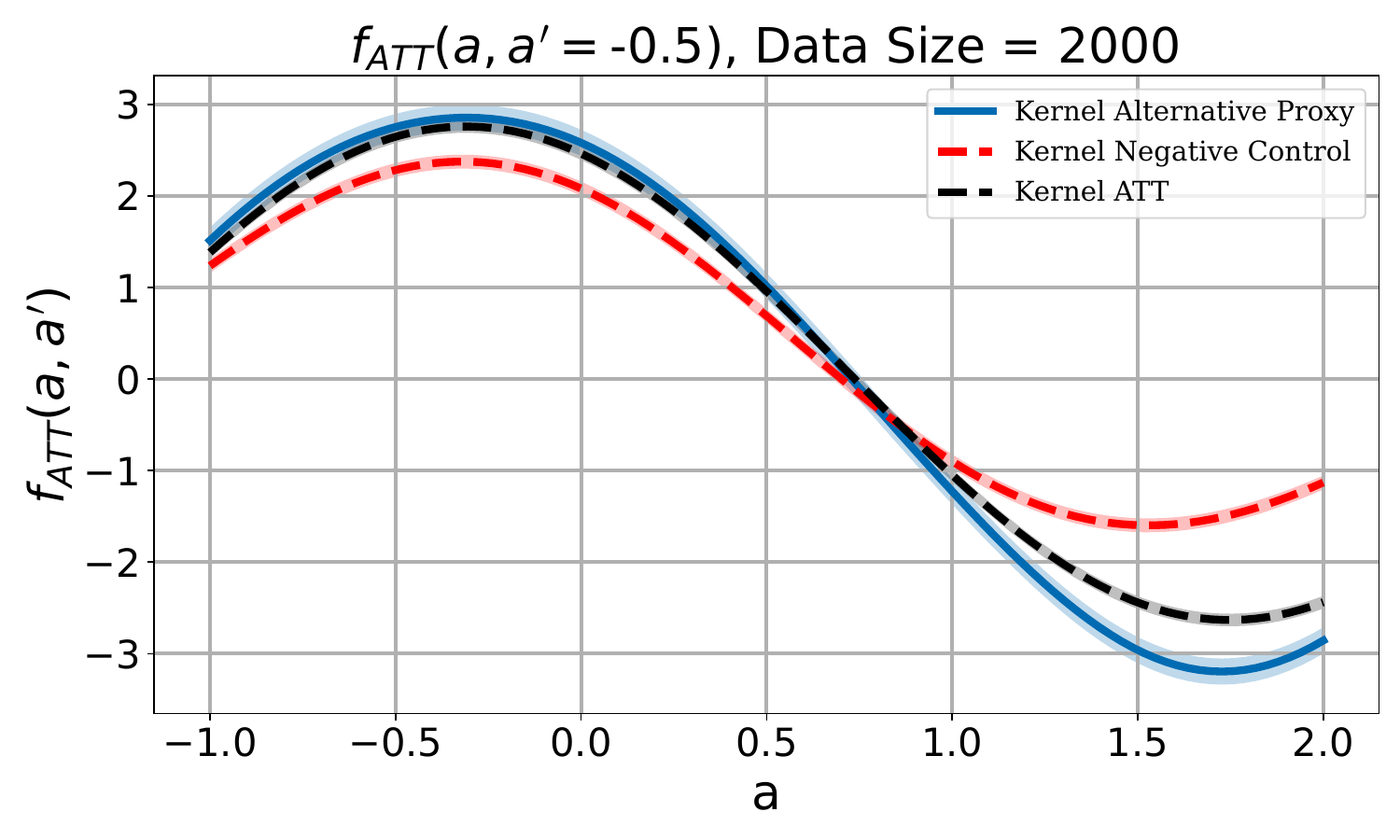}}

\centering 
\hspace{0.03cm}
\subfloat[]
{\includegraphics[trim = {0cm 0cm 0cm 0.0cm},clip,width=0.225\textwidth]{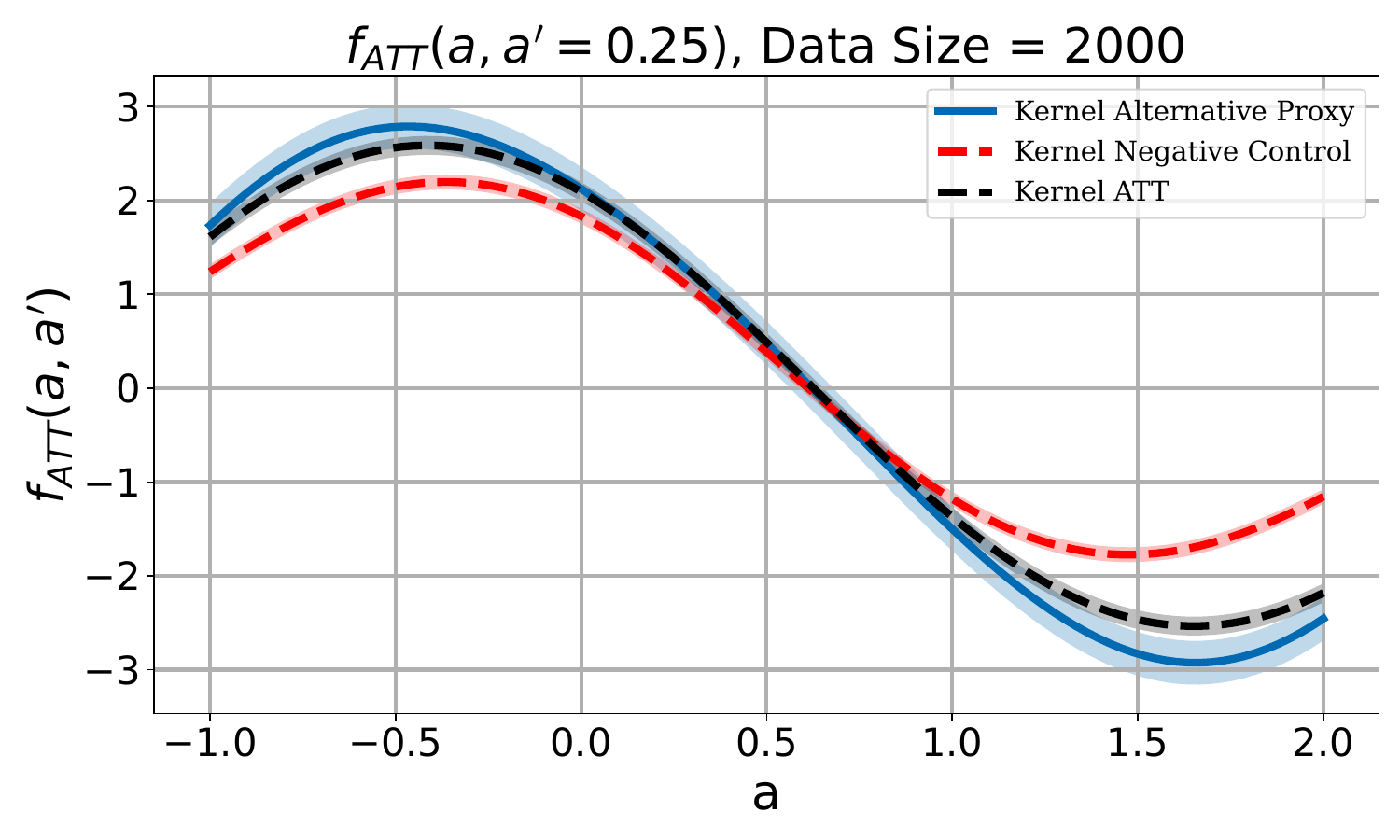}}
\subfloat[]
{\includegraphics[trim = {0cm 0cm 0cm 0.0cm},clip,width=0.225\textwidth]{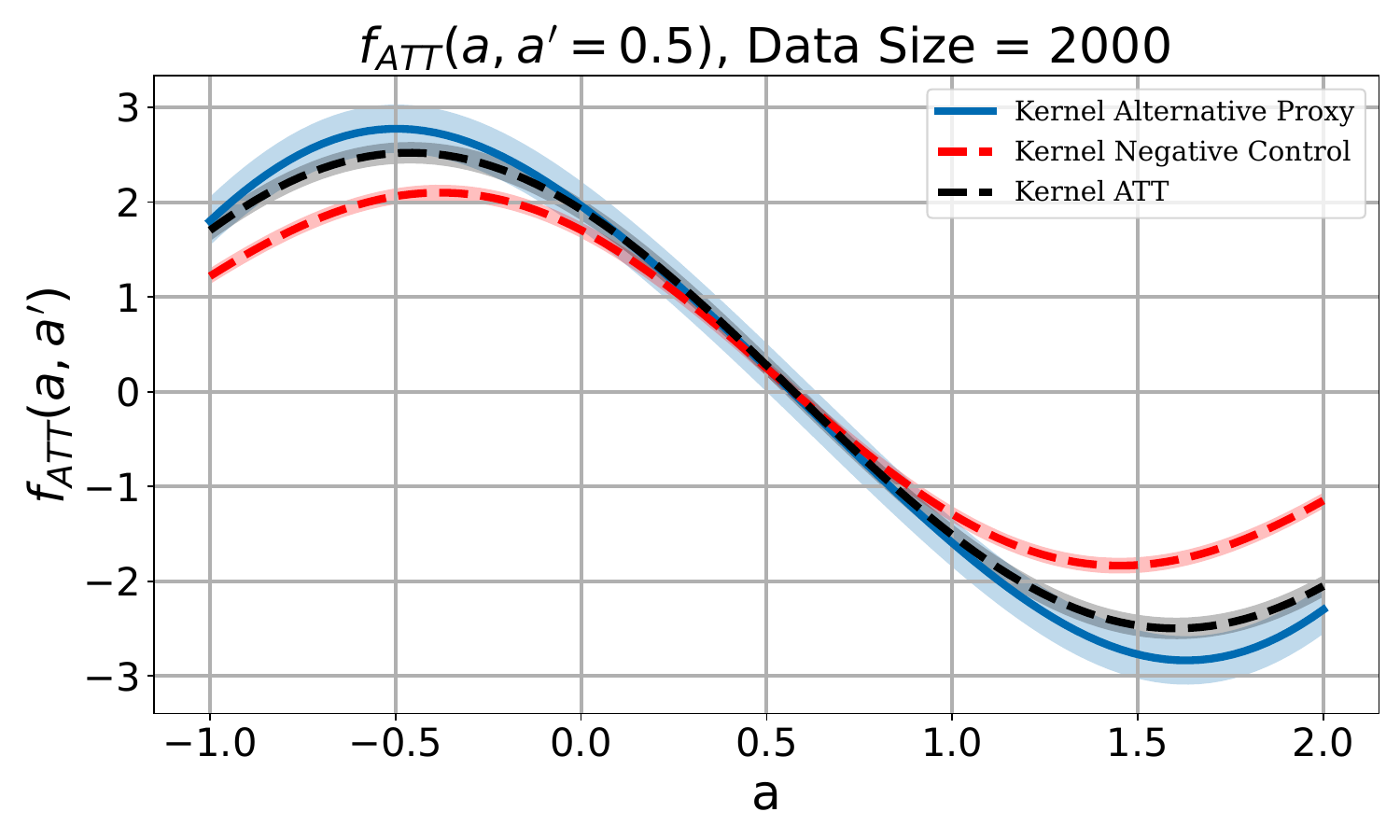}}
\newline\caption{Conditional dose-response curve estimation for synthetic low-dimensional data across $a'$ values and algorithms (averaged over 30 different runs) - mean solid line and standard deviation envelopes.\looseness=-1}
\label{fig:alternativeProxyMethodATTComparison}
\end{figure}
\section{CONCLUSION}
\label{sec:Conclusion}
We propose a methodology for proxy causal learning that leverages a treatment bridge function. Our method enables the recovery of causal effects in the graphical model illustrated in Figure (\ref{fig:ProxyCausalDAG}). It requires access to two proxy variables to the latent confounding variable, along with widely used completeness assumptions in the PCL setting. Our approach is practical for two key reasons. First, it avoids density ratio estimation - a challenging task in high dimensions. This enables strong performance even for high dimensional treatments, as demonstrated in the dSprite experiment. Second, the RKHS formulation of the problem allows for strong consistency guarantees, and closed form solutions that are easily implemented via matrix operations. 

\subsubsection*{Acknowledgements}
Bariscan Bozkurt, Dimitri Meunier, Liyuan Xu, and Arthur Gretton are supported by the Gatsby Charitable Foundation. We also thank the AISTATS 2025 reviewers for their valuable feedback and discussions.

\bibliography{aistats_bibliography}

\clearpage

\section*{Checklist}

 \begin{enumerate}

 \item For all models and algorithms presented, check if you include:
 \begin{enumerate}
   \item A clear description of the mathematical setting, assumptions, algorithm, and/or model. \bboz{[Yes]}. We provide a detailed account of the necessary assumptions and mathematical settings for the algorithms presented in Sections (\ref{sec:AlternativeProxyIdentification}) and (\ref{sec:AlternativeProxyAlgorithms}). The derivations of the algorithms are elaborated in S.M. (Sec. \ref{sec:structural_funcs_algorithm_proof}).
   \item An analysis of the properties and complexity (time, space, sample size) of any algorithm. \bboz{[Yes]}. We include ablation studies, hyperparameter tuning procedures, and discussions on complexity in the S.M. (Sec. \ref{sec:SupplementaryNumericalExperiments}).
   \item (Optional) Anonymized source code, with specification of all dependencies, including external libraries. \bboz{[Yes]}. We provide an anonymous GitHub URL for our implementation in the S.M. (Sec. \ref{sec:NumericalExperiments}). (For the camera-ready version of the paper, we include a non-anonymized GitHub URL in S.M.(Sec. \ref{sec:FurtherNotesOnNumericalExperiments}).)
 \end{enumerate}

 \item For any theoretical claim, check if you include:
 \begin{enumerate}
   \item Statements of the full set of assumptions of all theoretical results. \bboz{[Yes]}.
   \item Complete proofs of all theoretical results. \bboz{[Yes]}. For identifiability and existence proofs, see the S.M. (Sec. \ref{sec:Identification_Appendix}) and (Sec. \ref{sec:existenceOfBridgeFunction}). Consistency proofs are provided S.M.  (\ref{sec:consistency_appendix}). For algorithm derivations, refer to the S.M. (\ref{sec:SupplementaryNumericalExperiments})
   \item Clear explanations of any assumptions. \bboz{[Yes]}.   
 \end{enumerate}

 \item For all figures and tables that present empirical results, check if you include:
 \begin{enumerate}
   \item The code, data, and instructions needed to reproduce the main experimental results (either in the supplemental material or as a URL). \bboz{[Yes]}. See the anonymous GitHub URL for the implementation code, which includes a README.md file with instructions for reproducing the results.
   \item All the training details (e.g., data splits, hyperparameters, how they were chosen). \bboz{[Yes]}. Refer to the S.M. (\ref{sec:SupplementaryNumericalExperiments}) for hyperparameter tuning procedures, ablation studies, and further details on numerical experiments.
     \item A clear definition of the specific measure or statistics and error bars (e.g., with respect to the random seed after running experiments multiple times). \bboz{[Yes]}. Our experimental results are based on multiple realizations, with box plots and/or standard deviation envelopes illustrating the variability.
     \item A description of the computing infrastructure used. (e.g., type of GPUs, internal cluster, or cloud provider). \bboz{[No]}. We did not include details on computational resources since our method does not require advanced computing infrastructure. Each causal learning experiment can be run on a standard computer.
 \end{enumerate}

 \item If you are using existing assets (e.g., code, data, models) or curating/releasing new assets, check if you include:
 \begin{enumerate}
   \item Citations of the creator If your work uses existing assets. \bboz{[Yes]}. For 'Legalized Abortion and Crime Dataset', we cite \citep{LegalizedAbortionData, woody2020estimatingheterogeneouseffectscontinuous, Mastouri2021ProximalCL} in Section (\ref{sec:NumericalExperiments}). For 'Grade Retention Experiment', we cite \citep{Fruehwirth_GradeRetentions, deaner2023proxycontrolspaneldata} in Section (\ref{sec:NumericalExperiments}). The other experiments are based on synthetic data generation processes.
   \item The license information of the assets, if applicable. \bboz{[Not Applicable]}.
   \item New assets either in the supplemental material or as a URL, if applicable. \bboz{[Not Applicable]}
   \item Information about consent from data providers/curators. \bboz{[Not Applicable]}
   \item Discussion of sensible content if applicable, e.g., personally identifiable information or offensive content. \bboz{[Not Applicable]}
 \end{enumerate}

 \item If you used crowdsourcing or conducted research with human subjects, check if you include:
 \begin{enumerate}
   \item The full text of instructions given to participants and screenshots. \bboz{[Not Applicable]}
   \item Descriptions of potential participant risks, with links to Institutional Review Board (IRB) approvals if applicable. \bboz{[Not Applicable]}
   \item The estimated hourly wage paid to participants and the total amount spent on participant compensation. \bboz{[Not Applicable]}
 \end{enumerate}

 \end{enumerate}

\onecolumn
\aistatstitle{Supplementary Materials for Density Ratio-based Proxy Causal Learning Without Density Ratios}

Section (\ref{sec:Comparsion_to_Baselines_Appendix}) reviews the previous work on treatment bridge function-based and outcome bridge function-based proxy causal learning frameworks. In Section~(\ref{sec:Identification_Appendix}), we provide the proofs for identifiability of the dose-response and the conditional dose-response curves. Section~(\ref{sec:existenceOfBridgeFunction}) contains the proofs for existence of the bridge function. In Section (\ref{sec:structural_funcs_algorithm_proof}), we derive the algorithms for estimating the dose-response and the conditional dose-response curves. The proofs for the consistency results of our approach are provided in Section~(\ref{sec:consistency_appendix}). Section (\ref{sec:SupplementaryNumericalExperiments}) offers the additional details on the numerical experiments and ablation studies while Section (\ref{sec:IdentifiabilityOfDoseResponseDiscreteCase}) discusses  the identifiability of the dose-response in the discrete case as a specific example. Finally, Section (\ref{sec:KernelAlternativeProxyWithAdditionalCovariates}) discusses the extension of our method to settings with additional observed confounding variables.

\section{COMPARISON WITH OTHER METHODS}
\label{sec:Comparsion_to_Baselines_Appendix}
In Section (\ref{sec:NumericalExperiments}) and S.M. (Sec. \ref{sec:AdditionalNumericalExperiments}), we compare our method with other proxy causal learning algorithms: Proximal Kernel Inverse Probability Weighted (PKIPW) \citep{wu2024doubly}, Kernel Proxy Variable (KPV) \citep{Mastouri2021ProximalCL}, Proximal Maximum Moment Restriction (PMMR) \citep{Mastouri2021ProximalCL}, and Kernel Negative Control (KNC) \citep{singh2023kernelmethodsunobservedconfounding}. Here, we further examine these baselines, focusing on their assumptions, bridge functions, and suitability for high-dimensional settings.

\paragraph{Proximal Kernel Inverse Probability Weighted (PKIPW):}  PKIPW extends the binary treatment framework of \citep{semiparametricProximalCausalInference} to continuous treatments \citep{wu2024doubly}. \cite{semiparametricProximalCausalInference} introduced a treatment bridge function $\varphi_0(z, a)$ satisfying $$\E[\varphi_0(Z, a) | W, A = a] = \frac{1}{p(A = a| W)}$$ to identify the ATE in the binary case via
\begin{align*}
    f_{ATE}(a) = \E[\mathbf{1}[A = a] Y \varphi_0(Z, A)].
\end{align*}
This result relies on the following completeness assumptions (Assumptions (10) and (11) in \citep{semiparametricProximalCausalInference}):
\begin{assumption}
    Let $\ell_1 : \gU \rightarrow \R$ and $\ell_2 : \gW \rightarrow \R$ be square integrable functions. Assume that the following conditions hold for all $a \in \gA$,:
    \begin{itemize}
        \item[i.] $\E[\ell_1(U) | W = w, A = a] = 0 \quad \forall w \in \gW$ if and only if $\ell_1(U) = 0 \quad p(U) - $almost everywhere
        \item[ii.] $\E[\ell_2(W) | Z = z, A = a] = 0 \quad \forall z \in \gZ$ if and only if $\ell_2(W) = 0 \quad p(W) - $almost everywhere
    \end{itemize}
    \label{assum:PKIPW_CompletenessAssumptions}
\end{assumption}

Furthermore, they leverage the following completeness assumption in order to obtain the uniqueness of the treatment bridge function:
\begin{assumption}
    Let $\ell : \gZ \rightarrow \R$ be any square integrable function. Assume that the following conditions hold for all $a \in \gA$:
    $\E[\ell(Z) | W = w, A = a] = 0 \quad \forall w \in \gW$ if and only if $\ell(Z) = 0 \quad p(Z) - $almost everywhere.
    \label{assum:PKIPW_CompletenessAssumptions2}
\end{assumption}

PKIPW adapts this to continuous treatments by replacing the indicator function with a kernel function:
\begin{align}
    f_{ATE}(a) = \E[\mathbf{1}[A = a] Y \varphi_0(Z, A)] \approx \frac{1}{n} \sum_{i = 1}^n K_l(a_i - a) \varphi_0(z_i, a) y_i,
\end{align}
where $K_l$ is a kernel function with the bandwidth variable $l$. In particular, their approximation relies on their result
\begin{align*}
\E[\mathbf{1}[A = a] Y \varphi_0(Z, A)] = \lim_{l \rightarrow 0}[K_l(A - a) Y \varphi_0(Z, A)]    
\end{align*}
\newpage
under the condition that $\E[\mathbf{1}[A = a] Y \varphi_0(Z, A)]$ is continuous and bounded uniformly with respect to $a$ (see Theorem 4.2 in \citep{wu2024doubly}).

PKIPW requires estimating the policy function $\frac{1}{p(A = a| W)}$ via kernel density estimation or conditional normalizing flows (CNFs) which adds computational complexity. In contrast, our approach bypasses density ratio estimation by modifying the bridge function (see Section (\ref{sec:AlternativeProxyAlgorithms})). Moreover, while PKIPW assumes uniqueness of the treatment bridge function with Assumption (\ref{assum:PKIPW_CompletenessAssumptions2}), we show that convergence to the minimum RKHS norm solution suffices for consistency. Additionally, PKIPW’s completeness condition on $Z$ (Assumption (\ref{assum:PKIPW_CompletenessAssumptions}-ii)) is stronger than our Assumption (\ref{assum:ExistenceCompletenessAssumption}). Thus, our approach is more feasible and computationally desirable as reinforced by the experimental results shown in Figures (\ref{fig:alternativeProxyMethodATEComparisonLowDim}), (\ref{fig:alternativeProxyMethodATEComparisonAbortion}), and (\ref{fig:alternativeProxyMethodATEComparisonDeaner}). 
In Figure (\ref{fig:alternativeProxyMethodATEComparisonDSprite}) for dSprite experiment with high dimensional treatment, we could not produce result for PKIPW due to its published code being limited to univariate treatments. However, our method successfully handles high-dimensional treatments, as evidenced by the results in the dSprite experiment, showcasing another advantage.

\paragraph{Outcome Bridge Function-Based Approaches with Kernel Methods:} \citep{Mastouri2021ProximalCL} and \citep{singh2023kernelmethodsunobservedconfounding} introduced kernel-based PCL approaches leveraging an outcome bridge function $h(w, a)$ that satisfies
\begin{align*}
    \E[Y | A = a, Z = z] = \int_{\gW} h(a, w) p (w | a, z) dw.
\end{align*}
Given this outcome bridge function, \cite{Miao01102024} showed that the dose-response can be identified as
\begin{align*}
    f_{ATE}(a) = \int h(a, w) p(w) dw.
\end{align*}
Both \citep{Mastouri2021ProximalCL} and \citep{singh2023kernelmethodsunobservedconfounding} rely on the following completeness conditions:
\begin{assumption}
    Let $\ell_1 : \gU \rightarrow \R$ and $\ell_2 : \gW \rightarrow \R$ be square integrable functions. Assume that the following conditions hold for all $a \in \gA$, $x \in \gX$:
    \begin{itemize}
        \item[i.] $\E[\ell_1(U) | Z = z, A = a] = 0 \quad \forall z \in \gZ$ if and only if $\ell_1(U) = 0 \quad p(U) - $almost everywhere
        \item[ii.] $\E[\ell_2(Z) | W = w, A = a] = 0 \quad \forall w \in \gW$ if and only if $\ell_2(Z) = 0 \quad p(Z) - $almost everywhere.
    \end{itemize}
    \label{assum:KPV_CompletenessAssumptions}
\end{assumption}
\citep{xu2021deep, deaner2023proxycontrolspaneldata} later showed that Assumption (\ref{assum:KPV_CompletenessAssumptions}-ii) can be replaced by the weaker condition:
\begin{assumption}
    Let $\ell : \gU \rightarrow \R$ be any square integrable function. Assume that the following conditions hold for all $a \in \gA$:
    $\E[\ell(U) | W = w, A = a] = 0 \quad \forall w \in \gW$ if and only if $\ell(U) = 0 \quad p(U) - $almost everywhere.
    \label{assum:KPV_CompletenessAssumptions2}
\end{assumption}

Notably, these completeness conditions align with those in our approach: Assumption (\ref{assum:AlternativeProxyAssumptionCompleteness1}) corresponds to Assumption (\ref{assum:KPV_CompletenessAssumptions2}), while Assumption (\ref{assum:ExistenceCompletenessAssumption}) matches Assumption (\ref{assum:KPV_CompletenessAssumptions}-i). However, they serve different purposes. We use Assumption (\ref{assum:AlternativeProxyAssumptionCompleteness1}) to ensure identifiability of the causal structural function with treatment bridge function, while outcome bridge-based methods use it to guarantee the existence of the outcome bridge function. Similarly, we employ Assumption (\ref{assum:ExistenceCompletenessAssumption}) to establish the existence of the treatment bridge function, whereas outcome bridge-based methods use it for causal function identifiability.

\cite{Mastouri2021ProximalCL} introduced two estimation algorithms: KPV, which learns the bridge function in two stages, and PMMR, which leverages a closed-form solution under the Maximum Moment Restriction framework \citep{Muandet2020KernelCM}. \cite{singh2023kernelmethodsunobservedconfounding} proposed KNC, an alternative kernel-based approach with a different outcome bridge function representation. Since both methods employ kernel techniques, they are well-suited for high-dimensional settings. For algorithmic details, we refer readers to the respective papers.
 
\section{IDENTIFICATION PROOFS} \label{sec:Identification_Appendix}
\subsection{Identifiability of Dose-Response Curve}
\label{sec:IdentificationATE_Appendix}
In this section, we prove Theorem (\ref{thm:AlternativeProxyATEandATTIdentification}-1.) for $f_{\text{ATE}}$. For completeness, we state the theorem here again.
\begin{theorem}
Suppose that there exists a square integrable bridge function $\varphi_0: \gZ \times \gA \to \R$ that satisfies the following functional equation
\begin{align*}
    \E[\varphi_0(Z, a) | W, A = a] = \frac{p(W) p(A)}{p(W, A)}, \quad a \in \gA.
\end{align*}
Furthermore, suppose that Assumptions (\ref{assum:ProxyCausalAssumptions1}) and (\ref{assum:AlternativeProxyAssumptionCompleteness1}) hold. Then, the dose-response curve is given by 
\begin{align*}
    f_{\text{ATE}}(a) = \E[Y \varphi_0(Z, a) | A = a], \quad a \in \gA.
\end{align*} 
\end{theorem}

\begin{proof}
Suppose that 
\begin{align*}
    \E[\varphi_0(Z, A) | W, A] = \frac{p(W) p(A)}{p(W, A)}.
\end{align*}
Then, note the following, for $a \in \gA$,
$$
\begin{aligned}
\E[\varphi_0(Z, a) | W, A = a] &= \E_{U|W, A = a}[\E[\varphi_0(Z, a) | U, W, A = a]]  \quad \text{(by Law of Total Expectation)}\\
&= \E_{U|W, A = a}[\E[\varphi_0(Z, a) | U, A = a]] \quad \text{(since $Z \perp W | U$, Assumption (\ref{assum:ProxyCausalAssumptions1}))}
\end{aligned}
$$
On the other hand,
\begin{align*}
    p(w) &= \int p(w| u) p(u) d u = \int p(w| a, u) p(u) d u \quad \text{(since $W \perp A | U$, Assumption (\ref{assum:ProxyCausalAssumptions1}))}\\
    &= \int \frac{p(u | w, a) p(w | a)}{p(u | a)} p(u) d u \quad (\text{Baye's Rule})\\
    &= \int \frac{p(u | w, a) p(w , a)}{p(u , a)} p(u) d u = p(w, a) \int \frac{p(u)}{p(u, a)} p(u | w, a) d u\\
    &= p(w, a) \E \Bigg[  \frac{p(U)}{p(U, a)} \Bigg| W = w, A = a \Bigg]
\end{align*}
As a result,
\begin{equation} \label{eq:bridge_W_to_U}
    \frac{p(W)}{p(W, a)} = \E \Bigg[  \frac{p(U)}{p(U, a)} \Bigg| W, A = a \Bigg] \Rightarrow \frac{p(W) p(a)}{p(W, a)} = \E \Bigg[  \frac{p(U) p(a)}{p(U, a)} \Bigg|W, A = a\Bigg].
\end{equation}

For every value $a \in \gA$, it therefore holds $p(W)$-a.e. that 
$$
\E_{U|W, A = a}[\E[\varphi_0(Z, a) | U, A = a]] = \E \Bigg[  \frac{p(U) p(a)}{p(U, a)} \Bigg|W, A = a\Bigg]. 
$$
Therefore, due to Assumption~(\ref{assum:AlternativeProxyAssumptionCompleteness1}), we have for all $a \in \gA$,
\begin{align}
    \E[\varphi_0(Z, a) | U, A = a ] = \frac{p(U) p(a)}{p(U, a)} \quad p(U)-\text{a.e.}
\label{eq:AlternativeProxyBridgeFuncProposition2}
\end{align}
Next, we observe that
\begin{align*}
    &\E[ \E[Y | A = a, U]] = \int \E[Y | A = a, u] p(u) d u = \int \E[Y | A = a, u] \frac{p(u) p(a)}{p(u, a)} \frac{p(u, a)}{p(a)} d u\\
    &= \int \E[Y | A = a, u] \frac{p(u) p(a)}{p(u, a)} p(u| a) d u = \E_{U | A = a} \Bigg[ \E[Y | A = a, U] \frac{p(U) p(a)}{p(U, a)} \Bigg]\\
    &= \E_{U| A = a} \Bigg[ \E[Y | A = a, U] \E[\varphi_0(Z, a) | U, A = a ] \Bigg] \quad \text{(by Equation \ref{eq:AlternativeProxyBridgeFuncProposition2})}\\
    &= \E_{U | A = a} \Bigg[ \int y p(y | a, U) d y \int \varphi_0(z, a) p(z| U, a) d z \Bigg]\\
    &= \E_{U | A = a} \Bigg[\int \varphi_0(z, a) \Bigg( \int y p(y | a, U) d y \Bigg) p(z| U, a) d z \Bigg]\\
    &= \E_{U | A = a} \Bigg[\int \varphi_0(z, a) \Bigg( \int y p(y | a, U, {z}) d y \Bigg) p(z| U, a) d z \Bigg] \quad \text{(since $Y \perp Z | A, U $, Assump. (\ref{assum:ProxyCausalAssumptions1}))}\\
    &= \E_{U| A = a} \Bigg[\int \int\varphi_0(z, a)  y \underbrace{p(y | a, U, z)   p(z| U,  a)}_{p(y, z|  a, U)} d y  d z \Bigg] \\
    &= \int \int \int \varphi_0(z, a)  y \underbrace{p(y, z| a, u) p(u | a)}_{p(u, y, z | a)} d y  d z du\\
    &= \int \int \varphi_0(z, a)  y \underbrace{\int p(u, y, z | a) d u}_{p( y, z | a)} d y  d z \\
    &= \int \int \varphi_0(z,  a)  y p( y, z | a) d y  d z = \E[Y \varphi_0(Z, a) | A = a].
\end{align*}
As a result, we obtain
\begin{align*}
    \E[Y\varphi_0(Z, a) | A = a] = \E[ \E[Y | A = a, U]],
\end{align*}
which indicates that $f_{\text{ATE}}(a) = \E[Y \varphi_0(Z, a) | A = a]$.
\end{proof}

\subsection{Identifiability of Conditional Dose-Response Curve}
\label{sec:IdentificationATT_Appendix}
In this section, we prove Theorem (\ref{thm:AlternativeProxyATEandATTIdentification}-2.) for $f_{\text{ATT}}$. For completeness, we state the theorem here again.

\begin{theorem}
    Suppose there exists a square integrable bridge function $\varphi_0: \gZ \times \gA \times \gA \to \R$ that satisfies the following functional equation  
\begin{align*}
    \E[\varphi_0(Z, a, a') | W, A = a] = \frac{p(W, a') p(a)}{p(W, a) p(a')}, \qquad (a,a') \in \gA^2.
\end{align*} 
Furthermore, suppose that Assumptions (\ref{assum:ProxyCausalAssumptions1}) and (\ref{assum:AlternativeProxyAssumptionCompleteness1}) hold. Then, the conditional dose-response is given by 
\begin{align*}
    f_{\text{ATT}}(a, a') = \E[Y \varphi_0(Z, a, a') | A = a], \qquad (a,a') \in \gA^2.
\end{align*}
\end{theorem}
\begin{proof}
First, observe the following, for $(a,a') \in \gA^2$,
\begin{align}
    &\E[\varphi_0(Z, a, a') | W, A = a] = \E_{U| W, A = a}\big[\E[\varphi_0(Z, a, a') | U, W, A = a] \big] \quad \text{(by Law of Total Expectation)} \nonumber\\
    &= \E_{U| W, A = a}\big[\E[\varphi_0(Z, a, a') | U, A = a] \big] \quad \text{(since $Z \perp W | U, A = a$, Assumption (\ref{assum:ProxyCausalAssumptions1}))} \label{eq:AlternativeProxyATTPropositionImportantEquation1}
\end{align}
Furthermore, note that
\begin{align*}
    p(w, a') &= \int p(w, a' | u) p(u) d u = \int p(w | u) p(a' | u) p(u) d u \quad \text{(since $W \perp A | U$, Assumption (\ref{assum:ProxyCausalAssumptions1}))}\\
    &= \int p(w | a, u) p(a' | u) p(u) d u \quad \text{(again due to $W \perp A | U$, Assumption (\ref{assum:ProxyCausalAssumptions1}))}\\
    &= \int\frac{ p(u | w, a) p(w | a)}{p(u | a)} p(a' | u) p(u) d u \quad \text{(Baye's Rule)}\\
    &= \int\frac{ p(u | w, a) p(w , a)}{p(u , a)} p(u, a' )  d u  = p(w , a) \int \frac{p(u, a')}{p(u, a)} p(u | w, a) d u.
\end{align*}
As a result,
\begin{align}
    \frac{p(W, a')}{p(W , a)} = \E\left[\frac{p(U, a')}{p(U, a)} \Bigg|  W, A=a \right] \Rightarrow     \frac{p(W, a') p(a)}{p(W , a) p(a')} = \E\left[\frac{p(U, a')p(a)}{p(U, a)p(a')} \Bigg|  W, A=a \right].
\label{eq:AlternativeProxyATTPropositionImportantEquation2}
\end{align}
Recall that our assumption was
\begin{align*}
\E[\varphi_0(Z, a, a') | W, A = a] = \frac{p(W, a') p(a)}{p(W, a) p(a)}.
\end{align*}
Combining Equation (\ref{eq:AlternativeProxyATTPropositionImportantEquation1}) and (\ref{eq:AlternativeProxyATTPropositionImportantEquation2}), for every value $a,a' \in \gA$, it therefore holds $p(W)$-a.e. that 
\begin{align*}
    \E_{U | W, A = a} \Big[ \E[\varphi_0(Z, a, a') | U, A = a] \Big] = \E_{U | W, A = a} \Bigg[ \frac{p(U, a') p(a)}{p(U, a) p(a')} \Bigg].
\end{align*}
Using the completeness Assumption (\ref{assum:AlternativeProxyAssumptionCompleteness1}), we obtain, for all  $a,a' \in \gA$
\begin{align}
    \E[\varphi_0(Z, a, a') | U, A = a] = \frac{p(U, a') p(a)}{p(U, a) p(a')} \quad p(U)\text{-a.e.} \label{eq:AlternativeProxyATTPropositionImportantEquation3}
\end{align}
Next, to obtain the ATT function, consider
\begin{align*}
    &f_{\text{ATT}}(a, a') = \E_{U| A = a'}[\E[Y | A = a, U]] =\int \E[Y | A = a, u] p(u| a') d u\\
    &=\int \E[Y | A = a, u] \frac{p(u, a')}{p(a')} \frac{p(a)}{p(u, a)} \frac{p(u, a)}{p(a)} d u\\
    &=\int \E[Y | A = a, u] \frac{p(u, a') p(a)}{p(u, a ) p(a')}p(u| a) d u\\
    &= \E_{U| A = a} \Big[ \E[Y | A = a, U] \E[\varphi_0(Z, a, a') | U, A = a] \Big] \quad \text{(by Equation (\ref{eq:AlternativeProxyATTPropositionImportantEquation3}))}\\
    &= \E_{U| A = a} \Bigg[ \int y p(y| a, U) d y \int \varphi_0(z, a, a') p(z | a, U) d z
 \Bigg]\\
 &= \E_{U | A = a} \Bigg[ \int \int y \varphi_0(z, a, a') p(y| a, U)  p(z | a, U) d y d z
 \Bigg]\\
  &= \E_{U| A = a} \Bigg[ \int \int y \varphi_0(z, a, a') \underbrace{p(y| a, U, z)  p(z | a, U)}_{p(y, z|a, U)} d y d z
 \Bigg] \quad \text{($Y \perp Z | U, A$, Assumption (\ref{assum:ProxyCausalAssumptions1}) )}\\
 &= \E_{U| A=a} \Bigg[ \int \int y \varphi_0(z, a, a') p(y, z|a, U) d y d z
 \Bigg]\\
 &= \int \int \int y \varphi_0(z, a, a') \underbrace{p(y, z|a, u) p(u| a)}_{p(u, y, z |  a)}  d y d z d u\\
 &= \int \int y \varphi_0(z, a, a') \underbrace{\int  p(u, y, z | a) d u}_{p(y, z|a)} d y d z \\
 &=\int \int y \varphi_0(z, a, a') p(y, z|a) d y d z = \E [Y \varphi_0(Z, a, a') | A = a].
\end{align*}
Hence, we have shown that $f_{\text{ATT}}(a, a') = \E_{U| A = a'}[\E[Y | A = a, U]] = \E [Y \varphi_0(Z, a, a') | A = a]$.
\end{proof}

\section{EXISTENCE OF BRIDGE FUNCTIONS}
\label{sec:existenceOfBridgeFunction}

Our proofs follow the strategy of \citep{deaner2023proxycontrolspaneldata}. We note that Assumption~(\ref{assum:ExistenceCompletenessAssumption}) is weaker that the assumption used in \citep{semiparametricProximalCausalInference, wu2024doubly}. Namely they use the stronger assumption that for any square integrable function $\ell: \gW \rightarrow \R$, for all $a \in \gA$: $\E[\ell(W) | Z, A = a] = 0$ $p(Z)-$a.e. if and only if $\ell(W) = 0$ $p(W)$-a.e.

\subsection{Existence of Bridge Function for Dose-Response Curve}
\label{sec:existenceOfBridgeFunctionATE}
We will discuss conditions that will guarantee the existence of the bridge function $\varphi_0$ for the functional Equation (\ref{eq:AlternativeProxyATEBridgeFunction}). To this end, we consider the following conditional mean operator
\begin{align*}
    E_a: &\gL_2(\gW, p_{W|A = a}) \rightarrow \gL_2(\gZ, p_{Z|A=a}),\\
\end{align*}
such that
\begin{align*}
    E_a \ell(z) &= \E[\ell(W) | Z = z, A = a], \qquad z \in \gZ\\
\end{align*}
whose adjoint is given by 
\begin{align*}
    E_a^*: &\gL_2(\gZ, p_{Z|A=a}) \rightarrow \gL_2(\gW, p_{W|A = a})
\end{align*}
such that
\begin{align*}
    E_a^* \ell(w) &= \E[\ell(Z) | W = w, A = a], \qquad w \in \gW,
\end{align*}
where $\gL_2(\gW, p_{W|A = a})$ denotes the square integrable functions of $w \in \gW$ with respect to the distribution $p(W | A = a)$ and $\gL_2(\gZ, p_{Z|A = a})$ denotes the square integrable functions of $z \in \gZ$ with respect to the distribution $p(Z | A = a)$.
Our result will rely on Picard's Theorem as stated below.

\begin{theorem}[Picard's Theorem; Theorem 15.8 in \citep{linearIntegralEquationskress2013}]
    Let $\gX, \gY$ be Hilbert spaces and let $A : \gX \rightarrow \gY$ be a compact linear operator with singular system $(\mu_n, \varphi_n, g_n)_{n=1}^{\infty}$. The equation of the first kind
    \begin{align*}
        A \varphi = f
    \end{align*}
    admits a solution if and only if $f \in \operatorname{null} (A^*) ^\perp$ and 
    \begin{align*}
        \sum_{n = 1}^\infty \frac{1}{\mu_n^2} |\langle f, g_n  \rangle_\gY|^2 < \infty.
    \end{align*}
    Here, $A^*$ is the adjoint of the operator $A$ and $\operatorname{null}(A^*)$ represents the null-space of the operator $A^*$. Then a solution is given by 
    \begin{align*}
        \varphi = \sum_{n=1}^{\infty} \frac{1}{\mu_n} \langle  f, g_n \rangle_\gY \varphi_n.
    \end{align*}
    \label{thm:PicardsTheorem}
\end{theorem}

\begin{lemma} Suppose Assumptions (\ref{assum:ProxyCausalAssumptions1}) and (\ref{assum:ExistenceCompletenessAssumption}) hold. Then,
    $$\operatorname{null}(E_a) \subseteq \big\{\ell \in \gL_2(\gW, p_{W|A = a}) \enspace \big| \enspace \E[\ell(W) | U, A = a] = 0\big\}.$$
    \label{lemma:NullSpaceOfE_a}
\end{lemma}
\begin{proof}
    Observe that $\text{null}(E_a) =  \big\{\ell \in \gL_2(\gW, p_{W|A = a}) \enspace \big| \enspace \E[\ell(W) | Z, A = a] = 0\big\}$. Let $\ell \in \operatorname{null}(E_a)$, then,
    \begin{align*}
        0 &= E_a \ell = \E[\ell(W) | Z= \cdot, A = a]\\
        &= \E_{U|Z= \cdot, A = a}\big[\E[\ell(W) | U, Z= \cdot, A = a]\big]\\
        &= \E_{U|Z= \cdot, A = a}\big[\E[\ell(W) | U, A = a]\big] \quad \text{(since $W \perp Z|U, A$, Assumption (\ref{assum:ProxyCausalAssumptions1}))}\\
    \end{align*}
    Assumption (\ref{assum:ExistenceCompletenessAssumption}) implies that $\E[\ell(W) | U, A = a] = 0$ almost surely. Hence,
    \begin{align*}
        \ell \in \big\{\ell \in \gL_2(\gW, p_{W|A = a}) \enspace \big| \enspace \E[\ell(W) | U, A = a] = 0\big\}.
    \end{align*}
\end{proof}
Let us introduce the notation $\varphi_{0}^{a}:=\varphi_0(\cdot, a)$ and $r_{a}(W) := \frac{p(W) p(a)}{p(W, a)}$. The bridge function is then solution of 
$$
E_{a}^*\varphi_{0}^{a}=r_{a}.
$$
$E_{a}$ is well-defined, linear and bounded with $\|E_{a}\| \leq 1$. In order to apply Theorem (\ref{thm:PicardsTheorem}), we make the following additional assumptions.
\begin{assumption}
    For each $a \in \gA$, the operator $E_a^*$ is compact with singular system $\{\mu_{a, n}, \varphi_{a, n}, g_{a, n}\}_{n = 1}^\infty$.
    \label{assum:F_aIsCompact}
\end{assumption}
\begin{assumption}
    The density ratio function $r_a$ satisfies
    \begin{align*}
        \sum_{n = 1}^\infty \frac{1}{\mu_{a, n}^2} \big|\langle r_a, g_{a, n}\rangle_{\gL_2(\gW, p_{W|A = a})}\big|^2 < \infty.
    \end{align*}
    \label{assum:DensityRatioExistenceAssumption}
\end{assumption}
\begin{lemma}
    Suppose that Assumptions (\ref{assum:ProxyCausalAssumptions1}) and (\ref{assum:ExistenceCompletenessAssumption}) hold. Then, $r_a \in \operatorname{null}(E_a)^\perp$.
    \label{lemma:r_aInTheNullSpace}
\end{lemma}
\begin{proof}
    Recall that we previously proved that (Eq.~\ref{eq:bridge_W_to_U}),
    \begin{align*}
        \frac{p(W) p(a)}{p(W, a)} = \E \Big[\frac{p(U) p(a)}{p(U, a)} \Big| W, A = a\Big].
    \end{align*}
    Let $\ell \in \operatorname{null}(E_a)$. Then,
    \begin{align*}
        \big\langle \ell, r_a  \big\rangle_{\gL_2(\gW, p_{W|A = a})} &= \E[\ell(W) r_a(W) | A = a]\\
        &= \E\Big[\ell(W) \E\Big[\frac{p(U)p(a)}{p(U, a)} \Big| W, A = a\Big] \Big| A = a\Big]\\
        &= \iint \ell(w) \frac{p(u)p(a)}{p(u, a)} p(u | w, a) p(w | a) du dw\\
        &= \iint \ell(w) \frac{p(u)p(a)}{p(u, a)} p(u, w | a) du dw\\
        &= \iint \ell(w) \frac{p(u)p(a)}{p(u, a)} p(w | u, a) p(u | a) du dw\\
        &= \E\Big[\E[\ell(W) | U, A = a] \frac{p(U)p(a)}{p(U, a)} \Big| A = a\Big] = 0,
    \end{align*}
    where the last equality is due to Lemma (\ref{lemma:NullSpaceOfE_a}). Therefore, $r_a \in \operatorname{null}(E_a)^\perp$.
\end{proof}

\begin{theorem}
    Suppose that Assumptions (\ref{assum:ProxyCausalAssumptions1}), (\ref{assum:ExistenceCompletenessAssumption}), (\ref{assum:F_aIsCompact}), and (\ref{assum:DensityRatioExistenceAssumption}) hold. Then, there exists a solution to the functional equation 
\begin{align*}
    \E[\varphi_0(Z, a) | W, A = a] = \frac{p(W) p(a)}{p(W, a)}, \qquad a \in \gA.
\end{align*}
\end{theorem}
\begin{proof}
    Fix $a \in \gA$. The functional equation can be written as $E_a^* \varphi_0^a = r_a$. In Lemma (\ref{lemma:r_aInTheNullSpace}), we proved that $r_a \in \operatorname{null}(E_a)^{\perp}$. Furthermore, combined with Assumptions (\ref{assum:F_aIsCompact}) and (\ref{assum:DensityRatioExistenceAssumption}), Theorem (\ref{thm:PicardsTheorem}) implies the existence.
\end{proof}

\subsection{Existence of the Bridge Function for Conditional Dose-Response}
Let us fix $a,a' \in \gA$. In this section, we introduce the notations:
\begin{align*}
    \varphi_0^{a, a'} &= \varphi_0(\cdot, a, a'),\\
    r_{a, a'}(W) &= \frac{p(W, a') p(a)}{p(W, a) p(a')}.
\end{align*}
Then, the functional equation of the ATT bridge function
\begin{align*}
    \E[\varphi_0(Z, a, a') | W, A = a] = \frac{p(W, a') p(a)}{p(W, a) p(a')},
\end{align*} 
can be written as $E_a^* \varphi_0^{a, a'} = r_{a, a'}$, where $E_a^*$ is the conditional expectation operator that we defined in S.M. (\ref{sec:existenceOfBridgeFunctionATE}). Our result will again rely on Theorem (\ref{thm:PicardsTheorem}), so we make the following assumption.
\begin{assumption}
    The density ratio function $r_{a,a'}$ satisfies
    \begin{align*}
        \sum_{n = 1}^\infty \frac{1}{\mu_{a, n}^2} \big|\langle r_{a, a'}, g_{a, n}\rangle_{\gL_2(\gW, p_{W|A = a})}\big|^2 < \infty.
    \end{align*}
    \label{assum:DensityRatioExistenceAssumptionATT}
\end{assumption}

\begin{lemma}
    Suppose that Assumptions (\ref{assum:ProxyCausalAssumptions1}) and (\ref{assum:ExistenceCompletenessAssumption}) hold. Then, $r_{a, a'} \in \operatorname{null}(E_a)^\perp$.
    \label{lemma:r_aaPrimeInTheNullSpace}
\end{lemma}

\begin{proof}
    Recall that we previously proved that (Eq.~(\ref{eq:AlternativeProxyATTPropositionImportantEquation2}),
    \begin{align*}
        \frac{p(W, a') p(a)}{p(W, a) p(a')} = \E \Big[\frac{p(U, a') p(a)}{p(U, a) p(a')} \Big| W, A = a\Big].
    \end{align*}
    Let $\ell \in \operatorname{null}(E_a)$. Then,
    \begin{align*}
        \big\langle \ell, r_{a, a'}  \big\rangle_{\gL_2(\gW, p_{W|A = a})} &= \E[\ell(W) r_{a, a'}(W) | A = a]\\
        &= \E\Big[\ell(W) \E\Big[\frac{p(U, a') p(a)}{p(U, a) p(a')} \Big| W, A = a\Big] \Big| A = a\Big]\\
        &= \iint \ell(w) \frac{p(u, a') p(a)}{p(u, a) p(a')} p(u | w, a) p(w | a) du dw\\
        &= \iint \ell(w) \frac{p(u, a') p(a)}{p(u, a) p(a')} p(u, w | a) du dw\\
        &= \iint \ell(w) \frac{p(u, a') p(a)}{p(u, a) p(a')} p(w | u, a) p(u | a) du dw\\
        &= \E\Big[\E[\ell(W) | U, A = a] \frac{p(U, a') p(a)}{p(U, a) p(a')} \Big| A = a\Big] = 0,
    \end{align*}
    where the last equality is due to Lemma (\ref{lemma:NullSpaceOfE_a}). Therefore, $r_{a,a'} \in \operatorname{null}(E_a)^\perp$.
\end{proof}

\begin{theorem}
    Suppose that Assumptions (\ref{assum:ProxyCausalAssumptions1}), (\ref{assum:ExistenceCompletenessAssumption}), (\ref{assum:F_aIsCompact}), and (\ref{assum:DensityRatioExistenceAssumptionATT}) hold. Then, there exists a solution to the functional equation 
\begin{align*}
    \E[\varphi_0(Z, a, a') | W, A = a] = \frac{p(W, a') p(a)}{p(W, a) p(a')}.
\end{align*}
\end{theorem}
\begin{proof}
    Fix $a,a' \in \gA$. The functional equation can be written as $E_a^* \varphi_0^{a, a'} = r_{a, a'}$. In Lemma (\ref{lemma:r_aaPrimeInTheNullSpace}), we proved that $r_{a, a'} \in \operatorname{null}(E_a)^{\perp}$. Furthermore, combined with Assumptions (\ref{assum:F_aIsCompact}) and (\ref{assum:DensityRatioExistenceAssumptionATT}), Theorem (\ref{thm:PicardsTheorem}) implies the existence.
\end{proof}

\section{ALGORITHM DERIVATIONS}
\label{sec:structural_funcs_algorithm_proof}
\subsection{Dose-Response Curve Algorithm}
\label{sec:ATE_algorithm_proof}

Here, we will derive the Algorithm (\ref{algo:ATE_algorithm_with_confounders}). For completeness, we restate the algorithm first.

\begin{algorithm_state*}[Algorithm (\ref{algo:ATE_algorithm_with_confounders})]
Let $\{z_i, w_i, a_i\}_{i = 1}^n$ and $\{\Tilde{w}_i, \Tilde{a}_i\}_{i = 1}^m$ be the first-stage and second-stage data, respectively, and $(\lambda_1, \lambda_2, \lambda_3)$ be the regularization parameters of first, second, and third stage regressions. Furthermore, let $\{\alpha_i\}_{i = 1}^{m + 1}$ be the minimizer of the following cost function:
\begin{align*}
    &\hat{\gL}^{\text{2SR}}(\malpha) = \frac{1}{m} \begin{bmatrix}
        \alpha_{1:m} \\ \alpha_{m+1}
    \end{bmatrix}^T 
    \begin{bmatrix}
        \mB^T \mK_{Z Z}\mB \odot \mK_{\Tilde{A}\Tilde{A}} \nonumber\\ (\frac{\vone}{m})^T \big[{\mB}^T \mK_{Z Z} \bar{\mB} \odot \mK_{\Tilde{A}\Tilde{A}}\big]^T
    \end{bmatrix} 
    \begin{bmatrix}
        \mB^T \mK_{Z Z}\mB \odot \mK_{\Tilde{A}\Tilde{A}} & \big[{\mB}^T \mK_{Z Z} \bar{\mB} \odot \mK_{\Tilde{A}\Tilde{A}}\big] \frac{\vone}{m}
    \end{bmatrix} 
    \begin{bmatrix}
        \alpha_{1:m}\\ \alpha_{m+1}
    \end{bmatrix}\nonumber\\
    &-2 \begin{bmatrix}
    \alpha_{1:m } \\ \alpha_{m + 1}
    \end{bmatrix}^T \begin{bmatrix}
     [ \mB^T \mK_{Z Z} \Bar{\mB} \odot \mK_{\Tilde{A} \Tilde{A}} ] \frac{\vone}{m}\\
     (\frac{\vone}{m})^T \Big[ \Bar{\mB}^T \mK_{Z Z} \Bar{\mB} \odot \mK_{\Tilde{A} \Tilde{A}}\Big] \frac{\vone}{m}
    \end{bmatrix} \nonumber\\
    &+ \lambda_2 \begin{bmatrix}
        \alpha_{1:m } \\ \alpha_{m + 1}
    \end{bmatrix}^T 
    \begin{bmatrix}
        \mB^T \mK_{Z Z} \mB \odot \mK_{\Tilde{A} \Tilde{A}} & [ \mB^T \mK_{Z Z} \Bar{\mB} \odot \mK_{\Tilde{A} \Tilde{A}} ] \frac{\vone}{m}\\
        (\frac{\vone}{m})^T [ \mB^T \mK_{Z Z} \Bar{\mB} \odot \mK_{\Tilde{A} \Tilde{A}} ]^T & (\frac{\vone}{m})^T \Big[ \Bar{\mB}^T \mK_{Z Z} \Bar{\mB} \odot \mK_{\Tilde{A} \Tilde{A}}\Big] \frac{\vone}{m}
    \end{bmatrix} 
    \begin{bmatrix}
        \alpha_{1:m } \\ \alpha_{m + 1}
    \end{bmatrix} 
\end{align*}
where $\mI \in \R^{n \times n}$ is the identity matrix, $\vone$ is vector of ones, $\alpha_{1:m} = \begin{bmatrix}
        \alpha_1 & \alpha_2 & \ldots & \alpha_m
    \end{bmatrix}^T \in \R^{m}$, $\mB = (\mK_{W W} \odot \mK_{A A} + n \lambda_1 \mI)^{-1}(\mK_{W \Tilde{W}} \odot \mK_{A \Tilde{A}})\in \R^{n \times m},$
and $\bar{\mB}$ is the matrix whose $j$-th column is given by
\begin{align*}
    \Bar{\mB}_{:, j} &= \frac{1}{m} \sum_{\substack{l = 1 \\ l \ne j}}^m (\mK_{W W} \odot \mK_{A A} + n \lambda_1 \mI)^{-1}(\mK_{W \Tilde{w}_l} \odot \mK_{A \Tilde{a}_j}).
\end{align*}
Then, the dose-response curve estimation $f_{\text{ATE}}(a)$ is given by
\begin{align*}
    f_{\text{ATE}}(a) &= \alpha_{1:m}^T \Big( \mB^T \big(\mK_{Z Z} \text{diag}(\mY) [\mK_{A A} + n \lambda_3 \mI]^{-1} \mK_{A a}\big) \odot \mK_{\Tilde{A} a}\Big)\\
    &+ \alpha_{m + 1} \Big( \Bar{\mB}^T \big( \mK_{Z Z}  \text{diag}(\mY) [\mK_{A A} + n \lambda_3 \mI]^{-1} \mK_{A a} \big) \odot \mK_{\Tilde{A} a}\Big) \frac{\vone}{m}
\end{align*}
Here, for $F \in \{A, W, Z\}$ with domain $\gF$, the first-stage kernel matrices are denoted as $\mK_{FF} = [k_\gF(f_i, f_j)]_{ij} \in \R^{n \times n}$, $\mK_{Ff} = [k_\gF(f_i, f)]_{i} \in \R^{n}$, where $\{f_i\}_{i = 1}^n$ denotes the first-stage data samples. Similarly, with second-stage variables $\tilde{F} \in \{\Tilde{A}, \Tilde{W}\}$, the kernel matrices are denoted as follows: $\mK_{\Tilde{F} \Tilde{F}} = [k_\gF(\Tilde{f}_i, \Tilde{f}_j)]_{ij} \in \R^{m \times m}$, $\mK_{F \Tilde{F}} = [k_\gF(f_i, \Tilde{f}_j)]_{ij} \in \R^{n \times m}$, $\mK_{F \Tilde{f}} = [k_\gF(f_i, \Tilde{f})]_i \in \R^n$, and $\mK_{\Tilde{F} f} = [k_\gF(\Tilde{f}_j, f)]_j \in \R^m$.
\end{algorithm_state*}

\begin{proof}[Derivation of Algorithm (\ref{algo:ATE_algorithm_with_confounders})]
Let $r(W, a)$ denote $p(W) p(a) / p(W, a)$. We would like to find the optimum of the following loss function
\begin{align*}
    \gL^{2SR}(\varphi) = \E_{W, A}[(r(W, A) - \E[\varphi(Z, A) | W, A])^2] + \lambda_2 \|\varphi_0\|_{\gH_\gZ \otimes \gH_\gA}^2.
\end{align*}
This loss cannot be directly optimized as it involves the conditional mean $\E[\varphi(Z, A) | W, A]$. A similar problem in \citep{Mastouri2021ProximalCL, singh2023kernelmethodsunobservedconfounding, xu2024kernelsingleproxycontrol} is addressed using a two-stage regression approach: the first stage approximates the conditional expectation, and the second stage minimizes the loss. We build on the approach of \citet{xu2024kernelsingleproxycontrol} in our algorithm derivation, which is shown to be more numerically stable, and to require optimizing fewer parameters, than the earlier approaches: see \citet[Appendix F]{xu2024kernelsingleproxycontrol}

\emph{First Stage:} Assume that the bridge function $\varphi$ is the RKHS $\gH_\gZ \otimes \gH_\gA$. Then,
\begin{align*}
    \E[\varphi(Z, a) | W = w, A = a] &= \E[\langle \varphi, \phi_\gZ(Z) \otimes \phi_\gA(a) \rangle_{\gH_\gZ  \otimes \gH_\gA} | W = w, A = a]\\
    &= \langle \varphi, \E[\phi_\gZ(Z) | W = w, A = a] \otimes \phi_\gA(a) \rangle_{\gH_\gZ \otimes \gH_\gA} \\&= \langle \varphi, \mu_{Z|W, A} (w, a) \otimes \phi_\gA(a) \rangle_{\gH_\gZ  \otimes \gH_\gA}\\
\end{align*}
where $\mu_{Z|W, A} (w, a) = \E[\phi_\gZ(Z) | W = w, A = a]$ is the CME. Considering the sample-based first-stage regression with given data $\{z_i, w_i, a_i\}_{i = 1}^n$
\begin{align*}
    \hat{\gL}^{c}(C) = \frac{1}{n} \sum_{i = 1}^n \|\phi_\gZ(z_i) - C(\phi_\gW(w_i) \otimes \phi_\gA(a_i))\|_{\gH_\gZ}^2 + \lambda_1 \|C\|_{S_2(\gH_\gW \otimes \gH_\gA, \gH_\gZ)}^2
\end{align*}
The minimizer $\hat{C}_{Z|W, A}$ of the loss function $\hat{\gL}^{c}(C)$ is given by:
\begin{align*}
    \hat{C}_{Z|W, A} (\phi_\gW(w) \otimes \phi_\gA(a)) = \hat{\mu}_{Z| W, A} (w, a) = \sum_{i = 1}^n \beta_i(w, a) \phi_\gZ(z_i) = \Phi_\gZ {\mBeta}(w, a)
\end{align*}
where 
\begin{align*}
    {\mBeta}(w, a) &= (\mK_{W W} \odot \mK_{A A} + n \lambda \mI)^{-1}(\mK_{W w} \odot \mK_{A a})\\
    \Phi_\gZ &= \begin{bmatrix}
        \phi_\gZ(z_1) & \ldots & \phi_\gZ(z_n)
    \end{bmatrix}
\end{align*}

\emph{Second-Stage:} We first consider the simplification of the population loss for two-stage regression:
\begin{align*}
    &{\gL}^{2SR}(\varphi) = \E\big[\big(r(W, A) - \E[\varphi(Z, A) | W, A]\big)^2\big] + \lambda_2 \|\varphi\|^2_{\gH_\gZ \otimes \gH_\gA}\\
    &\propto\E\big[ \E[\varphi(Z, A) | W, A]^2 \big] - 2 \int \frac{p(w) p(a)}{p(w, a)} \E[\varphi(Z, a) | w, a] p(w, a) d w d a + \lambda_2 \|\varphi\|^2_{\gH_\gZ \otimes \gH_\gA} \\
    &=\E\big[ \E[\varphi(Z, A) | W, A]^2 \big] - 2 \E_{W} \E_{A} \big[ \E[\varphi(Z, A) | W, A]\big] + \lambda_2 \|\varphi\|^2_{\gH_\gZ \otimes \gH_\gA}\\
    &= \E\big[ \langle \varphi, \mu_{Z|W, A} (W, A)\otimes \phi_\gA(A) \rangle_{\gH_\gZ \otimes \gH_\gA}^2 \big] \\&- 2 \E_{W} \E_{A} \big[ \langle \varphi, \mu_{Z|W, A} (W, A)\otimes \phi_\gA(A) \rangle_{\gH_\gZ \otimes \gH_\gA}\big] + \lambda_2 \|\varphi\|^2_{\gH_\gZ \otimes \gH_\gA}.
\end{align*}
Recall that the set $\{\Tilde{z}_i, \Tilde{w}_i, \Tilde{a}_i\}_{i = 1}^m$ denotes the second-stage data. Then, the empirical objective can be written as
\begin{align}
    \hat{\gL}^{2SR}_m(\varphi) &= \frac{1}{m} \sum_{i = 1}^m \langle \varphi, \hat{\mu}_{Z|W, A} (\Tilde{w}_i, \Tilde{a}_i) \otimes \phi_\gA(\Tilde{a}_i) \rangle_{\gH_\gZ \otimes \gH_\gA}^2 \nonumber \\&-2 \frac{1}{m (m-1)} \sum_{i = 1}^m \sum_{\substack{j = 1 \\ j \ne i}}^m \Big\langle \varphi, \hat{\mu}_{Z|W, A} (\Tilde{w}_j, \Tilde{a}_i) \otimes \phi_\gA(\Tilde{a}_i) \Big\rangle_{\gH_\gZ \otimes \gH_\gA} + \lambda_2 \|\varphi\|^2_{\gH_\gZ \otimes \gH_\gA}
\label{eq:2StageRegressionObjectiveWithConfoundersFinal}
\end{align}
The minimizer of this objective should be in the span of the following set,
\begin{align*}
    \varphi \in \text{span} \Bigg\{ \{\hat{\mu}_{Z|W, A} (\Tilde{w}_i, \Tilde{a}_i) \otimes  \phi_\gA(\Tilde{a}_i) \}_{i = 1}^m \cup \Big\{ \frac{1}{m (m - 1)} \sum_{i = 1}^m \sum_{\substack{j = 1 \\ j \ne i}}^m \hat{\mu}_{Z|W, A} (\Tilde{w}_j, \Tilde{a}_i) \otimes \phi_\gA(\Tilde{a}_i) \Big\} \Bigg\}.
\end{align*}
Hence, we write
\begin{align*}
    {\varphi = \sum_{i = 1}^m \alpha_i  \hat{\mu}_{Z|W, A} (\Tilde{w}_i, \Tilde{a}_i) \otimes  \phi_\gA(\Tilde{a}_i) + \frac{\alpha_{m + 1}}{m (m - 1)} \sum_{j = 1}^m \sum_{\substack{l = 1 \\ l \ne j}}^m \hat{\mu}_{Z|W, A} (\Tilde{w}_l, \Tilde{a}_j) \otimes \phi_\gA(\Tilde{a}_j)}
\end{align*}

Let us notice the result of the following inner product which will come up a lot in our derivations,
\begin{align}
    &\Big\langle  \hat{\mu}_{Z|W, A} (\Tilde{w}_j, \Tilde{a}_i) \otimes  \phi_\gA(\Tilde{a}_i), \hat{\mu}_{Z|W, A} (\Tilde{w}_p, \Tilde{a}_l) \otimes \phi_\gA(\Tilde{a}_l) \Big\rangle_{\gH_\gZ \otimes \gH_\gA} \nonumber = \langle  \hat{\mu}_{Z|W, A} (\Tilde{w}_j, \Tilde{a}_i), \hat{\mu}_{Z|W, A} (\Tilde{w}_p, \Tilde{a}_l) \rangle 
    \langle \phi_\gA(\Tilde{a}_i), \phi_\gA(\Tilde{a}_l)\rangle \nonumber\\
    &= \mBeta(\Tilde{w}_j, \Tilde{a}_i)^T \Phi_\gZ^T \Phi_\gZ \mBeta(\Tilde{w}_p, \Tilde{a}_l) k_\gA(\Tilde{a}_i, \Tilde{a}_l) = \mBeta(\Tilde{w}_j, \Tilde{a}_i)^T \mK_{Z Z} \mBeta(\Tilde{w}_p, \Tilde{a}_l) k_\gA(\Tilde{a}_i, \Tilde{a}_l) \label{eq:AlternativeProxyMethodProofImportantEquation2}
\end{align}
Next, we will calculate the terms in our second-stage regression loss individually. We begin with the squared norm of $\varphi_0$:\looseness=-1
\begin{align*}
    \|\varphi\|^2  &= \langle \varphi, \varphi \rangle\nonumber\\
    &= \Bigg\langle 
    \sum_{i = 1}^m \alpha_i  \hat{\mu}_{Z|W, A} (\Tilde{w}_i, \Tilde{a}_i) \otimes \phi_\gA(\Tilde{a}_i) + \frac{\alpha_{m + 1}}{m (m - 1)} \sum_{j = 1}^m \sum_{\substack{l = 1 \\ l \ne j}}^m \hat{\mu}_{Z|W, A} (\Tilde{w}_l, \Tilde{a}_j) \otimes \phi_\gA(\Tilde{a}_j), \nonumber
    \\&\sum_{p = 1}^m \alpha_p  \hat{\mu}_{Z|W, A} (\Tilde{w}_p, \Tilde{a}_p)  \otimes \phi_\gA(\Tilde{a}_p) + \frac{\alpha_{m + 1}}{m (m - 1)}  \sum_{r = 1}^m \sum_{\substack{s = 1 \\ s \ne r}}^m \hat{\mu}_{Z|W, A} (\Tilde{w}_s, \Tilde{a}_r) \otimes  \phi_\gA(\Tilde{a}_r)
    \Bigg\rangle\nonumber\\
    &=\sum_{i = 1}^m \sum_{p = 1}^m \alpha_i \alpha_p \langle \hat{\mu}_{Z|W, A} (\Tilde{w}_i, \Tilde{a}_i)  \otimes \phi_\gA(\Tilde{a}_i), \hat{\mu}_{Z|W, A} (\Tilde{w}_p, \Tilde{a}_p)  \otimes \phi_\gA(\Tilde{a}_p)\rangle \\
    &+2 \alpha_{m + 1} \frac{1}{m (m - 1)} \sum_{i = 1}^m \sum_{j = 1}^m \sum_{\substack{l = 1 \\ l \ne j}}^m \alpha_i \langle \hat{\mu}_{Z|W, A} (\Tilde{w}_i, \Tilde{a}_i) \otimes \phi_\gA(\Tilde{a}_i),  \hat{\mu}_{Z|W, A} (\Tilde{w}_l, \Tilde{a}_j)  \otimes \phi_\gA(\Tilde{a}_j)\rangle\\
    &+ \alpha_{m + 1}^2 \frac{1}{m^2 (m - 1)^2} \sum_{j = 1}^m \sum_{\substack{l = 1 \\ l \ne j}}^m \sum_{r = 1}^m \sum_{\substack{s = 1 \\ s \ne r}}^m  \langle \hat{\mu}_{Z|W, A} (\Tilde{w}_l, \Tilde{a}_j)  \otimes \phi_\gA(\Tilde{a}_j), \hat{\mu}_{Z|W, A} (\Tilde{w}_s, \Tilde{a}_r) \otimes \phi_\gA(\Tilde{a}_r)\rangle
\end{align*}
Using Equation (\ref{eq:AlternativeProxyMethodProofImportantEquation2}), we can write
\begin{align}
&\|\varphi\|^2 = \langle \varphi, \varphi \rangle\nonumber\\
&=\sum_{i = 1}^m \sum_{p = 1}^m \alpha_i \alpha_p \mBeta(\Tilde{w}_i, \Tilde{a}_i)^T \mK_{Z Z} \mBeta(\Tilde{w}_p, \Tilde{a}_p) k_\gA(\Tilde{a}_i, \Tilde{a}_p) \label{eq:AlternativeProxyMethodVarphiNormWithConfounder1}\\&+2 \alpha_{m + 1} \frac{1}{m (m-1)} \sum_{i = 1}^m \sum_{j = 1}^m \sum_{\substack{l = 1 \\ l \ne j}}^m \alpha_i \mBeta(\Tilde{w}_i,\Tilde{a}_i)^T \mK_{Z Z} \mBeta(\Tilde{w}_l, \Tilde{a}_j) k_\gA(\Tilde{a}_i, \Tilde{a}_j)\label{eq:AlternativeProxyMethodVarphiNormWithConfounder2}\\
&+ \alpha_{m + 1}^2 \frac{1}{m^2 (m-1)^2} \sum_{j = 1}^m \sum_{\substack{l = 1 \\ l \ne j}}^m \sum_{r = 1}^m \sum_{\substack{s = 1 \\ s \ne r}}^m \mBeta(\Tilde{w}_l, \Tilde{a}_j)^T \mK_{Z Z} \mBeta(\Tilde{w}_s, \Tilde{a}_r) k_\gA(\Tilde{a}_j, \Tilde{a}_r) \label{eq:AlternativeProxyMethodVarphiNormWithConfounder3}
\end{align}
The component in Equation (\ref{eq:AlternativeProxyMethodVarphiNormWithConfounder1}):
\begin{align*}
    \sum_{i = 1}^m \sum_{p = 1}^m \alpha_i \alpha_p \mBeta(\Tilde{w}_i, \Tilde{a}_i)^T \mK_{Z Z} \mBeta(\Tilde{w}_p, \Tilde{a}_p)  k_\gA(\Tilde{a}_i, \Tilde{a}_p) = \alpha_{1:m}^T \Big[ \mB^T \mK_{Z Z} \mB \odot \mK_{\Tilde{A} \Tilde{A}} \Big] \alpha_{1:m}
\end{align*}
where 
\begin{align*}
    \alpha_{1:m} &= \begin{bmatrix}
        \alpha_1 & \alpha_2 & \ldots & \alpha_m
    \end{bmatrix}^T\\
    \mB &= (\mK_{W W} \odot \mK_{A A} + n \lambda_1 \mI)^{-1}(\mK_{W \Tilde{W}} \odot \mK_{A \Tilde{A}})
\end{align*}
The component in Equation (\ref{eq:AlternativeProxyMethodVarphiNormWithConfounder2}):
\begin{align*}
    &2 \alpha_{m + 1} \frac{1}{m (m - 1)} \sum_{i = 1}^m \sum_{j = 1}^m \sum_{\substack{l = 1 \\ l \ne j}}^m \alpha_i \mBeta(\Tilde{w}_i, \Tilde{a}_i)^T \mK_{Z Z} \mBeta(\Tilde{w}_l, \Tilde{a}_j) k_\gA(\Tilde{a}_i, \Tilde{a}_j)\\
    &= 2 \alpha_{m + 1} \frac{1}{m} \sum_{i = 1}^m \sum_{j = 1}^m \alpha_i \mBeta(\Tilde{w}_i, \Tilde{a}_i)^T \mK_{Z Z} \Big( \frac{1}{m - 1} \sum_{\substack{l = 1 \\ l \ne j}}^m \mBeta(\Tilde{w}_l, \Tilde{a}_j) \Big) k_\gA(\Tilde{a}_i, \Tilde{a}_j)\\
    &=2 \alpha_{m + 1} \alpha_{1:m}^T\Big[ \mB^T \mK_{Z Z} \Bar{\mB} \odot \mK_{\Tilde{A} \Tilde{A}} \Big] (\vone / m)
\end{align*}
where $\bar{\mB}$ is the matrix whose $j$-th column is given by
\begin{align*}
    \Bar{\mB}_{:, j} &= \frac{1}{m - 1} \sum_{\substack{l = 1 \\ l \ne j}}^m \mBeta(\Tilde{w}_l, \Tilde{a}_j) = \frac{1}{m - 1} \sum_{\substack{l = 1 \\ l \ne j}}^m (\mK_{W W} \odot \mK_{A A} + n \lambda_1 \mI)^{-1}(\mK_{W \Tilde{w}_l} \odot \mK_{A \Tilde{a}_j})
\end{align*}
Finally, the third component in Equation (\ref{eq:AlternativeProxyMethodVarphiNormWithConfounder3}):
\begin{align*}
    &\alpha_{m + 1}^2 \frac{1}{m^2 (m-1)^2} \sum_{j = 1}^m \sum_{\substack{l = 1 \\ l \ne j}}^m \sum_{r = 1}^m \sum_{\substack{s = 1 \\ s \ne r}}^m \mBeta(\Tilde{w}_l, \Tilde{a}_j)^T \mK_{Z Z} \mBeta(\Tilde{w}_s, \Tilde{a}_r) k_\gA(\Tilde{a}_j, \Tilde{a}_r) \\
    &=\alpha_{m + 1}^2  \frac{1}{m^2} \sum_{j = 1}^m \sum_{r = 1}^m \Big( \frac{1}{m - 1} \sum_{\substack{l = 1 \\ l \ne j}}^m \mBeta(\Tilde{w}_l, \Tilde{a}_j) \Big)^T \mK_{Z Z} \Big(\frac{1}{m - 1}  \sum_{\substack{s = 1 \\ s \ne r}}^m \mBeta(\Tilde{w}_s, \Tilde{a}_r) \Big) k_\gA(\Tilde{a}_j, \Tilde{a}_r)\\
    &= \alpha_{m + 1}^2 (\vone / m)^T \Big( \Bar{\mB}^T \mK_{Z Z} \Bar{\mB} \odot \mK_{\Tilde{A} \Tilde{A}} \Big) (\vone / m)
\end{align*}
As a result, 
\begin{align}
    &\|\varphi\|^2 = \langle \varphi, \varphi \rangle \nonumber \\&= \alpha_{1:m}^T \Big( \mB^T \mK_{Z Z} \mB \odot \mK_{\Tilde{A} \Tilde{A}}\Big) \alpha_{1:m} + 2 \alpha_{m + 1} \alpha_{1:m}^T\Big( \mB^T \mK_{Z Z} \Bar{\mB} \odot \mK_{\Tilde{A} \Tilde{A}} \Big) (\vone / m) \nonumber \\&+ \alpha_{m + 1}^2 (\vone / m)^T \Big( \Bar{\mB}^T \mK_{Z Z} \Bar{\mB} \odot \mK_{\Tilde{A} \Tilde{A}}\Big) (\vone / m)\nonumber
    \\&= \begin{bmatrix}
        \alpha_{1:m }^T & \alpha_{m + 1}
    \end{bmatrix} 
    \begin{bmatrix}
        \mB^T \mK_{Z Z} \mB \odot \mK_{\Tilde{A} \Tilde{A}} & ( \mB^T \mK_{Z Z} \Bar{\mB} \odot \mK_{\Tilde{A} \Tilde{A}} ) (\vone / m)\\
        (\vone / m)^T ( \mB^T \mK_{Z Z} \Bar{\mB} \odot \mK_{\Tilde{A} \Tilde{A}} )^T & (\vone / m)^T \Big( \Bar{\mB}^T \mK_{Z Z} \Bar{\mB} \odot \mK_{\Tilde{A} \Tilde{A}}\Big) (\vone / m)
    \end{bmatrix} 
    \begin{bmatrix}
        \alpha_{1:m } \\ \alpha_{m + 1}
    \end{bmatrix} \label{eq:AlternativeProxyMethodsVarphiNormWithConfoundersFinal}
\end{align}
Next, to derive the matrix-vector multiplication form for the first component of the objective given in Equation (\ref{eq:2StageRegressionObjectiveWithConfoundersFinal}), consider the following:.
\begin{align*}
    &\Big\langle \varphi,  \hat{\mu}_{Z| W, A}(\Tilde{w}_i, \Tilde{a}_i) \otimes \phi_\gA(\Tilde{a}_i)\Big\rangle\\
    &= \Bigg\langle \sum_{l = 1}^m \alpha_l  \hat{\mu}_{Z|W,A} (\Tilde{w}_l, \Tilde{a}_l)\otimes \phi_\gA(\Tilde{a}_l) + \frac{\alpha_{m + 1}}{m (m-1)}  \sum_{j = 1}^m \sum_{\substack{l = 1 \\ l \ne j}}^m \hat{\mu}_{Z|W, A} (\Tilde{w}_l, \Tilde{a}_j) \otimes \phi_\gA(\Tilde{a}_j), 
    \hat{\mu}_{Z| W, A}(\Tilde{w}_i, \Tilde{a}_i) \otimes \phi_\gA(\Tilde{a}_i)\Bigg\rangle\\
    &=\sum_{l = 1}^m \alpha_l \Big\langle  \hat{\mu}_{Z|W, A} (\Tilde{w}_l, \Tilde{a}_l)\otimes \phi_\gA(\Tilde{a}_l), \hat{\mu}_{Z| W, A}(\Tilde{w}_i, \Tilde{a}_i) \otimes \phi_\gA(\Tilde{a}_i)\Big\rangle\\
    &+ \frac{\alpha_{m + 1}}{m (m-1)} \sum_{j = 1}^m \sum_{\substack{l = 1 \\ l \ne j}}^m \Big\langle \hat{\mu}_{Z|W, A} (\Tilde{w}_l, \Tilde{a}_j) \otimes \phi_\gA(\Tilde{a}_j), 
    \hat{\mu}_{Z| W, A}(\Tilde{w}_i, \Tilde{a}_i) \otimes \phi_\gA(\Tilde{a}_i) \Big\rangle\\
    &= \sum_{l = 1}^m \alpha_l \mBeta(\Tilde{w}_l, \Tilde{a}_l)^T \mK_{Z Z} \mBeta(\Tilde{w}_i, \Tilde{a}_i) k_\gA(\Tilde{a}_l, \Tilde{a}_i)+ \frac{\alpha_{m + 1}}{m(m-1)} \sum_{j = 1}^m \sum_{\substack{l = 1 \\ l \ne j}}^m \mBeta(\Tilde{w}_l, \Tilde{a}_j)^T \mK_{Z Z} \mBeta(\Tilde{w}_i, \Tilde{a}_i) k_\gA(\Tilde{a}_j, \Tilde{a}_i)\\
    &= \sum_{l = 1}^m \alpha_l \mBeta(\Tilde{w}_l, \Tilde{a}_l)^T \mK_{Z Z} \mBeta(\Tilde{w}_i, \Tilde{a}_i) k_\gA(\Tilde{a}_l, \Tilde{a}_i)+ \frac{\alpha_{m + 1}}{m} \sum_{j = 1}^m \Big( \frac{1}{m - 1} \sum_{\substack{l = 1 \\ l \ne j}}^m \mBeta(\Tilde{w}_l, \Tilde{a}_j)\Big)^T \mK_{Z Z} \mBeta(\Tilde{w}_i, \Tilde{a}_i) k_\gA(\Tilde{a}_j, \Tilde{a}_i)\\
    &=\Big[\big(\mB^T \mK_{Z Z}\mB \odot \mK_{\Tilde{A}\Tilde{A}}\big) \alpha_{1:m}\Big]_i + \alpha_{m+1}\Big[\big({\mB}^T \mK_{Z Z}\bar{\mB} \odot \mK_{\Tilde{A}\Tilde{A}}\big) (\vone / m)\Big]_i\\
    &= \Bigg[ \begin{bmatrix}
        \mB^T \mK_{Z Z} \mB \odot \mK_{\Tilde{A}\Tilde{A}} & \big({\mB}^T \mK_{Z Z} \bar{\mB} \odot \mK_{\Tilde{A}\Tilde{A}}\big) (\vone / m)
    \end{bmatrix} \begin{bmatrix}
        \alpha_{1:m}\\ \alpha_{m+1}
    \end{bmatrix} \Bigg]_i
\end{align*}
As a result, the first component in Equation (\ref{eq:2StageRegressionObjectiveWithConfoundersFinal}) is given by
\begin{align}
    &\frac{1}{m} \sum_{i = 1}^m \Big\langle \varphi,  \hat{\mu}_{Z| W,A}(\Tilde{w}_i, \Tilde{a}_i) \otimes \phi_\gA(\Tilde{a}_i)\Big\rangle^2 \nonumber\\
    &= \frac{1}{m} \begin{bmatrix}
        \alpha_{1:m}^T & \alpha_{m+1}
    \end{bmatrix} 
    \begin{bmatrix}
        \mB^T \mK_{Z Z}\mB \odot \mK_{\Tilde{A}\Tilde{A}} \\ (\frac{\vone}{m})^T \big({\mB}^T \mK_{Z Z} \bar{\mB} \odot \mK_{\Tilde{A}\Tilde{A}}\big)^T
    \end{bmatrix} 
    \begin{bmatrix}
        \mB^T \mK_{Z Z}\mB \odot \mK_{\Tilde{A}\Tilde{A}} & \big({\mB}^T \mK_{Z Z} \bar{\mB} \odot \mK_{\Tilde{A}\Tilde{A}}\big) \frac{\vone}{m}
    \end{bmatrix} 
    \begin{bmatrix}
        \alpha_{1:m}\\ \alpha_{m+1}
    \end{bmatrix}
\label{eq:AlternativeProxyMethodsProofLossWithConfoundersFirstTerm}
\end{align}
Lastly, for the second component in Equation (\ref{eq:2StageRegressionObjectiveWithConfoundersFinal}), we note that
\begin{align}
    &\frac{1}{m(m-1)} \sum_{i = 1}^m \sum_{\substack{j = 1 \\ j \ne i}}^m \Big\langle \varphi, \hat{\mu}_{Z|W, A} (\Tilde{w}_j, \Tilde{a}_i) \otimes \phi_\gA(\Tilde{a}_i) \Big\rangle_{\gH_\gZ \otimes \gH_\gA} \nonumber\\
    &= \frac{1}{m(m-1)} \sum_{i = 1}^m \sum_{\substack{j = 1 \\ j \ne i}}^m \sum_{l = 1}^m \alpha_l \Big\langle 
     \hat{\mu}_{Z|W,A} (\Tilde{w}_l, \Tilde{a}_l)\otimes \phi_\gA(\Tilde{a}_l), \hat{\mu}_{Z|W, A} (\Tilde{w}_j, \Tilde{a}_i) \otimes \phi_\gA(\Tilde{a}_i)
    \Big\rangle_{\gH_\gZ \otimes \gH_\gA} \nonumber\\
    &+ \frac{\alpha_{m + 1} }{m^2 (m-1)^2} \sum_{i = 1}^m \sum_{\substack{j = 1 \\ j \ne i}}^m \sum_{r = 1}^m \sum_{\substack{s = 1 \\ s \ne r}}^m  \Big\langle \hat{\mu}_{Z|W, A} (\Tilde{w}_s, \Tilde{a}_r) \otimes \phi_\gA(\Tilde{a}_r), 
    \hat{\mu}_{Z|W,A} (\Tilde{w}_j, \Tilde{a}_i) \otimes \phi_\gA(\Tilde{a}_i)\Big\rangle_{\gH_\gZ \otimes \gH_\gA} \nonumber\\
    &= \frac{1}{m(m-1)} \sum_{i = 1}^m \sum_{\substack{j = 1 \\ j \ne i}}^m \sum_{l = 1}^m \alpha_l \mBeta(\Tilde{w}_l, \Tilde{a}_l)^T\mK_{Z Z}\mBeta(\Tilde{w}_j, \Tilde{a}_i) k_\gA(\Tilde{a}_l, \Tilde{a}_i) \nonumber\\
    &+ \frac{\alpha_{m + 1}}{m^2 (m-1)^2} \sum_{i = 1}^m \sum_{\substack{j = 1 \\ j \ne i}}^m \sum_{r = 1}^m \sum_{\substack{s = 1 \\ s \ne r}}^m \mBeta(\Tilde{w}_s, \Tilde{a}_r)^T \mK_{Z Z} \mBeta(\Tilde{w}_j, \Tilde{a}_i) k_\gA(\Tilde{a}_r, \Tilde{a}_i) \nonumber\\
    &= \frac{1}{m} \sum_{i = 1}^m \sum_{l = 1}^m \alpha_l \mBeta(\Tilde{w}_l, \Tilde{a}_l)^T\mK_{Z Z}\Big(\frac{1}{m - 1} \sum_{\substack{j = 1 \\ j \ne i}}^m  \mBeta(\Tilde{w}_j, \Tilde{a}_i)\Big) k_\gA(\Tilde{a}_l, \Tilde{a}_i) \nonumber\\
    &+ \frac{\alpha_{m + 1}}{m^2} \sum_{i = 1}^m  \sum_{r = 1}^m \Big(\frac{1}{m - 1} \sum_{\substack{s = 1 \\ s \ne r}}^m  \mBeta(\Tilde{w}_s, \Tilde{a}_r) \Big)^T \mK_{Z Z} \Big(\frac{1}{m - 1} \sum_{\substack{j = 1 \\ j \ne i}}^m \mBeta(\Tilde{w}_j, \Tilde{a}_i) \Big) k_\gA(\Tilde{a}_r, \Tilde{a}_i) \nonumber\\
    &=\frac{1}{m} \alpha_{1:m}^T \Big( \mB^T \mK_{Z Z} \Bar{\mB} \odot \mK_{\Tilde{A}\Tilde{A}}\Big) \vone + \alpha_{m+1} \frac{1}{m^2} \vone^T \Big( \bar{\mB} \mK_{Z Z} \Bar{\mB} \odot \mK_{\Tilde{A}\Tilde{A}} \Big) \vone \nonumber \\
&=  \begin{bmatrix}
    \alpha_{1:m } \\ \alpha_{m + 1}
\end{bmatrix}^T \begin{bmatrix}
 ( \mB^T \mK_{Z Z} \Bar{\mB} \odot \mK_{\Tilde{A} \Tilde{A}} ) \frac{\vone}{m}\\
 (\frac{\vone}{m})^T ( \Bar{\mB}^T \mK_{Z Z} \Bar{\mB} \odot \mK_{\Tilde{A} \Tilde{A}}) \frac{\vone}{m}
\end{bmatrix} \label{eq:AlternativeProxyMethodsProofLossWithConfounderSecondTerm}
\end{align}
Using Equation (\ref{eq:AlternativeProxyMethodsVarphiNormWithConfoundersFinal}), (\ref{eq:AlternativeProxyMethodsProofLossWithConfoundersFirstTerm}) and (\ref{eq:AlternativeProxyMethodsProofLossWithConfounderSecondTerm}), we can write
\begin{align}
    &\hat{\gL}^{2SR}_m(\varphi) = \frac{1}{m} \sum_{i = 1}^m \langle \varphi, \hat{\mu}_{Z|W, A} (\Tilde{w}_i, \Tilde{a}_i) \otimes \phi_\gA(\Tilde{a}_i) \rangle_{\gH_\gZ \otimes \gH_\gA}^2 \nonumber \\&-2 \frac{1}{m(m-1)} \sum_{i = 1}^m \sum_{\substack{j = 1 \\ j \ne i}}^m \Big\langle \varphi, \hat{\mu}_{Z|W, A} (\Tilde{w}_j, \Tilde{a}_i) \otimes \phi_\gA(\Tilde{a}_i)\Big\rangle_{\gH_\gZ \otimes \gH_\gA} + \lambda_2 \|\varphi_0\|^2_{\gH_\gZ \otimes \gH_\gA} \nonumber\\
    &= \frac{1}{m} \begin{bmatrix}
        \alpha_{1:m}^T & \alpha_{m+1}
    \end{bmatrix} 
    \begin{bmatrix}
        \mB^T \mK_{Z Z}\mB \odot \mK_{\Tilde{A}\Tilde{A}} \nonumber\\ (\frac{\vone}{m})^T \big({\mB}^T \mK_{Z Z} \bar{\mB} \odot \mK_{\Tilde{A}\Tilde{A}}\big)^T
    \end{bmatrix} 
    \begin{bmatrix}
        \mB^T \mK_{Z Z}\mB \odot \mK_{\Tilde{A}\Tilde{A}} & \big({\mB}^T \mK_{Z Z} \bar{\mB} \odot \mK_{\Tilde{A}\Tilde{A}}\big) \frac{\vone}{m}
    \end{bmatrix} 
    \begin{bmatrix}
        \alpha_{1:m}\\ \alpha_{m+1}
    \end{bmatrix}\nonumber\\
    &-2 \begin{bmatrix}
    \alpha_{1:m } \\ \alpha_{m + 1}
    \end{bmatrix}^T \begin{bmatrix}
     ( \mB^T \mK_{Z Z} \Bar{\mB} \odot \mK_{\Tilde{A} \Tilde{A}} ) \frac{\vone}{m}\\
     (\frac{\vone}{m})^T \Big( \Bar{\mB}^T \mK_{Z Z} \Bar{\mB} \odot \mK_{\Tilde{A} \Tilde{A}}\Big) \frac{\vone}{m}
    \end{bmatrix} \nonumber\\
    &+ \lambda_2 \begin{bmatrix}
        \alpha_{1:m }^T & \alpha_{m + 1}
    \end{bmatrix} 
    \begin{bmatrix}
        \mB^T \mK_{Z Z} \mB \odot \mK_{\Tilde{A} \Tilde{A}} & ( \mB^T \mK_{Z Z} \Bar{\mB} \odot \mK_{\Tilde{A} \Tilde{A}} ) \frac{\vone}{m}\\
        (\frac{\vone}{m})^T ( \mB^T \mK_{Z Z} \Bar{\mB} \odot \mK_{\Tilde{A} \Tilde{A}} )^T & (\frac{\vone}{m})^T \Big( \Bar{\mB}^T \mK_{Z Z} \Bar{\mB} \odot \mK_{\Tilde{A} \Tilde{A}}\Big) \frac{\vone}{m}
    \end{bmatrix} 
    \begin{bmatrix}
        \alpha_{1:m } \\ \alpha_{m + 1}
    \end{bmatrix} 
\label{eq:AlternativeProxyMethodProofLossMatrixVectorProductFormWithConfounder}
\end{align}
The optimal coefficients $\{\alpha_{1:m}, \alpha_{m+1}\}$ can be found by setting the derivative of Equation (\ref{eq:AlternativeProxyMethodProofLossMatrixVectorProductFormWithConfounder}) to zero. With these optimal coefficients, let $\hat{\varphi}_{\lambda_2, m}$ denote the minimizer of $\hat{\gL}^{2SR}_m(\varphi)$. Using $\hat{\varphi}_{\lambda_2, m}$, we can estimate $\E[Y \hat{\varphi}_{\lambda_2, m}(Z, a) | A = a]$, thus obtaining the desired ATE function estimation. First, observe that
\begin{align*}
    \E[Y \hat{\varphi}_{\lambda_2, m}(Z, a) | A = a] &= \E [Y \langle \hat{\varphi}_{\lambda_2, m}, \phi_\gZ(Z) \otimes \phi_\gA(a) \rangle | A = a]\\
    &= \Big\langle  \hat{\varphi}_{\lambda_2, m}, \E [Y \phi_\gZ(Z) | A = a] \otimes \phi_\gA(a) \Big\rangle \\
    &= \langle \hat{\varphi}_{\lambda_2, m}, C_{YZ|A} \phi_\gA(a) \otimes \phi_\gA(a)\rangle
\end{align*}
where $C_{YZ|A}$ is the conditional mean operator, i.e., $C_{YZ|A} \phi_\gA(a) =  \E [Y \phi_\gZ(Z) | A = a]$ and it is estimated by kernel ridge regression:
\begin{align}
    \hat{C}_{YZ|A}  &= \argmin_{C} \frac{1}{n} \sum_{i = 1}^n \|y_i \phi_\gZ(z_i) - C_{YZ|A} \phi_\gA(a_i)\|^2 + \lambda_3 \|C_{YZ|A}\|^2 \nonumber\\
    &= \argmin_{C} \frac{1}{n}\|\Phi_\gZ \text{diag}(\mY) - C_{YZ|A} \Phi_\gA\|^2 + \lambda_3 \|C_{YZ|A}\|^2.
    \label{eq:AlternativeProxyThirdStage}
\end{align}
The solution for Equation (\ref{eq:AlternativeProxyThirdStage}) is given by
\begin{align*}
    \hat{C}_{YZ|A} \phi_\gA(a) = \Phi_\gZ \text{diag}(\mY) [\mK_{A A} + n \lambda_3 \mI]^{-1} \mK_{A a}
\end{align*}
As a result,
\begin{align*}
    \E[Y \hat{\varphi}_{\lambda_2, m}(Z, a) | A = a] &= \langle \hat{\varphi}_{\lambda_2, m}, \hat{C}_{YZ|A} \phi_\gA(a) \otimes \phi_\gA(a)\rangle \\
    &= \Big\langle \sum_{l = 1}^m \alpha_l  \hat{\mu}_{Z|W, A} (\Tilde{w}_l, \Tilde{a}_l)\otimes \phi_\gA(\Tilde{a}_l),  \hat{C}_{YZ|A}  \phi_\gA(a) \otimes \phi_\gA(a)\Big\rangle \\
    &+ \Big\langle  \alpha_{m + 1} \frac{1}{m (m-1)} \sum_{j = 1}^m \sum_{\substack{l = 1 \\ l \ne j}}^m \hat{\mu}_{Z|W, A} (\Tilde{w}_l, \Tilde{a}_j) \otimes \phi_\gA(\Tilde{a}_j), \hat{C}_{YZ|A}  \phi_\gA(a) \otimes \phi_\gA(a)\Big\rangle \\
    &= \sum_{l = 1}^m \alpha_l \langle \hat{\mu}_{Z|W, A} (\Tilde{w}_l, \Tilde{a}_l)\otimes \phi_\gA(\Tilde{a}_l) ,  \hat{C}_{YZ|A}  \phi_\gA(a) \otimes \phi_\gA(a)\rangle\\&+ \alpha_{m + 1} \frac{1}{m (m - 1)} \sum_{j = 1}^m \sum_{\substack{l = 1 \\ l \ne j}}^m \langle \hat{\mu}_{Z|W, A} (\Tilde{w}_l, \Tilde{a}_j) \otimes \phi_\gA(\Tilde{a}_j),\hat{C}_{YZ|A}  \phi_\gA(a) \otimes \phi_\gA(a) \rangle \\
    &= \sum_{l = 1}^m \alpha_l \langle \hat{\mu}_{Z|W, A} (\Tilde{w}_l, \Tilde{a}_l),  \hat{C}_{YZ|A}  \phi_\gA(a) \rangle \langle \phi_\gA(\Tilde{a}_l), \phi_\gA(a)\rangle\\&+ \alpha_{m + 1} \frac{1}{m (m-1)} \sum_{j = 1}^m \sum_{\substack{l = 1 \\ l \ne j}}^m \langle \hat{\mu}_{Z|W, A} (\Tilde{w}_l, \Tilde{a}_j), \hat{C}_{YZ|A}  \phi_\gA(a) \rangle \langle \phi_\gA(\Tilde{a}_j), \phi_\gA(a)\rangle\\
    &= \sum_{l = 1}^m \alpha_l \mBeta(\Tilde{w}_l, \Tilde{a}_l)^T \Phi_\gZ^T \Phi_\gZ \text{diag}(\mY)[\mK_{A A} + n \lambda_3 \mI]^{-1} \mK_{A a} k_\gA(\Tilde{a}_l, a)\\
    &+ \alpha_{m + 1} \frac{1}{m(m-1)} \sum_{j = 1}^m \sum_{\substack{l = 1 \\ l \ne j}}^m \mBeta(\Tilde{w}_l, \Tilde{a}_j)^T \Phi_\gZ^T \Phi_\gZ \text{diag}(\mY) [\mK_{A A} + n \lambda_3 \mI]^{-1} \mK_{A a} k_\gA(\Tilde{a}_j, a)\\
    &= \sum_{l = 1}^m \alpha_l \mBeta(\Tilde{w}_l, \Tilde{a}_l)^T \mK_{Z Z} \text{diag}(\mY)[\mK_{A A} + n \lambda_3 \mI]^{-1} \mK_{A a} k_\gA(\Tilde{a}_l, a)\\
    &+ \alpha_{m + 1} \frac{1}{m(m-1)} \sum_{j = 1}^m \sum_{\substack{l = 1 \\ l \ne j}}^m \mBeta(\Tilde{w}_l, \Tilde{a}_j)^T \mK_{Z Z} \text{diag}(\mY) [\mK_{A A} + n \lambda_3 \mI]^{-1} \mK_{A a} k_\gA(\Tilde{a}_j, a)\\
    &= \alpha_{1:m}^T \Big( \mB^T \big(\mK_{Z Z} \text{diag}(\mY) [\mK_{A A} + n \lambda_3 \mI]^{-1} \mK_{A a}\big) \odot \mK_{\Tilde{A} a}\Big)\\
    &+ \alpha_{m + 1} \Big( \Bar{\mB}^T \big( \mK_{Z Z}  \text{diag}(\mY) [\mK_{A A} + n \lambda_3 \mI]^{-1} \mK_{A a} \big) \odot \mK_{\Tilde{A} a}\Big) \frac{\vone}{m}
\end{align*}
As a result, 
\begin{align*}
    f_{\text{ATE}}(a) &= \alpha_{1:m}^T \Big( \mB^T \big(\mK_{Z Z} \text{diag}(\mY) [\mK_{A A} + n \lambda_3 \mI]^{-1} \mK_{A a}\big) \odot \mK_{\Tilde{A} a}\Big)\\
    &+ \alpha_{m + 1} \Big( \Bar{\mB}^T \big( \mK_{Z Z}  \text{diag}(\mY) [\mK_{A A} + n \lambda_3 \mI]^{-1} \mK_{A a} \big) \odot \mK_{\Tilde{A} a}\Big) \frac{\vone}{m}
\end{align*}
\end{proof}

One observation that is a result of the Algorithm (\ref{algo:ATE_algorithm_with_confounders}) is that we can compute the bridge function in the closed-form as stated in the following remark.
\begin{remark}
    Given the optimal coefficients $\{\alpha_{1:m}, \alpha_{m+1}\}$ from Algorithm (\ref{algo:ATE_algorithm_with_confounders}), the bridge function can be written in closed-form as
    \begin{align*}
        \hat{\varphi}_{\lambda_2, m}(z, a) = \alpha_{1:m}^T [(\mB^T \mK_{Z z}) \odot \mK_{\Tilde{A}a}] + \alpha_{m + 1} \Big(\frac{\vone}{m}\Big)^T [(\Bar{\mB}^T \mK_{Z z}) \odot \mK_{\Tilde{A} a}]
    \end{align*}
    \label{remark:ClosedFormForBridgeFunction}
\end{remark}
\begin{proof}
    \begin{align*}
    &\hat{\varphi}_{\lambda_2, m}(z, a) =  \langle \hat{\varphi}_{\lambda_2, m}, \phi_\gZ(z) \otimes \phi_\gA(a) \rangle\\
    &= \Bigg\langle \sum_{l = 1}^m \alpha_l  \hat{\mu}_{Z|W,A} (\Tilde{w}_l, \Tilde{a}_l) \otimes \phi_\gA(\Tilde{a}_l) + \frac{\alpha_{m + 1}}{m(m-1)} \sum_{j = 1}^m \sum_{\substack{l = 1 \\ l \ne j}}^m \hat{\mu}_{Z|W, A} (\Tilde{w}_l, \Tilde{a}_j) \otimes \phi_\gA(\Tilde{a}_j), \phi_\gZ(z) \otimes \phi_\gA(a)\Bigg\rangle\\
    &= \sum_{l = 1}^m \alpha_l  \langle \hat{\mu}_{Z|W, A} (\Tilde{w}_l, \Tilde{a}_l)\otimes \phi_\gA(\Tilde{a}_l),\phi_\gZ(z) \otimes \phi_\gA(a) \rangle\\
    &+ \frac{\alpha_{m + 1} }{m(m-1)} \sum_{j = 1}^m \sum_{\substack{l = 1 \\ l \ne j}}^m \langle \hat{\mu}_{Z|W, A} (\Tilde{w}_l, \Tilde{a}_j) \otimes \phi_\gA(\Tilde{a}_j),\phi_\gZ(z) \otimes \phi_\gA(a) \rangle\\
    &= \sum_{l = 1}^m \alpha_l  \langle \hat{\mu}_{Z|W, A} (\Tilde{w}_l, \Tilde{a}_l),\phi_\gZ(z) \rangle k_\gA(\Tilde{a}_l, a) + \frac{\alpha_{m + 1}}{m(m-1)} \sum_{j = 1}^m \sum_{\substack{l = 1 \\ l \ne j}}^m \langle \hat{\mu}_{Z|W, A} (\Tilde{w}_l, \Tilde{a}_j),\phi_\gZ(z) \rangle k_\gA(\Tilde{a}_j, a)\\
    &= \sum_{l = 1}^m \alpha_l  \langle \Phi_\gZ {\mBeta}(\Tilde{w}_l, \tilde{a}_l),\phi_\gZ(z) \rangle k_\gA(\Tilde{a}_l, a) + \frac{\alpha_{m + 1}}{m(m-1)} \sum_{j = 1}^m \sum_{\substack{l = 1 \\ l \ne j}}^m \langle \Phi_\gZ {\mBeta}(\Tilde{w}_l, \tilde{a}_j),\phi_\gZ(z) \rangle k_\gA(\Tilde{a}_j, a)\\
    &= \sum_{l = 1}^m \alpha_l \mK_{Z z}^T {\mBeta}(\Tilde{w}_l, \tilde{a}_l) k_\gA(\Tilde{a}_l, a) + \frac{\alpha_{m + 1}}{m(m-1)} \sum_{j = 1}^m \sum_{\substack{l = 1 \\ l \ne j}}^m \mK_{Z z}^T \mBeta(\Tilde{w}_l, \tilde{a}_j)k_\gA(\Tilde{a}_j, a)\\
    &= \sum_{l = 1}^m \alpha_l \mK_{Z z}^T (\mK_{W W} \odot \mK_{A A} + n \lambda_1 \mI)^{-1}(\mK_{W \Tilde{w}_l} \odot \mK_{A \tilde{a}_l}) k_\gA(\Tilde{a}_l, a)\\
    &+ \frac{\alpha_{m + 1}}{m^2} \sum_{j = 1}^m \sum_{\substack{l = 1 \\ l \ne j}}^m \mK_{Z z}^T (\mK_{W W} \odot \mK_{A A} + n \lambda_1 \mI)^{-1}(\mK_{W \Tilde{w}_l} \odot \mK_{A \tilde{a}_j}) k_\gA(\Tilde{a}_j, a)\\
    &= \sum_{l = 1}^m \alpha_l \mK_{Z z}^T (\mK_{W W} \odot \mK_{A A} + n \lambda_1 \mI)^{-1}(\mK_{W \Tilde{w}_l} \odot \mK_{A \tilde{a}_l}) k_\gA(\Tilde{a}_l, a)\\
    &+ \frac{\alpha_{m + 1}}{m}  \sum_{j = 1}^m  \mK_{Z z}^T (\mK_{W W} \odot \mK_{A A} + n \lambda_1 \mI)^{-1}\Big(\frac{1}{m - 1}\sum_{\substack{l = 1 \\ l \ne j}}^m \mK_{W \Tilde{w}_l}  \odot \mK_{A \tilde{a}_j} \Big) k_\gA(\Tilde{a}_j, a)\\
    &= \alpha_{1:m}^T [(\mB^T \mK_{Z z}) \odot \mK_{\Tilde{A} a}] + \alpha_{m + 1} \Big(\frac{\vone}{m} \Big)^T [(\Bar{\mB}^T \mK_{Z z}) \odot \mK_{\Tilde{A} a}]
\end{align*}
\end{proof}

\subsection{Conditional Dose-Response Curve Algorithm}
\label{sec:ATT_algorithm_proof}
In this section, we provide the derivation of Algorithm (\ref{algo:ATT_algorithm_with_confounders}). For the completeness, we write the algorithm first.

\begin{algorithm_state*}[Algorithm (\ref{algo:ATT_algorithm_with_confounders})]
Let $\{z_i, w_i, a_i\}_{i = 1}^n$ and $\{\Tilde{w}_i, \Tilde{a}_i\}_{i = 1}^m$ be the first-stage and second-stage data, respectively, and $(\lambda_1, \lambda_2, \lambda_3, \zeta)$ be the regularization parameters. Furthermore, let $\{\alpha_i\}_{i = 1}^{m + 1}$ be the minimizer of the following cost function for a given $a'$: 
\begin{align*}
    &\hat{\gL}^{2SR}_m(\malpha) =\\
    & \frac{1}{m} \begin{bmatrix}
        \alpha_{1:m} \\ \alpha_{m+1}
    \end{bmatrix}^T 
    \begin{bmatrix}
        \mB^T \mK_{Z Z}\mB \odot \mK_{\Tilde{A}\Tilde{A}} \nonumber\\ (\frac{\vone}{m})^T \big[{\mB}^T \mK_{Z Z} \Tilde{\mB} \odot \mK_{\Tilde{A}\Tilde{A}}\big]^T
    \end{bmatrix} 
    \begin{bmatrix}
        \mB^T \mK_{Z Z}\mB \odot \mK_{\Tilde{A}\Tilde{A}} & \big[{\mB}^T \mK_{Z Z} \Tilde{\mB} \odot \mK_{\Tilde{A}\Tilde{A}}\big] \frac{\vone}{m}
    \end{bmatrix} 
    \begin{bmatrix}
        \alpha_{1:m}\\ \alpha_{m+1}
    \end{bmatrix} k_\gA(a', a')^2\nonumber\\
    &-2 \begin{bmatrix}
    \alpha_{1:m } \\ \alpha_{m + 1}
    \end{bmatrix}^T \begin{bmatrix}
     [ \mB^T \mK_{Z Z} \Tilde{\mB} \odot \mK_{\Tilde{A} \Tilde{A}} ] \frac{\vone}{m}\\
     (\frac{\vone}{m})^T \Big[ \Tilde{\mB}^T \mK_{Z Z} \Tilde{\mB} \odot \mK_{\Tilde{A} \Tilde{A}}\Big] \frac{\vone}{m}
    \end{bmatrix} k_\gA(a', a')\nonumber\\
    &+ \lambda_2 \begin{bmatrix}
        \alpha_{1:m } \\ \alpha_{m + 1}
    \end{bmatrix}^T 
    \begin{bmatrix}
        \mB^T \mK_{Z Z} \mB \odot \mK_{\Tilde{A} \Tilde{A}} & [ \mB^T \mK_{Z Z} \Tilde{\mB} \odot \mK_{\Tilde{A} \Tilde{A}} ] \frac{\vone}{m}\\
        (\frac{\vone}{m})^T [ \mB^T \mK_{Z Z} \Tilde{\mB} \odot \mK_{\Tilde{A} \Tilde{A}} ]^T & (\frac{\vone}{m})^T \Big[ \Tilde{\mB}^T \mK_{Z Z} \Tilde{\mB} \odot \mK_{\Tilde{A} \Tilde{A}}\Big] \frac{\vone}{m}
    \end{bmatrix} 
    \begin{bmatrix}
        \alpha_{1:m } \\ \alpha_{m + 1}
    \end{bmatrix} k_\gA(a', a')
\end{align*}
where $\Tilde{\mB}$ is the matrix whose $j$-th column is given by $\Tilde{\mB}_{:, j} = \sum_{\substack{l = 1 \\ l \ne j}}^m (\mK_{W W} \odot \mK_{A A} + n \lambda_1 \mI)^{-1}(\theta_l \mK_{W \Tilde{w}_l} \odot \mK_{A \Tilde{a}_j})$ 
and $\theta_l = [(\mK_{\Tilde{A} \Tilde{A}} + m \zeta \mI)^{-1} \mK_{\Tilde{A} a'}]_l$.
Furthermore, the kernel matrices and the matrix $\mB$ are as defined in Algorithm (\ref{algo:ATE_algorithm_with_confounders}). Then, the conditional dose-response curve estimation can be written in the closed-form as
\begin{align*}
    f_{\text{ATT}}(a, a') &= \alpha_{1:m}^T \Big( \mB^T \big(\mK_{Z Z} \text{diag}(\mY) [\mK_{A A} + n \lambda_3 \mI]^{-1} \mK_{A a}\big) \odot \mK_{\Tilde{A} a}\Big) k_\gA(a', a')\\
    &+ \alpha_{m + 1} \Big( \Tilde{\mB}^T \big( \mK_{Z Z}  \text{diag}(\mY) [\mK_{A A} + n \lambda_3 \mI]^{-1} \mK_{A a} \big) \odot \mK_{\Tilde{A} a}\Big) \frac{\vone}{m} k_\gA(a', a')
\end{align*}
\label{algo:ATT_algorithm_with_confounders_appendix}
\end{algorithm_state*}

\begin{proof}[Derivation of Algorithm (\ref{algo:ATT_algorithm_with_confounders})]
Let $r(W, a, a')$ denote $\frac{p(W, a') p(a) }{p(W, a) p (a')}$. We aim to find the optimum of the following loss function:
\begin{align*}
    \gL^{\text{2SR}}(\varphi) = \E[(r(W, A, a') - \E[\varphi(Z, A, a') | W, A])^2] + \lambda_2 \|\varphi\|_{\gH_\gZ \otimes \gH_\gA \otimes \gH_\gA}^2
\end{align*}
in which we assume that the bridge function $\varphi$ is the RKHS $\gH_\gZ \otimes \gH_\gA \otimes \gH_\gA$. Similar to the dose-response algorithm, this loss function cannot be directly optimized as it involves the conditional expectation $\E[\varphi(Z, A, a') | W, A]$. Following \citep{xu2024kernelsingleproxycontrol}, we similarly employ a two-stage regression approach. The first-stage is identical to the one used for the dose-response curve as described in S.M. (\ref{sec:ATE_algorithm_proof}). In this stage, we find the conditional mean embedding:
\begin{align*}
    \hat{\mu}_{Z|W, A}(w, a) = \sum_{i = 1}^n \beta_i(w, a) \phi_\gZ(z_i) = \Phi_\gZ \mBeta(w, a)
\end{align*}
where 
\begin{align*}
    {\mBeta}(w, a) &= (\mK_{W W} \odot \mK_{A A} + n \lambda_1 \mI)^{-1}(\mK_{W w} \odot \mK_{A a})\\
    \Phi_\gZ &= \begin{bmatrix}
        \phi_\gZ(z_1) & \ldots & \phi_\gZ(z_n)
    \end{bmatrix}
\end{align*}
For the second-stage, we note the following:
\begin{align*}
    &\E[(r(W, A, a') - \E[\varphi(Z, A, a') | W, A])^2] \\&= \E[ \E[\varphi(Z, A, a') | W, A]^2] - 2\E[r(W, A, a') \E[\varphi(Z, A, a') |W, A] ] + \text{const.}\\
    &= \E[ \E[\varphi(Z, A, a') | A, W]^2] - 2 \int \frac{p(w, a') p(a) }{p(w, a) p (a')} \E[\varphi(Z, A, a') | A=a, W=w] p(w, a) d w d a + \text{const.}\\
    &= \E[ \E[\varphi(Z, A, a') | W=w, A=a]^2] - 2 \int p(w | a') p(a) \E[\varphi(Z, A, a') | W=w, A=a] d w d a + \text{const.}\\
    &= \E[ \E[\varphi(Z, A, a') | W=w, A=a]^2] - 2 \E_{W| A' = a'} \E_{A}[ \E[\varphi(Z, A, a') | W=w, A=a] ] + \text{const.}
\end{align*}
Recall that under Assumption~\ref{asst:well_specifiedness}-(1), $\mu_{Z|W, A}(w, a) = C_{Z|W, A} (\phi_\gW(w) \otimes \phi_\gA(a))$. Hence, we can write,
\begin{align*}
     &\E[(r(W, A, a') - \E[\varphi(Z, A, a') | W, A])^2] \\&=\E[ \langle \varphi, \mu_{Z|W, A} (W, A) \otimes \phi_\gA(A) \otimes \phi_\gA(a') \rangle^2]\\
     &-2 \E_{W | A' = a'} \E_{A}[\langle \varphi, C_{Z|W, A}(\phi_\gW(W) \otimes \phi_\gA(A)) \otimes \phi_\gA(A) \otimes \phi_\gA(a') \rangle] + \text{const.}\\
     &=\E[ \langle \varphi, \mu_{Z|W, A} (W, A) \otimes \phi_\gA(A) \otimes \phi_\gA(a') \rangle^2]\\
     &-2 \E_{A}[\langle \varphi, C_{Z|W, A}(\E_{W| A = a'}[\phi_\gW(W)] \otimes \phi_\gA(A)) \otimes \phi_\gA(A) \otimes \phi_\gA(a') \rangle] + \text{const.}
\end{align*}
This expectation can be estimated from the second stage data $\{\Tilde{w}_i, \Tilde{z}_i, \Tilde{a}_i\}$ using the following expression, up to a constant factor:
\begin{align*}
    &\E[(r(W, A, a') - \E[\varphi(Z, A, a') | W, A])^2] \\&\approx \frac{1}{m} \sum_{i = 1}^m \langle \varphi, \hat{\mu}_{Z|W, A} (\Tilde{w}_i, \Tilde{a}_i) \otimes \phi_\gA(\Tilde{a}_i) \otimes \phi_\gA(a') \rangle^2\\
    &- \frac{2}{m} \sum_{j = 1}^m \langle \varphi, \hat{C}_{Z|W, A}(\hat{\E}_{W| A = a'}[\phi_\gW(W)] \otimes \phi_\gA(\Tilde{a}_j)) \otimes \phi_\gA(\Tilde{a}_j) \otimes \phi_\gA(a') \rangle.
\end{align*}
Here, $\hat{\E}_{W| A' = a'}[\phi_\gW(W)]$ is the estimation of conditional mean embedding for $p(W| A = a')$, and can be expressed in the closed-form as the result of kernel ridge regression with regularization parameter $\zeta$:
\begin{align*}
    \hat{\E}[\phi_\gW(W) | A = a'] = \Phi_{\gW}(\mK_{\Tilde{A} \Tilde{A}} + m \zeta \mI)^{-1} \mK_{\Tilde{A} a'} = \sum_{i = 1}^m \theta_i \phi_\gW(\Tilde{w}_i) = \Phi_{\gW} \vtheta
\end{align*}
where $\vtheta = (\mK_{\Tilde{A} \Tilde{A}} + m \zeta \mI)^{-1} \mK_{\Tilde{A} a'}$ and $\Phi_{\gW} = \begin{bmatrix}
    \phi_\gW(\Tilde{w}_1) & \ldots & \phi_\gW(\Tilde{w}_m)
\end{bmatrix}$. Hence, the least-squares loss can be estimated by second-stage data $\{\Tilde{w}_i, \Tilde{a}_i\}_{i = 1}^m$ as follows:
\begin{align*}
    &\E[(r(W, A, a') - \E[\varphi(Z, A, a') | W, A])^2] \\&\approx \frac{1}{m} \sum_{i = 1}^m \langle \varphi, \hat{\mu}_{Z|W, A} (\Tilde{w}_i, \Tilde{a}_i) \otimes \phi_\gA(\Tilde{a}_i) \otimes \phi_\gA(a') \rangle^2\\
    &-2 \frac{1}{m} \sum_{j = 1}^m \langle \varphi, \hat{C}_{Z|W, A}(\sum_{\substack{i = 1 \\ i \ne j}}^m \theta_i \phi_\gW(\Tilde{w}_i) \otimes \phi_\gA(\Tilde{a}_j)) \otimes \phi_\gA(\Tilde{a}_j) \otimes \phi_\gA(a') \rangle\\
    &= \frac{1}{m} \sum_{i = 1}^m \langle \varphi, \hat{\mu}_{Z|W, A} (\Tilde{w}_i, \Tilde{a}_i) \otimes \phi_\gA(\Tilde{a}_i) \otimes \phi_\gA(a') \rangle^2\\
    &-2 \frac{1}{m} \sum_{j = 1}^m \sum_{\substack{i = 1 \\ i \ne j}}^m \langle \varphi, \hat{C}_{Z|W, A}( \theta_i \phi_\gW(\Tilde{w}_i) \otimes \phi_\gA(\Tilde{a}_j)) \otimes \phi_\gA(\Tilde{a}_j) \otimes \phi_\gA(a') \rangle\\
    &= \frac{1}{m} \sum_{i = 1}^m \langle \varphi, \hat{\mu}_{Z|W,A} (\Tilde{w}_i, \Tilde{a}_i) \otimes \phi_\gA(\Tilde{a}_i) \otimes \phi_\gA(a') \rangle^2\\
    &-2 \frac{1}{m} \sum_{j = 1}^m \sum_{\substack{i = 1 \\ i \ne j}}^m \langle \varphi,  \theta_i \hat{\mu}_{Z | W, A}(\Tilde{w}_i, \Tilde{a}_j) \otimes \phi_\gA(\Tilde{a}_j) \otimes \phi_\gA(a') \rangle.
\end{align*}
As a result, we can write the sample loss for two-stage regression with Tikhonov regularization as
\begin{align}
    \hat{\mathcal{L}}^{\text{2SR}}_m(\varphi) &= \frac{1}{m} \sum_{i = 1}^m \langle \varphi, \hat{\mu}_{Z|W, A} (\Tilde{w}_i, \Tilde{a}_i) \otimes \phi_\gA(\Tilde{a}_i) \otimes \phi_\gA(a') \rangle^2 \nonumber\\ &-2 \frac{1}{m} \sum_{j = 1}^m \sum_{\substack{i = 1 \\ i \ne j}}^m \langle \varphi,  \theta_i \hat{\mu}_{Z | W, A}(\Tilde{w}_i, \Tilde{a}_j) \otimes \phi_\gA(\Tilde{a}_j) \otimes \phi_\gA(a') \rangle + \lambda_2 \|\varphi\|_{\gH_\gZ \otimes \gH_\gA \otimes \gH_\gA}^2.
\label{eq:2StageRegressionATTObjectiveFinal}
\end{align}
We can see that the minimizer of this objective should be in the span of the following set,
\begin{align*}
    \varphi \in \text{span} \Bigg\{ \{\hat{\mu}_{Z|W, A} (\Tilde{w}_i, \Tilde{a}_i)\otimes \phi_\gA(\Tilde{a}_i) \otimes \phi_\gA(a') \}_{i = 1}^m \cup \Big\{ \frac{1}{m} \sum_{i = 1}^m \sum_{\substack{j = 1 \\ j \ne i}}^m \theta_j \hat{\mu}_{Z|W, A} (\Tilde{w}_j, \Tilde{a}_i) \otimes \phi_\gA(\Tilde{a}_i) \otimes \phi_\gA(a') \Big\} \Bigg\}.
\end{align*}
Hence, we write
\begin{align*}
    {\varphi = \sum_{i = 1}^m \alpha_i  \hat{\mu}_{Z|W, A} (\Tilde{w}_i, \Tilde{a}_i)\otimes \phi_\gA(\Tilde{a}_i) \otimes \phi_\gA(a') + \frac{\alpha_{m + 1}}{m} \sum_{j = 1}^m \sum_{\substack{l = 1 \\ l \ne j}}^m \theta_l \hat{\mu}_{Z|W,A} (\Tilde{w}_l, \Tilde{a}_j) \otimes \phi_\gA(\Tilde{a}_j) \otimes \phi_\gA(a')}
\end{align*}
Let us notice the result of the following inner product which will come up a lot
\begin{align}
    &\Big\langle  \hat{\mu}_{Z|W, A} (\Tilde{w}_j, \Tilde{a}_i)\otimes \phi_\gA(\Tilde{a}_i) \otimes \phi_\gA(a'), \hat{\mu}_{Z|W, A} (\Tilde{w}_p, \Tilde{a}_l)\otimes \phi_\gA(\Tilde{a}_l) \otimes \phi_\gA(a') \Big\rangle_{\gH_\gZ \otimes \gH_\gA \otimes \gH_\gA}\nonumber\\
    &= \langle  \hat{\mu}_{Z|W, A} (\Tilde{w}_j, \Tilde{a}_i), \hat{\mu}_{Z|W,A} (\Tilde{w}_p, \Tilde{a}_l) \rangle \langle \phi_\gA(\Tilde{a}_i), \phi_\gA(\Tilde{a}_l)\rangle \langle \phi_\gA(a'), \phi_\gA(a')\rangle\nonumber\\
    &= \mBeta(\Tilde{w}_j, \Tilde{a}_i)^T \Phi_\gZ^T \Phi_\gZ \mBeta(\Tilde{w}_p, \Tilde{a}_l) k_\gA(\Tilde{a}_i, \Tilde{a}_l) k_\gA(a', a')\nonumber\\
    &= \mBeta(\Tilde{w}_j, \Tilde{a}_i)^T \mK_{Z Z} \mBeta(\Tilde{w}_p, \Tilde{a}_l) k_\gA(\Tilde{a}_i, \Tilde{a}_l) k_\gA(a', a')
\label{eq:AlternativeProxyMethodATTProofImportantEquation1}
\end{align}
We will now compute the individual terms in $\gL^{\text{2SR}}_m(\varphi)$ one by one. We start by the squared norm of $\varphi$:
\begin{align*}
    &\|\varphi\|^2 = \langle \varphi, \varphi \rangle\nonumber\\
    &= \Bigg\langle 
    \sum_{i = 1}^m \alpha_i  \hat{\mu}_{Z|W, A} (\Tilde{w}_i, \Tilde{a}_i) \otimes \phi_\gA(\Tilde{a}_i) \otimes \phi_\gA(a') + \frac{\alpha_{m + 1}}{m} \sum_{j = 1}^m \sum_{\substack{l = 1 \\ l \ne j}}^m \theta_l \hat{\mu}_{Z|W, A} (\Tilde{w}_l, \Tilde{a}_j) \otimes \phi_\gA(\Tilde{a}_j) \otimes \phi_\gA(a'), \nonumber
    \\&\sum_{p = 1}^m \alpha_p  \hat{\mu}_{Z|W, A} (\Tilde{w}_p, \Tilde{a}_p)\otimes \phi_\gA(\Tilde{a}_p) \otimes \phi_\gA(a') + \frac{\alpha_{m + 1}}{m}  \sum_{r = 1}^m \sum_{\substack{s = 1 \\ s \ne r}}^m \theta_s \hat{\mu}_{Z|W, A} (\Tilde{w}_s, \Tilde{a}_r) \otimes \phi_\gA(\Tilde{a}_r) \otimes \phi_\gA(a')
    \Bigg\rangle\nonumber\\
    &=\sum_{i = 1}^m \sum_{p = 1}^m \alpha_i \alpha_p \langle \hat{\mu}_{Z|W, A} (\Tilde{w}_i, \Tilde{a}_i)\otimes \phi_\gA(\Tilde{a}_i) \otimes \phi_\gA(a'), \hat{\mu}_{Z|W, A} (\Tilde{w}_p, \Tilde{a}_p)\otimes \phi_\gA(\Tilde{a}_p) \otimes \phi_\gA(a')\rangle \\
    &+2 \frac{\alpha_{m + 1}}{m} \sum_{i = 1}^m \sum_{j = 1}^m \sum_{\substack{l = 1 \\ l \ne j}}^m \alpha_i \langle \hat{\mu}_{Z|W, A} (\Tilde{w}_i, \Tilde{a}_i)\otimes \phi_\gA(\Tilde{a}_i) \otimes \phi_\gA(a'),  \theta_l \hat{\mu}_{Z|W, A} (\Tilde{w}_l, \Tilde{a}_j) \otimes \phi_\gA(\Tilde{a}_j) \otimes \phi_\gA(a')\rangle\\
    &+ \frac{\alpha_{m + 1}^2}{m^2} \sum_{j = 1}^m \sum_{\substack{l = 1 \\ l \ne j}}^m \sum_{r = 1}^m \sum_{\substack{s = 1 \\ s \ne r}}^m \langle \theta_l \hat{\mu}_{Z|W, A} (\Tilde{w}_l, \Tilde{a}_j) \otimes \phi_\gA(\Tilde{a}_j) \otimes \phi_\gA(a'), \theta_s \hat{\mu}_{Z|W, A} (\Tilde{w}_s, \Tilde{a}_r) \otimes \phi_\gA(\Tilde{a}_r) \otimes \phi_\gA(a')\rangle
\end{align*}
Using Equation (\ref{eq:AlternativeProxyMethodATTProofImportantEquation1}), we can write
\begin{align}
&\|\varphi\|^2 = \langle \varphi, \varphi \rangle\nonumber\\
&=\sum_{i = 1}^m \sum_{p = 1}^m \alpha_i \alpha_p \mBeta(\Tilde{w}_i, \Tilde{a}_i)^T \mK_{Z Z} \mBeta(\Tilde{w}_p, \Tilde{a}_p) k_\gA(\Tilde{a}_i, \Tilde{a}_p) k_\gA(a', a')\label{eq:AlternativeProxyMethodATTVarphiNorm1}\\&+2 \frac{\alpha_{m + 1}}{m} \sum_{i = 1}^m \sum_{j = 1}^m \sum_{\substack{l = 1 \\ l \ne j}}^m \alpha_i \mBeta(\Tilde{w}_i, \Tilde{a}_i)^T \mK_{Z Z} \theta_l \mBeta(\Tilde{w}_l, \Tilde{a}_j) k_\gA(\Tilde{a}_i, \Tilde{a}_j) k_\gA(a', a')\label{eq:AlternativeProxyMethodATTVarphiNorm2}\\
&+ \frac{\alpha_{m + 1}^2}{m^2} \sum_{j = 1}^m \sum_{\substack{l = 1 \\ l \ne j}}^m \sum_{r = 1}^m \sum_{\substack{s = 1 \\ s \ne r}}^m \theta_l \mBeta(\Tilde{w}_l, \Tilde{a}_j)^T \mK_{Z Z} \theta_s \mBeta(\Tilde{w}_s, \Tilde{a}_r) k_\gA(\Tilde{a}_j, \Tilde{a}_r) k_\gA(a', a')\label{eq:AlternativeProxyMethodATTVarphiNorm3}
\end{align}
The component in Equation (\ref{eq:AlternativeProxyMethodATTVarphiNorm1}) is equal to:
\begin{align*}
    \sum_{i = 1}^m \sum_{p = 1}^m \alpha_i \alpha_p \mBeta(\Tilde{w}_i, \Tilde{a}_i)^T \mK_{Z Z} \mBeta(\Tilde{w}_p, \Tilde{a}_p) k_\gA(\Tilde{a}_i, \Tilde{a}_p) k_\gA(a', a') = \alpha_{1:m}^T \Big[ \mB^T \mK_{Z Z} \mB \odot \mK_{\Tilde{A} \Tilde{A}}\Big] \alpha_{1:m} k_\gA(a', a')
\end{align*}
where 
\begin{align*}
    \alpha_{1:m} &= \begin{bmatrix}
        \alpha_1 & \alpha_2 & \ldots & \alpha_m
    \end{bmatrix}^T\\
    \mB &= (\mK_{W W} \odot \mK_{A A} + n \lambda_1 \mI)^{-1}(\mK_{W \Tilde{W}} \odot \mK_{A \Tilde{A}})
\end{align*}
The component in Equation (\ref{eq:AlternativeProxyMethodATTVarphiNorm2}) is equal to:
\begin{align*}
    &2 \frac{\alpha_{m + 1}}{m} \sum_{i = 1}^m \sum_{j = 1}^m \sum_{\substack{l = 1 \\ l \ne j}}^m \alpha_i \mBeta(\Tilde{w}_i, \Tilde{a}_i)^T \mK_{Z Z} \theta_l \mBeta(\Tilde{w}_l, \Tilde{a}_j) k_\gA(\Tilde{a}_i, \Tilde{a}_j) k_\gA(a', a')\\
    &= 2 \frac{\alpha_{m + 1}}{m} \sum_{i = 1}^m \sum_{j = 1}^m \alpha_i  \mBeta(\Tilde{w}_i, \Tilde{a}_i)^T \mK_{Z Z} \Big( \sum_{\substack{l = 1 \\ l \ne j}}^m \theta_l \mBeta(\Tilde{w}_l, \Tilde{a}_j) \Big) k_\gA(\Tilde{a}_i, \Tilde{a}_j) k_\gA(a', a')\\
    &=2 \alpha_{m + 1} \alpha_{1:m}^T\Big[ \mB^T \mK_{Z Z} \Tilde{\mB} \odot \mK_{\Tilde{A} \Tilde{A}} \Big] (\vone / m) k_\gA(a', a')
\end{align*}
where
\begin{align*}
    \Tilde{\mB}_{:, j} &=  \sum_{\substack{l = 1 \\ l \ne j}}^m \theta_l \mBeta(\Tilde{w}_l, \Tilde{a}_j) = \sum_{l = 1}^m (\mK_{W W} \odot \mK_{A A} + n \lambda \mI)^{-1}(\theta_l \mK_{W \Tilde{w}_l} \odot \mK_{A \Tilde{a}_j})\\&=(\mK_{W W} \odot \mK_{A A} + n \lambda_1 \mI)^{-1}\Big( \sum_{\substack{l = 1 \\ l \ne j}}^m \theta_l \mK_{W \Tilde{w}_l} \odot \mK_{A \Tilde{a}_j} \Big)\\
\end{align*}
Finally, the third component in Equation (\ref{eq:AlternativeProxyMethodATTVarphiNorm3}):
\begin{align*}
    &\frac{\alpha_{m + 1}^2}{m^2} \sum_{j = 1}^m \sum_{\substack{l = 1 \\ l \ne j}}^m \sum_{r = 1}^m \sum_{\substack{s = 1 \\ s \ne r}}^m \theta_l \mBeta(\Tilde{w}_l, \Tilde{a}_j)^T \mK_{Z Z} \theta_s \mBeta(\Tilde{w}_s, \Tilde{a}_r) k_\gA(\Tilde{a}_j, \Tilde{a}_r) k_\gA(a', a')\\
    &=\frac{\alpha_{m + 1}^2}{m^2} \sum_{j = 1}^m \sum_{r = 1}^m \Big( \sum_{\substack{l = 1 \\ l \ne j}}^m \theta_l \mBeta(\Tilde{w}_l, \Tilde{a}_j) \Big)^T \mK_{Z Z} \Big( \sum_{\substack{s = 1 \\ s \ne r}}^m \theta_s \mBeta(\Tilde{w}_s, \Tilde{a}_r) \Big) k_\gA(\Tilde{a}_j, \Tilde{a}_r) k_\gA(a', a')\\
    &= \alpha_{m + 1}^2 (\vone / m)^T \Big[ \Tilde{\mB}^T \mK_{Z Z} \Tilde{\mB} \odot \mK_{\Tilde{A} \Tilde{A}} \Big] (\vone / m) k_\gA(a', a')
\end{align*}
As a result, 
\begin{align}
    &\|\varphi\|^2 = \langle \varphi, \varphi \rangle \nonumber \\&= \alpha_{1:m}^T \Big[ \mB^T \mK_{Z Z} \mB \odot \mK_{\Tilde{A} \Tilde{A}}\Big] \alpha_{1:m} k_\gA(a', a') + 2 \alpha_{m + 1} \alpha_{1:m}^T\Big[ \mB^T \mK_{Z Z} \Bar{\mB} \odot \mK_{\Tilde{A} \Tilde{A}} \Big] (\vone / m) k_\gA(a', a') \nonumber \\&+ \alpha_{m + 1}^2 (\vone / m)^T \Big[ \Tilde{\mB}^T \mK_{Z Z} \Tilde{\mB} \odot \mK_{\Tilde{A} \Tilde{A}}\Big] (\vone / m) k_\gA(a', a')\nonumber
    \\&= \begin{bmatrix}
        \alpha_{1:m } \\ \alpha_{m + 1}
    \end{bmatrix}^T 
    \begin{bmatrix}
        \mB^T \mK_{Z Z} \mB \odot \mK_{\Tilde{A} \Tilde{A}} & [ \mB^T \mK_{Z Z} \Tilde{\mB} \odot \mK_{\Tilde{A} \Tilde{A}} ] (\vone / m)\\
        (\vone / m)^T [ \mB^T \mK_{Z Z} \Tilde{\mB} \odot \mK_{\Tilde{A} \Tilde{A}} ]^T & (\vone / m)^T \Big[ \Tilde{\mB}^T \mK_{Z Z} \Tilde{\mB} \odot \mK_{\Tilde{A} \Tilde{A}}\Big] (\vone / m)
    \end{bmatrix} 
    \begin{bmatrix}
        \alpha_{1:m } \\ \alpha_{m + 1}
    \end{bmatrix} k_\gA(a', a')\label{eq:AlternativeProxyMethodsATTVarphiNormFinal}
\end{align}
Next, to derive the matrix-vector multiplication form for the first component of the objective given in Equation (\ref{eq:2StageRegressionATTObjectiveFinal}), consider the following:
\begin{align*}
    &\Big\langle \varphi,  \hat{\mu}_{Z| W, A}(\Tilde{w}_i, \Tilde{a}_i) \otimes \phi_\gA(\Tilde{a}_i) \otimes \phi_\gA(a') \Big\rangle\\
    &= \Bigg\langle \sum_{l = 1}^m \alpha_l  \hat{\mu}_{Z|W, A} (\Tilde{w}_l, \Tilde{a}_l)\otimes \phi_\gA(\Tilde{a}_l) \otimes \phi_\gA(a') + \frac{\alpha_{m + 1}}{m}  \sum_{j = 1}^m \sum_{\substack{l = 1 \\ l \ne j}}^m \theta_l \hat{\mu}_{Z|W, A} (\Tilde{w}_l, \Tilde{a}_j) \otimes \phi_\gA(\Tilde{a}_j) \otimes \phi_\gA(a'),\\ 
    &\hat{\mu}_{Z| W, A}(\Tilde{w}_i, \Tilde{a}_i) \otimes \phi_\gA(\Tilde{a}_i) \otimes \phi_\gA(a')\Bigg\rangle\\
    &=\sum_{l = 1}^m \alpha_l \Big\langle  \hat{\mu}_{Z|W, A} (\Tilde{w}_l, \Tilde{a}_l)\otimes \phi_\gA(\Tilde{a}_l) \otimes \phi_\gA(a'), \hat{\mu}_{Z| W, A}(\Tilde{w}_i, \Tilde{a}_i) \otimes \phi_\gA(\Tilde{a}_i) \otimes \phi_\gA(a')\Big\rangle\\
    &+ \frac{\alpha_{m + 1}}{m} \sum_{j = 1}^m \sum_{\substack{l = 1 \\ l \ne j}}^m \Big\langle \theta_l \hat{\mu}_{Z|W, A} (\Tilde{w}_l, \Tilde{a}_j) \otimes \phi_\gA(\Tilde{a}_j) \otimes \phi_\gA(a'), 
    \hat{\mu}_{Z| W, A}(\Tilde{w}_i, \Tilde{a}_i) \otimes \phi_\gA(\Tilde{a}_i) \otimes \phi_\gA(a') \Big\rangle\\
    &= \sum_{l = 1}^m \alpha_l \mBeta(\Tilde{w}_l, \Tilde{a}_l)^T \mK_{Z Z} \mBeta(\Tilde{w}_i, \Tilde{a}_i) k_\gA(\Tilde{a}_l, \Tilde{a}_i) k_\gA(a', a') \\&+ \frac{\alpha_{m + 1}}{m} \sum_{j = 1}^m \Big(\sum_{\substack{l = 1 \\ l \ne j}}^m \theta_l \mBeta(\Tilde{w}_l, \Tilde{a}_j)\Big)^T \mK_{Z Z} \mBeta(\Tilde{w}_i, \Tilde{a}_i) k_\gA(\Tilde{a}_j, \Tilde{a}_i) k_\gA(a', a')\\
    &=\Big[\big[\mB^T \mK_{Z Z}\mB \odot \mK_{\Tilde{A}\Tilde{A}}\big] \alpha_{1:m}\Big]_i k_\gA(a', a') + \alpha_{m+1}\Big[\big[{\mB}^T \mK_{Z Z} \Tilde{\mB} \odot \mK_{\Tilde{A}\Tilde{A}}\big] (\vone / m)\Big]_i k_\gA(a', a')\\
    &= \Bigg[\begin{bmatrix}
        \mB^T \mK_{Z Z} \mB \odot \mK_{\Tilde{A}\Tilde{A}} & \big[{\mB}^T \mK_{Z Z} \Tilde{\mB} \odot \mK_{\Tilde{A}\Tilde{A}}\big] (\vone / m)
    \end{bmatrix} \begin{bmatrix}
        \alpha_{1:m}\\ \alpha_{m+1}
    \end{bmatrix} \Bigg]_i k_\gA(a', a')
\end{align*}
As a result, the first component in Equation (\ref{eq:2StageRegressionATTObjectiveFinal}) is given by
\begin{align}
    &\frac{1}{m} \sum_{i = 1}^m \Big\langle \varphi,  \hat{\mu}_{Z| W, A}(\Tilde{w}_i, \Tilde{a}_i) \otimes \phi_\gA(\Tilde{a}_i) \otimes \phi_\gA(a') \Big\rangle^2 \nonumber\\
    &= \frac{1}{m} \begin{bmatrix}
        \alpha_{1:m} \\ \alpha_{m+1}
    \end{bmatrix}^T 
    \begin{bmatrix}
        \mB^T \mK_{Z Z}\mB \odot \mK_{\Tilde{A}\Tilde{A}} \\ (\frac{\vone}{m})^T \big[{\mB}^T \mK_{Z Z} \Tilde{\mB} \odot \mK_{\Tilde{A}\Tilde{A}}\big]^T
    \end{bmatrix} 
    \begin{bmatrix}
        \mB^T \mK_{Z Z}\mB \odot \mK_{\Tilde{A}\Tilde{A}} & \big[{\mB}^T \mK_{Z Z} \Tilde{\mB} \odot \mK_{\Tilde{A}\Tilde{A}}\big] \frac{\vone}{m}
    \end{bmatrix} 
    \begin{bmatrix}
        \alpha_{1:m}\\ \alpha_{m+1}
    \end{bmatrix}k_\gA(a', a')^2
\label{eq:AlternativeProxyMethodsATTProofLossFirstTerm}
\end{align}
Lastly, for the second component in Equation (\ref{eq:2StageRegressionATTObjectiveFinal}), we note that
\begin{align}
    &\frac{1}{m} \sum_{i = 1}^m \sum_{\substack{j = 1 \\ j \ne i}}^m \Big\langle \varphi, \theta_j \hat{\mu}_{Z|W, A} (\Tilde{w}_j, \Tilde{a}_i) \otimes \phi_\gA(\Tilde{a}_i) \otimes \phi_\gA(a') \Big\rangle \nonumber\\
    &= \frac{1}{m} \sum_{i = 1}^m \sum_{\substack{j = 1 \\ j \ne i}}^m \sum_{l = 1}^m \alpha_l \Big\langle 
     \hat{\mu}_{Z|W, A} (\Tilde{w}_l, \Tilde{a}_l)\otimes \phi_\gA(\Tilde{a}_l) \otimes \phi_\gA(a'), \theta_j \hat{\mu}_{Z|W, A} (\Tilde{w}_j, \Tilde{a}_i) \otimes \phi_\gA(\Tilde{a}_i) \otimes \phi_\gA(a')
    \Big\rangle \nonumber\\
    &+ \frac{\alpha_{m + 1}}{m^2} \sum_{i = 1}^m \sum_{\substack{j = 1 \\ j \ne i}}^m \sum_{r = 1}^m \sum_{\substack{s = 1 \\ s \ne r}}^m  \Big\langle \theta_s \hat{\mu}_{Z|W, A} (\Tilde{w}_s, \Tilde{a}_r) \otimes \phi_\gA(\Tilde{a}_r) \otimes \phi_\gA(a'), 
    \theta_j \hat{\mu}_{Z|W, A} (\Tilde{w}_j, \Tilde{a}_i) \otimes \phi_\gA(\Tilde{a}_i) \otimes \phi_\gA(a') \Big\rangle \nonumber\\
    &= \frac{1}{m} \sum_{i = 1}^m \sum_{\substack{j = 1 \\ j \ne i}}^m \sum_{l = 1}^m \alpha_l \mBeta(\Tilde{w}_l, \Tilde{a}_l)^T\mK_{Z Z} \theta_j \mBeta(\Tilde{w}_j, \Tilde{a}_i) k_\gA(\Tilde{a}_l, \Tilde{a}_i) k_\gA(a', a') \nonumber\\
    &+ \frac{\alpha_{m + 1}}{m^2} \sum_{i = 1}^m \sum_{\substack{j = 1 \\ j \ne i}}^m \sum_{r = 1}^m \sum_{\substack{s = 1 \\ s \ne r}}^m \theta_s \mBeta(\Tilde{w}_s, \Tilde{a}_r)^T \mK_{Z Z} \theta_j \mBeta(\Tilde{w}_j, \Tilde{a}_i) k_\gA(\Tilde{a}_r, \Tilde{a}_i) k_\gA(a', a') \nonumber\\
    &= \frac{1}{m} \sum_{i = 1}^m \sum_{l = 1}^m \alpha_l \mBeta(\Tilde{w}_l, \Tilde{a}_l)^T\mK_{Z Z}\Big(\sum_{\substack{j = 1 \\ j \ne i}}^m  \theta_j \mBeta(\Tilde{w}_j, \Tilde{a}_i)\Big) k_\gA(\Tilde{a}_l, \Tilde{a}_i) k_\gA(a', a') \nonumber\\
    &+ \frac{\alpha_{m + 1}}{m^2} \sum_{i = 1}^m  \sum_{r = 1}^m \Big( \sum_{\substack{s = 1 \\ s \ne r}}^m \theta_s \mBeta(\Tilde{w}_s, \Tilde{a}_r) \Big)^T \mK_{Z Z} \Big( \sum_{\substack{j = 1 \\ j \ne i}}^m \theta_j \mBeta(\Tilde{w}_j, \Tilde{a}_i) \Big) k_\gA(\Tilde{a}_r, \Tilde{a}_i) k_\gA(a', a') \nonumber\\
    &=\frac{1}{m} \alpha_{1:m}^T \Big[ \mB^T \mK_{Z Z} \Tilde{\mB} \odot \mK_{\Tilde{A}\Tilde{A}}\Big] \vone k_\gA(a', a') + \alpha_{m+1} \frac{1}{m^2} \vone^T \Big[ \Tilde{\mB} \mK_{Z Z} \Tilde{\mB} \odot \mK_{\Tilde{A}\Tilde{A}} \Big] \vone k_\gA(a', a') \nonumber \\
&=  \begin{bmatrix}
    \alpha_{1:m } \\ \alpha_{m + 1}
\end{bmatrix}^T \begin{bmatrix}
 [ \mB^T \mK_{Z Z} \Tilde{\mB} \odot \mK_{\Tilde{A} \Tilde{A}} ] \frac{\vone}{m}\\
 (\frac{\vone}{m})^T \Big[ \Tilde{\mB}^T \mK_{Z Z} \Tilde{\mB} \odot \mK_{\Tilde{A} \Tilde{A}}\Big] \frac{\vone}{m}
\end{bmatrix} k_\gA(a', a') \label{eq:AlternativeProxyMethodsATTProofLossSecondTerm}
\end{align}
Now, we are ready to combine our findings and write the loss function in terms of matrix-vector multiplications. Using Equation (\ref{eq:AlternativeProxyMethodsATTVarphiNormFinal}),  (\ref{eq:AlternativeProxyMethodsATTProofLossFirstTerm}) and (\ref{eq:AlternativeProxyMethodsATTProofLossSecondTerm}), the loss function can be expressed as 
\begin{align}
    &\hat{\gL}^{2SR}_m(\varphi) = \frac{1}{m} \sum_{i = 1}^m \langle \varphi, \hat{\mu}_{Z|W, A} (\Tilde{w}_i, \Tilde{a}_i)\otimes \phi_\gA(\Tilde{a}_i) \otimes \phi_\gA(a') \rangle_{\gH_\gZ \otimes \gH_\gA \otimes \gH_\gA}^2 \nonumber \\&-2 \frac{1}{m} \sum_{i = 1}^m \sum_{\substack{j = 1 \\ j \ne i}}^m \Big\langle \varphi, \theta_j \hat{\mu}_{Z|W, A} (\Tilde{w}_j, \Tilde{a}_i) \otimes \phi_\gA(\Tilde{a}_i) \otimes \phi_\gA(a') \Big\rangle_{\gH_\gZ \otimes \gH_\gA \otimes \gH_\gA} + \lambda_2 \|\varphi_0\|^2_{\gH_\gZ \otimes \gH_\gA \otimes \gH_\gA} \nonumber\\
    &= \frac{1}{m} \begin{bmatrix}
        \alpha_{1:m} \\ \alpha_{m+1}
    \end{bmatrix}^T 
    \begin{bmatrix}
        \mB^T \mK_{Z Z}\mB \odot \mK_{\Tilde{A}\Tilde{A}} \nonumber\\ (\frac{\vone}{m})^T \big[{\mB}^T \mK_{Z Z} \Tilde{\mB} \odot \mK_{\Tilde{A}\Tilde{A}}\big]^T
    \end{bmatrix} 
    \begin{bmatrix}
        \mB^T \mK_{Z Z}\mB \odot \mK_{\Tilde{A}\Tilde{A}} & \big[{\mB}^T \mK_{Z Z} \Tilde{\mB} \odot \mK_{\Tilde{A}\Tilde{A}}\big] \frac{\vone}{m}
    \end{bmatrix} 
    \begin{bmatrix}
        \alpha_{1:m}\\ \alpha_{m+1}
    \end{bmatrix} k_\gA(a', a')^2\nonumber\\
    &-2 \begin{bmatrix}
    \alpha_{1:m } \\ \alpha_{m + 1}
    \end{bmatrix}^T \begin{bmatrix}
     [ \mB^T \mK_{Z Z} \Tilde{\mB} \odot \mK_{\Tilde{A} \Tilde{A}} ] \frac{\vone}{m}\\
     (\frac{\vone}{m})^T \Big[ \Tilde{\mB}^T \mK_{Z Z} \Tilde{\mB} \odot \mK_{\Tilde{A} \Tilde{A}}\Big] \frac{\vone}{m}
    \end{bmatrix} k_\gA(a', a')\nonumber\\
    &+ \lambda_2 \begin{bmatrix}
        \alpha_{1:m } \\ \alpha_{m + 1}
    \end{bmatrix}^T 
    \begin{bmatrix}
        \mB^T \mK_{Z Z} \mB \odot \mK_{\Tilde{A} \Tilde{A}} & [ \mB^T \mK_{Z Z} \Tilde{\mB} \odot \mK_{\Tilde{A} \Tilde{A}} ] \frac{\vone}{m}\\
        (\frac{\vone}{m})^T [ \mB^T \mK_{Z Z} \Tilde{\mB} \odot \mK_{\Tilde{A} \Tilde{A}} ]^T & (\frac{\vone}{m})^T \Big[ \Tilde{\mB}^T \mK_{Z Z} \Tilde{\mB} \odot \mK_{\Tilde{A} \Tilde{A}}\Big] \frac{\vone}{m}
    \end{bmatrix} 
    \begin{bmatrix}
        \alpha_{1:m } \\ \alpha_{m + 1}
    \end{bmatrix} k_\gA(a', a')
\label{eq:AlternativeProxyMethodATTProofLossMatrixVectorProductForm}
\end{align}
The optimal coefficients $\{\alpha_{1:m}, \alpha_{m+1}\}$ can be found by setting the derivative of Equation (\ref{eq:AlternativeProxyMethodATTProofLossMatrixVectorProductForm}) to zero. With these optimal coefficients, let $\hat{\varphi}_{\lambda_2, m}$ denote the minimizer of $\hat{\gL}^{2SR}_m(\varphi)$. Using $\hat{\varphi}_{\lambda_2, m}$, we can estimate $\E[Y \hat{\varphi}_{\lambda_2, m}(Z, a, a') | A = a]$. First, observe that
\begin{align*}
    \E[Y \hat{\varphi}_{\lambda_2, m}(Z, a, a') | A = a] &= \E [Y \langle \hat{\varphi}_{\lambda_2, m}, \phi_\gZ(Z) \otimes \phi_\gA(a) \otimes \phi_\gA(a') \rangle | A = a]\\
    &= \Big\langle  \hat{\varphi}_{\lambda_2, m}, \E [Y \phi_\gZ(Z) | A = a] \otimes \phi_\gA(a) \otimes \phi_\gA(a') \Big\rangle \\
    &\approx \langle \hat{\varphi}_{\lambda_2, m}, \hat{C}_{YZ|A}  \phi_\gA(a) \otimes \phi_\gA(a) \otimes \phi_\gA(a') \rangle
\end{align*}
where $\hat{C}_{YZ|A} $ is the estimation of the conditional mean $\E [Y \phi_\gZ(Z) | A = \cdot]$, as used in the dose-response curve algorithm. It is found by kernel ridge regression: 
\begin{align*}
    \hat{C}_{YZ|A} \phi_\gA(a) = \Phi_\gZ \text{diag}(\mY) [\mK_{A A} + n \lambda_3 \mI]^{-1} \mK_{A a}
\end{align*}
Thus,
\begin{align*}
    \hat{\E}[Y \hat{\varphi}_{\lambda_2, m}(Z, a, a') | A = a] &= \langle \hat{\varphi}_{\lambda_2, m}, \hat{C}_{YZ|A} \phi_\gA(a) \otimes \phi_\gA(a) \otimes \phi_\gA(a') \rangle \\
    &= \Big\langle \sum_{l = 1}^m \alpha_l  \hat{\mu}_{Z|W, A} (\Tilde{w}_l, \Tilde{a}_l)\otimes \phi_\gA(\Tilde{a}_l) \otimes \phi_\gA(a'),  \hat{C}_{YZ|A} \phi_\gA(a) \otimes \phi_\gA(a) \otimes \phi_\gA(a') \Big\rangle \\
    &+ \Big\langle  \frac{\alpha_{m + 1}}{m} \sum_{j = 1}^m \sum_{\substack{l = 1 \\ l \ne j}}^m \theta_l \hat{\mu}_{Z|W, A} (\Tilde{w}_l, \Tilde{a}_j) \otimes \phi_\gA(\Tilde{a}_j) \otimes \phi_\gA(a'), \hat{C}_{YZ|A} \phi_\gA(a) \otimes \phi_\gA(a) \otimes \phi_\gA(a') \Big\rangle \\
    &= \sum_{l = 1}^m \alpha_l \langle \hat{\mu}_{Z|W, A} (\Tilde{w}_l, \Tilde{a}_l)\otimes \phi_\gA(\Tilde{a}_l) \otimes \phi_\gA(a') ,  \hat{C}_{YZ|A}\phi_\gA(a) \otimes \phi_\gA(a) \otimes \phi_\gA(a') \rangle\\&+ \frac{\alpha_{m + 1}}{m} \sum_{j = 1}^m \sum_{\substack{l = 1 \\ l \ne j}}^m \theta_l \langle \hat{\mu}_{Z|W, A} (\Tilde{w}_l, \Tilde{a}_j) \otimes \phi_\gA(\Tilde{a}_j) \otimes \phi_\gA(a'), \hat{C}_{YZ|A} \phi_\gA(a) \otimes \phi_\gA(a) \otimes \phi_\gA(a') \rangle \\
    &= \sum_{l = 1}^m \alpha_l \langle \hat{\mu}_{Z|W, A} (\Tilde{w}_l, \Tilde{a}_l),  \hat{C}_{YZ|A} \phi_\gA(a) \rangle \langle \phi_\gA(\Tilde{a}_l), \phi_\gA(a)\rangle \langle \phi_\gA(a'), \phi_\gA(a') \rangle\\&+ \frac{\alpha_{m + 1}}{m} \sum_{j = 1}^m \sum_{\substack{l = 1 \\ l \ne j}}^m \langle \theta_l \hat{\mu}_{Z|W, A} (\Tilde{w}_l, \Tilde{a}_j), \hat{C}_{YZ|A} \phi_\gA(a) \rangle \langle \phi_\gA(\Tilde{a}_j), \phi_\gA(a)\rangle \langle \phi_\gA(a'), \phi_\gA(a') \rangle\\
    &= \sum_{l = 1}^m \alpha_l \mBeta(\Tilde{w}_l, \Tilde{a}_l)^T \Phi_\gZ^T \Phi_\gZ \text{diag}(\mY)[\mK_{A A} + n \lambda_3 \mI]^{-1} \mK_{A a} k_\gA(\Tilde{a}_l, a) k_\gA(a', a')\\
    &+ \frac{\alpha_{m + 1} }{m} \sum_{j = 1}^m \sum_{\substack{l = 1 \\ l \ne j}}^m \theta_l \mBeta(\Tilde{w}_l, \Tilde{a}_j)^T \Phi_\gZ^T \Phi_\gZ \text{diag}(\mY) [\mK_{A A} + n \lambda_3 \mI]^{-1} \mK_{A a} k_\gA(\Tilde{a}_j, a) k_\gA(a', a')\\
    &= \sum_{l = 1}^m \alpha_l \mBeta(\Tilde{w}_l, \Tilde{a}_l)^T \mK_{Z Z} \text{diag}(\mY)[\mK_{A A} + n \lambda_3 \mI]^{-1} \mK_{A a} k_\gA(\Tilde{a}_l, a) k_\gA(a', a')\\
    &+ \frac{\alpha_{m + 1}}{m} \sum_{j = 1}^m \sum_{\substack{l = 1 \\ l \ne j}}^m \theta_l \mBeta(\Tilde{w}_l, \Tilde{a}_j)^T \mK_{Z Z} \text{diag}(\mY) [\mK_{A A} + n \lambda_3 \mI]^{-1} \mK_{A a} k_\gA(\Tilde{a}_j, a) k_\gA(a', a')\\
    &= \alpha_{1:m}^T \Big( \mB^T \big(\mK_{Z Z} \text{diag}(\mY) [\mK_{A A} + n \lambda_3 \mI]^{-1} \mK_{A a}\big) \odot \mK_{\Tilde{A} a}\Big) k_\gA(a', a')\\
    &+ \alpha_{m + 1} \Big( \Tilde{\mB}^T \big( \mK_{Z Z}  \text{diag}(\mY) [\mK_{A A} + n \lambda_3 \mI]^{-1} \mK_{A a} \big) \odot \mK_{\Tilde{A} a}\Big) \frac{\vone}{m} k_\gA(a', a')
\end{align*}
The conditional dose-response curve can therefore be expressed in the closed-form as
\begin{align*}
    f_{\text{ATT}}(a, a') &= \alpha_{1:m}^T \Big( \mB^T \big(\mK_{Z Z} \text{diag}(\mY) [\mK_{A A} + n \lambda_2 \mI]^{-1} \mK_{A a}\big) \odot \mK_{\Tilde{A} a}\Big) k_\gA(a', a')\\
    &+ \alpha_{m + 1} \Big( \Tilde{\mB}^T \big( \mK_{Z Z}  \text{diag}(\mY) [\mK_{A A} + n \lambda_2 \mI]^{-1} \mK_{A a} \big) \odot \mK_{\Tilde{A} a}\Big) \frac{\vone}{m} k_\gA(a', a')
\end{align*}
\end{proof}

\begin{remark}
    Given the optimal coefficients $\{\alpha_{1:m}, \alpha_{m+1}\}$ from Algorithm (\ref{algo:ATT_algorithm_with_confounders}), the bridge function can be written in the closed-form as
    \begin{align*}
        \hat{\varphi}_{\lambda_2, m}(z, a, a') = k_\gA(a', a')\alpha_{1:m}^T [(\mB^T \mK_{Z z}) \odot \mK_{\Tilde{A} a}] + k_\gA(a', a')\alpha_{m + 1} \Big(\frac{\vone}{m} \Big)^T [(\Tilde{\mB}^T \mK_{Z z}) \odot \mK_{\Tilde{A} a}]
    \end{align*}
    \label{remark:ClosedFormForBridgeFunctionATT}
\end{remark}
\begin{proof}
    \begin{align*}
    &\hat{\varphi}_{\lambda_2, m}(z, a, a') =  \langle \hat{\varphi}_{\lambda_2, m}, \phi_\gZ(z) \otimes \phi_\gA(a) \otimes \phi_\gA(a') \rangle\\
    &= \Bigg\langle \sum_{l = 1}^m \alpha_l  \hat{\mu}_{Z|W,A} (\Tilde{w}_l, \Tilde{a}_l) \otimes \phi_\gA(\Tilde{a}_l) \otimes \phi_\gA(a') + \frac{\alpha_{m + 1}}{m} \sum_{j = 1}^m \sum_{\substack{l = 1 \\ l \ne j}}^m \theta_l \hat{\mu}_{Z|W, A} (\Tilde{w}_l, \Tilde{a}_j) \otimes \phi_\gA(\Tilde{a}_j)  \otimes \phi_\gA(a'),\\& \phi_\gZ(z) \otimes \phi_\gA(a)  \otimes \phi_\gA(a') \Bigg\rangle\\
    &= \sum_{l = 1}^m \alpha_l  \langle \hat{\mu}_{Z|W, A} (\Tilde{w}_l, \Tilde{a}_l)\otimes \phi_\gA(\Tilde{a}_l)  \otimes \phi_\gA(a'),\phi_\gZ(z) \otimes \phi_\gA(a)  \otimes \phi_\gA(a') \rangle\\
    &+ \frac{\alpha_{m + 1} }{m} \sum_{j = 1}^m \sum_{\substack{l = 1 \\ l \ne j}}^m \theta_l\langle \hat{\mu}_{Z|W, A} (\Tilde{w}_l, \Tilde{a}_j) \otimes \phi_\gA(\Tilde{a}_j) \otimes \phi_\gA(a'),\phi_\gZ(z) \otimes \phi_\gA(a)  \otimes \phi_\gA(a') \rangle\\
    &= \sum_{l = 1}^m \alpha_l  \langle \hat{\mu}_{Z|W, A} (\Tilde{w}_l, \Tilde{a}_l),\phi_\gZ(z) \rangle k_\gA(\Tilde{a}_l, a) k_\gA(a', a') \\&+ \frac{\alpha_{m + 1}}{m(m-1)} \sum_{j = 1}^m \sum_{\substack{l = 1 \\ l \ne j}}^m \theta_l \langle \hat{\mu}_{Z|W, A} (\Tilde{w}_l, \Tilde{a}_j),\phi_\gZ(z) \rangle k_\gA(\Tilde{a}_j, a) k_\gA(a', a')\\
    &= \sum_{l = 1}^m \alpha_l  \langle \Phi_\gZ {\mBeta}(\Tilde{w}_l, \tilde{a}_l),\phi_\gZ(z) \rangle k_\gA(\Tilde{a}_l, a) k_\gA(a', a') \\&+ \frac{\alpha_{m + 1}}{m} \sum_{j = 1}^m \sum_{\substack{l = 1 \\ l \ne j}}^m \theta_l \langle \Phi_\gZ {\mBeta}(\Tilde{w}_l, \tilde{a}_j),\phi_\gZ(z) \rangle k_\gA(\Tilde{a}_j, a) k_\gA(a', a')\\
    &= \sum_{l = 1}^m \alpha_l \mK_{Z z}^T {\mBeta}(\Tilde{w}_l, \tilde{a}_l) k_\gA(\Tilde{a}_l, a) k_\gA(a', a') + \frac{\alpha_{m + 1}}{m} \sum_{j = 1}^m \sum_{\substack{l = 1 \\ l \ne j}}^m \theta_l \mK_{Z z}^T \mBeta(\Tilde{w}_l, \tilde{a}_j)k_\gA(\Tilde{a}_j, a) k_\gA(a', a')\\
    &= \sum_{l = 1}^m \alpha_l \mK_{Z z}^T (\mK_{W W} \odot \mK_{A A} + n \lambda_1 \mI)^{-1}(\mK_{W \Tilde{w}_l} \odot \mK_{A \tilde{a}_l}) k_\gA(\Tilde{a}_l, a) k_\gA(a', a')\\
    &+ \frac{\alpha_{m + 1}}{m} \sum_{j = 1}^m \sum_{\substack{l = 1 \\ l \ne j}}^m \theta_l \mK_{Z z}^T (\mK_{W W} \odot \mK_{A A} + n \lambda_1 \mI)^{-1}(\mK_{W \Tilde{w}_l} \odot \mK_{A \tilde{a}_j}) k_\gA(\Tilde{a}_j, a) k_\gA(a', a')\\
    &= \sum_{l = 1}^m \alpha_l \mK_{Z z}^T (\mK_{W W} \odot \mK_{A A} + n \lambda_1 \mI)^{-1}(\mK_{W \Tilde{w}_l} \odot \mK_{A \tilde{a}_l}) k_\gA(\Tilde{a}_l, a) k_\gA(a', a')\\
    &+ \frac{\alpha_{m + 1}}{m}  \sum_{j = 1}^m  \mK_{Z z}^T (\mK_{W W} \odot \mK_{A A} + n \lambda_1 \mI)^{-1}\Big(\sum_{\substack{l = 1 \\ l \ne j}}^m \theta_l \mK_{W \Tilde{w}_l}  \odot \mK_{A \tilde{a}_j} \Big) k_\gA(\Tilde{a}_j, a) k_\gA(a', a')\\
    &= k_\gA(a', a')\alpha_{1:m}^T [(\mB^T \mK_{Z z}) \odot \mK_{\Tilde{A} a}] + k_\gA(a', a')\alpha_{m + 1} \Big(\frac{\vone}{m} \Big)^T [(\Tilde{\mB}^T \mK_{Z z}) \odot \mK_{\Tilde{A} a}]
\end{align*}
\end{proof}

\section{CONSISTENCY RESULTS} \label{sec:consistency_appendix}

In this section, we provide the consistency result of our proposed method. We make the following assumptions on the kernels and on the noise between the outcome and the treatment.

\begin{assumption}[Replication of Assumption (\ref{assum:Stage1ConsistencyKernelAssumptions_main})] 
For $\gF \in \{\gA, \gW, \gZ\}$, we assume that
    \begin{itemize}
        \item  $\gF$ is a Polish space;
        \item $k_\gF(f, .)$, is continuous for almost every $f \in \gF $;
        \item 
        $k_\gF(f, .)$, is bounded by $\kappa$ for almost every $f \in \gF $, i.e., 
        \begin{align*}
            \sup_{f \in \gF} \|k_\gF(f, .)\|_{\gH_\gF} \le \kappa.
        \end{align*}
        \item There exists $R,\sigma > 0$ such that for all $q \geq 2$, $P_A-$almost surely, \begin{equation}\label{eq:mom_stage_3}
    \E[(Y - \E[Y \mid A])^q \mid A] \leq \frac{1}{2}q!\sigma^2R^{q-2}.
    \end{equation}
    \end{itemize}

    \label{assum:Stage1ConsistencyKernelAssumptions}
\end{assumption}

The last assumption is a Bernstein moment condition used to control the noise of the observations
(see \citet{caponnetto2007optimal,fischer2020sobolev} for more details). If $Y$
is almost surely bounded, then the condition is automatically satisfied. It is possible to prove that the Bernstein condition is equivalent to
sub-exponentiality, see \citet[][Remark 4.9]{mollenhauer2022learning}.

\subsection{Consistency Results for Dose-Response Curve}

\subsubsection{Assumptions for ATE}

Under Assumption~(\ref{asst:well_specifiedness}-2), there exists a solution to the bridge equation within the RKHS $\gH_{\gZ \gA}$. However, this solution might not be unique. We therefore introduce the following minimum norm solution in the set of valid bridge functions. 


\begin{definition}[Bridge solution with minimum RKHS norm] \label{def:min_norm_bridge}
We define 
\begin{align*}
\bar{\varphi}_0 = \argmin_{\varphi 
    \in \gH_{\gZ\gA}} \|\varphi\|_{\gH_{\gZ\gA}} \quad \text{ s.t. } \quad \E[\varphi(Z, A) | W, A] = r(W, A),
\end{align*}
with $r(W, A) = \frac{p(W) p(A)}{p(W, A)}$.
\end{definition}

Under Assumption~(\ref{asst:well_specifiedness}-2), $\bar{\varphi}_0$ is well-defined. We will show that the estimator from stage 2 converges to $\bar{\varphi}_0$ in RKHS norm and we will therefore be able to obtain consistency guarantees for the dose-response function. 

\begin{remark}[Uniqueness of the bridge function]
    We notice that previous works on Proximal Causal Learning \citep{Mastouri2021ProximalCL,xu2021deep, singh2023kernelmethodsunobservedconfounding, semiparametricProximalCausalInference,wu2024doubly} require the bridge solution to be unique. However, we show that this assumption is not needed and convergence to the minimum norm bridge solution is enough to obtain consistency on the estimation of the dose-response curve. Our results build upon recent advances in instrumental variable regression with kernel methods \cite{meunier2024nonparametric}.
\end{remark}

We now introduce the covariance operators associated to stages 1, 2 and 3.

\begin{definition}
The covariance operators are defined as 
\begin{enumerate}
    \item (Stage 1) $\Sigma_1 := \E\left[\phi_{\gW\gA}(W, A) \otimes \phi_{\gW\gA}(W, A)\right], \qquad \phi_{\gW\gA}(W, A) = \phi_{\gW}(W) \otimes \phi_{\gA}(A)$;
    \item (Stage 2) 
    $\Sigma_2 := \E_{W, A}\left[ \left(\big(\mu_{Z|W, A} (W, A) \otimes \phi_\gA(A)\big) \otimes \big(\mu_{Z|W, A} (W, A)\otimes \phi_\gA(A) \big)\right)\right];$
    \item (Stage 3) $\Sigma_3 := \E[\phi_\gA(A) \otimes \phi_\gA(A)]$.
\end{enumerate}
\label{definition:CovarianceOperators_ATE}
\end{definition}
$\Sigma_1, \Sigma_2, \Sigma_3$ are self-adjoint and positive semi-definite operators. Under Assumption~(\ref{assum:Stage1ConsistencyKernelAssumptions}), they are trace class, and therefore compact, which implies that they have a countable spectrum \citep{Ingo_support_vector_machines}.

The next proposition relates $\bar{\varphi}_0$ to $\Sigma_2$.

\begin{proposition} \label{prop:min_norm_charac}
Under Assumption~(\ref{asst:well_specifiedness}-2), $\bar{\varphi}_0$ is well-defined and is the unique element of $\gH_{\gZ\gA}$ satisfying $\E[\bar{\varphi}_0(Z, A) | W, A] = r(W, A)$ and such that $\bar{\varphi}_0 \in \operatorname{null}(\Sigma_2)^{\perp}$.
\end{proposition}

\begin{proof}
Note that the bridge equation, for an element $\varphi \in \gH_{\gZ\gA}$, can be written as
$$
r(W, A) = E[\varphi(Z, A) | W, A] = \langle \varphi, \mu_{Z \mid W,A}(W,A) \otimes \phi_{\gA}(A) \rangle_{\gH_{\gZ\gA}} = A\varphi,
$$
by using the reproducing property and introducing the operator $A: \gH_{\gZ\gA} \to \gL_2(\gW \times \gA, p_{W, A}), \varphi \mapsto \langle \varphi, \mu_{Z \mid W,A}(W,A) \otimes \phi_{\gA}(A) \rangle_{\gH_{\gZ\gA}}$ where $\gL_2(\gW \times \gA, p_{W, A})$ denotes the square integrable functions with respect to the measure $p(W, A)$. By Assumption~(\ref{asst:well_specifiedness}-2), $A^{-1}(\{r\}) \subseteq  \gH_{\gZ\gA}$ is not empty. Fix $\varphi$ an element of $A^{-1}(\{r\})$. Since $\gH_{\gZ\gA} = \operatorname{null}(A) \oplus \operatorname{null}(A)^{\perp}$ there exists a unique pair $(\varphi',\varphi'') \in \operatorname{null}(A)^{\perp} \times \operatorname{null}(A)$ such that $\varphi = \varphi' + \varphi''$. Since $\varphi \in A^{-1}(\{r\})$ and $\varphi'' \in \operatorname{null}(A)$, we have:
$$
r = A \varphi = A \varphi' + A \varphi'' = A \varphi'.
$$
Therefore $\varphi' \in A^{-1}(\{r\})$. Furthermore, $\|\varphi\|^2_{\gH_{\gZ\gA}} = \|\varphi'\|^2_{\gH_{\gZ\gA}} + \|\varphi''\|^2_{\gH_{\gZ\gA}} \geq \|\varphi'\|^2_{\gH_{\gZ\gA}}$. This proves that the minimum norm solution in $\gH_{\gZ\gA}$ exists and is uniquely defined as $\varphi'$ and belongs to $\operatorname{null}(A)^{\perp} \cap A^{-1}(\{r\})$. We then show that $\operatorname{null}(A)^{\perp} \cap A^{-1}(\{r\})$ contains only one element. Assume that there exists $\varphi,\tilde{\varphi} \in \operatorname{null}(A)^{\perp} \cap A^{-1}(\{r\})$, then $A (\varphi-\tilde{\varphi}) = r - r = 0$, therefore $\varphi-\tilde{\varphi} \in \operatorname{null}(A)$. But since we also have $\varphi-\tilde{\varphi}  \in \operatorname{null}(A)^{\perp}$, it implies $\varphi = \tilde{\varphi}$. To conclude, observe that $A$ is such that $\Sigma_2 = A^*A$, therefore $\operatorname{null}(A) = \operatorname{null}(A^*A) = \operatorname{null}(\Sigma_2
)$.
\end{proof}

To characterize the smoothness of the target functions for each respective stage we employ the following source assumption.


\begin{assumption}[Replication of Assumption (\ref{asst:src_all_stages_main})] We assume that the following conditions hold:\label{asst:src_all_stages}
\begin{enumerate}
    \item  There exists a constant $B_1 < \infty$ such that for a given $\beta_1 \in (1, 3]$,
    $$
    \|{C}_{Z|W,A}\Sigma_1^{-\frac{\beta_1-1}{2}}\|_{S_2(\gH_{\gW\gA}, \gH_\gZ)} \le B_1
    $$   
    \item There exists a constant $B_2 < \infty$ such that for a given $\beta_2 \in (1, 3]$, 
    $$
        \|\Sigma_2^{-\frac{\beta_2-1}{2}}\bar{\varphi}_0\|_{\gH_{\gZ\gA}} \le B_2.
    $$
    \item There exists a constant $B_3 < \infty$ such that for a given $\beta_3 \in (1,3]$,
    $$
    \|C_{YZ \mid A}\Sigma_3^{-\frac{\beta_3-1}{2}}\|_{S_2(\gH_{\gA}, \gH_{\gZ})} \leq B_3.
    $$
\end{enumerate}
\end{assumption}

This assumption is referred to as the source condition in the literature \citep{caponnetto2007optimal, fischer2020sobolev}. It measures the smoothness of the regression functions with respect to the covariance operators. The inverse covariance operators have to be understood as Moore–Penrose pseudoinverses \citep{ben2006generalized}. In particular, $\Sigma_2^{\frac{\beta_2-1}{2}}\Sigma_2^{-\frac{\beta_2-1}{2}} = P_2$, with $P_2$ the orthogonal projection onto $\operatorname{null}\left(\Sigma_2^{\frac{\beta_2-1}{2}}\right)^{\perp} = \operatorname{null}(\Sigma_2)^{\perp}.$ Combined with Proposition~(\ref{prop:min_norm_charac}), we obtain the following result.

\begin{proposition} \label{prop:min_norm_charac2}
    Under Assumption~(\ref{asst:well_specifiedness}-2), 
    $\bar{\varphi}_0 = \Sigma_2^{\frac{\beta_2-1}{2}}\Sigma_2^{-\frac{\beta_2-1}{2}}\bar{\varphi}_0.$
\end{proposition}

\begin{remark}[Smoothness and minimum norm solution] $\bar{\varphi}_0$ is the unique RKHS solution to the bridge equation such that Proposition~(\ref{prop:min_norm_charac2}) holds. Indeed by Proposition~(\ref{prop:min_norm_charac}), $\bar{\varphi}_0$ is the unique RKHS solution to the bridge equation such that $\bar{\varphi}_0 \in \operatorname{null}(\Sigma_2)^{\perp}$ and therefore such that $\bar{\varphi}_0 = P_2\bar{\varphi}_0$. This crucial property will allow us to show that our bridge estimator converges to $\bar{\varphi}_0$.

\end{remark}

The next assumption characterize the effective dimension of the RKHSs associated to each stage. It is a standard assumption on the eigenvalue decay of the covariance operators (see
more details in \citet{caponnetto2007optimal,fischer2020sobolev}).

\begin{assumption} We assume the following conditions hold\label{asst:evd_all_stages}
    \begin{enumerate}
        \item  Let $(\lambda_{1,i})_{i\geq 1}$ be the eigenvalues of $\Sigma_1$. For some constant $c_1 >0$ and parameter $p_1 \in (0,1]$ and for all $i \geq 1$,
    \begin{equation*}
        \lambda_{1,i} \leq c_1i^{-1/p_1}.
    \end{equation*}    
    \item 
    Let $(\lambda_{2,i})_{i\geq 1}$ be the eigenvalues of $\Sigma_2$. For some constant $c_2 >0$ and parameter $p_2 \in (0,1]$ and for all $i \geq 1$,
    \begin{equation*}
        \lambda_{2,i} \leq c_2i^{-1/p_2}.
    \end{equation*}   
    \item Let $(\lambda_{3,i})_{i \geq 1}$ be the eigenvalues of $\Sigma_3$. For some constant $c_3 >0$ and parameter $p_3 \in (0,1]$ and for all $i \geq 1$,
    \begin{equation*}
        \lambda_{3,i} \leq c_3i^{-1/p_3}.
    \end{equation*}    
    \end{enumerate}
\end{assumption}

\subsubsection{Proof sketch for ATE}

We provide non-asymptotic uniform consistency guarantees for the dose-response curve. Below we provide a proof sketch. 

\paragraph{First Stage regression.} The estimator from stage 1 aims at estimating the conditional mean embedding $\mu_{Z|W,A} = \E[\phi_\gZ(Z) \mid W, A]$. We recall that under the well-specifiedness assumption (Assumption~(\ref{asst:well_specifiedness}-1)), we have $\mu_{Z|W,A}(\cdot, \cdot) = C_{Z|W, A}(\phi_\gW(\cdot) \otimes \phi_\gA(\cdot))$, with $C_{Z|W, A} \in S_2(\gH_{\gW\gA}, \gH_\gZ)$. We point out that Assumption~(\ref{asst:well_specifiedness}-1) is equivalent to the assumption that $\mu_{Z|W,A}(\cdot, \cdot)$ belong to the vector-valued RKHS associated to the vector-valued kernel $K((w,a),(w',a')) = \langle \phi_{\gW\gA}(w,a), \phi_{\gW\gA}(w',a') \rangle_{\gH_{\gW\gA}} \operatorname{Id}_{\gH_{\gZ}},$ where $\operatorname{Id}_{\gH_{\gZ}}$ denotes the identity operator in $\gH_{\gZ}$ (see \cite{li2022optimal} for a detailed discussion).

Given the regularization parameter $\lambda_1 > 0$, we recall that the objective to learn the conditional mean embedding operator is,
\begin{align*}
    \hat{\gL}^{c}(C) = \frac{1}{n} \sum_{i = 1}^n \|\phi_\gZ(z_i) - C(\phi_\gW(w_i) \otimes \phi_\gA(a_i))\|_{\gH_\gZ}^2 + \lambda_1 \|C\|_{S_2(\gH_{\gW\gA}, \gH_\gZ)}^2, \qquad C \in S_2(\gH_{\gW\gA}, \gH_\gZ),
\end{align*}
whose minimizer is denoted as,
\begin{align*}
    \hat{C}_{Z|W, A} = \argmin_{C \in S_2(\gH_{\gW\gA}, \gH_\gZ)} \hat{\gL}^{c}(C).
\end{align*}
The conditional mean embedding is then approximated as
\begin{align*}
    \hat{\mu}_{Z|W,A}(w, a) = \hat{C}_{Z|W, A}(\phi_\gW(w) \otimes \phi_\gA(a)), \qquad w \in \gW, \quad a \in \gA.
\end{align*}
We bound $\|\hat{C}_{Z|W, A} - C_{Z|W, A}\|_{S_2}$ in S.M. (Sec. ~\ref{sec:first_stage_ate}), using the main result from \citep{li2022optimal}. We then convert this bound to a bound on $\|\hat{\mu}_{Z|W, A} - \mu_{Z|W, A}\|_{\infty}$, that will allow us to obtain the uniform consistency of the dose-response function.

\paragraph{Second Stage regression.}
 We recall that for the second stage we have the following loss at the population level:
$$
    \gL^{2SR}(\varphi)=\E\big[ (r(W,A) - \E[\varphi(Z,A) \mid W,A])^2 ], \qquad \varphi \in \gH_{\gZ \gA}.
$$
We showed in Eq.~(\ref{eq:MainTextSimplifiedLossATEPopulation}) that $\gL^{2SR}$ can be equivalently written as 
$$
    \gL^{2SR}(\varphi) = \E\big[\E[\varphi(Z,A) \mid W,A]^2 \big] - 2 \E_{W} \E_{A} \big[\E[\varphi(Z,A) \mid W,A]\big] + \text{const.}
$$
We introduce the regularized version of the population loss, for $\lambda_2 > 0$,
$$
 \gL^{2SR}_{\lambda_2}(\varphi) =  \gL^{2SR}(\varphi) + \lambda_2 \|\varphi\|^2_{\gH_{\gZ\gA}}.
$$
Let us introduce $g_{2} = \E_{W} \E_{A} \left[\mu_{Z \mid W,A} (W, A) \otimes \phi_{\gA}(A)\right]$.
 
\begin{proposition} \label{prop:pop_estim_stage_2}
    $g_2$ can be alternatively written as 
    $$
    g_{2} = \E[r(W,A)\mu_{Z \mid W,A} (W, A) \otimes \phi_{\gA}(A)].
    $$
    Furthermore, for any element $\varphi_0 \in \gH_{\gZ \gA}$ solution to the bridge equation, we have $g_2 = \Sigma_2 \varphi_0.$ Finally,
    $$
    \varphi_{\lambda_2} := \argmin_{\varphi \in \gH_{\gZ \gA}} \gL^{2SR}_{\lambda_2}(\varphi) = \left(\Sigma_2 + \lambda_2 \text{Id}_{\gH_{\gZ\gA}} \right)^{-1} g_2.
    $$
\end{proposition}
\begin{proof}
The first part follows from the same derivations as in  Eq.~(\ref{eq:MainTextSimplifiedLossATEPopulation}). For the second part, if $\varphi_0$ is such that $r(W,A) = \langle \varphi_0, \mu_{Z|W, A} (W, A)\otimes \phi_\gA(A) \rangle_{\gH_{\gZ\gA}} $, then,
$$
g_2 = \E_{W,A}[r(W, A) \mu_{Z|W, A} (W, A)\otimes \phi_\gA(A)] =  \Sigma_2 \varphi_0.
$$
For the final part, notice that for $\varphi \in \gH_{\gZ \gA}$, by the reproducing property,
$$
\begin{aligned}
 \gL^{2SR}_{\lambda_2}(\varphi) &= \E_{W, A}\big[ \langle {\varphi}, \mu_{Z|W, A} (W, A)\otimes \phi_\gA(A) \rangle_{\gH_{\gZ\gA}}^2 \big] - 2 \E_{W} \E_{A} \big[ \langle \varphi, \mu_{Z|W, A} (W, A)\otimes \phi_\gA(A) \rangle_{\gH_{\gZ\gA}}\big] + \lambda_2 \|\varphi\|^2_{\gH_{\gZ\gA}} \\  &+ \text{const} \\ &=   \langle \varphi, \Sigma_2 \varphi \rangle_{\gH_{\gZ\gA}} - 2 \langle \varphi, g_2 \rangle_{\gH_{\gZ\gA}} + \lambda_2 \|\varphi\|^2_{\gH_{\gZ\gA}}  + \text{const} \\ &= \langle \varphi, (\Sigma_2 + \lambda_2 \operatorname{Id}) \varphi \rangle_{\gH_{\gZ\gA}} - 2 \langle \varphi, g_2 \rangle_{\gH_{\gZ\gA}}  + \text{const.}
 \end{aligned}
$$
We conclude by setting the Fréchet derivative to $0$.
\end{proof}

We would now like to consider the empirical version of the previous loss. The standard kernel ridge regression estimator would obtain an estimator by replacing both the population covariance $\Sigma_2$ and the term $g_2$ by their empirical counterpart in Proposition~(\ref{prop:pop_estim_stage_2}). However, as $g_2$ takes a different form, we need to specify how we build its empirical counterpart.

\begin{itemize}
    \item Option 1: take the empirical counterpart of $ g_2 = \Sigma_2 \varphi_0$; this is not feasible as it would require the knowledge of a bridge function.
    \item Option 2: take the empirical counterpart of $g_2 =  \E[r(W,A)\mu_{Z \mid W,A} (W, A) \otimes \phi_{\gA}(A)]$; this would require the estimation of the density ratio $r(W,A)$ and would be inefficient in high dimension.
    \item Option 3 (\textbf{ours}): take the empirical counterpart of $g_2 = \E_{W} \E_{A} \left[\mu_{Z \mid W,A} \otimes \phi_{\gA}(A)\right]$; this is the estimator suggested in Section~(\ref{sec:main_algo}) that allows us to by-pass density ratio estimation.
\end{itemize}

We therefore introduce 

\begin{align*}
    &\bar{\Sigma}_{2,m} = \frac{1}{m} \sum_{i = 1}^m \Big( \mu_{Z|W, A}(\Tilde{w}_i, \Tilde{a}_i) \otimes \phi_\gA(\Tilde{a}_i)\Big) \otimes \Big( \mu_{Z|W, A}(\Tilde{w}_i, \Tilde{a}_i) \otimes \phi_\gA(\Tilde{a}_i)\Big),\\
    &\bar{g}_{2,m} = \frac{1}{m(m-1)}\sum_{\substack{i,j \\ i\ne j}}^m  \mu_{Z| W, A}(\Tilde{w}_j, \Tilde{a}_i) \otimes \phi_\gA(\Tilde{a}_i).
\end{align*}
However, as we do not directly observe the conditional mean embedding $\mu_{Z|W, A}$, we plug-in its approximation obtained in the first stage regression. This leads us to, 
\begin{align*}
    &\hat{\Sigma}_{2,m} = \frac{1}{m} \sum_{i = 1}^m \Big( \hat{\mu}_{Z|W, A}(\Tilde{w}_i, \Tilde{a}_i) \otimes \phi_\gA(\Tilde{a}_i)\Big) \otimes \Big( \hat{\mu}_{Z|W, A}(\Tilde{w}_i, \Tilde{a}_i) \otimes \phi_\gA(\Tilde{a}_i)\Big),\\
    &\hat{g}_{2, m} = \frac{1}{m(m-1)}\sum_{\substack{i,j \\ i\ne j}}^m  \hat{\mu}_{Z|W, A}(\Tilde{w}_j, \Tilde{a}_i) \otimes \phi_\gA(\Tilde{a}_i).
\end{align*}
Let us then introduce the following empirical losses, 
$$
    \bar{\gL}^{2SR}_m(\varphi) = \frac{1}{m} \sum_{i = 1}^m \langle \varphi, \mu_{Z|W, A} (\Tilde{w}_i, \Tilde{a}_i) \otimes \phi_\gA(\Tilde{a}_i) \rangle_{\gH_{\gZ\gA}}^2 - \frac{2}{m(m-1)} \sum_{\substack{i,j = 1 \\ j \ne i}}^m \langle \varphi, \mu_{Z|W, A} (\Tilde{w}_j, \Tilde{a}_i) \otimes \phi_\gA(\Tilde{a}_i) \rangle_{\gH_{\gZ\gA}} + \lambda_2 \|\varphi\|^2_{\gH_{\gZ\gA}},
$$
$$
    \hat{\gL}^{2SR}_{m}(\varphi) = \frac{1}{m} \sum_{i = 1}^m \langle \varphi, \hat{\mu}_{Z|W, A} (\Tilde{w}_i, , \Tilde{a}_i) \otimes \phi_\gA(\Tilde{a}_i) \rangle_{\gH_{\gZ\gA}}^2 - \frac{2}{m(m-1)} \sum_{\substack{i,j = 1 \\ j \ne i}}^m \langle \varphi, \hat{\mu}_{Z|W, A} (\Tilde{w}_j, \Tilde{a}_i) \otimes \phi_\gA(\Tilde{a}_i) \rangle_{\gH_{\gZ\gA}} + \lambda_2 \|\varphi_0\|^2_{\gH_{\gZ\gA}}.
$$
We can observe that the minimizers of the objective functions are given by
\begin{align*}
    \varphi_{\lambda_2} &= (\Sigma_2 + \lambda_2 I)^{-1} g_2 = \argmin_{\varphi \in \gH_{\gZ\gA}} {\gL}^{2SR}_{\lambda_2}(\varphi),\\
    \bar{\varphi}_{\lambda_2,m} &= (\bar{\Sigma}_{2,m} + \lambda_2 I)^{-1} \bar{g}_{2,m} = \argmin_{\varphi \in \gH_{\gZ\gA}} \bar{\gL}^{2SR}_m(\varphi),\\
    \hat{\varphi}_{\lambda_2,m} &= (\hat{\Sigma}_{2,m} + \lambda_2 I)^{-1} \hat{g}_{2,m} = \argmin_{\varphi \in \gH_{\gZ\gA}} \hat{\gL}_{m}^{2SR}(\varphi).\\
\end{align*}
$\hat{\varphi}_{\lambda_2,m}$ is the final estimator presented in Section~(\ref{sec:main_algo}). In S.M. (Sec. ~\ref{sec:second_stage_ate}), we show the convergence of $\hat{\varphi}_{\lambda_2,m}$ to the minimum norm bridge $\bar{\varphi}_0$ introduced in Definition~(\ref{def:min_norm_bridge}). $\bar{\varphi}_{\lambda_2,m}$ and $\varphi_{\lambda_2}$ are introduced for theoretical reasons. Indeed, we will consider the following decomposition, 
\begin{align*}
    \|\hat{\varphi}_{\lambda_2, m} - \bar{\varphi}_0\|_{\gH_{\gZ\gA}} \le \|\hat{\varphi}_{\lambda_2, m} - \bar{\varphi}_{\lambda_2,m}\|_{\gH_{\gZ\gA}} + \|\bar{\varphi}_{\lambda_2,m} -  \varphi_{\lambda_2}\|_{\gH_{\gZ\gA}}+  \| \varphi_{\lambda_2} - \bar{\varphi}_0\|_{\gH_{\gZ\gA}}.
\end{align*}

\paragraph{Third Stage regression.} The estimator from stage 3 aims at estimating $\Psi(A) := \E[Y \phi_\gZ(Z) | A]$. We recall that under the well-specifiedness assumption (Assumption~(\ref{asst:well_specifiedness}-3)), we have $\Psi(\cdot) = C_{YZ|A}\phi_\gA(\cdot)$, with $C_{YZ|A} \in S_2(\gH_{\gA}, \gH_\gZ)$.

Given the regularization parameter $\lambda_3 > 0$, we recall that the objective to learn $C_{YZ|A}$ is,
\begin{align*}
    \hat{\gL}_3(C) = \frac{1}{t} \sum_{i = 1}^t \|\bar{y}_i\phi_\gZ(\bar{z}_i) - C\phi_\gA(\bar{a}_i))\|_{\gH_\gZ}^2 + \lambda_3 \|C\|_{S_2(\gH_{\gA}, \gH_\gZ)}^2, \qquad C \in S_2(\gH_{\gA}, \gH_\gZ),
\end{align*}
where $\{(\bar{a}_i, \bar{z}_i, \bar{y}_i)\}_{i=1}^t$ can re-use data from stage 1 or 2 or both. The minimizer is denoted as,
\begin{align*}
    \hat{C}_{YZ|A} = \argmin_{C \in S_2(\gH_{\gA}, \gH_\gZ)} \hat{\gL}_{3}(C),
\end{align*}
leading to 
\begin{align*}
    \hat{\Psi}(a) = \hat{C}_{YZ|A}\phi_\gA(a), \qquad a \in \gA.
\end{align*}

We note that $\Psi(A)$ can be interpreted as the conditional kernel mean embedding of the random variable $(Y,Z)$ given $A$ with the following kernel: $k_{YZ}((y,z),(y',z')) = yy' k_{\gZ}(z,z')$, $(y,z), (y',z') \in \R \times \gZ$. Indeed, the canonical feature map of $k_{YZ}$ is $y\phi_{\gZ}(z)$ for $(y,z) \in \R \times \gZ$. We could therefore proceed as for stage 1 and apply results from \citep{li2022optimal}. However, the analysis of \citep{li2022optimal} would require $k_{YZ}$ to be bounded on the support of $(Y,Z)$ and therefore requires $Y$ to be almost surely bounded. Instead, in Section~(\ref{sec:third_stage_ate}), we apply results from \citep{li2024towards} which generalize consistency guarantees for conditional mean operator learning to general vector-valued regression. The results applies with the weaker assumption that $Y$ is sub-exponential (Assumption~(\ref{assum:Stage1ConsistencyKernelAssumptions}), Equation ~(\ref{eq:mom_stage_3})).

\paragraph{Uniform consistency guarantees for ATE.}

We recall that after obtaining the estimators $\hat{\varphi}_{\lambda_2,m}$ from stage 2 and $\hat \Psi$ from stage 3, we have
$$
\hat{f}_{ATE}(\cdot) = \langle \hat{\varphi}_{\lambda_2,m}, \hat{\Psi}(\cdot) \otimes \phi_\gA(\cdot)\rangle_{\gH_{\gZ\gA}}.
$$
On the other hand, under Assumption~(\ref{asst:well_specifiedness}),
$$
f_{ATE}(\cdot) = \langle \bar{\varphi}_{0}, \Psi(\cdot) \otimes \phi_\gA(\cdot)\rangle_{\gH_{\gZ\gA}}.
$$
For any $a \in \gA$, we apply the following decomposition,
    \begin{align}
        |&\hat{f}_{ATE}(a) - f_{ATE}(a)| = |\langle \hat{\varphi}_{\lambda_2,m}, \hat{\Psi}(a) \otimes \phi_\gA(a)\rangle_{\gH_{\gZ\gA}} - \langle \bar{\varphi}_{0}, \Psi(a) \otimes \phi_\gA(a) \rangle_{\gH_{\gZ\gA}}| \nonumber\\
        &= |\langle \hat{\varphi}_{\lambda_2,m}, (\hat{\Psi} - \Psi)(a) \otimes \phi_\gA(a)\rangle_{\gH_{\gZ\gA}} + \langle (\hat{\varphi}_{\lambda_2,m} - \bar{\varphi}_{0}), \Psi(a) \otimes \phi_\gA(a) \rangle_{\gH_{\gZ\gA}}| \nonumber\\
        &= |\langle \hat{\varphi}_{\lambda_2,m} - \bar{\varphi}_{0}, (\hat{\Psi} - \Psi)(a) \otimes \phi_\gA(a)\rangle_{\gH_{\gZ\gA}} + \langle \bar{\varphi}_{0}, (\hat{\Psi} - \Psi)(a) \otimes \phi_\gA(a) \rangle_{\gH_{\gZ\gA}} + \langle (\hat{\varphi}_{\lambda_2,m} - \bar{\varphi}_{0}), \Psi(a) \otimes \phi_\gA(a) \rangle_{\gH_{\gZ\gA}}| \nonumber\\
         &\le \kappa\left(\|\hat{\varphi}_{\lambda_2,m} - \bar{\varphi}_{0}\|_{\gH_{\gZ\gA}} \|\hat{\Psi}(a) - \Psi(a)\|_{\gH_{\gZ}} + \|\bar{\varphi}_{0}\| \| \hat{\Psi}(a) - \Psi(a) \|_{\gH_{\gZ\gA}} + \| \hat{\varphi}_{\lambda_2,m} - \bar{\varphi}_{0} \|_{\gH_{\gZ\gA}} \|\Psi(a) \|_{\gH_{\gZ}}\right) \label{eq:ate_decomposition}
    \end{align}

By plugging the consistency results from stage 2 and 3, we obtain the final bound that leads to Theorem~(\ref{th:final_rate_ate_main}). See S.M. (Sec. ~\ref{sec:final_stage_ate}) for details. In the next sections, we detail each step of the proof. 

\subsubsection{First-Stage Regression Consistency Result}  \label{sec:first_stage_ate}

We adapt \citet[][Theorem~2]{li2022optimal} to our setting.

\begin{theorem}[Theorem~2 \cite{li2022optimal}] \label{th:CME_rate}
    Suppose Assumptions~(\ref{asst:well_specifiedness}-1), (\ref{assum:Stage1ConsistencyKernelAssumptions}), 
    (\ref{asst:src_all_stages}-1) and (\ref{asst:evd_all_stages}-1) hold and take $\lambda_{1} = \Theta \left(n^{-\frac{1}{\beta_1 + p_1}}\right)$. There is a constant $J_1 > 0$ independent of $n \geq 1$ and $\delta \in (0,1)$ such that \[\left\|\hat{C}_{Z | W, A} - {C}_{Z | W, A}\right\|_{S_2(\gH_{\gW\gA}, \gH_\gZ)} \leq  J_1 \log(4/\delta) \left(\frac{1}{\sqrt{n}}\right)^{\frac{\beta_1-1}{\beta_1 + p_1}} =: r_1(\delta, n, \beta_1, p_1),\] is satisfied for sufficiently large $n \geq 1$ with probability at least $1-\delta$.
\end{theorem}
\begin{proof}
    We apply \citet[][Theorem~2-case 2.]{li2022optimal}, with $\gamma = 1$ which corresponds to the Hilbert-Schmidt norm $S_2(\gH_{\gW\gA}, \gH_\gZ)$. As we focus on the well-specified setting with $\beta_1 \geq 1,$ we can apply case 2. of \citet[][Theorem~2]{li2022optimal} as in their paper $\alpha \leq 1$, hence $\beta_1 + p_1 \geq \alpha$. Note that \citet[][Theorem~2]{li2022optimal} applies under the assumption that $k_{\gZ}$ is bounded, which is the case under Assumption~\ref{assum:Stage1ConsistencyKernelAssumptions}. We note that in \citet[][Theorem~2]{li2022optimal}, the bound is valid for $\beta_1 \in (1,2]$ while we allow for $\beta_1 \in (1,3]$ (see Remark~\ref{ref:higher_saturation} below).      
\end{proof}


\begin{corollary} \label{cor:CME_rate_pointwise}
Under the same assumptions as Theorem~\ref{th:CME_rate}, with $\lambda_{1} = \Theta \left(n^{-\frac{1}{\beta_1 + p_1}}\right)$, for any $\delta \in (0, 1)$, the following holds with probability at least $1 - \delta$:
    \begin{align*}
        \sup_{(w, a) \in \gW \times \gA}\|\hat{\mu}_{Z|W, A}(w, a) - {\mu}_{Z|W, A}(w, a)\|_{\gH_\gZ} \le \kappa^2 r_1(\delta, n, \beta_1, p_1).
    \end{align*}
\end{corollary}

\begin{proof}
    For any $(w, a) \in \gW \times \gA$, under Assumption~(\ref{assum:Stage1ConsistencyKernelAssumptions}), we have 
    \begin{align*}
        \|\phi_\gW(w) \otimes \phi_\gA(a)\|_{\gH_{\gW\gA}} = \|\phi_\gW(w)\|_{\gH_\gW} \|\phi_\gA(a)\|_{\gH_\gA} \le \kappa^2.
    \end{align*}
    As a result, we observe that,
    \begin{align*}
        \|\hat{\mu}_{Z|W, A}(w, a) - {\mu}_{Z|W, A}(w, a)\|_{\gH_\gZ} &= \|(\hat{C}_{Z | W, A} - {C}_{Z | W, A})\phi_\gW(w) \otimes \phi_\gA(a)\|_{\gH_\gZ}\\
        & \le \|\hat{C}_{Z | W, A} - {C}_{Z | W, A}\|_{S_2(\gH_{\gW\gA}, \gH_\gZ)} \|\phi_\gW(w) \otimes \phi_\gA(a)\|_{\gH_{\gW\gA}}\\
        &\le \kappa^2 \|\hat{C}_{Z | W, A} - {C}_{Z | W, A}\|_{S_2(\gH_{\gW\gA}, \gH_\gZ)},\\
    \end{align*}
    and the conclusion follows from Theorem~(\ref{th:CME_rate}).
\end{proof}

\subsubsection{Second-Stage Regression Consistency Results} \label{sec:second_stage_ate}
We recall that we consider the following decomposition,
\begin{align*}
    \|\hat{\varphi}_{\lambda_2, m} - \bar{\varphi}_0\|_{\gH_{\gZ\gA}} \le \|\hat{\varphi}_{\lambda_2, m} - \bar{\varphi}_{\lambda_2,m}\|_{\gH_{\gZ\gA}} + \|\bar{\varphi}_{\lambda_2,m} -  \varphi_{\lambda_2}\|_{\gH_{\gZ\gA}}+  \| \varphi_{\lambda_2} - \bar{\varphi}_0\|_{\gH_{\gZ\gA}}.
\end{align*}
We first consider an upper bound for $\|\varphi_{\lambda_2} - \bar{\varphi}_0\|_{\gH_{\gZ\gA}}$. 
\begin{lemma}
    Suppose that Assumption~(\ref{asst:src_all_stages}-2.) holds with parameter $\beta_2 \in (1,3]$. Then, for any $\lambda_2 > 0$,
    \begin{align*}
        \|\varphi_{\lambda_2} - \bar{\varphi}_0\|_{\gH_{\gZ\gA}} \le B_2\lambda_2^{\frac{\beta_2-1}{2}}.
    \end{align*}
    \label{lemma:SecondStageBiasBound}
\end{lemma}

\begin{proof}
    We saw in Proposition~(\ref{prop:pop_estim_stage_2}) that
    $$
         \varphi_{\lambda_2} = \left(\Sigma_2 + \lambda_2 \operatorname{Id} \right)^{-1}g_2 = \left(\Sigma_2 + \lambda_2 \operatorname{Id} \right)^{-1}\Sigma_2\bar{\varphi}_0 = \bar{\varphi}_0 - \lambda_2 \left(\Sigma_2 + \lambda_2 \operatorname{Id} \right)^{-1} \bar{\varphi}_0.
    $$
    Therefore, under Assumption~(\ref{asst:src_all_stages}-2.), 
    $$
    \left\|\varphi_{\lambda_2} - \bar{\varphi}_0\right\|_{\gH_{\gZ\gA}} = \lambda_2  \left\|\left(\Sigma_2 + \lambda_2\operatorname{Id} \right)^{-1} \bar{\varphi}_0\right\|_{\gH_{\gZ\gA}} \leq B_2 \lambda_2  \left\|\left(\Sigma_2 + \lambda_2\operatorname{Id} \right)^{-1}\Sigma_2^{\frac{\beta_2-1}{2}} \right\|_{op},
    $$
    where we used Proposition~(\ref{prop:min_norm_charac2}) and $\|.\|_{op}$ denotes the \emph{operator norm}.. Note that, by Lemma~(\ref{lma:steinwart_lemma}),
    $$
    \left\|\left(\Sigma_2 + \lambda_2\operatorname{Id} \right)^{-1}\Sigma_2^{\frac{\beta_2-1}{2}} \right\|_{op} = \sup_{i \geq 1} \frac{\lambda_{2,i}^{\frac{\beta_2-1}{2}}}{\lambda_{2,i} + \lambda_2} \leq \lambda_2^{\frac{\beta_2-1}{2} - 1},
    $$
    as long as $\frac{\beta_2-1}{2} \in (0,1]$, i.e. $\beta_2 \in (1,3]$. By merging the bounds, we obtain the final result.  
\end{proof}

\begin{remark} \label{ref:higher_saturation}
    The commonly known saturation effect of Tikhonov regularization comes for the approximation error bound. As demonstrated above, the range of smoothness is limited to $\beta_2 \leq 3$. However, past works on kernel ridge regression (e.g. \cite{fischer2020sobolev}) or kernel PCL (e.g. \citet{Mastouri2021ProximalCL, singh2023kernelmethodsunobservedconfounding}) observed a saturation effect at $\beta_2 = 2$. It was observed in \citet[][Remark 7 \& Proposition 7]{meunier2023nonlinear} -- see also \cite{blanchard2018optimal} -- that saturation happens at $\beta_2=2$ when we measure the error in the $L_2-$norm while saturation happens at $\beta_2=3$ when we measure the error in the RKHS norm, as seen in the previous proof. As the error is measured in RKHS norm in both works \citet{Mastouri2021ProximalCL, singh2023kernelmethodsunobservedconfounding}, they can apply the same reasoning to extend their results from the range $\beta_2 \in (1,2]$ to the range $\beta_2 \in (1,3]$.
\end{remark}

We will need the following result to pursue our proof. 

\begin{lemma}
For any $\lambda_2 > 0$, $\|\varphi_{\lambda_2}\|_{\gH_{\gZ\gA}} \le \|\bar{\varphi}_0\|_{\gH_{\gZ\gA}}$.
    \label{lemma:varphi0lessthanvarphi}
\end{lemma}

\begin{proof}
We saw in Proposition~\ref{prop:pop_estim_stage_2} that
    $$
         \left\|\varphi_{\lambda_2}\right\|_{\gH_{\gZ\gA}}  = \left\|\left(\Sigma_2 + \lambda_2 \operatorname{Id} \right)^{-1}\Sigma_2\bar{\varphi}_0\right\|_{\gH_{\gZ\gA}} \leq  \left\|\left(\Sigma_2 + \lambda_2 \operatorname{Id} \right)^{-1}\Sigma_2\right\|_{op}\left\|\bar{\varphi}_0\right\|_{\gH_{\gZ\gA}}\leq \left\|\bar{\varphi}_0\right\|_{\gH_{\gZ\gA}}.
    $$

\end{proof}

Our proof of the convergence result relies on an Hoeffding concentration inequality (Corollary~\ref{corr:hoeffding}) and a Bernstein concentration inequality (Theorem~\ref{theo:bernstein}) for Hilbert space-valued random variables. We introduce the effective dimension for the stage 2 error: for $\lambda_2 > 0$, $\gN(\lambda_2) := \operatorname{Tr}((\Sigma_2 + \lambda_2 \operatorname{Id})^{-1}\Sigma_2)$ \cite{caponnetto2007optimal}.

\begin{proposition}[Lemma 11 \& Lemma 13 \cite{fischer2020sobolev}]  \label{prop:eff_dim_properties}
Under Assumption~(\ref{asst:evd_all_stages}-2), there is a constant $D>0$ such that the following inequality is satisfied, for $\lambda_2 > 0$, $\gN(\lambda_2) \leq D\lambda_2^{-p_2}$. Furthermore, we have the equality,
$$
\E\left[\left\|(\Sigma_2 + \lambda_2 \operatorname{Id})^{-1/2} \mu_{Z|W, A} (W, A) \otimes \phi_\gA(A)\right\|^2_{\gH_{\gZ\gA}}\right] = \gN(\lambda_2).
$$
\end{proposition}

\begin{lemma}
Let us introduce $g_{\lambda_2} = \log \left( 2e\mathcal{N}(\lambda_2) \frac{\|\Sigma_2\|_{op}+\lambda_2}{\|\Sigma_2\|_{op}}\right).$ Suppose Assumption (\ref{assum:Stage1ConsistencyKernelAssumptions}) holds. Then, with probability at least $1-\delta$ for all $\delta \in (0,1)$, for $m \geq 8\kappa^4\log(2/\delta) g_{\lambda_2} \lambda_2^{-1}$,
\begin{align*}
    \|\bar{\varphi}_{\lambda_2,m} -  \varphi_{\lambda_2}\|_{\gH_{\gZ\gA}} \le \frac{3}{\sqrt{\lambda_2}} \Bigg( \log(2/\delta)\sqrt{\frac{32}{m}\left(\gN(\lambda_2)\kappa^4 \|\bar{\varphi}_0\|_{\gH_{\gZ\gA}}^2+\frac{\kappa^8 \|\bar{\varphi}_0\|_{\gH_{\gZ\gA}}^2}{m\lambda_2}\right)} + \frac{\|g_2 - \bar{g}_{2,m}\|_{\gH_{\gZ\gA}}}{\sqrt{\lambda_2}}\Bigg).
\end{align*}
    \label{theorem:SecondStageBound1}
\end{lemma}

\begin{proof}
    We decompose the error as 
    \begin{align*}
        \|\bar{\varphi}_{\lambda_2,m} -  \varphi_{\lambda_2}\|_{\gH_{\gZ\gA}} &= \big\|(\bar{\Sigma}_{2,m} + \lambda_2 I)^{-1} \big(\bar{g}_{2,m}-  (\bar{\Sigma}_{2,m} + \lambda_2 I) \varphi_{\lambda_2}\big) \big\|_{\gH_{\gZ\gA}}\\ &\leq \left\|(\Sigma_2 + \lambda_2 I)^{-1/2} \right\|_{op}\left\|(\Sigma_2 + \lambda_2 I)^{1/2}(\bar{\Sigma}_{2,m} + \lambda_2 I)^{-1}(\Sigma_2 + \lambda_2 I)^{1/2} \right\|_{op}\\&\times \big\|(\Sigma_2 + \lambda_2 I)^{-1/2} \big( \bar{g}_{2,m}  - \bar{\Sigma}_{2,m} \varphi_{\lambda_2} - \underbrace{\lambda_2 \varphi_{\lambda_2}}_{= g_2 - \Sigma_2 \varphi_{\lambda_2}} \big) \big\|_{\gH_{\gZ\gA}}\\
        &\le \lambda_2^{-1/2} \left\|(\Sigma_2 + \lambda_2 I)^{1/2}(\bar{\Sigma}_{2,m} + \lambda_2 I)^{-1}(\Sigma_2 + \lambda_2 I)^{1/2} \right\|_{op} \\&\times \Big(\|(\Sigma_2 + \lambda_2 I)^{-1/2}(\bar{\Sigma}_{2,m} - \Sigma_2) \varphi_{\lambda_2}\|_{\gH_{\gZ\gA}} + \lambda_2^{-1/2}\|\bar{g}_{2,m} - g_2\|_{\gH_{\gZ\gA}} \Big),
    \end{align*}
     The first term is bounded by Lemma~(\ref{lemma:steinwart_cov_conc}),
    $$
    \left\|(\Sigma_2 + \lambda_2 I)^{1/2}(\bar{\Sigma}_{2,m} + \lambda_2 I)^{-1}(\Sigma_2 + \lambda_2 I)^{1/2} \right\|_{op} \leq 3,
    $$
    for $m \geq 8\kappa^4\log(2/\delta) g_{\lambda_2} \lambda_2^{-1}$ with probability at least $1-\delta$ for all $\delta \in (0,1)$. 
To bound the remaining term, we wish to apply Theorem~(\ref{theo:bernstein}) with $\gH = \gH_{\gZ\gA}$. Consider the measurable map $\xi: \gW \times \gA \rightarrow \gH_{\gZ\gA}$ defined by
\begin{equation*} \label{eq:xi2}  
\xi(w, a) := (\Sigma_2 + \lambda_2 I)^{-1/2}\langle \varphi_{\lambda_2}, \mu_{Z|W,A}(w, a) \otimes \phi_\gA(a) \rangle_{\gH_{\gZ\gA}} \mu_{Z|W,A}(w, a) \otimes \phi_\gA(a),
\end{equation*}
inducing random variables such that
$$
\frac{1}{m} \sum_{i=1}^m \left(\xi(\Tilde{w}_i, \Tilde{a}_i) - \E[\xi(W, A)]\right) = (\Sigma_2 + \lambda_2 I)^{-1/2}(\bar{\Sigma}_{2,m} - \Sigma_2) \varphi_{\lambda_2}.
$$
By Assumption (\ref{assum:Stage1ConsistencyKernelAssumptions}), Lemma (\ref{lemma:varphi0lessthanvarphi}) and Cauchy-Schwarz inequality, 
$$
|\langle \varphi_{\lambda_2}, \mu_{Z|W,A}(w, a) \otimes \phi_\gA(a) \rangle_{\gH_{\gZ\gA}}| \leq \kappa^2 \|\bar{\varphi}_0\|_{\gH_{\gZ\gA}}. 
$$
We can now bound the $q$-th moment of $\xi$, for $q \geq 2$,
$$
\begin{aligned}
    \mathbb{E}\left\|\xi(W,A)\right\|_{\gH_{\gZ\gA}}^{q} & \leq \left(\kappa^2 \|\bar{\varphi}_0\|_{\gH_{\gZ\gA}}\right)^q \mathbb{E}\left\|(\Sigma_2 + \lambda_2 I)^{-1/2} \mu_{Z|W,A}(W, A) \otimes \phi_\gA(A)\right\|_{\gH_{\gZ\gA}}^{q} \\ &\leq \left(\kappa^2 \|\bar{\varphi}_0\|_{\gH_{\gZ\gA}}\right)^q\left(\frac{\kappa^2}{\sqrt{\lambda_2}} \right)^{q-2} \mathbb{E}\left\|(\Sigma_2 + \lambda_2 I)^{-1/2} \mu_{Z|W,A}(W, A) \otimes \phi_\gA(A)\right\|_{\gH_{\gZ\gA}}^{2} \\ &= \left(\kappa^2 \|\bar{\varphi}_0\|_{\gH_{\gZ\gA}}\right)^q\left(\frac{\kappa^2}{\sqrt{\lambda_2}} \right)^{q-2}\gN(\lambda_2) \\ &\leq \frac{1}{2}q! \left(\frac{\kappa^4}{\sqrt{\lambda_2}}\|\bar{\varphi}_0\|_{\gH_{\gZ\gA}} \right)^{q-2} \gN(\lambda_2)\kappa^4 \|\bar{\varphi}_0\|_{\gH_{\gZ\gA}}^2,
\end{aligned}
$$
where in the equality, we used Proposition~(\ref{prop:eff_dim_properties}). An application of Bernstein's inequality from Theorem~(\ref{theo:bernstein}) with 
$$
L=\frac{\kappa^4}{\sqrt{\lambda_2}}\|\bar{\varphi}_0\|_{\gH_{\gZ\gA}}, \qquad \sigma^{2} =  \gN(\lambda_2)\kappa^4 \|\bar{\varphi}_0\|_{\gH_{\gZ\gA}}^2,
$$ 
yields the final bound.
\end{proof}

We will now derive a bound for $\|g_2 - \bar{g}_{2,m}\|_{\gH_{\gZ\gA}}$.

\begin{lemma}
    With probability at least $1 - \delta$ for $\delta \in (0,1)$ the following bound holds:
    $$\|g_2 - \bar{g}_{2,m}\|_{\gH_{\gZ\gA}} \le 2\kappa^2\sqrt{\frac{2  \log(2/\delta)}{m (m-1)}}.$$
    \label{lemma:g_2Bound}
\end{lemma}
\begin{proof}
Observe that, by Proposition~(\ref{prop:pop_estim_stage_2}),
    \begin{align*}
        &\bar{g}_{2,m} = \frac{1}{m(m-1)}\sum_{\substack{i,j \\ i\ne j}}^m  \mu_{Z|W, A}(\Tilde{w}_j, \Tilde{a}_i) \otimes \phi_\gA(\Tilde{a}_i),\\
        &\E_{W}\E_A [\bar{g}_{2,m}] = \E_{W} \E_{A} \big[ \mu_{Z|W, A} (W, A) \otimes \phi_\gA(A)\big] = g_2.
    \end{align*}
    Let 
    \begin{align*}
        \xi(W, A) := \mu_{Z|W, A} (W, A) \otimes \phi_\gA(A).
    \end{align*}
    Then, note that
    \begin{align*}
        \|\xi(W,A)\|_{\gH_{\gZ\gA}} \le \kappa^2 \quad (\text{by Assumption (\ref{assum:Stage1ConsistencyKernelAssumptions})}).
    \end{align*}
    Now, we apply Corollary~(\ref{corr:hoeffding}) such that with probability at least $1 - \delta$,
    $$
        \|g_2 - \bar{g}_{2,m}\|_{\gH_{\gZ\gA}} \le  2\kappa^2\sqrt{\frac{2  \log(2/\delta)}{m (m-1)}}.
    $$
\end{proof}

\begin{theorem}
    Suppose Assumptions (\ref{assum:Stage1ConsistencyKernelAssumptions}), (\ref{asst:src_all_stages}-2.) and (\ref{asst:evd_all_stages}-2) hold. Then, with probability at least $1-2\delta$ for all $\delta \in (0,1/2)$, for $m \geq 8\kappa^4\log(2/\delta) g_{\lambda_2} \lambda_2^{-1}$,
    \begin{align*}
        \|\bar{\varphi}_{\lambda_2,m} - \bar{\varphi}_0\|_{\gH_{\gZ\gA}} \le J_2 \left(\frac{\log(2/\delta)}{\sqrt{\lambda_2}} \Bigg( \sqrt{\frac{1}{m}\left(\frac{1}{\lambda_2^{p_2}}+\frac{1}{m\lambda_2}\right)} + \sqrt{\frac{1}{m (m-1)\lambda_2}} \Bigg) + \lambda_2^{\frac{\beta_2-1}{2}} \right),
    \end{align*}
    where $J_2$ is a constant depending on $\kappa, \beta_2, B_2, D$.
    \label{theorem:SecondStageBound2}
\end{theorem}
\begin{proof}
    Combining the bounds in Lemma (\ref{lemma:SecondStageBiasBound}), Lemma (\ref{theorem:SecondStageBound1}) and Lemma (\ref{lemma:g_2Bound}) with a union bound, we obtain that with probability at least $1-2\delta$,
    $$
     \|\bar{\varphi}_{\lambda_2,m} - \bar{\varphi}_0\|_{\gH_{\gZ\gA}} \leq \mathring{J} \left(\frac{\log(2/\delta)}{\sqrt{\lambda_2}}\left(\sqrt{\frac{1}{m}\left(\gN(\lambda_2) \|\bar{\varphi}_0\|_{\gH_{\gZ\gA}}^2+\frac{\|\bar{\varphi}_0\|_{\gH_{\gZ\gA}}^2}{m\lambda_2}\right)}  +\sqrt{\frac{1}{m (m-1)\lambda_2}}\right)  + \lambda_2^{\frac{\beta_2-1}{2}}\right),
    $$
    where $\mathring{J}$ is a constant depending on $\kappa, B_2.$ Under Assumption~(\ref{asst:evd_all_stages}-2), using Proposition~(\ref{prop:eff_dim_properties}), there is a constant $D > 0$ such that $\gN(\lambda_2) \leq D\lambda_2^{-p_2}$. Furthermore, under Assumption~(\ref{asst:src_all_stages}-2.),
    $$
    \|\bar{\varphi}_0\|_{\gH_{\gZ\gA}} = \|\Sigma_2^{\frac{\beta_2-1}{2}}\Sigma_2^{-\frac{\beta_2-1}{2}}\bar{\varphi}_0\|_{\gH_{\gZ\gA}} \leq \kappa^{\frac{\beta_2-1}{2}}B_2.
    $$
\end{proof}

Next, we will derive a bound for $\|\hat{\varphi}_{\lambda_2, m} - \bar{\varphi}_{\lambda_2,m}\|$.

\begin{lemma}
Under the same assumptions as Theorem~(\ref{th:CME_rate}), with probability at least $1 - \delta$ for $\delta \in (0, 1)$, the following bound holds
    \begin{align*}
        \|\hat{\varphi}_{\lambda_2, m} - \bar{\varphi}_{\lambda_2,m}\| \le  \frac{1}{\lambda_2} \kappa^3 r_1(\delta, n, \beta_1, p_1) + \frac{1}{\lambda_2} \left( \kappa^6 r_1(\delta, n, \beta_1, p_1)^2 + 2B_1\kappa^{6+\frac{\beta_1-1}{2}}  r_1(\delta, n, \beta_1, p_1) \right)\|\bar{\varphi}_{\lambda_2,m}\|_{\gH_{\gZ\gA}}.
    \end{align*}
    \label{thm:SecondStageBound3}
\end{lemma}
\begin{proof}
    \begin{align*}
        &\|\hat{\varphi}_{\lambda_2, m} - \bar{\varphi}_{\lambda_2,m}\|_{\gH_{\gZ\gA}} = \|(\hat{\Sigma}_{2,m} + \lambda_2 I)^{-1} (\hat{g}_{2, m} - \bar{g}_{2,m} + \bar{g}_{2,m}) - \bar{\varphi}_{\lambda_2,m}\|_{\gH_{\gZ\gA}}\\
        &= \|(\hat{\Sigma}_{2,m} + \lambda_2 I)^{-1} (\hat{g}_{2, m} - \bar{g}_{2,m}) +  (\hat{\Sigma}_{2,m} + \lambda_2 I)^{-1}\bar{g}_{2,m} - \bar{\varphi}_{\lambda_2,m}\|_{\gH_{\gZ\gA}}\\
        &= \|(\hat{\Sigma}_{2,m} + \lambda_2 I)^{-1} (\hat{g}_{2, m} - \bar{g}_{2,m}) + (\hat{\Sigma}_{2,m} + \lambda_2 I)^{-1}\bar{g}_{2,m}- (\hat{\Sigma}_{2,m} + \lambda_2 I)^{-1} (\hat{\Sigma}_{2,m} + \lambda_2 I)\bar{\varphi}_{\lambda_2,m}\|_{\gH_{\gZ\gA}}\\
        &= \|(\hat{\Sigma}_{2,m} + \lambda_2 I)^{-1} (\hat{g}_{2, m} - \bar{g}_{2,m}) + (\hat{\Sigma}_{2,m} + \lambda_2 I)^{-1}\bar{g}_{2,m}- (\hat{\Sigma}_{2,m} + \lambda_2 I)^{-1} (\hat{\Sigma}_{2,m}\bar{\varphi}_{\lambda_2,m} + \underbrace{\lambda_2 \bar{\varphi}_{\lambda_2,m}}_{\bar{g}_{2,m} - \bar{\Sigma}_{2,m} \bar{\varphi}_{\lambda_2,m}})\|_{\gH_{\gZ\gA}}\\
        &= \|(\hat{\Sigma}_{2,m} + \lambda_2 I)^{-1} (\hat{g}_{2, m} - \bar{g}_{2,m}) + (\hat{\Sigma}_{2,m} + \lambda_2 I)^{-1}\bar{g}_{2,m}- (\hat{\Sigma}_{2,m} + \lambda_2 I)^{-1} \bar{g}_{2,m} \\ &- (\hat{\Sigma}_{2,m} + \lambda_2 I)^{-1}(\hat{\Sigma}_{2,m}\bar{\varphi}_{\lambda_2,m} - \bar{\Sigma}_{2,m} \bar{\varphi}_{\lambda_2,m})\|_{\gH_{\gZ\gA}}\\
        &= \|(\hat{\Sigma}_{2,m} + \lambda_2 I)^{-1} (\hat{g}_{2, m} - \bar{g}_{2,m})  -(\hat{\Sigma}_{2,m} + \lambda_2 I)^{-1}(\hat{\Sigma}_{2,m} - \bar{\Sigma}_{2,m} )\bar{\varphi}_{\lambda_2,m}\|_{\gH_{\gZ\gA}}\\
        &\le \|(\hat{\Sigma}_{2,m} + \lambda_2 I)^{-1} \|_{op} \|\hat{g}_{2, m} - \bar{g}_{2,m} \|_{\gH_{\gZ\gA}} + \|(\hat{\Sigma}_{2,m} + \lambda_2 I)^{-1} \|_{op} \|\hat{\Sigma}_{2,m} - \bar{\Sigma}_{2,m} 
 \|_{op} \|\bar{\varphi}_{\lambda_2,m}\|_{\gH_{\gZ\gA}} \\ &\le \lambda_2^{-1}\left( \|\hat{g}_{2, m} - \bar{g}_{2,m} \|_{\gH_{\gZ\gA}} + \|\hat{\Sigma}_{2,m} - \bar{\Sigma}_{2,m} 
 \|_{op} \|\bar{\varphi}_{\lambda_2,m}\|_{\gH_{\gZ\gA}} \right).
    \end{align*}
    We have two terms to bound. First, we observe that
    \begin{align}
        \|\hat{g}_{2, m} - \bar{g}_{2,m} \|_{\gH_{\gZ\gA}} &= \Bigg\| \frac{1}{m(m-1)} \sum_{\substack{i, j = 1 \\ j \ne i}} \hat{\mu}_{Z| W, A} (\Tilde{w}_j, \Tilde{a}_i) \otimes \phi_\gA(\Tilde{a}_i) - \frac{1}{m(m-1)} \sum_{\substack{i, j = 1 \\ j \ne i}} {\mu}_{Z| W, A} (\Tilde{w}_j, \Tilde{a}_i) \otimes \phi_\gA(\Tilde{a}_i)\Bigg\|_{\gH_{\gZ\gA}}\nonumber\\
        &= \Bigg\| \frac{1}{m(m-1)} \sum_{\substack{i, j = 1 \\ j \ne i}}( \hat{\mu}_{Z| W, A} (\Tilde{w}_j, \Tilde{a}_i)  -  {\mu}_{Z| W, A} (\Tilde{w}_j, \Tilde{a}_i)) \otimes \phi_\gA(\Tilde{a}_i)\Bigg\|_{\gH_{\gZ\gA}}\nonumber\\
        &\le  \frac{\kappa}{m(m-1)} \sum_{\substack{i, j = 1 \\ j \ne i}} \big\|  \hat{\mu}_{Z| W, A} (\Tilde{w}_j, \Tilde{a}_i)  -  {\mu}_{Z| W, A} (\Tilde{w}_j, \Tilde{a}_i) \big\|_{\gH_{\gZ}} \quad \text{(by Assumption (\ref{assum:Stage1ConsistencyKernelAssumptions}))}\nonumber\\
        &\le \kappa^3 r_1(\delta, n, \beta_1, p_1) \quad \text{(with probability $1 - \delta$, Corollary~(\ref{cor:CME_rate_pointwise}))}.  \nonumber
    \end{align}
    
    For the second component, we note that for $i = 1, \ldots, m$,
    $$
    \begin{aligned}
        \xi_i &:= (\hat{\mu}_{Z | W, A}(\Tilde{w}_i, \Tilde{a}_i) \otimes \phi_\gA(\Tilde{a}_i)) \otimes (\hat{\mu}_{Z | W, A}(\Tilde{w}_i, \Tilde{a}_i) \otimes \phi_\gA(\Tilde{a}_i)) -  ({\mu}_{Z | W, A}(\Tilde{w}_i, \Tilde{a}_i) \otimes \phi_\gA(\Tilde{a}_i)) \otimes ({\mu}_{Z | W, A}(\Tilde{w}_i, \Tilde{a}_i) \otimes \phi_\gA(\Tilde{a}_i)) \\ &= ((\hat{\mu}_{Z | W, A}(\Tilde{w}_i, \Tilde{a}_i) - \mu_{Z | W, A}(\Tilde{w}_i, \Tilde{a}_i)) \otimes \phi_\gA(\Tilde{a}_i)) \otimes ((\hat{\mu}_{Z | W, A}(\Tilde{w}_i, \Tilde{a}_i)-\mu_{Z | W, A}(\Tilde{w}_i, \Tilde{a}_i)) \otimes \phi_\gA(\Tilde{a}_i)) \\ &+ ((\hat{\mu}_{Z | W, A}(\Tilde{w}_i, \Tilde{a}_i) - \mu_{Z | W, A}(\Tilde{w}_i, \Tilde{a}_i)) \otimes \phi_\gA(\Tilde{a}_i)) \otimes (\mu_{Z | W, A}(\Tilde{w}_i, \Tilde{a}_i) \otimes \phi_\gA(\Tilde{a}_i)) \\ &+ (\mu_{Z | W, A}(\Tilde{w}_i, \Tilde{a}_i) \otimes \phi_\gA(\Tilde{a}_i)) \otimes ((\hat{\mu}_{Z | W, A}(\Tilde{w}_i, \Tilde{a}_i)-\mu_{Z | W, A}(\Tilde{w}_i, \Tilde{a}_i)) \otimes \phi_\gA(\Tilde{a}_i)),
    \end{aligned}
    $$
    and therefore, under Assumption~(\ref{assum:Stage1ConsistencyKernelAssumptions}), by the triangular inequality and by Corollary~(\ref{cor:CME_rate_pointwise}),
    $$
    \begin{aligned}   
        \|\xi_i\|_{op} &\leq \kappa^2 \big\|  \hat{\mu}_{Z| W, A} (\Tilde{w}_j, \Tilde{a}_i)  -  {\mu}_{Z| W, A} (\Tilde{w}_j, \Tilde{a}_i) \big\|_{\gH_{\gZ}}^2 + 2 \kappa^2 \big\|   {\mu}_{Z| W, A} (\Tilde{w}_j, \Tilde{a}_i) \big\|_{\gH_{\gZ}} \big\|  \hat{\mu}_{Z| W, A} (\Tilde{w}_j, \Tilde{a}_i)  -  {\mu}_{Z| W, A} (\Tilde{w}_j, \Tilde{a}_i) \big\|_{\gH_{\gZ}} \\ &\leq  \kappa^6 r_1(\delta, n, \beta_1, p_1)^2 + 2\kappa^4 \big\|   {\mu}_{Z| W, A} (\Tilde{w}_j, \Tilde{a}_i) \big\|_{\gH_{\gZ}} r_1(\delta, n, \beta_1, p_1).
     \end{aligned}
     $$
     with probability at least $1 - \delta$. We also note that under Assumptions~(\ref{asst:well_specifiedness}), (\ref{assum:Stage1ConsistencyKernelAssumptions}) and (\ref{asst:src_all_stages}-1), for all $(w,a) \in \gW \times \gA$,
    \begin{align*}
        \|{\mu}_{Z|W, A}(w, a)\|_{\gH_\gZ} &= \|{C}_{Z | W, A}(\phi_\gW(w) \otimes \phi_\gA(a))\|_{\gH_\gZ}\\
        & \le \|{C}_{Z | W, A}\|_{S_2(\gH_{\gW\gA}, \gH_\gZ)} \|\phi_\gW(w) \otimes \phi_\gA(a)\|_{\gH_{\gW\gA}}\\
        &\le \kappa^2 \|{C}_{Z | W, A}\|_{S_2(\gH_{\gW\gA}, \gH_\gZ)}\\
        &\le B_1 \kappa^{2+\frac{\beta_1-1}{2}}.
    \end{align*}
     Finally, we obtain that with probability at least $1-\delta$,
     \begin{align}
        \|\hat{\Sigma}_{2,m} - \bar{\Sigma}_{2,m}\|_{op} &= \left\| \frac{1}{m} \sum_{i = 1}^m \xi_i\right\|_{op} \leq \frac{1}{m} \sum_{i = 1}^m \left\|\xi_i\right\|_{op}\nonumber\\
        &\leq \kappa^6 r_1(\delta, n, \beta_1, p_1)^2 + 2B_1\kappa^{6+\frac{\beta_1-1}{2}}  r_1(\delta, n, \beta_1, p_1).
    \end{align}
    
\end{proof}

The following Theorem provides convergence rates in RKHS norm for the estimation of the bridge solution with minimum RKHS norm.

\begin{theorem}
    Suppose Assumptions~(\ref{asst:well_specifiedness}-1 \& 2), (\ref{assum:Stage1ConsistencyKernelAssumptions}), 
    (\ref{asst:src_all_stages}-1 \& 2) and (\ref{asst:evd_all_stages}-1 \& 2) hold and set $\lambda_{1} = \Theta \left(n^{-\frac{1}{\beta_1 + p_1}}\right)$ and $n = m^{\iota\frac{\beta_1+p_1}{\beta_1 - 1}}$ where $\iota > 0$. Then,
    \begin{itemize}
        \item[i.]If $\iota \le \frac{\beta_2+1}{\beta_2 + p_2}$ then $ \|\hat{\varphi}_{\lambda_2, m} - \bar{\varphi}_0\|_{\gH_{\gZ\gA}} = O_p \bigg( m^{-\frac{\iota}{2}\frac{\beta_2-1}{\beta_2+1}}\bigg)$ with $\lambda_{2} = \Theta \left(m^{-\frac{\iota}{\beta_2+1}}\right)$;
        \item[ii.] If $\iota \ge \frac{\beta_2+1}{\beta_2 + p_2}$ then $ \|\hat{\varphi}_{\lambda_2, m} - \bar{\varphi}_0\|_{\gH_{\gZ\gA}} = O_p \bigg(m^{-\frac{1}{2}\frac{\beta_2-1}{\beta_2 + p_2}} \bigg)$ with $\lambda_{2} = \Theta \left(m^{-\frac{1}{\beta_2 + p_2}}\right)$.
    \end{itemize}
    \label{theorem:ATEBridgeFunctionFinalBound}
\end{theorem}
\begin{proof}
Let us abbreviate $r_1(n) = r_1(\delta, n, \beta_1, p_1)$. From Lemma~(\ref{thm:SecondStageBound3}), we obtain with high probability that
\begin{align*}
    \|\hat{\varphi}_{\lambda_2, m} - \bar{\varphi}_0\|_{\gH_{\gZ\gA}} 
    &\le  \|\bar{\varphi}_{\lambda_2,m} - \bar{\varphi}_0\|_{\gH_{\gZ\gA}} + \frac{\bar{J}}{\lambda_2} r_1(n)\left(1 +  (r_1(n)+1) \|\bar{\varphi}_{\lambda_2,m}\|_{\gH_{\gZ\gA}}\right) \\ &\leq  \|\bar{\varphi}_{\lambda_2,m} - \bar{\varphi}_0\|_{\gH_{\gZ\gA}} + \frac{\bar{J}}{\lambda_2} r_1(n) + \frac{\bar{J}}{\lambda_2}  (1 + r_1(n)) r_1(n)(\|\bar{\varphi}_{\lambda_2,m} - \bar{\varphi}_0 \|_{\gH_{\gZ\gA}} + \|\bar{\varphi}_0\|_{\gH_{\gZ\gA}}),
\end{align*}
    where $\bar{J}$ is a constant depending on $\kappa, B_1, \beta_1$. Furthermore, from Theorem~(\ref{theorem:SecondStageBound2}),
    \begin{align*}
        \|\bar{\varphi}_{\lambda_2,m} - \bar{\varphi}_0\|_{\gH_{\gZ\gA}} = O_p\Bigg( r_2(m)\Bigg) \quad \text{ with } \quad r_2(m) =  \frac{1}{\sqrt{\lambda_2}} \Bigg( \sqrt{\frac{1}{m}\left(\frac{1}{\lambda_2^{p_2}}+\frac{1}{m\lambda_2}\right)}  + \sqrt{\frac{1}{m (m-1)}}\Bigg) + \lambda_2^{\frac{\beta_2-1}{2}}.
    \end{align*}
    as long as $m \geq 8\kappa^4\log(2/\delta) g_{\lambda_2} \lambda_2$. Note that the term $[m(m-1)]^{-1/2}$ is of {faster order} and can be removed. Putting it together, we obtain,
    \begin{align*}
        \|\hat{\varphi}_{\lambda_2, m} - \bar{\varphi}_0\|_{\gH_{\gZ\gA}} &= O_p\Bigg( \sqrt{\frac{1}{m\lambda_2^{1+p_2}}\left(1+\frac{1}{m\lambda_2^{1-p_2}}\right)} + \lambda_2^{\frac{\beta_2-1}{2}} + \frac{r_1(n)}{\lambda_2} + \frac{r_1(n)(1 + r_1(n)) (\|\bar \varphi_{\lambda_2,m} - \bar{\varphi}_0\| + \|\bar{\varphi}_0\|)}{\lambda_2}\Bigg)\\
        &= O \Bigg( \sqrt{\frac{1}{m\lambda_2^{1+p_2}}\left(1+\frac{1}{m\lambda_2^{1-p_2}}\right)} + \lambda_2^{\frac{\beta_2-1}{2}} + \frac{1}{\lambda_2}\left(\frac{1}{\sqrt{n}}\right)^{\frac{\beta_1-1}{\beta_1 + p_1}} \Bigg),\\
    \end{align*}
    where in the second equation we removed the last term that is of {faster order} with respect to $r_1(n)\lambda_2^{-1}$ and we plugged the expression for $r_1(n)$ given in Theorem~(\ref{th:CME_rate}). We are now ready to prove cases i. and ii. For each choice of $\lambda_2$ we are required to check the condition $g_{\lambda_2} \lambda_2^{-1} m^{-1} = O(1)$ where $g_{\lambda_2} = \log \left( 2e\mathcal{N}(\lambda_2) \frac{\|\Sigma_2\|_{op}+\lambda_2}{\|\Sigma_2\|_{op}}\right)$. Let us fix a lower bound $0 < c \leq 1$ with $c \leq \|\Sigma_{2}\|_{op}$. Since $\lambda_2 \rightarrow 0$ we choose $m_0 \geq 1$ such that $\lambda_2 \leq c \leq \min\{1, \|\Sigma_2\|_{op}\}$ for all $m \geq m_0$. We get for $m \geq m_0$, by Proposition~(\ref{prop:eff_dim_properties}),
    \begin{equation*}
    \begin{aligned}
    \frac{g_{\lambda_{2}}}{m \lambda_{2}} &= \frac{1}{m \lambda_{2}} \cdot \log \left(2 e \mathcal{N}\left(\lambda_{2}\right) \frac{\|\Sigma_2\|_{op}+\lambda_2}{\|\Sigma_2\|_{op}}\right) \\
    & \leq \frac{1}{m \lambda_{2}} \cdot \log \left(4 D e \lambda_{2}^{-p_2}\right) \\
    &= \frac{\log\left(4 D e\right) }{m \lambda_{2}}+\frac{p_2\log \lambda_{2}^{-1}}{m \lambda_{2}}.
    \end{aligned}
    \end{equation*}
    Therefore, to check $g_{\lambda_2} \lambda_2^{-1} m^{-1} = O(1)$, it is sufficient to check $ \lambda_2^{-1} m^{-1}\log \lambda_{2}^{-1} = O(1)$.
    \paragraph{Case i.} Let $n = m^{\iota \frac{\beta_1+p_1}{\beta_1 - 1}}$ with $\iota \leq \frac{\beta_2+1}{\beta_2 + p_2}$ and $\lambda_{2} = m^{-\frac{\iota}{\beta_2 + 1}}$. We first check $ \lambda_2^{-1} m^{-1}\log \lambda_{2}^{-1} = O(1)$, with this choice of $\lambda_{2}$, we have,
    $$
    \frac{\log \lambda_{2}^{-1}}{m \lambda_{2}}=\frac{\iota}{\beta_2 + 1}\frac{\log(m)}{m}m^{\frac{\iota}{\beta_2 + 1}}  =  \frac{\iota}{\beta_2 + 1}\frac{\log(m)}{m^{1-\frac{\iota}{\beta_2 + 1}}}.
    $$
    As $\iota \leq (\beta_2+1)(\beta_2 + p_2)^{-1}$ implies $1-\iota(\beta_2 + 1)^{-1} > 0$, we have $\log(\lambda_{2}^{-1})/(m \lambda_{2}) \rightarrow 0,$ as $m \rightarrow \infty$. Next note that with this choice of $\lambda_{2}$, we have,
    $$
    \lambda_{2}^{\frac{\beta_2-1}{2}} = m^{-\frac{\iota}{2}\frac{\beta_2-1}{\beta_2+1}} = \frac{1}{\lambda_{2}}\left(\frac{1}{\sqrt{n}}\right)^{\frac{\beta_1-1}{\beta_1 + p_1}}.
    $$
    Furthermore,
    $$
    \frac{1}{m\lambda_{2}^{1+p_2}} \leq \lambda_{2}^{\beta_2-1} \iff \iota \leq \frac{\beta_2+1}{\beta_2 + p_2},
    $$
    and
    $$
    \frac{1}{m\lambda_{2}^{1-p_2}} \leq \frac{1}{m\lambda_{2}^{1+p_2}} \leq 1.
    $$
    Therefore, 
    $$
        \|\hat{\varphi}_{\lambda_2, m} - \bar{\varphi}_0\|_{\gH_{\gZ\gA}} = O\Bigg(  \lambda_{2}^{\frac{\beta_2-1}{2}} \Bigg) = O\Bigg(  m^{-\frac{\iota}{2}\frac{\beta_2-1}{\beta_2+1}} \Bigg).
    $$

    \paragraph{Case ii.} Let $n = m^{\iota \frac{\beta_1+p_1}{\beta_1 - 1}}$ with $\iota \geq \frac{\beta_2+1}{\beta_2 + p_2}$ and  $\lambda_{2} = m^{-\frac{1}{\beta_2 + p_2}}$.  We first check $ \lambda_2^{-1} m^{-1}\log \lambda_{2}^{-1} = O(1)$, with this choice of $\lambda_{2}$, we have,
    $$
    \frac{\log \lambda_{2}^{-1}}{m \lambda_{2}}=\frac{1}{\beta_2 + p_2}\frac{\log(m)}{m}m^{\frac{1}{\beta_2 + p_2}}  =  \frac{1}{\beta_2 + p_2}\frac{\log(m)}{m^{1-\frac{1}{\beta_2 + p_2}}}.
    $$
    As $\beta_2+p_2 > 1$, we have $\log(\lambda_{2}^{-1})/(m \lambda_{2}) \rightarrow 0,$ as $m \rightarrow \infty$. Next note that with this choice of $\lambda_{2}$, we have,
    $$
    \frac{1}{\lambda_{2}}\left(\frac{1}{\sqrt{n}}\right)^{\frac{\beta_1-1}{\beta_1 + p_1}} = \lambda_{2}^{\frac{\beta_2-1}{2}}\left(\lambda_{2}^{-\frac{\beta_2+1}{2}}\left(\frac{1}{\sqrt{n}}\right)^{\frac{\beta_1-1}{\beta_1 + p_1}} \right),
    $$ 
    and 
    $$
    \lambda_{2}^{-\frac{\beta_2+1}{2}}\left(\frac{1}{\sqrt{n}}\right)^{\frac{\beta_1-1}{\beta_1 + p_1}} =  \sqrt{m}^{\frac{\beta_2+1}{\beta_2 + p_2}}\left(\frac{1}{\sqrt{n}}\right)^{\frac{\beta_1-1}{\beta_1 + p_1}} \leq 1 \iff n \geq  m^{\frac{\beta_1+p_1}{\beta_1 - 1}\frac{\beta_2+1}{\beta_2 + p_2}}.
    $$
    Therefore, $\lambda_{2}^{-1}\sqrt{n}^{-\frac{\beta_1-1}{\beta_1 + p_1}} \leq \lambda_{2}^{\frac{\beta_2-1}{2}}$ since we have assumed $n = m^{\iota \frac{\beta_1+p_1}{\beta_1 - 1}} \geq m^{\frac{\beta_1+p_1}{\beta_1 - 1}\frac{\beta_2+1}{\beta_2 + p_2}}.$ Furthermore,
    $$
    \frac{1}{m\lambda_{2}^{1+p_2}} = \left(\frac{1}{m} \right)^{\frac{\beta_2-1}{\beta_2 + p_2}} = \lambda_{2}^{\beta_2 - 1},
    $$
    and
    $$
    \frac{1}{m\lambda_{2}^{1-p_2}} = \left(\frac{1}{m}\right)^{\frac{\beta_2-1+2p_2}{\beta_2+p_2}} \leq \frac{1}{m\lambda_{2}^{1+p_2}}.
    $$
    Therefore, 
    $$
    \|\hat{\varphi}_{\lambda_2, m} - \bar{\varphi}_0\|_{\gH_{\gZ\gA}} = O\Bigg(  \lambda_{2}^{\frac{\beta_2-1}{2}} \Bigg) = O\Bigg(  \sqrt{m}^{-\frac{\beta_2-1}{\beta_2 + p_2}} \Bigg).
    $$
\end{proof}

\subsubsection{Third-Stage Regression Consistency Results} \label{sec:third_stage_ate}

The following theorem is obtained from \cite{li2024towards} that provides convergence guarantees for vector-valued regression with Tikhonov regularization. 

\begin{theorem}[Theorem~3 \cite{li2024towards}] \label{th:CME_rate_JMLR}
    Suppose Assumptions~(\ref{asst:well_specifiedness}-3), (\ref{assum:Stage1ConsistencyKernelAssumptions}), 
    (\ref{asst:src_all_stages}-3) and (\ref{asst:evd_all_stages}-3) hold and take $\lambda_{3} = \Theta \left(t^{-\frac{1}{\beta_3 + p_3}}\right)$. There is a constant $J_3 > 0$ independent of $t \geq 1$ and $\delta \in (0,1)$ such that \[\left\|\hat{C}_{YZ | A} - {C}_{YZ | A}\right\|_{S_2(\gH_{\gA}, \gH_\gZ)} \leq  J_3 \log(5/\delta) \left(\frac{1}{\sqrt{t}}\right)^{\frac{\beta_3-1}{\beta_3 + p_3}} =: r_3(\delta, t, \beta_3, p_3),\] is satisfied for sufficiently large $t \geq 1$ with probability at least $1-\delta$. 
\end{theorem}

\begin{proof}
    To apply Theorem 3 \cite{li2024towards}, we prove that the following noise condition is satisfied, for $q \geq 2$, $P_A-$almost surely, 
    \begin{equation} \label{eq:MOM_proof}
    \E[\|Y\phi_\gZ(Z) - \Psi(A)\|^q_{\gH_{\gZ}} \mid A] \leq \frac{1}{2}q!\tilde{\sigma}^2\tilde{R}^{q-2},
    \end{equation}
    for some $\tilde{\sigma} > 0$ and $\tilde{R} > 0$. We first notice that the Bernstein condition given by Assumption~\ref{assum:Stage1ConsistencyKernelAssumptions} Eq.~(\ref{eq:mom_stage_3}) implies that conditionally on $A$, $|Y - \E[Y \mid A]|$ is sub-exponential \citep[Section 1.4][]{buldygin2000metric}, which then implies sub-exponentiality of $|Y|$ given $A$. As, $\|\phi_\gZ(Z)\|_{\gH_{\gZ}} \leq \kappa$ is almost surely bounded, we obtain that $\|Y\phi_\gZ(Z)\|_{\gH_{\gZ}}$ given $A$ is sub-exponential. Finally, \citet[][Remark 4.9 and Appendix A.2]{mollenhauer2022learning} shows that $\|Y\phi_\gZ(Z)\|_{\gH_{\gZ}}$ given $A$ being sub-exponential implies Eq.~(\ref{eq:MOM_proof}). We can then apply Theorem 3 \cite{li2024towards} with $\gamma=1$ which corresponds to the Hilbert-Schmidt norm of $S_2(\gH_{\gA}, \gH_\gZ)$. Similarly to the proof of Theorem~\ref{th:CME_rate}, we can take the smoothness up to $\beta_3=3$ as we work with the RKHS norm. 
\end{proof}

\begin{corollary} \label{cor:CME_rate_pointwise_stage3}
Under the same assumptions as Theorem~(\ref{th:CME_rate_JMLR}), with $\lambda_{3} = \Theta \left(n^{-\frac{1}{\beta_3 + p_3}}\right)$, for any $\delta \in (0, 1)$, the following holds with probability at least $1 - \delta$:
    \begin{align*}
        \sup_{a \in \gA}\|\hat{\mu}_{YZ|A}(a) - {\mu}_{YZ|A}(a)\|_{\gH_\gZ} \le \kappa r_3(\delta, t, \beta_3, p_3).
    \end{align*}
\end{corollary}

\subsubsection{Consistency for Dose-Response Curve}\label{sec:final_stage_ate}

\begin{theorem}
Suppose Assumptions~(\ref{asst:well_specifiedness}), (\ref{assum:Stage1ConsistencyKernelAssumptions}), (\ref{asst:src_all_stages}) and (\ref{asst:evd_all_stages}) hold and set $\lambda_1 = \Theta \left(n^{-\frac{1}{\beta_1 + p_1}}\right)$, $\lambda_{3} = \Theta \left(t^{-\frac{1}{\beta_3 + p_3}}\right)$ and $n = m^{\iota\frac{\beta_1+p_1}{\beta_1 - 1}}$ where $\iota > 0$. Then,
    \begin{itemize}
        \item[i.]If $\iota \le \frac{\beta_2+1}{\beta_2 + p_2}$ then $\|\hat{f}_{ATE} - f_{ATE}\|_\infty = O_p \bigg(\sqrt{t}^{-\frac{\beta_3-1}{\beta_3 + p_3}} + m^{-\frac{\iota}{2}\frac{\beta_2-1}{\beta_2+1}}\bigg)$ with $\lambda_2 = \Theta \left(m^{-\frac{\iota}{\beta_2+1}}\right)$;
        \item[ii.] If $\iota \ge \frac{\beta_2+1}{\beta_2 + p_2}$ then $ \|\hat{f}_{ATE} - f_{ATE}\|_\infty = O_p \bigg(\sqrt{t}^{-\frac{\beta_3-1}{\beta_3 + p_3}} + m^{-\frac{1}{2}\frac{\beta_2-1}{\beta_2 + p_2}} \bigg)$ with $\lambda_2 = \Theta \left(m^{-\frac{1}{\beta_2 + p_2}}\right)$.
    \end{itemize}
    \label{theorem:ATEB_FINAL}
\end{theorem}

\begin{proof} We recall that we showed in Equation~(\ref{eq:ate_decomposition}) that for any $a \in \gA$,
$$
\begin{aligned}
    |&\hat{f}_{ATE}(a) - f_{ATE}(a)| \\ &\leq \kappa\left(\|\hat{\varphi}_{\lambda_2,m} - \bar{\varphi}_0\|_{\gH_{\gZ\gA}} \|\hat{\Psi}(a) - \Psi(a)\|_{\gH_{\gZ}} + \|\bar{\varphi}_0\|_{\gH_{\gZ\gA}} \| \hat{\Psi}(a) - \Psi(a) \|_{\gH_{\gZ\gA}} + \| \hat{\varphi}_{\lambda_2,m} - \varphi_0 \|_{\gH_{\gZ\gA}} \|\Psi(a) \|_{\gH_{\gZ}}\right).
\end{aligned}    
$$
Note that, under Assumption~(\ref{asst:well_specifiedness}-3),  Assumption~(\ref{assum:Stage1ConsistencyKernelAssumptions}) and Assumption~(\ref{asst:src_all_stages}-3),
    \begin{align*}
        \|\Psi(a)\|_{\gH_{\gZ}} = \|C_{YZ \mid A}\phi_\gA(a)\|_{\gH_{\gZ}} \leq \kappa \|C_{YZ \mid A}\Sigma_3^{-\frac{\beta_3-1}{2}}\Sigma_3^{\frac{\beta_3-1}{2}}\|_{S_2(\gH_{\gA}, \gH_{\gZ})} \leq B_3 \kappa^{1 + \frac{\beta_3-1}{2}} =: \alpha_1.
    \end{align*}
    Furthermore, under Assumption~(\ref{asst:src_all_stages}-2), by Proposition~\ref{prop:min_norm_charac2},
    $$
    \|\bar{\varphi}_0\|_{\gH_{\gZ\gA}} = \|\Sigma_2^{\frac{\beta_2-1}{2}}\Sigma_2^{-\frac{\beta_2-1}{2}}\bar{\varphi}_0\|_{\gH_{\gZ\gA}}  \leq B_2 \kappa^{\frac{\beta_2-1}{2}} =: \alpha_2.
    $$
    
    Therefore,
    $$
|\hat{f}_{ATE}(a) - f_{ATE}(a)| \leq \kappa\left(\|\hat{\varphi}_{\lambda_2,m} - \bar{\varphi}_0\|_{\gH_{\gZ\gA}} \|\hat{\Psi}(a) - \Psi(a)\|_{\gH_{\gZ}} + \alpha_2 \| \hat{\Psi}(a) - \Psi(a) \|_{\gH_{\gZ\gA}} + \alpha_1\| \hat{\varphi}_{\lambda_2,m} - \varphi_0 \|_{\gH_{\gZ\gA}}\right) .
$$
As the first term is faster than the two last terms, plugging the results of Theorem~\ref{theorem:ATEBridgeFunctionFinalBound} and Corollary~\ref{cor:CME_rate_pointwise_stage3}, we obtain the final bound.
\end{proof}

\subsection{Consistency Results for Conditional Dose-Response}
In this section, we present the consistency result for the estimation of $f_{\text{ATT}}$. We note that the first-stage and third-stage regressions are identical to the ones for $f_{\text{ATE}}$. Hence, we first derive the consistency of the second-stage regression of conditional dose-response curve. Then, by incorporating the first and third-stage consistency results for dose-response curve from previous sections, we prove the non-asymptotic uniform consistency guarantees for conditional dose-response curve estimation.

Throughout the whole section $a' \in \gA$ is fixed.

\subsubsection{Second-Stage Regression Consistency Results}

First, we will assume that the problem is well-defined. 

\begin{assumption}[RKHS bridge for ATT]
    There exists $\varphi_0 \in \gH_{\gZ\gA\gA}$ that is a solution of Equation (\ref{eq:AlternativeProxyATTBridgeFunction}). 
    \label{asst:well_specifiedness_ATT}
\end{assumption}

Similarly to ATE we define the minimum norm bridge solution for ATT.

\begin{definition}[Bridge solution with minimum RKHS norm] \label{def:min_norm_bridge_ATT}
We define 
\begin{align*}
\bar{\varphi}_0 = \argmin_{\varphi 
    \in \gH_{\gZ\gA \gA}} \|\varphi\|_{\gH_{\gZ\gA \gA}} \quad \text{ s.t. } \quad \E[\varphi(Z, A, a') | W, A] = r(W, A, a'),
\end{align*}
where $r(W, A, a') = \frac{p(W, a') p(a)}{p(W, a) p(a')}$.
\end{definition}

\begin{definition}
We define the second stage covariance operator for the conditional dose-response: 
$$
\Sigma_4 := \E_{W, A}\left[ \left(\big(\mu_{Z|W, A} (W, A) \otimes \phi_\gA(A) \otimes \phi_\gA(a')\big) \otimes \big(\mu_{Z|W, A} (W, A)\otimes \phi_\gA(A) \otimes \phi_\gA(a') \big)\right)\right].
$$
\end{definition}

\begin{proposition} \label{prop:min_norm_charac_ATT} Under Assumption~(\ref{asst:well_specifiedness_ATT}),
 $\bar{\varphi}_0$ is well-defined and is the unique element of $\gH_{\gZ\gA\gA}$ satisfying Equation (\ref{eq:AlternativeProxyATTBridgeFunction}) and such that $\bar{\varphi}_0 \in \operatorname{null}(\Sigma_4)^{\perp}$.
\end{proposition}

\begin{proof}
Note that the bridge equation, for an element $\varphi \in \gH_{\gZ\gA\gA}$, can be written as
$$
r(W, A, a') = E[\varphi(Z, A, a') | W, A] = \langle \varphi, \mu_{Z \mid W,A}(W,A) \otimes \phi_{\gA}(A) \otimes \phi_\gA(a')\rangle_{\gH_{\gZ\gA}} = A\varphi,
$$
by using the reproducing property and introducing the operator $A: \gH_{\gZ\gA\gA} \to \gL_2(\gW \times \gA, p_{W, A}), \varphi \mapsto \langle \varphi, \mu_{Z \mid W,A}(W,A) \otimes \phi_{\gA}(A) \otimes \phi_\gA(a') \rangle_{\gH_{\gZ\gA\gA}}$. The rest of the proof follows from the steps of the proof of Proposition~(\ref{prop:min_norm_charac}).
\end{proof}

Recall that we have the following loss at the population level:
\begin{align*}
    \gL^{2SR}(\varphi) = \E[(r(W, A, a') - \E[\varphi(Z, A, a') | W, A])^2].
\end{align*} This loss can be equivalently written as
\begin{align*}
    \gL^{2SR}(\varphi) = \E[ \E[\varphi(Z, A, a') | W, A]^2] - 2 \E_{W| A' = a'} \E_{A}[ \E[\varphi(Z, A, a') | W, A] ] + \text{const.}
\end{align*}
We introduce the regularized version of the population loss, for $\lambda_2 > 0$,
$$
 \gL^{2SR}_{\lambda_2}(\varphi) =  \gL^{2SR}(\varphi) + \lambda_2 \|\varphi\|^2_{\gH_{\gZ\gA\gA}}.
$$

We modify Assumptions~(\ref{asst:src_all_stages}) and (\ref{asst:evd_all_stages}) for ATT as follows.

\begin{assumption}\label{asst:src_all_stages_ATT}  We assume that the following condition holds: there exists a constant $B_4 < \infty$ such that for a given $\beta_4 \in (1, 3]$, 
$$
    \|\Sigma_4^{-\frac{\beta_4-1}{2}}\bar{\varphi}_0\|_{\gH_{\gZ\gA\gA}} \le B_4.
$$
\end{assumption}

\begin{assumption} We assume that the following condition holds\label{asst:evd_all_stages_ATT}: let $(\lambda_{4,i})_{i\geq 1}$ be the eigenvalues of $\Sigma_4$, for some constant $c_4 >0$ and parameter $p_4 \in (0,1]$ and for all $i \geq 1$,
\begin{equation*}
    \lambda_{4,i} \leq c_4i^{-1/p_4}.
\end{equation*}   
\end{assumption}

Let us also introduce $g_4$ defined as
\begin{align*}
    g_4 = \E_{W| A' = a'} \E_{A}[\mu_{Z|W, A}(W, A) \otimes \phi_\gA(A) \otimes \phi_\gA(a')].
\end{align*}
\begin{proposition}
    $g_4$ can be equivalently written as 
    \begin{align*}
        g_4 = \E[r(W, A, a') \mu_{Z|W, A}(W, A) \otimes \phi_\gA(A) \otimes \phi_\gA(a')].
    \end{align*}
    Furthermore, for any element $\varphi_0 \in \gH_{\gZ\gA\gA}$ solution to the bridge function equation, we have $g_4 = \Sigma_4 \varphi_0$. Finally, 
    \begin{align*}
        \varphi_{\lambda_2} := \argmin_{\varphi \in \gH_{\gZ \gA \gA}} \gL^{2SR}_{\lambda_2}(\varphi) = \left(\Sigma_4 + \lambda_2 \text{Id}_{\gH_{\gZ\gA \gA}} \right)^{-1} g_4.
    \end{align*}
    \label{prop:pop_estim_stage_2_ATT}
\end{proposition}
\begin{proof}
    {The proof follows the same steps of the proof of Proposition (\ref{prop:pop_estim_stage_2})}
\end{proof}
Combined with Proposition~(\ref{prop:min_norm_charac_ATT}), we obtain the following result.
\begin{proposition} \label{prop:min_norm_charac2_ATT}
    Under Assumption~(\ref{asst:well_specifiedness_ATT}), 
    $\bar{\varphi}_0 = \Sigma_4^{\frac{\beta_4-1}{2}}\Sigma_4^{-\frac{\beta_4-1}{2}}\bar{\varphi}_0.$
\end{proposition}

Now, we introduce the empirical version of the loss function. Define 
\begin{align*}
    \bar{\gL}^{2SR}_m(\varphi) &= \frac{1}{m} \sum_{i = 1}^m \big\langle \varphi, \mu_{Z| W, A}(\Tilde{w}_i, \Tilde{a}_i) \otimes \phi_\gA(\Tilde{a}_i) \otimes \phi_\gA(a')\big\rangle^2\\
    &-\frac{2}{m} \sum_{\substack{i, j = 1\\ i \ne j}}^m \big\langle \varphi, \theta_i \mu_{Z| W, A}(\Tilde{w}_i, \Tilde{a}_j) \otimes \phi_\gA(\Tilde{a}_j) \otimes \phi_\gA(a')\big\rangle + \lambda_2\|\varphi\|^2_{\gH_{\gZ\gA\gA}},
\end{align*}
\begin{align*}
    \hat{\gL}^{2SR}_m(\varphi) &= \frac{1}{m} \sum_{i = 1}^m \big\langle \varphi, \hat{\mu}_{Z| W, A}(\Tilde{w}_i, \Tilde{a}_i) \otimes \phi_\gA(\Tilde{a}_i) \otimes \phi_\gA(a')\big\rangle^2\\
    &-\frac{2}{m} \sum_{\substack{i, j = 1\\ i \ne j}}^m \big\langle \varphi, \theta_i \hat{\mu}_{Z| W, A}(\Tilde{w}_i, \Tilde{a}_j) \otimes \phi_\gA(\Tilde{a}_j) \otimes \phi_\gA(a')\big\rangle + \lambda_2\|\varphi\|^2_{\gH_{\gZ\gA\gA}}.
\end{align*}
where $\theta_i = [(\mK_{\Tilde{A} \Tilde{A}} + m \zeta \mI)^{-1} \mK_{\Tilde{A} a'}]_i$. To write down the minimizers of these loss functions, we introduce the following sample based operators
\begin{align*}
    &\bar{\Sigma}_{4,m} = \frac{1}{m} \sum_{i = 1}^m \Big( \mu_{Z|W, A}(\Tilde{w}_i, \Tilde{a}_i) \otimes \phi_\gA(\Tilde{a}_i) \otimes \phi_\gA(a')\Big) \otimes \Big( \mu_{Z|W, A}(\Tilde{w}_i, \Tilde{a}_i) \otimes \phi_\gA(\Tilde{a}_i) \otimes \phi_\gA(a')\Big),\\
    &\bar{g}_{4,m} = \frac{1}{m}\sum_{\substack{i,j \\ i\ne j}}^m  \theta_i \mu_{Z| W, A}(\Tilde{w}_i, \Tilde{a}_j) \otimes \phi_\gA(\Tilde{a}_j) \otimes \phi_\gA(a'),\\
    &\hat{\Sigma}_{4,m} = \frac{1}{m} \sum_{i = 1}^m \Big( \hat{\mu}_{Z|W, A}(\Tilde{w}_i, \Tilde{a}_i) \otimes \phi_\gA(\Tilde{a}_i) \otimes \phi_\gA(a')\Big) \otimes \Big( \hat{\mu}_{Z|W, A}(\Tilde{w}_i, \Tilde{a}_i) \otimes \phi_\gA(\Tilde{a}_i) \otimes \phi_\gA(a')\Big),\\
    &\hat{g}_{4, m} = \frac{1}{m}\sum_{\substack{i,j \\ i\ne j}}^m  \theta_i\hat{\mu}_{Z|W, A}(\Tilde{w}_i, \Tilde{a}_j) \otimes \phi_\gA(\Tilde{a}_j) \otimes \phi_\gA(a').
\end{align*}
We can observe that the minimizers of the objective functions are given by
\begin{align*}
    \varphi_{\lambda_2} &= (\Sigma_4 + \lambda_2 I)^{-1} g_4 = \argmin_{\varphi \in \gH_{\gZ\gA}} {\gL}^{2SR}_{\lambda_2}(\varphi),\\
    \bar{\varphi}_{\lambda_2,m} &= (\bar{\Sigma}_{4,m} + \lambda_2 I)^{-1} \bar{g}_{4,m} = \argmin_{\varphi \in \gH_{\gZ\gA}} \bar{\gL}^{2SR}_m(\varphi),\\
    \hat{\varphi}_{\lambda_2,m} &= (\hat{\Sigma}_{4,m} + \lambda_2 I)^{-1} \hat{g}_{4,m} = \argmin_{\varphi \in \gH_{\gZ\gA}} \hat{\gL}_{m}^{2SR}(\varphi).\\
\end{align*}
$\hat{\varphi}_{\lambda_2,m}$ is the final estimator presented in Section~(\ref{sec:main_algo}).
To show the convergence of $\hat{\varphi}_{\lambda_2,m}$ to the minimum norm bridge $\bar{\varphi}_0$ introduced in Definition~(\ref{def:min_norm_bridge_ATT}), we will consider the following decomposition, 
\begin{align*}
    \|\hat{\varphi}_{\lambda_2, m} - \bar{\varphi}_0\|_{\gH_{\gZ\gA\gA}} \le \|\hat{\varphi}_{\lambda_2, m} - \bar{\varphi}_{\lambda_2,m}\|_{\gH_{\gZ\gA \gA}} + \|\bar{\varphi}_{\lambda_2,m} -  \varphi_{\lambda_2}\|_{\gH_{\gZ\gA \gA}}+  \| \varphi_{\lambda_2} - \bar{\varphi}_0\|_{\gH_{\gZ\gA\gA}}.
\end{align*}
First, consider an upper bound for $\|\varphi_{\lambda_2} - \bar{\varphi}_0\|_{\gH_{\gZ\gA\gA}}$ 
\begin{lemma}
    Suppose that Assumption~(\ref{asst:src_all_stages_ATT}) holds with parameter $\beta_4 \in (1,3]$. Then, for any $\lambda_2 > 0$,
    \begin{align*}
        \|\varphi_{\lambda_2} - \bar{\varphi}_0\|_{\gH_{\gZ\gA\gA}} \le B_4\lambda_2^{\frac{\beta_4-1}{2}}.
    \end{align*}
    \label{lemma:SecondStageBiasBoundATT}
\end{lemma}
\begin{proof}
    We saw in Proposition~(\ref{prop:pop_estim_stage_2_ATT}) that
    $$
         \varphi_{\lambda_2} = \left(\Sigma_4 + \lambda_2 \operatorname{Id} \right)^{-1}g_4 = \left(\Sigma_4 + \lambda_2 \operatorname{Id} \right)^{-1}\Sigma_4\bar{\varphi}_0 = \bar{\varphi}_0 - \lambda_2 \left(\Sigma_4 + \lambda_2 \operatorname{Id} \right)^{-1} \bar{\varphi}_0.
    $$
    Therefore, under Assumption~(\ref{asst:src_all_stages_ATT}), 
    $$
    \left\|\varphi_{\lambda_2} - \bar{\varphi}_0\right\|_{\gH_{\gZ\gA\gA}} = \lambda_2  \left\|\left(\Sigma_4 + \lambda_2\operatorname{Id} \right)^{-1} \bar{\varphi}_0\right\|_{\gH_{\gZ\gA}} \leq B_4 \lambda_2  \left\|\left(\Sigma_4 + \lambda_2\operatorname{Id} \right)^{-1}\Sigma_4^{\frac{\beta_4-1}{2}} \right\|_{op},
    $$
    where we used Proposition~(\ref{prop:min_norm_charac2_ATT}). Note that, by Lemma~(\ref{lma:steinwart_lemma}),
    $$
    \left\|\left(\Sigma_4 + \lambda_2\operatorname{Id} \right)^{-1}\Sigma_4^{\frac{\beta_4-1}{2}} \right\|_{op} = \sup_{i \geq 1} \frac{\lambda_{4,i}^{\frac{\beta_4-1}{2}}}{\lambda_{4,i} + \lambda_2} \leq \lambda_2^{\frac{\beta_4-1}{2} - 1},
    $$
    as long as $\frac{\beta_4-1}{2} \in (0,1]$, i.e. $\beta_4 \in (1,3]$. By merging the bounds, we obtain the final result.  
\end{proof}

\begin{lemma}
For any $\lambda_2 > 0$, $\|\varphi_{\lambda_2}\|_{\gH_{\gZ\gA \gA}} \le \|\bar{\varphi}_0\|_{\gH_{\gZ\gA \gA}}$.
    \label{lemma:varphi0lessthanvarphi_ATT}
\end{lemma}
\begin{proof}
We saw in Proposition~(\ref{prop:pop_estim_stage_2_ATT}) that
    $$
         \left\|\varphi_{\lambda_2}\right\|_{\gH_{\gZ\gA\gA}}  = \left\|\left(\Sigma_4 + \lambda_2 \operatorname{Id} \right)^{-1}\Sigma_4\bar{\varphi}_0\right\|_{\gH_{\gZ\gA\gA}} \leq  \left\|\left(\Sigma_4 + \lambda_2 \operatorname{Id} \right)^{-1}\Sigma_4\right\|_{op}\left\|\bar{\varphi}_0\right\|_{\gH_{\gZ\gA \gA}}\leq \left\|\bar{\varphi}_0\right\|_{\gH_{\gZ\gA \gA}}.
    $$
\end{proof}

Similar to the consistency result for the dose-response curve, our convergence proof uses both the Hoeffding concentration inequality (Corollary~\ref{corr:hoeffding})  and the Bernstein concentration inequality (Theorem~\ref{theo:bernstein}) for Hilbert space-valued random variables. We also define the effective dimension for the stage 2 error: for $\lambda_2 > 0$, $\gN(\lambda_2) := \operatorname{Tr}((\Sigma_4 + \lambda_2 \operatorname{Id})^{-1}\Sigma_4)$ \cite{caponnetto2007optimal}.

\begin{proposition}[Lemma 11 \& Lemma 13 \cite{fischer2020sobolev}]  \label{prop:eff_dim_properties_ATT}
Under Assumption~(\ref{asst:evd_all_stages_ATT}), there is a constant $D>0$ such that the following inequality is satisfied, for $\lambda_2 > 0$, $\gN(\lambda_2) \leq D\lambda_2^{-p_4}$. Furthermore, we have the equality,
$$
\E\left[\left\|(\Sigma_4 + \lambda_2 \operatorname{Id})^{-1/2} \mu_{Z|W, A} (W, A) \otimes \phi_\gA(A) \otimes \phi_\gA(a')\right\|^2_{\gH_{\gZ\gA \gA}}\right] = \gN(\lambda_2).
$$
\end{proposition}

\begin{lemma}
Let us introduce $g_{\lambda_2} = \log \left( 2e\mathcal{N}(\lambda_2) \frac{\|\Sigma_4\|_{op}+\lambda_2}{\|\Sigma_4\|_{op}}\right).$ Suppose Assumption (\ref{assum:Stage1ConsistencyKernelAssumptions}) holds. Then, with probability at least $1-\delta$ for all $\delta \in (0,1)$, for $m \geq 8\kappa^6\log(2/\delta) g_{\lambda_2} \lambda_2^{-1}$,
\begin{align*}
    \|\bar{\varphi}_{\lambda_2,m} -  \varphi_{\lambda_2}\|_{\gH_{\gZ\gA \gA}} \le \frac{3}{\sqrt{\lambda_2}} \Bigg( \log(2/\delta)\sqrt{\frac{32}{m}\left(\gN(\lambda_2)\kappa^6 \|\bar{\varphi}_0\|_{\gH_{\gZ\gA\gA}}^2+\frac{\kappa^{12} \|\bar{\varphi}_0\|_{\gH_{\gZ\gA\gA}}^2}{m\lambda_2}\right)} + \frac{\|g_4 - \bar{g}_{4,m}\|_{\gH_{\gZ\gA\gA}}}{\sqrt{\lambda_2}}\Bigg).
\end{align*}
    \label{theorem:SecondStageBound1_ATT}
\end{lemma}

\begin{proof}
    We decompose the error as 
    \begin{align*}
        \|\bar{\varphi}_{\lambda_2,m} -  \varphi_{\lambda_2}\|_{\gH_{\gZ\gA\gA}} &= \big\|(\bar{\Sigma}_{4,m} + \lambda_2 I)^{-1} \big(\bar{g}_{4,m}-  (\bar{\Sigma}_{4,m} + \lambda_2 I) \varphi_{\lambda_2}\big) \big\|_{\gH_{\gZ\gA\gA}}\\ &\leq \left\|(\Sigma_4 + \lambda_2 I)^{-1/2} \right\|_{op}\left\|(\Sigma_4 + \lambda_2 I)^{1/2}(\bar{\Sigma}_{4,m} + \lambda_2 I)^{-1}(\Sigma_4 + \lambda_2 I)^{1/2} \right\|_{op}\\&\times \big\|(\Sigma_4 + \lambda_2 I)^{-1/2} \big( \bar{g}_{4,m}  - \bar{\Sigma}_{4,m} \varphi_{\lambda_2} - \underbrace{\lambda_2 \varphi_{\lambda_2}}_{= g_4 - \Sigma_4 \varphi_{\lambda_2}} \big) \big\|_{\gH_{\gZ\gA\gA}}\\
        &\le \lambda_2^{-1/2} \left\|(\Sigma_4 + \lambda_2 I)^{1/2}(\bar{\Sigma}_{4,m} + \lambda_2 I)^{-1}(\Sigma_4 + \lambda_2 I)^{1/2} \right\|_{op} \\&\times \Big(\|(\Sigma_4 + \lambda_2 I)^{-1/2}(\bar{\Sigma}_{4,m} - \Sigma_4) \varphi_{\lambda_2}\|_{\gH_{\gZ\gA\gA}} + \lambda_2^{-1/2}\|\bar{g}_{4,m} - g_4\|_{\gH_{\gZ\gA\gA}} \Big),
    \end{align*}
    where $\|.\|_{op}$ denotes the \emph{operator norm}. In the same fashion as in Lemma~(\ref{lemma:steinwart_cov_conc}), we obtain that
$$
\left\|(\Sigma_4 + \lambda_2 I)^{1/2}(\bar{\Sigma}_{4,m} + \lambda_2 I)^{-1}(\Sigma_4 + \lambda_2 I)^{1/2} \right\|_{op} \leq 3,
$$
for $m \geq 8\kappa^6\log(2/\delta) g_{\lambda_2} \lambda_2^{-1}$ with probability at least $1-\delta$ for all $\delta \in (0,1)$. 
To bound the remaining term, we wish to apply Theorem~(\ref{theo:bernstein}) with $\gH = \gH_{\gZ\gA\gA}$. Consider the measurable map $\xi: \gW \times \gA \rightarrow \gH_{\gZ\gA\gA}$ defined by
\begin{equation*} \label{eq:xi2_ATT}  
\xi(w, a) := (\Sigma_4 + \lambda_2 I)^{-1/2}\langle \varphi_{\lambda_2}, \mu_{Z|W,A}(w, a) \otimes \phi_\gA(a) \otimes \phi_\gA(a') \rangle_{\gH_{\gZ\gA}} \mu_{Z|W,A}(w, a) \otimes \phi_\gA(a) \otimes \phi_\gA(a'),
\end{equation*}
inducing random variables such that
$$
\frac{1}{m} \sum_{i=1}^m \left(\xi(\Tilde{w}_i, \Tilde{a}_i) - \E[\xi(W, A)]\right) = (\Sigma_4 + \lambda_2 I)^{-1/2}(\bar{\Sigma}_{4,m} - \Sigma_4) \varphi_{\lambda_2}.
$$
By Assumption (\ref{assum:Stage1ConsistencyKernelAssumptions}), Lemma (\ref{lemma:varphi0lessthanvarphi_ATT}) and Cauchy-Schwarz inequality, 
$$
|\langle \varphi_{\lambda_2}, \mu_{Z|W,A}(w, a) \otimes \phi_\gA(a) \otimes \phi_\gA(a') \rangle_{\gH_{\gZ\gA\gA}}| \leq \kappa^3 \|\bar{\varphi}_0\|_{\gH_{\gZ\gA\gA}}. 
$$
We can now bound the $q$-th moment of $\xi$, for $q \geq 2$,
$$
\begin{aligned}
    \mathbb{E}\left\|\xi(W,A)\right\|_{\gH_{\gZ\gA}}^{q} & \leq \left(\kappa^3 \|\bar{\varphi}_0\|_{\gH_{\gZ\gA \gA}}\right)^q \mathbb{E}\left\|(\Sigma_4 + \lambda_2 I)^{-1/2} \mu_{Z|W,A}(W, A) \otimes \phi_\gA(A) \otimes \phi_\gA(a') \right\|_{\gH_{\gZ\gA\gA}}^{q} \\ &\leq \left(\kappa^3 \|\bar{\varphi}_0\|_{\gH_{\gZ\gA\gA}}\right)^q\left(\frac{\kappa^3}{\sqrt{\lambda_2}} \right)^{q-2} \mathbb{E}\left\|(\Sigma_4 + \lambda_2 I)^{-1/2} \mu_{Z|W,A}(W, A) \otimes \phi_\gA(A) \otimes \phi_\gA(a')\right\|_{\gH_{\gZ\gA\gA}}^{2} \\ &= \left(\kappa^3 \|\bar{\varphi}_0\|_{\gH_{\gZ\gA\gA}}\right)^q\left(\frac{\kappa^3}{\sqrt{\lambda_2}} \right)^{q-2}\gN(\lambda_2) \\ &\leq \frac{1}{2}q! \left(\frac{\kappa^6}{\sqrt{\lambda_2}}\|\bar{\varphi}_0\|_{\gH_{\gZ\gA}} \right)^{q-2} \gN(\lambda_2)\kappa^6 \|\bar{\varphi}_0\|_{\gH_{\gZ\gA}}^2,
\end{aligned}
$$
where in the equality, we used Proposition~(\ref{prop:eff_dim_properties_ATT}). An application of Bernstein's inequality from Theorem~(\ref{theo:bernstein}) with 
$$
L=\frac{\kappa^6}{\sqrt{\lambda_2}}\|\bar{\varphi}_0\|_{\gH_{\gZ\gA}}, \qquad \sigma^{2} =  \gN(\lambda_2)\kappa^6 \|\bar{\varphi}_0\|_{\gH_{\gZ\gA}}^2,
$$ 
yields the final bound.
\end{proof}

We will now derive a bound for $\|g_4 - \bar{g}_{4,m}\|_{\gH_{\gZ\gA}}$. We need the following assumption.

\begin{assumption}
    Suppose that there exists $C_{W|A} \in S_2(\gH_\gA, \gH_\gW)$ such that $\E[\phi_\gW(W) | A] = C_{W|A} \phi_\gA(A)$.
    \label{assum:C_WA_well_specifiedness}
\end{assumption}

Since we also need the consider the variables $\theta_i$ in deriving this bound, we will introduce an intermediate operator as follows:
\begin{align*}
    \Tilde{g}_4 &= \frac{1}{m} \sum_{j = 1}^m C_{Z | W, A} \Big(\E_{W | A = a'}[\phi_\gW(W)] \otimes \phi_\gA(\Tilde{a}_j)\Big) \otimes \phi_\gA(\Tilde{a}_j) \otimes \phi_\gA(a')\\
    &= \frac{1}{m} \sum_{j = 1}^m C_{Z | W, A} \Big(C_{W|A} \phi_\gA(a') \otimes \phi_\gA(\Tilde{a}_j)\Big) \otimes \phi_\gA(\Tilde{a}_j) \otimes \phi_\gA(a')
\end{align*}
where $C_{W|A} \in S_2(\gH_\gA, \gH_\gW)$ is the conditional mean operator, i.e., $\E[\phi_\gW(W) | A = a] = C_{W|A} \phi_\gA(a)$. Recall that, in our conditional dose-response algorithm, we estimate this operator using second-stage data. Therefore, we need the following bound for the estimation error of $C_{W|A}$.

\begin{assumption}
    We assume that the following condition holds: there exist a constant $B_5 < \infty$ such that for a given $\beta_5 \in (1,3]$,
    \begin{align*}
        \|C_{W|A} \Sigma_3^{-\frac{\beta_5 - 1}{2}}\|_{S_2(\gH_\gA, \gH_\gW)} \le B_5.
    \end{align*}
    \label{asst:src_ATT_for_E_WA}
\end{assumption}
\begin{theorem}[Theorem~2 \cite{li2022optimal}]
    Suppose Assumptions (\ref{assum:Stage1ConsistencyKernelAssumptions}), (\ref{asst:evd_all_stages}-3), (\ref{assum:C_WA_well_specifiedness}), (\ref{asst:src_ATT_for_E_WA}), hold, and take $\zeta = \Theta \left(m^{-\frac{1}{\beta_5 + p_3}}\right)$. Then, there is a constant $J_5 > 0$ independent of $m \ge 1$ and $\delta \in (0, 1)$ such that
    \begin{align*}
        \|\hat{C}_{W|A} - C_{W|A}\|_{S_2(\gH_\gA, \gH_\gW)} \le J_5 \log(4/\delta) \left(\frac{1}{\sqrt{m}}\right)^{\frac{\beta_5 - 1}{\beta_5 + p_3}}
    \end{align*}
is satisfied for sufficiently large $m \geq 1$ with probability at least $1-\delta$.
\label{theorem:E_WA_CME_Rate_for_ATT}
\end{theorem}

\begin{lemma}
With probability at least $1 - \delta$ for $\delta \in (0, 1)$ the following bound holds:
    \begin{align*}
        \|\bar{g}_{4,m} - g_4\| \le \kappa^4 J_5 \log(4/\delta) \left(\frac{1}{\sqrt{m}}\right)^{\frac{\beta_5 - 1}{\beta_5 + p_3}} \big\| C_{Z|W,A}\big\| + 2 \kappa^3 \sqrt{\frac{2 \log(2/ \delta)}{m}}
    \end{align*}
\label{lemma:g_4Bound_ATT}    
\end{lemma}
\begin{proof}
    Let 
    \begin{align*}
        \xi(A) &= C_{Z|W, A} \left(\E[\phi_\gW(W) |A = a'] \otimes \phi_\gA(A)\right) \otimes \phi_\gA(A) \otimes \phi_\gA(a')\\
        \xi_i &= C_{Z|W, A} \left(\E[\phi_\gW(W) | A = a'] \otimes \phi_\gA(a_i)\right) \otimes \phi_\gA(a_i) \otimes \phi_\gA(a')
    \end{align*}
    Note that $\E[\xi_i] = g_4$. Furthermore, we observe that 
    \begin{align*}
        \|\xi(A)\| &= \|C_{Z|W, A} \left(\E[\phi_\gW(W) |A = a'] \otimes \phi_\gA(A)\right) \otimes \phi_\gA(A) \otimes \phi_\gA(a')\|\\
        &= \|\E_{W|A = a'}\E[\phi_\gZ(Z)| W, A ] \otimes \phi_\gA(A) \otimes \phi_\gA(a')\| \le \kappa^3.
    \end{align*}
    Now, we apply Corollary~(\ref{corr:hoeffding}) such that with probability at least $1 - \delta$, 
    \begin{align*}
        \|\Tilde{g}_4 - g_4\| \le 2 \kappa^3 \sqrt{\frac{2 \log(2/ \delta)}{m}}.
    \end{align*}
    Now, consider the decomposition
    \begin{align*}
        \|\bar{g}_{4,m} - g_4\| &\le \|\bar{g}_{4,m} - \Tilde{g}_4\| + \|\Tilde{g}_4 - g_4\|\\
        &\le \|\bar{g}_{4,m} - \Tilde{g}_4\| + 2 \kappa^3 \sqrt{\frac{2 \log(2/ \delta)}{m}}
    \end{align*}
    Next, we bound the component $ \|\bar{g}_{4,m} - \Tilde{g}_4\|$. 
    \begin{align*}
        \|\bar{g}_{4,m} - \Tilde{g}_4\| &= \Bigg\|\frac{1}{m} \sum_{j = 1}^m C_{Z|W,A} \left( \hat{E}[\phi_\gW(W) | A = a'] \otimes \phi_\gA(a_j)\right) \otimes \phi_\gA(a_j) \otimes \phi_\gA(a') \\
        &- \frac{1}{m} \sum_{j = 1}^m C_{Z|W,A} \left({E}[\phi_\gW(W) | A = a'] \otimes \phi_\gA(a_j)\right) \otimes \phi_\gA(a_j) \otimes \phi_\gA(a') \Bigg\|\\
        &= \Bigg\|\frac{1}{m} \sum_{j = 1}^m C_{Z|W,A} \left( (\hat{E}[\phi_\gW(W) | A = a'] - {E}[\phi_\gW(W) | A = a']) \otimes \phi_\gA(a_j)\right) \otimes \phi_\gA(a_j) \otimes \phi_\gA(a')\Bigg\|\\
        &\le \frac{1}{m} \sum_{j = 1}^m \kappa^3\big\| C_{Z|W,A}\big\| \big\|\hat{E}[\phi_\gW(W) | A = a'] - {E}[\phi_\gW(W) | A = a']\big\|\\
        &= \kappa^3\big\| C_{Z|W,A}\big\| \big\|\hat{E}[\phi_\gW(W) | A = a'] - {E}[\phi_\gW(W) | A = a']\big\|
    \end{align*}
    Appealing to the Theorem (\ref{theorem:E_WA_CME_Rate_for_ATT}), 
    \begin{align*}
        \big\|\hat{E}[\phi_\gW(W) | A = a'] - {E}[\phi_\gW(W) | A = a']\big\| \le \kappa J_5 \log(4/\delta) \left(\frac{1}{\sqrt{m}}\right)^{\frac{\beta_5 - 1}{\beta_5 + p_3}},
    \end{align*}
    with probability at least $1 - \delta$ for $\delta \in (0, 1)$. As a result,
    \begin{align*}
         \|\bar{g}_{4,m} - \Tilde{g}_4\| \le \kappa^4 J_5 \log(4/\delta) \left(\frac{1}{\sqrt{m}}\right)^{\frac{\beta_5 - 1}{\beta_5 + p_3}} \big\| C_{Z|W,A}\big\|.
    \end{align*}
    This implies that
    \begin{align*}
        \|\bar{g}_{4,m} - g_4\| \le \kappa^4 J_5 \log(4/\delta) \left(\frac{1}{\sqrt{m}}\right)^{\frac{\beta_5 - 1}{\beta_5 + p_3}} \big\| C_{Z|W,A}\big\| + 2 \kappa^3 \sqrt{\frac{2 \log(2/ \delta)}{m}}
    \end{align*}
    with probability $1 - 2 \delta$ for $\delta \in (0, 1)$.
\end{proof}

\begin{theorem}
    Suppose Assumptions (\ref{assum:Stage1ConsistencyKernelAssumptions}), (\ref{asst:src_all_stages_ATT}), and (\ref{asst:evd_all_stages_ATT}) hold. Then, with probability at least $1-2\delta$ for all $\delta \in (0,1)$, for $m \geq 8\kappa^6\log(2/\delta) g_{\lambda_2} \lambda_2^{-1}$,
    \begin{align*}
        \|\bar{\varphi}_{\lambda_2,m} - \bar{\varphi}_0\|_{\gH_{\gZ\gA\gA}} \le {J_4\frac{\log(4/\delta)}{\sqrt{\lambda_2}} \Bigg( \sqrt{\frac{1}{m}\left(\frac{1}{\lambda_2^{p_4}}+\frac{1}{m\lambda_2}\right)} + \frac{1}{\sqrt{\lambda_2}}\left(\frac{1}{\sqrt{m}}\right)^{\frac{\beta_5 - 1}{\beta_5 + p_3}} + \sqrt{\frac{1}{{\lambda_2 m}}} \Bigg) + \lambda_2^{\frac{\beta_4-1}{2}},}
    \end{align*}
    where $J_4$ is a constant depending on $\kappa, \beta_4, B_4, D$.    \label{theorem:SecondStageBound2_ATT}
\end{theorem}
\begin{proof}
    Combining the bounds in Lemma (\ref{lemma:SecondStageBiasBoundATT}), Lemma (\ref{theorem:SecondStageBound1_ATT}) and Lemma (\ref{lemma:g_4Bound_ATT}) with a union bound, we obtain that with probability at least $1-2\delta$,
    \begin{align*}
\|\bar{\varphi}_{\lambda_2,m} - \bar{\varphi}_0\|_{\gH_{\gZ\gA\gA}} &\leq \mathring{J} \Bigg(\frac{\log(4/\delta)}{\sqrt{\lambda_2}}\Bigg(\sqrt{\frac{1}{m}\left(\gN(\lambda_2) \|\bar{\varphi}_0\|_{\gH_{\gZ\gA\gA}}^2+\frac{\|\bar{\varphi}_0\|_{\gH_{\gZ\gA\gA}}^2}{m\lambda_2}\right)} +\frac{1}{\sqrt{\lambda_2}} \left(\frac{1}{\sqrt{m}}\right)^{\frac{\beta_5 - 1}{\beta_5 + p_3}} \big\| C_{Z|W,A}\big\| \\&+ \sqrt{\frac{1}{{\lambda_2} m}}\Bigg)  + \lambda_2^{\frac{\beta_4-1}{2}}\Bigg),
\end{align*}
    where $\mathring{J}$ is a constant depending on $\kappa, B_4.$ Under Assumption~(\ref{asst:evd_all_stages_ATT}), using Proposition~(\ref{prop:eff_dim_properties_ATT}), there is a constant $D > 0$ such that $\gN(\lambda_2) \leq D\lambda_2^{-p_4}$. Furthermore, under Assumption~(\ref{asst:src_all_stages_ATT}),
    $$
    \|\bar{\varphi}_0\|_{\gH_{\gZ\gA\gA}} = \|\Sigma_4^{\frac{\beta_4-1}{2}}\Sigma_4^{-\frac{\beta_4-1}{2}}\bar{\varphi}_0\|_{\gH_{\gZ\gA\gA}} \leq \kappa^{\frac{\beta_4-1}{2}}B_4.
    $$
\end{proof}

Next, we will derive a bound for $\|\hat{\varphi}_{\lambda_2, m} - \bar{\varphi}_{\lambda_2,m}\|$.

\begin{lemma}
Under the assumptions of Theorems~(\ref{th:CME_rate}) and (\ref{theorem:E_WA_CME_Rate_for_ATT}), with probability at least $1 - 2 \delta$ for $\delta \in (0, 1)$, the following bound holds
\begin{align*}
    \|\hat{\varphi}_{\lambda_2, m} - \bar{\varphi}_{\lambda_2,m}\| &\le \frac{1}{\lambda_2} {\kappa^3} r_1(\delta, n, \beta_1, p_1) \left(\kappa J_5 \log(4/\delta) \left(\frac{1}{\sqrt{m}}\right)^{\frac{\beta_5 - 1}{\beta_5 + p_3}} + \kappa \right) \\&+ \frac{1}{\lambda_2} \left(\kappa^8 r_1(\delta, n, \beta_1, p_1)^2 + 2B_1\kappa^{8+\frac{\beta_1-1}{2}}  r_1(\delta, n, \beta_1, p_1) \right)\|\bar{\varphi}_{\lambda_2,m}\|_{\gH_{\gZ\gA}}.
\end{align*}
\label{thm:SecondStageBound3_ATT}
\end{lemma}
\begin{proof}
    \begin{align*}
        &\|\hat{\varphi}_{\lambda_2, m} - \bar{\varphi}_{\lambda_2,m}\|_{\gH_{\gZ\gA\gA}} = \|(\hat{\Sigma}_{4,m} + \lambda_2 I)^{-1} (\hat{g}_{4, m} - \bar{g}_{4,m} + \bar{g}_{4,m}) - \bar{\varphi}_{\lambda_2,m}\|_{\gH_{\gZ\gA\gA}}\\
        &= \|(\hat{\Sigma}_{4,m} + \lambda_2 I)^{-1} (\hat{g}_{4, m} - \bar{g}_{4,m}) +  (\hat{\Sigma}_{2,m} + \lambda_2 I)^{-1}\bar{g}_{4,m} - \bar{\varphi}_{\lambda_2,m}\|_{\gH_{\gZ\gA\gA}}\\
        &= \|(\hat{\Sigma}_{4,m} + \lambda_2 I)^{-1} (\hat{g}_{4, m} - \bar{g}_{4,m}) + (\hat{\Sigma}_{4,m} + \lambda_2 I)^{-1}\bar{g}_{4,m}- (\hat{\Sigma}_{4,m} + \lambda_2 I)^{-1} (\hat{\Sigma}_{4,m} + \lambda_2 I)\bar{\varphi}_{\lambda_2,m}\|_{\gH_{\gZ\gA\gA}}\\
        &= \|(\hat{\Sigma}_{4,m} + \lambda_2 I)^{-1} (\hat{g}_{4, m} - \bar{g}_{4,m}) + (\hat{\Sigma}_{4,m} + \lambda_2 I)^{-1}\bar{g}_{4,m}- (\hat{\Sigma}_{4,m} + \lambda_2 I)^{-1} (\hat{\Sigma}_{4,m}\bar{\varphi}_{\lambda_4,m} + \underbrace{\lambda_2 \bar{\varphi}_{\lambda_2,m}}_{\bar{g}_{4,m} - \bar{\Sigma}_{4,m} \bar{\varphi}_{\lambda_2,m}})\|_{\gH_{\gZ\gA\gA}}\\
        &= \|(\hat{\Sigma}_{4,m} + \lambda_2 I)^{-1} (\hat{g}_{4, m} - \bar{g}_{4,m}) + (\hat{\Sigma}_{4,m} + \lambda_2 I)^{-1}\bar{g}_{4,m}- (\hat{\Sigma}_{4,m} + \lambda_2 I)^{-1} \bar{g}_{4,m} \\ &- (\hat{\Sigma}_{4,m} + \lambda_2 I)^{-1}(\hat{\Sigma}_{4,m}\bar{\varphi}_{\lambda_2,m} - \bar{\Sigma}_{4,m} \bar{\varphi}_{\lambda_2,m})\|_{\gH_{\gZ\gA\gA}}\\
        &= \|(\hat{\Sigma}_{4,m} + \lambda_2 I)^{-1} (\hat{g}_{4, m} - \bar{g}_{4,m})  -(\hat{\Sigma}_{4,m} + \lambda_2 I)^{-1}(\hat{\Sigma}_{4,m} - \bar{\Sigma}_{4,m} )\bar{\varphi}_{\lambda_2,m}\|_{\gH_{\gZ\gA\gA}}\\
        &\le \|(\hat{\Sigma}_{4,m} + \lambda_2 I)^{-1} \|_{op} \|\hat{g}_{4, m} - \bar{g}_{4,m} \|_{\gH_{\gZ\gA}} + \|(\hat{\Sigma}_{4,m} + \lambda_2 I)^{-1} \|_{op} \|\hat{\Sigma}_{4,m} - \bar{\Sigma}_{4,m} 
 \|_{op} \|\bar{\varphi}_{\lambda_2,m}\|_{\gH_{\gZ\gA\gA}} \\ &\le \lambda_2^{-1}\left( \|\hat{g}_{4, m} - \bar{g}_{4,m} \|_{\gH_{\gZ\gA\gA}} + \|\hat{\Sigma}_{4,m} - \bar{\Sigma}_{4,m} 
 \|_{op} \|\bar{\varphi}_{\lambda_2,m}\|_{\gH_{\gZ\gA\gA}} \right).
    \end{align*}
    We have two terms to bound. First, we observe that
\begin{align*}
       \|\hat{g}_{4, m} - \bar{g}_{4,m}\| &= \Bigg\|\frac{1}{m} \sum_{\substack{i, j = 1 \\ j \ne i}}^m \theta_i \hat{\mu}_{Z| W, A}(\Tilde{w}_i, \Tilde{a}_j) \otimes \phi_\gA(\Tilde{a}_j) \otimes \phi_\gA(a') - \frac{1}{m} \sum_{\substack{i, j = 1 \\ j \ne i}}^m \theta_i {\mu}_{Z| W, A}(\Tilde{w}_i, \Tilde{a}_j) \otimes  \phi_\gA(\Tilde{a}_j) \otimes \phi_\gA(a') \Bigg\|\\
       &\le \kappa^2 \Bigg\| \frac{1}{m} \sum_{\substack{i, j = 1 \\ j \ne i}}^m \theta_i \hat{C}_{Z| W, A}(\phi_\gW(\Tilde{w}_i) \otimes \phi_\gA(\Tilde{a}_j)) - \frac{1}{m} \sum_{\substack{i, j = 1 \\ j \ne i}}^m \theta_i {C}_{Z| W, A}(\phi_\gW(\Tilde{w}_i)  \otimes \phi_\gA(\Tilde{a}_j)) \Bigg\|\\
       &= \kappa^2 \Bigg\| \frac{1}{m} \sum_{\substack{j = 1}}^m \hat{C}_{Z| W, A}\Big(\sum_{\substack{i = 1 \\ i \ne j}}^m(\theta_i\phi_\gW(\Tilde{w}_i) ) \otimes \phi_\gA(\Tilde{a}_j)\Big) - \frac{1}{m} \sum_{\substack{j = 1}}^m  {C}_{Z| W, A}\Big(\sum_{\substack{i = 1 \\ i \ne j}}^m(\theta_i \phi_\gW(\Tilde{w}_i)) \otimes \phi_\gA(\Tilde{a}_j)\Big) \Bigg\|\\
       &= \kappa^2 \Bigg\| \frac{1}{m} \sum_{\substack{j = 1}}^m \hat{C}_{Z| W, A}\Big(\hat{\E}[\phi_\gW(W) | A = a'] \otimes \phi_\gA(\Tilde{a}_j) \Big) - \frac{1}{m} \sum_{\substack{j = 1}}^m  {C}_{Z| W, A}\Big(\hat{\E}[\phi_\gW(W) | A = a'] \otimes \phi_\gA(\Tilde{a}_j)\Big) \Bigg\|\\
       &\le \frac{\kappa^2}{m}\sum_{j = 1}^m \|\hat{C}_{Z|W,A} - C_{Z|W,A}\| \|\hat{\E}[\phi_\gW(W) | A = a']\| \|\phi_\gA(a_j)\|\\
       &\le {\kappa^3} r_1(\delta, n, \beta_1, p_1) \Big\|\hat{\E}[\phi_\gW(W) | A = a'] - {\E}[\phi_\gW(W)| A = a'] + {\E}[\phi_\gW(W) | A = a']\Big\|\\
       &\le {\kappa^3} r_1(\delta, n, \beta_1, p_1) \Big(\Big\|\hat{\E}[\phi_\gW(W)| A = a'] - {\E}[\phi_\gW(W) | A = a']\Big\| + \Big\| {\E}[\phi_\gW(W) | A = a']\Big\| \Big)\\
       &\le {\kappa^3} r_1(\delta, n, \beta_1, p_1) \Big(\Big\|\hat{\E}[\phi_\gW(W) | A = a'] - {\E}[\phi_\gW(W) | A = a']\Big\| + \kappa\Big) \\&\quad \text{(with probability $1 - \delta$, Corollary~(\ref{theorem:E_WA_CME_Rate_for_ATT}))}\\
       &\le {\kappa^3} r_1(\delta, n, \beta_1, p_1) \left(\kappa J_5 \log(4/\delta) \left(\frac{1}{\sqrt{m}}\right)^{\frac{\beta_5 - 1}{\beta_5 + p_3}} + \kappa \right) \quad \text{(with probability $1 - 2\delta$, Theorem (\ref{theorem:E_WA_CME_Rate_for_ATT})}.
    \end{align*}
    For the second component, we note that for $i = 1, \ldots, m$,
    $$
    \begin{aligned}
        \xi_i &:= (\hat{\mu}_{Z | W, A}(\Tilde{w}_i, \Tilde{a}_i) \otimes \phi_\gA(\Tilde{a}_i) \otimes \phi_\gA(a')) \otimes (\hat{\mu}_{Z | W, A}(\Tilde{w}_i, \Tilde{a}_i) \otimes \phi_\gA(\Tilde{a}_i) \otimes \phi_\gA(a')) \\&-  ({\mu}_{Z | W, A}(\Tilde{w}_i, \Tilde{a}_i) \otimes \phi_\gA(\Tilde{a}_i) \otimes \phi_\gA(a')) \otimes ({\mu}_{Z | W, A}(\Tilde{w}_i, \Tilde{a}_i) \otimes \phi_\gA(\Tilde{a}_i) \otimes \phi_\gA(a')) \\ &= ((\hat{\mu}_{Z | W, A}(\Tilde{w}_i, \Tilde{a}_i) - \mu_{Z | W, A}(\Tilde{w}_i, \Tilde{a}_i)) \otimes \phi_\gA(\Tilde{a}_i) \otimes \phi_\gA(a')) \otimes ((\hat{\mu}_{Z | W, A}(\Tilde{w}_i, \Tilde{a}_i)-\mu_{Z | W, A}(\Tilde{w}_i, \Tilde{a}_i)) \otimes \phi_\gA(\Tilde{a}_i) \otimes \phi_\gA(a')) \\ &+ ((\hat{\mu}_{Z | W, A}(\Tilde{w}_i, \Tilde{a}_i) - \mu_{Z | W, A}(\Tilde{w}_i, \Tilde{a}_i)) \otimes \phi_\gA(\Tilde{a}_i) \otimes \phi_\gA(a')) \otimes (\mu_{Z | W, A}(\Tilde{w}_i, \Tilde{a}_i) \otimes \phi_\gA(\Tilde{a}_i) \otimes \phi_\gA(a')) \\ &+ (\mu_{Z | W, A}(\Tilde{w}_i, \Tilde{a}_i) \otimes \phi_\gA(\Tilde{a}_i) \otimes \phi_\gA(a')) \otimes ((\hat{\mu}_{Z | W, A}(\Tilde{w}_i, \Tilde{a}_i)-\mu_{Z | W, A}(\Tilde{w}_i, \Tilde{a}_i)) \otimes \phi_\gA(\Tilde{a}_i) \otimes \phi_\gA(a')),
    \end{aligned}
    $$
and therefore, under Assumption~(\ref{assum:Stage1ConsistencyKernelAssumptions}), by the triangular inequality and by Corollary~(\ref{cor:CME_rate_pointwise})),
    $$
    \begin{aligned}   
        \|\xi_i\|_{op} &\leq \kappa^4 \big\|  \hat{\mu}_{Z| W, A} (\Tilde{w}_j, \Tilde{a}_i)  -  {\mu}_{Z| W, A} (\Tilde{w}_j, \Tilde{a}_i) \big\|_{\gH_{\gZ}}^2 + 2 \kappa^4 \big\|   {\mu}_{Z| W, A} (\Tilde{w}_j, \Tilde{a}_i) \big\|_{\gH_{\gZ}} \big\|  \hat{\mu}_{Z| W, A} (\Tilde{w}_j, \Tilde{a}_i)  -  {\mu}_{Z| W, A} (\Tilde{w}_j, \Tilde{a}_i) \big\|_{\gH_{\gZ}} \\ &\leq  \kappa^8 r_1(\delta, n, \beta_1, p_1)^2 + 2\kappa^6 \big\|   {\mu}_{Z| W, A} (\Tilde{w}_j, \Tilde{a}_i) \big\|_{\gH_{\gZ}} r_1(\delta, n, \beta_1, p_1).
     \end{aligned}
     $$
     with probability at least $1 - \delta$. Also, recall that we observe in the proof of Lemma~(\ref{thm:SecondStageBound3}) that
    \begin{align*}
         \|{\mu}_{Z|W, A}(w, a)\|_{\gH_\gZ} &\le B_1 \kappa^{2+\frac{\beta_1-1}{2}}. 
    \end{align*}
    As a result, we have the following bound with probability at least $1-\delta$,
     \begin{align}
        \|\hat{\Sigma}_{4,m} - \bar{\Sigma}_{4,m}\|_{op} &= \left\| \frac{1}{m} \sum_{i = 1}^m \xi_i\right\|_{op} \leq \frac{1}{m} \sum_{i = 1}^m \left\|\xi_i\right\|_{op}\nonumber\\
        &\leq \kappa^8 r_1(\delta, n, \beta_1, p_1)^2 + 2B_1\kappa^{8+\frac{\beta_1-1}{2}}  r_1(\delta, n, \beta_1, p_1).
    \end{align}
    
\end{proof}

The following Theorem provides convergence rates in RKHS norm for the estimation of the bridge solution with minimum RKHS norm.

\begin{theorem}
    Suppose Assumptions~(\ref{asst:well_specifiedness}-1), (\ref{assum:Stage1ConsistencyKernelAssumptions}), (\ref{asst:src_all_stages}-1), (\ref{asst:evd_all_stages}-1), 
    (\ref{asst:well_specifiedness_ATT}),(\ref{asst:src_all_stages_ATT}), (\ref{asst:evd_all_stages_ATT}), (\ref{assum:C_WA_well_specifiedness}). (\ref{asst:src_ATT_for_E_WA}) hold and set $\lambda_{1} = \Theta \left(n^{-\frac{1}{\beta_1 + p_1}}\right)$, $\zeta = \Theta\left(m^{-\frac{1}{\beta_5 + p_3}}\right)$ and $n = m^{\iota\frac{(\beta_4+1)(\beta_1 + p_1)}{(\beta_4 + p_4)(\beta_1 - 1)}}$ where $\iota > 0$. Then,
    \begin{itemize}
        \item[i.] If $\iota \le \frac{(\beta_5 - 1)(\beta_4 + p_4)}{(\beta_5 + p_3)(\beta_4 + 1)}$ then $\|\hat{\varphi}_{\lambda_2, m} - \bar{\varphi}_0\|_{\gH_{\gZ\gA\gA}} = O_p\left(m^{-\frac{\iota (\beta_4 - 1)}{2(\beta_4 + p_4)}}\right)$ with $\lambda_2 = \Theta\left(m^{\frac{-\iota}{\beta_4 + p_4}}\right)$,
        \item[ii.] If $\iota \ge \frac{(\beta_5 - 1)(\beta_4 + p_4)}{(\beta_5 + p_3)(\beta_4 + 1)}$ then $\|\hat{\varphi}_{\lambda_2, m} - \bar{\varphi}_0\|_{\gH_{\gZ\gA\gA}} = O_p\left(m^{-\frac{(\beta_4 - 1)(\beta_5 - 1)}{2(\beta_4 + 1)(\beta_5 + p_3)}}\right)$ with $\lambda_2 = \Theta\left(m^{\frac{-(\beta_5 - 1)}{(\beta_4 + 1)(\beta_5 + p_3)}}\right)$.
    \end{itemize}
    \label{theorem:ATTBridgeFunctionFinalBound}
\end{theorem}
\begin{proof}
    Let us abbreviate $r_1(n) = r_1(\delta, n, \beta_1, p_1)$. From Lemma~(\ref{thm:SecondStageBound3_ATT}), we obtain with high probability that
\begin{align*}
    \|\hat{\varphi}_{\lambda_2, m} - \bar{\varphi}_0\|_{\gH_{\gZ\gA\gA}} 
    &\leq  \|\bar{\varphi}_{\lambda_2,m} - \bar{\varphi}_0\|_{\gH_{\gZ\gA\gA}} + \frac{\bar{J}}{\lambda_2}r_1(n) \left( 1 + \left(\frac{1}{\sqrt{m}}\right)^{\frac{\beta_5 - 1}{\beta_5 + p_3}}\right) \\&+ \frac{\bar{J}}{\lambda_2}  (1 + r_1(n)) r_1(n)(\|\bar{\varphi}_{\lambda_2,m} - \bar{\varphi}_0 \|_{\gH_{\gZ\gA}} + \|\bar{\varphi}_0\|_{\gH_{\gZ\gA}}),
\end{align*}
where $\bar{J}$ is a constant depending on $\kappa, B_1, \beta_1, J_5$. Furthermore, from Theorem~(\ref{theorem:SecondStageBound2_ATT}),
    \begin{align*}
        &\|\bar{\varphi}_{\lambda_2,m} - \bar{\varphi}_0\|_{\gH_{\gZ\gA\gA}} = O_p\Bigg( r_2(m)\Bigg),
    \end{align*}
    with
    \begin{align*}
       r_2(m) &= \frac{1}{\sqrt{\lambda_2}} \Bigg( \sqrt{\frac{1}{m}\left(\frac{1}{\lambda_2^{p_4}}+\frac{1}{m\lambda_2}\right)}  + \frac{1}{\sqrt{\lambda_2}}\left(\frac{1}{\sqrt{m}}\right)^{\frac{\beta_5 - 1}{\beta_5 + p_3}} + \sqrt{\frac{1}{{\lambda_2 m}}}\Bigg) + \lambda_2^{\frac{\beta_4-1}{2}}\\
       &= \sqrt{\frac{1}{m \lambda^{p_4 + 1}}\left(1+\frac{1}{m\lambda_2^{1 - p_4}}\right)}  + \frac{1}{{\lambda_2}}\left(\frac{1}{\sqrt{m}}\right)^{\frac{\beta_5 - 1}{\beta_5 + p_3}} + {\frac{1}{\lambda_2\sqrt{m}}} + \lambda_2^{\frac{\beta_4-1}{2}}
    \end{align*}
as long as $m \geq 8\kappa^6\log(2/\delta) g_{\lambda_2} \lambda_2$. {Note that the term ${\frac{1}{\lambda_2\sqrt{m}}}$ is of faster order with respect to $\frac{1}{{\lambda_2}}\left(\frac{1}{\sqrt{m}}\right)^{\frac{\beta_5 - 1}{\beta_5 + p_3}}$ and can be removed.} Discarding the other faster terms similar to Theorem (\ref{theorem:ATEBridgeFunctionFinalBound}) and putting it together, we obtain
\begin{align*}
    \|\hat{\varphi}_{\lambda_2, m} - \bar{\varphi}_0\|_{\gH_{\gZ\gA\gA}} = O_p\left( \sqrt{\frac{1}{m \lambda^{p_4 + 1}}\left(1+\frac{1}{m\lambda_2^{1 - p_4}}\right)}  + \frac{1}{{\lambda_2}}\left(\frac{1}{\sqrt{m}}\right)^{\frac{\beta_5 - 1}{\beta_5 + p_3}} + \lambda_2^{\frac{\beta_4-1}{2}} + \frac{1}{\lambda_2}\left(\frac{1}{\sqrt{n}}\right)^{\frac{\beta_1-1}{\beta_1 + p_1}}\right)
\end{align*}
In this proof, similar to Theorem (\ref{theorem:ATEBridgeFunctionFinalBound}) we check, $ \lambda_2^{-1} m^{-1}\log \lambda_{2}^{-1} = O(1)$.

\textbf{Case i.} Let $\iota \le \frac{(\beta_5 - 1)(\beta_4 + p_4)}{(\beta_5 + p_3)(\beta_4 + 1)}$ and $\lambda_2 = m^{\frac{-\iota}{\beta_4 + p_4}}$. We need to check that $ \lambda_2^{-1} m^{-1}\log \lambda_{2}^{-1} = O(1)$. 
\begin{align*}
    \frac{\log \lambda_2^{-1}}{m \lambda_2} = \frac{\iota}{\beta_4 + p_4} \frac{\log m}{m^{1 - \frac{\iota}{\beta_4 + p_4}}}.
\end{align*}
As $\frac{\iota}{\beta_4 + p_4} \le \frac{\beta_5 - 1}{(\beta_5 + p_3)(\beta_4 + 1)} \le 1$, we have $\lambda_2^{-1} m^{-1}\log \lambda_{2}^{-1} \rightarrow 0$ as $m \rightarrow \infty$. Next,
note that with this choice of $\lambda_2$ we have 
\begin{align*}
    \lambda_2^{\frac{\beta_4 - 1}{2}} = m^{-\frac{\iota}{2}\frac{\beta_4 + 1}{\beta_4 + p_4}} = \frac{1}{\lambda_2} n^{-\frac{1}{2}\frac{\beta_1 - 1}{\beta_1 + p_1}}.
\end{align*}
Furthermore,
\begin{itemize}
    \item[i.] \begin{align*}
        \lambda_2^{\beta_4 - 1} \ge \frac{1}{m \lambda_2^{p_4 + 1}} &\iff \lambda_2^{\beta_4 + p_4} \ge m^{-1}\\
        &\iff m^{\iota \frac{\beta_4 + p_4}{\beta_4 + p_4}} \le m \iff \iota \le 1
    \end{align*}
    that is true due to our assumption in this condition $\iota \le \frac{(\beta_5 - 1)(\beta_4 + p_4)}{(\beta_5 + p_3)(\beta_4 + 1)} \le 1$.

    \item[ii.] \begin{align*}
        \lambda_2^{\frac{\beta_4 - 1}{2}} \ge \frac{1}{\lambda_2} m^{-\frac{1}{2}\frac{\beta_5 - 1}{\beta_5 + p_3}} &\iff \lambda_2^{\frac{\beta_4 + 1}{2}} \ge m^{-\frac{1}{2}\frac{\beta_5 - 1}{\beta_5 + p_3}}\\
        &\iff m^{-\frac{\iota}{2}\frac{\beta_4 + 1}{\beta_4 + p_4}} \ge m^{-\frac{1}{2}\frac{\beta_5 - 1}{\beta_5 + p_3}}\\
        &\iff \iota \le \frac{(\beta_4 + p_4)(\beta_5 - 1)}{(\beta_4 + 1)(\beta_5 + p_3)}
    \end{align*}

    \item[iii.]     \begin{align*}
        \frac{1}{m \lambda_2^{1 - p_4}} = \frac{1}{m m^{- \frac{\iota (1 - p_4)}{\beta_4 + p_4}}} = \frac{1}{m^{1 - \frac{\iota (1 - p_4)}{\beta_4 + p_4}}} \le 1
    \end{align*}
    since $\beta_4 + p_4 - \iota(1 - p_4) \ge 0$ due to the fact that $\iota \le \frac{(\beta_5 - 1)(\beta_4 + p_4)}{(\beta_5 + p_3)(\beta_4 + 1)} \le 1$.

    \item[iv.] Also
    \begin{align*}
        \frac{1}{m \lambda_2^{1 - p_4}} \le \frac{1}{m \lambda^{1 + p_4}}.
    \end{align*}
\end{itemize}
Therefore
\begin{align*}
    \|\hat{\varphi}_{\lambda_2, m} - \bar{\varphi}_0\|_{\gH_{\gZ\gA}} = O\left(\lambda_2^{\frac{\beta_4 - 1}{2}}\right) = O\left(m^{-\frac{\iota}{2}\frac{\beta_4 - 1}{\beta_4 + p_4}}\right).
\end{align*}

\textbf{Case ii.} Let $\iota \ge \frac{(\beta_5 - 1)(\beta_4 + p_4)}{(\beta_5 + p_3)(\beta_4 + 1)}$ and $\lambda_2 = m^{\frac{-(\beta_5 - 1)}{(\beta_4 + 1)(\beta_5 + p_3)}}$. Let us first check $\lambda_2^{-1}m^{-1} \log(\lambda_2^{-1}) = O(1)$.
\begin{align*}
    \frac{\log \lambda_2^{-1}}{m \lambda_2} = \frac{\beta_5 - 1}{(\beta_5 + p_3) (\beta_4 + 1)}\frac{\log m}{m^{1 - \frac{\beta_5 -1}{(\beta_5 + p_3)(\beta_4 + 1)}}}.
\end{align*}
As $\frac{\beta_5 - 1}{(\beta_5 + p_3)(\beta_4 + 1)} < 1$, we have $\lambda_2^{-1} m^{-1}\log \lambda_{2}^{-1} \rightarrow 0$ as $m \rightarrow \infty$.
Next, note that with this choice of $\lambda_2$ we have 
\begin{align*}
    \lambda_2^{\frac{\beta_4 - 1}{2}} = \frac{1}{\lambda_2} m^{-\frac{1}{2} \frac{\beta_5 - 1}{\beta_5 + p_3}}. 
\end{align*}

Furthermore,
\begin{itemize}
    \item[i.] \begin{align*}
        \lambda_2^{\beta_4 - 1} \ge \frac{1}{m \lambda_2^{p_4 + 1}} &\iff \lambda_2^{\beta_4 + p_4} \ge m^{-1}\\
        &\iff m^{- \frac{(\beta_5 - 1)(\beta_4 + p _4)}{(\beta_5 + p_3)(\beta_4 + 1)}} \ge m^{-1}
    \end{align*}
    which is true since $\frac{(\beta_5 - 1)(\beta_4 + p _4)}{(\beta_5 + p_3)(\beta_4 + 1)} \le 1$

\item[ii.] \begin{align*}
    \lambda_2^{\frac{\beta_4 - 1}{2}} \ge \frac{1}{\lambda_2} n^{-\frac{1}{2}\frac{\beta_1 - 1}{\beta_1 + p_1}} &\iff \lambda_2^{\frac{\beta_4 + 1}{2}} \ge n^{-\frac{1}{2}\frac{\beta_1 - 1}{\beta_1 + p_1}}\\
    &\iff m^{-\frac{1}{2}\frac{(\beta_5 - 1)(\beta_4 + 1)}{(\beta_5 + p_3)(\beta_4 + 1)}} \ge m^{-\frac{\iota}{2} \frac{(\beta_1 - 1)(\beta_4 + 1)(\beta_1 + p_1)}{(\beta_1 + p_1)(\beta_4 + p_4)(\beta_1 - 1)}}
\end{align*} 
which is true since $\iota \ge \frac{(\beta_5 - 1)(\beta_4 + p_4)}{(\beta_5 + p_3)(\beta_4 + 1)}$.

\item[iii.] \begin{align*}
    \frac{1}{m \lambda_2^{1 - p_4}} = \frac{1}{m m^{\frac{-(\beta_5 - 1)}{(\beta_4 + 1)(\beta_5 + p_3)}}} = \frac{1}{ m^{1 - \frac{(1 - p_4)(\beta_5 - 1)}{(\beta_4 + 1)(\beta_5 + p_3)}}} \le 1
\end{align*} 
since $\frac{(1 - p_4)(\beta_5 - 1)}{(\beta_4 + 1)(\beta_5 + p_3)} \le 1$.
\item[iv.] Also, 
\begin{align*}
    \frac{1}{m \lambda_2^{1 - p_4}} \le \frac{1}{m \lambda_2^{1 + p_4}}.
\end{align*}
\end{itemize}
Hence, 
\begin{align*}
    \|\hat{\varphi}_{\lambda_2, m} - \bar{\varphi}_0\|_{\gH_{\gZ\gA}} = O\left(\lambda_2^{\frac{\beta_4 - 1}{2}}\right) = O\left(m^{-\frac{1}{2}\frac{(\beta_5 - 1)(\beta_4 - 1)}{(\beta_4 + 1)(\beta_5 + p_3)}}\right).
\end{align*}
\end{proof}

\subsubsection{Consistency for Conditional Dose-Response Curve}
\begin{theorem}
    Suppose that Assumptions (\ref{asst:well_specifiedness}), (\ref{assum:Stage1ConsistencyKernelAssumptions}), (\ref{asst:src_all_stages}-1 \& 3) and (\ref{asst:evd_all_stages}-1 \& 3), (\ref{asst:src_all_stages_ATT}), (\ref{asst:evd_all_stages_ATT})  hold and set $\lambda_1 = \Theta \left(n^{-\frac{1}{\beta_1 + p_1}}\right)$, $\lambda_{3} = \Theta \left(t^{-\frac{1}{\beta_3 + p_3}}\right)$, $\zeta = \Theta\left(m^{-\frac{1}{\beta_5 + p_3}}\right)$, and $n = m^{\iota\frac{(\beta_4+1)(\beta_1 + p_1)}{(\beta_4 + p_4)(\beta_1 - 1)}}$ where $\iota > 0$. Then,
    \begin{itemize}
        \item[i.] If $\iota \le \frac{(\beta_5 - 1)(\beta_4 + p_4)}{(\beta_5 + p_3)(\beta_4 + 1)}$ then $\sup_a |\hat{f}_{ATT}(a, a') - f_{ATT}(a, a')| = O_p\left(\sqrt{t}^{-\frac{\beta_3-1}{\beta_3 + p_3}} + m^{-\frac{\iota (\beta_4 - 1)}{2(\beta_4 + p_4)}}\right)$ with $\lambda_2 = \Theta\left(m^{\frac{-\iota}{\beta_4 + p_4}}\right)$,\looseness=-1
        \item[ii.] If $\iota \ge \frac{(\beta_5 - 1)(\beta_4 + p_4)}{(\beta_5 + p_3)(\beta_4 + 1)}$ then $\sup_a |\hat{f}_{ATT}(a, a') - f_{ATT}(a, a')| = O_p\left(\sqrt{t}^{-\frac{\beta_3-1}{\beta_3 + p_3}} + m^{-\frac{(\beta_4 - 1)(\beta_5 - 1)}{2(\beta_4 + 1)(\beta_5 + p_3)}}\right)$ with $\lambda_2 =\Theta\left(m^{\frac{-(\beta_5 - 1)}{(\beta_4 + 1)(\beta_5 + p_3)}}\right)$.\looseness=-1
    \end{itemize}
\end{theorem}
\begin{proof}
For fix $a'$ and for any $a \in \gA$, we apply the following decomposition,
    \begin{align}
        |&\hat{f}_{ATT}(a, a') - f_{ATT}(a, a')| = |\langle \hat{\varphi}_{\lambda_2,m}, \hat{\Psi}(a) \otimes \phi_\gA(a) \otimes \phi_\gA(a')\rangle_{\gH_{\gZ\gA\gA}} - \langle \bar{\varphi}_{0}, \Psi(a) \otimes \phi_\gA(a) \otimes \phi_\gA(a') \rangle_{\gH_{\gZ\gA\gA}}| \nonumber\\
        &= |\langle \hat{\varphi}_{\lambda_2,m}, (\hat{\Psi} - \Psi)(a) \otimes \phi_\gA(a) \otimes \phi_\gA(a')\rangle_{\gH_{\gZ\gA\gA}} + \langle (\hat{\varphi}_{\lambda_2,m} - \bar{\varphi}_{0}), \Psi(a) \otimes \phi_\gA(a) \otimes \phi_\gA(a')\rangle_{\gH_{\gZ\gA\gA}}| \nonumber\\
        &= |\langle \hat{\varphi}_{\lambda_2,m} - \bar{\varphi}_{0}, (\hat{\Psi} - \Psi)(a) \otimes \phi_\gA(a)\otimes \phi_\gA(a')\rangle_{\gH_{\gZ\gA\gA}} + \langle \bar{\varphi}_{0}, (\hat{\Psi} - \Psi)(a) \otimes \phi_\gA(a) \otimes \phi_\gA(a')\rangle_{\gH_{\gZ\gA\gA}} \\&+ \langle (\hat{\varphi}_{\lambda_2,m} - \bar{\varphi}_{0}), \Psi(a) \otimes \phi_\gA(a) \otimes \phi_\gA(a')\rangle_{\gH_{\gZ\gA\gA}}| \nonumber\\
         &\le \kappa^2\left(\|\hat{\varphi}_{\lambda_2,m} - \bar{\varphi}_{0}\|_{\gH_{\gZ\gA}} \|\hat{\Psi}(a) - \Psi(a)\|_{\gH_{\gZ}} + \|\bar{\varphi}_{0}\| \| \hat{\Psi}(a) - \Psi(a) \|_{\gH_{\gZ\gA}} + \| \hat{\varphi}_{\lambda_2,m} - \bar{\varphi}_{0} \|_{\gH_{\gZ\gA}} \|\Psi(a) \|_{\gH_{\gZ}}\right) \label{eq:att_decomposition}
    \end{align}
Note that, under Assumption~(\ref{asst:well_specifiedness}-3),  Assumption~(\ref{assum:Stage1ConsistencyKernelAssumptions}) and Assumption~(\ref{asst:src_all_stages}-3),
    \begin{align*}
        \|\Psi(a)\|_{\gH_{\gZ}} = \|C_{YZ \mid A}\phi_\gA(a)\|_{\gH_{\gZ}} \leq \kappa \|C_{YZ \mid A}\Sigma_3^{-\frac{\beta_3-1}{2}}\Sigma_3^{\frac{\beta_3-1}{2}}\|_{S_2(\gH_{\gA}, \gH_{\gZ})} \leq B_3 \kappa^{1 + \frac{\beta_3-1}{2}} =: \alpha_1.
    \end{align*}
    Furthermore, under Assumption~(\ref{asst:src_all_stages_ATT}), by Proposition~(\ref{prop:min_norm_charac2_ATT}),
    $$
    \|\bar{\varphi}_0\|_{\gH_{\gZ\gA\gA}} = \|\Sigma_2^{\frac{\beta_4-1}{2}}\Sigma_2^{-\frac{\beta_4-1}{2}}\bar{\varphi}_0\|_{\gH_{\gZ\gA\gA}}  \leq B_2 \kappa^{\frac{\beta_4-1}{2}} =: \alpha_2.
    $$
    
    Therefore,
    $$
|\hat{f}_{ATT}(a, a') - f_{ATT}(a, a')| \leq \kappa^2\left(\|\hat{\varphi}_{\lambda_2,m} - \bar{\varphi}_0\|_{\gH_{\gZ\gA}} \|\hat{\Psi}(a) - \Psi(a)\|_{\gH_{\gZ}} + \alpha_2 \| \hat{\Psi}(a) - \Psi(a) \|_{\gH_{\gZ\gA}} + \alpha_1\| \hat{\varphi}_{\lambda_2,m} - \varphi_0 \|_{\gH_{\gZ\gA}}\right) .
$$
As the first term is faster than the two last terms, plugging the results of Theorem~(\ref{theorem:ATTBridgeFunctionFinalBound}) and Corollary~(\ref{cor:CME_rate_pointwise_stage3}), we obtain the final bound.
\end{proof}

\subsection{Additional results}
Here, we present the results that we used in deriving the consistency bounds for dose-response and conditional dose-response estimations. 

The following is a direct consequence of \citet[][Lemma $17$]{fischer2020sobolev}.

\begin{lemma} \label{lemma:steinwart_cov_conc} For any $\delta \in (0,1)$ and $m \geq 8\kappa^4\log(2/\delta) g_{\lambda_2} \lambda_2^{-1}$ where $g_{\lambda_2} =\log \left( 2e\mathcal{N}(\lambda_2) \frac{\|\Sigma_2\|_{op}+\lambda_2}{\|\Sigma_2\|_{op}}\right)$ and under Assumption~\ref{assum:Stage1ConsistencyKernelAssumptions}
$$
\left\|(\Sigma_2 + \lambda_2 I)^{1/2}(\bar{\Sigma}_{2,m} + \lambda_2 I)^{-1}(\Sigma_2 + \lambda_2 I)^{1/2} \right\|_{op} \leq 3,
$$  
\end{lemma}

\begin{proof}
    By \citet[][Lemma $17$]{fischer2020sobolev}, for $\delta \in (0,1), \lambda_2>0$, and $m \geq 1$, the following operator norm bound is satisfied probability not less than $1-\delta,$
    $$
    \left\|(\Sigma_2 + \lambda_2 I)^{-1 / 2}\left(\Sigma_2 - \bar{\Sigma}_{2,m}\right)(\Sigma_2 + \lambda_2 I)^{-1 / 2}\right\| \leq \frac{4\kappa^4 \log(2/\delta) g_{\lambda_2}}{3 m \lambda_2}+\sqrt{\frac{2\kappa^4 \log(2/\delta) g_{\lambda_2}}{m \lambda_2}},
    $$
    where we took $\alpha=1$ in their result since the kernels are bounded. In their notations, for $\alpha=1$ (see \citet[][Eq.~(15) \& Eq.~(16)]{fischer2020sobolev} 
    $$
    \left\|k_\nu^\alpha\right\|_{\infty} = \sup_{(w,a) \in \gW \times \gA} \|\mu_{Z \mid W,A}(w,a) \otimes \phi_{\gA}(a) \|_{\gH_{\gW\gA}} \leq \kappa^2.
    $$
    Therefore, with $m \geq 8\kappa^4 \log(2/\delta) g_{\lambda_2} \lambda_2^{-1}$, 
    $$
    \left\|(\Sigma_2 + \lambda_2 I)^{-1 / 2}\left(\Sigma_2 - \bar{\Sigma}_{2,m}\right)(\Sigma_2 + \lambda_2 I)^{-1 / 2}\right\| \leq \frac{4}{3} \cdot \frac{1}{8}+\sqrt{2 \cdot \frac{1}{8}}=\frac{2}{3}
    $$
    with probability not less than $1-\delta$. Consequently, the inverse of
    $$
    \operatorname{Id}-(\Sigma_2 + \lambda_2 I)^{-1 / 2}\left(\Sigma_2 - \bar{\Sigma}_{2,m}\right)(\Sigma_2 + \lambda_2 I)^{-1 / 2}
    $$
    can be represented by the Neumann series. In particular, the Neumann series gives us the following bound,
    $$
    \begin{aligned}
    & \left\|(\Sigma_2 + \lambda_2 I)^{1/2}(\bar{\Sigma}_{2,m} + \lambda_2 I)^{-1}(\Sigma_2 + \lambda_2 I)^{1/2}\right\|^2 \\
    = & \left\|\operatorname{Id}-(\Sigma_2 + \lambda_2 I)^{-1 / 2}\left(\Sigma_2 - \bar{\Sigma}_{2,m}\right)(\Sigma_2 + \lambda_2 I)^{-1 / 2}\right\|^2 \\
    \leq & \left(\sum_{k=0}^{\infty}\left\|(\Sigma_2 + \lambda_2 I)^{-1 / 2}\left(\Sigma_2 - \bar{\Sigma}_{2,m}\right)(\Sigma_2 + \lambda_2 I)^{-1 / 2}\right\|^k\right)^2 \\
    \leq & \left(\sum_{k=0}^{\infty}\left(\frac{2}{3}\right)^k\right)^2=9
    \end{aligned}
    $$
    \end{proof}
    
\begin{lemma}[Lemma 25~\cite{fischer2020sobolev}] \label{lma:steinwart_lemma}
Let, for $\lambda>0$ and $0 \leq \alpha \leq 1$, the function $f_{\lambda, \alpha}:[0, \infty) \rightarrow \mathbb{R}$ be defined by $f_{\lambda, \alpha}(t):=t^\alpha /(\lambda+t)$. The supremum of $f_{\lambda, \alpha}$ satisfies the following bound
$$
\sup _{t \geq 0} f_{\lambda, \alpha}(t) \leq \lambda^{\alpha-1}.
$$
\end{lemma}

The following result is the direct consequence of
\cite[Theorem 3.5]{pinelis1994optimum}.

\begin{proposition}[Hoeffing's inequality in Hilbert space]
\label{prop:hoeffding}
    Let $\xi_1, \dots, \xi_m$ be independent centered 
    random variables taking values in a separable Hilbert space $\gH$
    such that $\|\xi_i\|_\gH \leq M$ almost surely for all $1 \leq i \leq m$. Then for all $\delta \in (0,1)$, with probability at least $1-\delta$, we have,
    \begin{equation*}
    \left\|\frac{1}{m} \sum_{i=1}^n \xi_i \right\|_\gH \leq M\sqrt{\frac{2\log(2/\delta)}{m}}
    \end{equation*}
\end{proposition}

We obtain the following inequality for non-centered random variables.

\begin{corollary}[Hoeffing's inequality in Hilbert space]
\label{corr:hoeffding}
    Let $\xi_1, \dots, \xi_m$ be independent (not necessarily centered) 
    random variables taking values in a separable Hilbert space $\gH$
    such that $\|\xi_i\|_\gH \leq M$ almost surely for all $1 \leq i \leq m$. Then for all $\delta \in (0,1)$, with probability at least $1-\delta$, we have,
    $$
    \left\|\frac{1}{m}\left(\sum_{i=1}^n \xi_i - \E[\xi_i]\right)\right\|_\gH \leq 2M\sqrt{\frac{2\log(2/\delta)}{m}}
    $$
\end{corollary}
\begin{proof}
    By Jensen's inequality and the triangular inequality, for all $1 \leq i \leq m$,
    $$
    \|\xi_i - \E[\xi_i]\|_\gH \leq \|\xi_i\|_\gH + \|\E[\xi_i]\|_\gH  \leq \|\xi_i\|_\gH + \E[\|\xi_i\|_\gH]\leq 2M,
    $$
    and the result follows from Proposition~(\ref{prop:hoeffding}).
\end{proof}

The following Bernstein's inequality can be found in \citet[Theorem $26$]{fischer2020sobolev}.
\begin{theorem}[Bernstein's inequality in Hilbert space]\label{theo:bernstein}
Let $\xi_1, \dots, \xi_m$ be independent and identically distributed random variables taking values in a separable Hilbert space $\gH$ such that
$$
\mathbb{E}[\|\xi_1\|_{H}^{q}] \leq \frac{1}{2} q ! \sigma^{2} L^{q-2}
$$
for all $q \geq 2$. Then for all $\delta \in (0,1)$, with probability at least $1-\delta$, we have,
$$
\left\|\frac{1}{m} \sum_{i=1}^{m} \xi_i-\mathbb{E} \xi_i\right\|_{H} \leq \log(2/\delta)\sqrt{\frac{32}{m}\left(\sigma^{2}+\frac{L^{2}}{m}\right)}.
$$
\end{theorem}

\section{SUPPLEMENTARY ON NUMERICAL EXPERIMENTS}
\label{sec:SupplementaryNumericalExperiments}
In this section, we provide more detail on the numerical experiments as well as further ablation studies. We first provide details on the kernel function that we used in our experiments and the procedure to tune the regularization parameters. After that, we provide further information about the experimental setups and discuss the complexity of our proposed methods.

\subsection{Kernel}
\label{sec:kernelDetails_Appendix}
In our experiments, we employed the Gaussian kernel function
\begin{align}
    k_\gF(f_i, f_j) = \exp\Bigg(\frac{-\|f_i - f_j\|^2_2}{2 l ^2}\Bigg)
    \label{eq:gaussian_kernel_expression}
\end{align}
for $f_i, f_j \in \R^{d_\gF}$. Gaussian kernel is bounded, continuous and characteristic. The parameter $l$ is called the length scale of the kernel and there is a simple heuristic to determine the length scale called \emph{median length scale heuristic}. In particular, consider the data $\{f_i\}_{i = 1}^n$. Then, we set the length scale squared $l^2$ to the half of the median value of the pairwise squared distances, i.e.,
\begin{align*}
    l^2 = \frac{1}{2}\text{median}(\{\|f_i - f_j\|_2^2 : 1 \le i < j \le n\}).
\end{align*}
This heuristic has also been utilized in causal inference literature, e.g., \citep{RahulKernelCausalFunctions, Mastouri2021ProximalCL, singh2023kernelmethodsunobservedconfounding, xu2024kernelsingleproxycontrol}. 

The Gaussian kernel in Equation (\ref{eq:gaussian_kernel_expression}) can also be considered as multiplication of Gaussian kernels for each dimension, i.e.,
\begin{align}
    k_\gF(f_i, f_j) = \prod_{k = 1}^{d_\gF}\exp\Bigg(\frac{-\|f_i^{(k)} - f_j^{(k)}\|^2_2}{2 l^{{(k)}^2}}\Bigg)
    \label{eq:columnwise_gaussian_kernel_expression}
\end{align}
where $f_i^{(k)}$ is the $k$-th dimension of $f_i \in \R^{d_\gF}$. In that case, each of the length scale in the set $\{l^{(k)}\}_{k = 1}^{d_\gF}$ can be determined by the median distance length-scale heuristic for each dimension separately. We will refer to the kernel in Equation (\ref{eq:columnwise_gaussian_kernel_expression}) as \emph{columnwise Gaussian kernel}. In our synthetic low-dimensional experiment in Section (\ref{sec:DoseResponseExperiments}) we used columnwise Gaussian kernel for outcome proxy variable $W$. For all the other experiments and variables with our proposed method, we used Gaussian kernel.

\subsection{Hyperparameter Selection}
\label{sec:HyperparameterSelection_Appendix}

\subsubsection{Tuning \texorpdfstring{$\lambda_1$}{Lg} and \texorpdfstring{$\lambda_2$}{Lg} Regularization Parameters}
\label{section:LOOCV_KernelRidgeRegression}
{To tune the regularization parameters $\lambda_1$ and $\lambda_3$, we employ leave-one-out cross validation (LOOCV) technique since it has a closed-form expression in the kernel ridge regression setup. The following theorem provides the closed-form expression for the LOOCV loss in kernel ridge regression.

\begin{theorem}[Theorem F.1 in \citep{RahulKernelCausalFunctions}]
    Consider the kernel ridge regression setup from measurements $\{x_i\}_{i = 1}^n$ to the outcomes $\{y_i\}_{i = 1}^n$, and we minimize the regularized squared loss:
\begin{align*}
    \argmin_{f \in \gH_\gX} \gL(f) &= \argmin_{f \in \gH_\gX} \Big[ \frac{1}{n}\sum_{i=1}^n(y_i - \langle f, \phi_\gX(x_i) \rangle_\gH)^2 + \lambda \|f\|_\gH^2\Big]\\
    &= \argmin_{f \in \gH_\gX} \Big[\frac{1}{n} \|\mY - \Phi_\gX^T f\|_2^2 + \lambda \|f\|_\gH^2\Big]
\end{align*}
where $\phi_\gX(.)$ is the canonical feature map for the assumed kernel function $k_\gX(.,.)$, $\gH_\gX$ is the corresponding RKHS, $\mY = \begin{bmatrix}
y_1 & y_2 & \ldots & y_n
\end{bmatrix}^T \in \R^{n}$ and $\Phi_\gX = \begin{bmatrix}
\phi_\gX(x_1) & \phi_\gX(x_2) & \ldots & \phi_\gX(x_n)
\end{bmatrix}$. Then, the LOOCV loss is given by
\begin{align}
\text{LOOCV}_f(\lambda) &= \frac{1}{n} \|\Tilde{\mH}_\lambda^{-1} \mH_\lambda \mY\|_2^2
    \label{eq:KRR_LOOCV}
\end{align}
where 
\begin{align*}
    \mH_\lambda &= \mI - \mK_{\mX \mX} (\mK_{\mX\mX} + n \lambda \mI)^{-1} \in \R^{n \times n}\\
    \Tilde{\mH}_\lambda &= \text{diag}(\mH_\lambda) \in \R^{n \times n}
\end{align*}
\end{theorem}
The proof of this theorem can be found in \citep{RahulKernelCausalFunctions} (see Algorithm F.1 in that paper). As a result o this theorem, one can tune the hyperparameter $\lambda$ of kernel ridge regression over a grid $\Lambda \subset \R$, i.e.,
\begin{align*}
    \lambda^* = \argmin_{\Lambda \subset \R} \frac{1}{n} \|\Tilde{\mH}_\lambda^{-1} \mH_\lambda \mY\|_2^2.
\end{align*}

\paragraph{Tuning of First Stage Regression Regularization ($\lambda_1$):} Recall that in first-stage regression, we solve the following optimization problem (see derivation of Algorithm (\ref{algo:ATE_algorithm_with_confounders}) in S.M. (\ref{sec:ATE_algorithm_proof}):
\begin{align*}
    \hat{\gL}^{c}(C) = \frac{1}{n} \sum_{i = 1}^n \|\phi_\gZ(z_i) - C(\phi_\gW(w_i) \otimes \phi_\gA(a_i))\|_{\gH_\gZ}^2 + \lambda_1 \|C\|_{S_2(\gH_\gW \otimes \gH_\gA, \gH_\gZ)}^2.
\end{align*}
This is a kernel ridge regression from the (infinite dimensional) measurements $\{\phi_\gW(w_i) \otimes \phi_\gA(a_i)\}_{i = 1}^n$ to the (infinite dimensional) outcomes $\{\phi_\gZ(z_i)\}_{i = 1}^n$. Since the kernel function for the tensor product features $\phi_\gW(w_i) \otimes \phi_\gA(a_i)$ is $k_\gW(w_i, .) k_\gA(a_i, .)$, we can write
\begin{align*}
    \mH_{\lambda_1} &= \mI - (\mK_{\mW \mW} \odot \mK_{\mA \mA}) (\mK_{\mW \mW} \odot \mK_{\mA \mA} + n \lambda_1 \mI)^{-1} \in \R^{n \times n}\\
    \Tilde{\mH}_{\lambda_1} &= \text{diag}(\mH_{\lambda_1}) \in \R^{n \times n}.
\end{align*}
Furthermore, we see that the LOOCV can be written as
\begin{align}
    \text{LOOCV}_C(\lambda) &= \frac{1}{n} \|\Tilde{\mH}_{\lambda_1}^{-1} \mH_{\lambda_1} \Phi_\gZ^T\|_2^2 \nonumber\\
    &= \frac{1}{n} \text{Tr}(\Tilde{\mH}_{\lambda_1}^{-1} \mH_{\lambda_1} \Phi_\gZ^T \Phi_\gZ \mH_{\lambda_1}^T \Tilde{\mH}_{\lambda_1}^{-T})\nonumber\\
    &= \frac{1}{n} \text{Tr}(\Tilde{\mH}_{\lambda_1}^{-1} \mH_{\lambda_1} \mK_{\mZ \mZ} \mH_{\lambda_1}^T \Tilde{\mH}_{\lambda_1}^{-T}). \label{eq:LOOCV_lambda}
\end{align}
Hence, we can tune $\lambda_1$ over a grid $\Lambda_1 \subset \R$ that minimizes the LOOCV loss in Equation (\ref{eq:LOOCV_lambda}). In each of our experiments, we generated the grid $\Lambda_1$ with \emph{logspace} with maximum and minimum values of $1.0$ and $10^{-7}$, respectively, and we used $150$ grid points. This procedure applies to both Algorithms for dose-response curve and conditional dose-response curve estimations.

\paragraph{Tuning of Third Stage Regression Regularization ($\lambda_3$):} Recall that in order to estimate the causal functions after second-stage regression, we solve the following optimization problem (see derivation of Algorithm (\ref{algo:ATE_algorithm_with_confounders}) in S.M. (\ref{sec:ATE_algorithm_proof}):
\begin{align*}
    \hat{C}_{YZ|A} &= \argmin_{C} \frac{1}{n} \sum_{i = 1}^n \|y_i \phi_\gZ(z_i) - C \phi_\gA(a_i)\|^2 + \lambda_2 \|C\|^2 \\
    &= \argmin_{C} \frac{1}{n}\|\Phi_\gZ \text{diag}(\mY) - C \Phi_\gA\|^2 + \lambda_3 \|C\|^2.
\end{align*}
Again, this is kernel ridge regression from (infinite dimensional) measurement $\{\phi_\gA(a_i)\}_{i = 1}^n$ to the (infinite dimensional) outcomes $y_i \phi_\gZ(z_i)$. Then, we construct
\begin{align*}
        \mH_{\lambda_3} &= \mI - \mK_{\mA \mA} (\mK_{\mA \mA} + n \lambda_3 \mI)^{-1} \in \R^{n \times n}\\
    \Tilde{\mH}_{\lambda_3} &= \text{diag}(\mH_{\lambda_3}) \in \R^{n \times n}.
\end{align*}
Therefore, LOOCV loss can be written as 
\begin{align}
    \text{LOOCV}_F(\lambda_3) &= \frac{1}{n} \|\Tilde{\mH}_{\lambda_3}^{-1} \mH_{\lambda_3} (\Phi_\gZ \text{diag}(\mY))^T\|_2^2 \nonumber\\
&= \frac{1}{n} \|\Tilde{\mH}_{\lambda_3}^{-1} \mH_{\lambda_3} \text{diag}(\mY)^T \Phi_\gZ^T\|_2^2 \nonumber\\
&= \frac{1}{n} \text{Tr}\Big(\Tilde{\mH}_{\lambda_3}^{-1} \mH_{\lambda_3} \text{diag}(\mY)^T \Phi_\gZ^T \Phi_\gZ \text{diag}(\mY) \mH_{\lambda_3}^T \Tilde{\mH}_{\lambda_3}^{-T}\Big) \nonumber   \\
&= \frac{1}{n} \text{Tr}\Big(\Tilde{\mH}_{\lambda_3}^{-1} \mH_{\lambda_3} (\mK_{\mZ \mZ} \odot \mY \mY^T) \mH_{\lambda_3}^T \Tilde{\mH}_{\lambda_3}^{-T}\Big) \label{eq:LOOCV_lambda2}
\end{align}
As a result, we can tune $\lambda_3$ over a grid $\Lambda_3 \subset \R$ that minimizes the LOOCV loss in Equation (\ref{eq:LOOCV_lambda2}). In each of our experiments, we generated the grid $\Lambda_3$ with \emph{logspace} with maximum and minimum values of $1.0$ and $10^{-7}$, respectively, and we used $150$ grid points. This procedure applies to both Algorithms for dose-response curve and conditional dose-response curve estimations.
}

\subsubsection{Tuning \texorpdfstring{$\lambda_2$}{Lg} 
\label{sec:SecondStageRegularizationTuning}
Regularization Parameter in ATE Algorithm}

Recall that in the second-stage regression, we use the stage 2 data $\{\Tilde{z}_i, \Tilde{w}_i, \Tilde{a}_i\}_{i = 1}^n$. We can estimate the out-of-sample loss of the second-stage regression using the data from  first-stage $\{z_i, w_i, a_i\}_{i=1}^n$ in order to tune the regularization parameter $\lambda_2$. Then, the out-of-sample loss can be expressed as follows:
\begin{align}
    \hat{\gL}^{\text{Val}}(\varphi) &= \frac{1}{n} \sum_{i = 1}^n \langle \varphi, \hat{\mu}_{Z|W, X, A} ({w}_i, {a}_i) \otimes \phi_\gA({a}_i) \rangle_{\gH_\gZ \otimes \gH_\gA}^2 \\&-2 \frac{1}{n ( n-1)} \sum_{i = 1}^n \sum_{\substack{j = 1 \\ j \ne i}}^n \Big\langle \varphi, \hat{\mu}_{Z|W, X, A} ({w}_j, {a}_i) \otimes \phi_\gA({a}_i) \Big\rangle_{\gH_\gZ \otimes \gH_\gA}
\end{align}
where
\begin{align*}
    {\varphi = \sum_{i = 1}^m \alpha_i  \hat{\mu}_{Z|W, A} (\Tilde{w}_i, \Tilde{a}_i) \otimes  \phi_\gA(\Tilde{a}_i) + \frac{\alpha_{m + 1}}{m (m - 1)} \sum_{j = 1}^m \sum_{\substack{l = 1 \\ l \ne j}}^m \hat{\mu}_{Z|W, A} (\Tilde{w}_l, \Tilde{a}_j) \otimes \phi_\gA(\Tilde{a}_j)},
\end{align*}
and the set $\{\alpha_i\}_{i = 1}^{m + 1}$ are the optimizer of the loss function in Equation (\ref{eq:AlternativeProxyMethodProofLossMatrixVectorProductFormWithConfounder}). Now, let us compute the closed-form expression for out-of-loss. First, consider the following inner product:
\begin{align*}
    &\langle \varphi, \hat{\mu}_{Z | W, A}(w_i, a_i) \otimes \phi_\gA(a_i)\rangle \\&= \sum_{l = 1}^m \alpha_l \langle \hat{\mu}_{Z|W, A} (\Tilde{w}_l, \Tilde{a}_l) \otimes  \phi_\gA(\Tilde{a}_l), \hat{\mu}_{Z | W, A}(w_i, a_i) \otimes \phi_\gA(a_i) \rangle\\
    &+ \frac{\alpha_{m + 1}}{m (m - 1)} \sum_{j = 1}^m \sum_{\substack{l = 1 \\ l \ne j}}^m \langle \hat{\mu}_{Z|W, A} (\Tilde{w}_l, \Tilde{a}_j) \otimes \phi_\gA(\Tilde{a}_j), \hat{\mu}_{Z | W, A}(w_i, a_i) \otimes \phi_\gA(a_i)\rangle\\
    &= \sum_{l = 1}^m \alpha_l \langle \hat{\mu}_{Z|W, A} (\Tilde{w}_l, \Tilde{a}_l), \hat{\mu}_{Z | W, A}(w_i, a_i) \rangle \langle \phi_\gA(\Tilde{a}_l), \phi_\gA(a_i) \rangle\\
    &+ \frac{\alpha_{m + 1}}{m (m - 1)} \sum_{j = 1}^m \sum_{\substack{l = 1 \\ l \ne j}}^m \langle \hat{\mu}_{Z|W, A} (\Tilde{w}_l, \Tilde{a}_j), \hat{\mu}_{Z | W, A}(w_i, a_i) \rangle \langle \phi_\gA(\Tilde{a}_j), \phi_\gA(a_i)\rangle\\
    &= \sum_{l = 1}^m \alpha_l \mBeta(\Tilde{w}_l, \Tilde{a}_l)^\top \mK_{Z Z} \mBeta(w_i, a_i) k_\gA(\Tilde{a}_l, a_i)\\
    &+ \frac{\alpha_{m + 1}}{m} \sum_{j = 1}^m \Bigg(\sum_{\substack{l = 1 \\ l \ne j}}^m \frac{1}{m -1 } \mBeta(\Tilde{w}_l, \Tilde{a}_j) \Bigg)^\top \mK_{Z Z} \mBeta(w_i, a_i) k_\gA(\Tilde{a}_j, a_i)\\
    &=\Bigg[ \Bigg( \mC^\top \mK_{Z Z} \mB \odot \mK_{A \Tilde{A}}\Bigg) \alpha_{1:m} + \alpha_{m + 1}\Bigg( \mC^\top \mK_{Z Z} \Bar{B} \odot \mK_{A \Tilde{A}}\Bigg) \frac{\vone}{m}\Bigg]_i
\end{align*}
where
\begin{align*}
    \mC = \Big( \mK_{W W} \odot \mK_{A A} + n \lambda_1 \mI\Big)^{-1} (\mK_{W W} \odot \mK_{X X} \odot \mK_{A A}).
\end{align*}
As a result,
\begin{align*}
    &\frac{1}{n} \sum_{i = 1}^n \langle \varphi, \hat{\mu}_{Z|W, A} ({w}_i, {a}_i) \otimes \phi_\gA({a}_i) \rangle_{\gH_\gZ \otimes \gH_\gA}^2\\
    &= \frac{1}{n} \begin{bmatrix}
        \alpha_{1:m} \\ \alpha_{m + 1}
    \end{bmatrix}
    \begin{bmatrix}
        \mB^T \mK_{Z Z}\mC \odot \mK_{\Tilde{A} {A}} \\ (\frac{\vone}{m})^T \big(\bar{\mB}^T \mK_{Z Z} {\mC} \odot \mK_{\Tilde{A}{A}}\big)
    \end{bmatrix} 
    \begin{bmatrix}
        \mC^T \mK_{Z Z}\mB \odot \mK_{{A}\Tilde{A}} & \big({\mC}^T \mK_{Z Z} \bar{\mB} \odot \mK_{{A}\Tilde{A}}\big) \frac{\vone}{m}
    \end{bmatrix} 
    \begin{bmatrix}
        \alpha_{1:m}\\ \alpha_{m+1}
    \end{bmatrix}
\end{align*}

Secondly, consider the following sum of inner products:
\begin{align*}
    &\frac{1}{n ( n-1)} \sum_{i = 1}^n \sum_{\substack{j = 1 \\ j \ne i}}^n \Big\langle \varphi, \hat{\mu}_{Z|W, A} ({w}_j, {a}_i) \otimes \phi_\gA({a}_i) \Big\rangle_{\gH_\gZ \otimes \gH_\gA}\\
    &=\frac{1}{n (n-1)} \sum_{i = 1}^n \sum_{\substack{j = 1 \\ j \ne i}}^n \sum_{l = 1}^m \alpha_l \Big\langle  \hat{\mu}_{Z|W, A} (\Tilde{w}_l, \Tilde{a}_l) \otimes  \phi_\gA(\Tilde{a}_l), \hat{\mu}_{Z|W, A} ({w}_j, {a}_i) \otimes \phi_\gA({a}_i) \Big\rangle_{\gH_\gZ \otimes \gH_\gA}
    \\&+ \frac{\alpha_{m + 1}}{ m n (m - 1) ( n-1)} \sum_{i = 1}^n \sum_{\substack{j = 1 \\ j \ne i}}^n \sum_{r = 1}^m \sum_{\substack{s = 1 \\ l \ne r}}^m \Big\langle \hat{\mu}_{Z|W, A} (\Tilde{w}_s, \Tilde{a}_r) \otimes \phi_\gA(\Tilde{a}_r), \hat{\mu}_{Z|W, A} ({w}_j, {a}_i) \otimes \phi_\gA({a}_i) \Big\rangle_{\gH_\gZ \otimes \gH_\gA}\\
    &=\frac{1}{n ( n-1)} \sum_{i = 1}^n \sum_{\substack{j = 1 \\ j \ne i}}^n \sum_{l = 1}^m \alpha_l \mBeta(\Tilde{w}_l, \Tilde{a}_l)^\top \mK_{Z Z} \mBeta (w_j, a_i) k_\gA(\Tilde{a}_l, a_i)\\
    &+ \frac{\alpha_{m + 1}}{ m n (m - 1) ( n-1)} \sum_{i = 1}^n \sum_{\substack{j = 1 \\ j \ne i}}^n \sum_{r = 1}^m \sum_{\substack{s = 1 \\ l \ne r}}^m \mBeta(\Tilde{w}_s, \Tilde{a}_r)^\top \mK_{Z Z} \mBeta(w_j, a_i) k_\gA(\Tilde{a}_r, a_i)\\
    &= \frac{1}{n} \sum_{i = 1}^n \sum_{l = 1}^m \alpha_l \mBeta(\Tilde{w}_l, \Tilde{a}_l)^\top \mK_{Z Z} \Bigg(\frac{1}{n - 1} \sum_{\substack{j = 1 \\ j \ne i}}^n \mBeta (w_j, a_i) \Bigg) k_\gA(\Tilde{a}_l, a_i)\\
    &+  \frac{\alpha_{m + 1}}{ m n} \sum_{i = 1}^n \sum_{r = 1}^m \Bigg(\frac{1}{m - 1}\sum_{\substack{s = 1 \\ l \ne r}}^m \mBeta(\Tilde{w}_s, \Tilde{a}_r) \Bigg)^\top \mK_{Z Z} \Bigg(\sum_{\substack{j = 1 \\ j \ne i}}^n \mBeta(w_j, a_i) \Bigg) k_\gA(\Tilde{a}_r, a_i)\\
    &= \frac{1}{n} \alpha_{1:m}^T \big( \mB^\top \mK_{Z Z} \bar{\mC} \odot \mK_{\Tilde{A} A}\big) \vone + \alpha_{m + 1} \frac{1}{n m } \vone^\top \big(\bar{\mB}^\top \mK_{Z Z} \bar{\mC} \odot \mK_{\Tilde{A} A}) \vone\\
    &= \begin{bmatrix}
        \alpha_{1:m} \\ \alpha_{m + 1}
    \end{bmatrix}^\top 
    \begin{bmatrix}
        \big( \mB^\top \mK_{Z Z} \bar{\mC} \odot \mK_{\Tilde{A} A}\big) \frac{\vone}{n}\\
        (\frac{\vone}{m})^\top \big(\bar{\mB}^\top \mK_{Z Z} \bar{\mC} \odot \mK_{\Tilde{A} A}) \frac{\vone}{n}
    \end{bmatrix}
\end{align*}
where $\bar{\mC}$ is the matrix whose $j$-th column is given by
\begin{align*}
    \bar{\mC} = \frac{1}{n} \sum_{\substack{l = 1 \\ l \ne j}}^n\Big( \mK_{W W} \odot \mK_{X X} \odot \mK_{A A} + n \lambda_1 \mI\Big)^{-1} (\mK_{W w_l} \odot \mK_{X x_l} \odot \mK_{A a_j}).
\end{align*}
As a result, we can write the hold-out sample loss as:
\begin{align}
    \hat{\gL}^{\text{Val}}(\varphi) &= \frac{1}{n} \sum_{i = 1}^n \langle \varphi, \hat{\mu}_{Z|W, A} ({w}_i, {a}_i) \otimes \phi_\gA({a}_i) \rangle_{\gH_\gZ \otimes \gH_\gA}^2 \nonumber\\&-2 \frac{1}{n ( n-1)} \sum_{i = 1}^n \sum_{\substack{j = 1 \\ j \ne i}}^n \Big\langle \varphi, \hat{\mu}_{Z|W, A} ({w}_j, {a}_i) \otimes \phi_\gA({a}_i) \Big\rangle_{\gH_\gZ \otimes \gH_\gA}\nonumber\\
    &=\frac{1}{n} \begin{bmatrix}
        \alpha_{1:m} \\ \alpha_{m + 1}
    \end{bmatrix}
    \begin{bmatrix}
        \mB^T \mK_{Z Z}\mC \odot \mK_{\Tilde{A} {A}} \\ (\frac{\vone}{m})^T \big(\bar{\mB}^T \mK_{Z Z} {\mC} \odot \mK_{\Tilde{A}{A}}\big)
    \end{bmatrix} 
    \begin{bmatrix}
        \mC^T \mK_{Z Z}\mB \odot \mK_{{A}\Tilde{A}} & \big({\mC}^T \mK_{Z Z} \bar{\mB} \odot \mK_{{A}\Tilde{A}}\big) \frac{\vone}{m}
    \end{bmatrix} 
    \begin{bmatrix}
        \alpha_{1:m}\\ \alpha_{m+1}
    \end{bmatrix}\nonumber
    \\& -2\begin{bmatrix}
        \alpha_{1:m} \\ \alpha_{m + 1}
    \end{bmatrix}^\top 
    \begin{bmatrix}
        \big( \mB^\top \mK_{Z Z} \bar{\mC} \odot \mK_{\Tilde{A} A}\big) \frac{\vone}{n}\\
        (\frac{\vone}{m})^\top \big(\bar{\mB}^\top \mK_{Z Z} \bar{\mC} \odot \mK_{\Tilde{A} A}) \frac{\vone}{n}
    \end{bmatrix}
    \label{eq:secondstage-holdout_loss}
\end{align}

One can choose the regularization parameter $\lambda_2$ that will minimize $\hat{\gL^{\text{Val}}}$ in Equation (\ref{eq:secondstage-holdout_loss}). However, even though validation error is an estimator of the test error, its variance may cause overfitting \citep{pmlr-v151-meanti22a}. Hence, we will propose a similar method to \citep{pmlr-v151-meanti22a} that will avoid overfitting via utilizing the additional complexity regularization cost. Now, recall our original population level cost function for the second-stage (before simplification):
\begin{align}
    &{\gL}^{2SR}(\varphi) = \E\big[\big(r(W, A) - \E[\varphi(Z, A) | W, A]\big)^2\big] + \lambda_2 \|\varphi\|^2_{\gH_\gZ \otimes \gH_\gA}
    \label{eq:secondStagePopulationLoss}
\end{align}
Recall that our analysis has shown that $\varphi$ must be in the form of 
\begin{align*}
    {\varphi = \sum_{i = 1}^m \alpha_i  \hat{\mu}_{Z|W, A} (\Tilde{w}_i,\Tilde{a}_i) \otimes  \phi_\gA(\Tilde{a}_i) + \alpha_{m + 1} \frac{1}{m (m - 1)} \sum_{j = 1}^m \sum_{\substack{l = 1 \\ l \ne j}}^m \hat{\mu}_{Z|W, A} (\Tilde{w}_l, \Tilde{a}_j) \otimes \phi_\gA(\Tilde{a}_j)},
\end{align*}
if we optimize the corresponding sample loss in Equation (\ref{eq:2StageRegressionObjectiveWithConfoundersFinal}). Now, suppose that we observe the noisy version of the target $r(W, A)$ (even though we do not observe the density-ratios, for theoretical analysis here we can see that Equation (\ref{eq:secondStagePopulationLoss}) is regression on the target variable $r(W, A)$). We write the sample-based counterpart of the loss in Equation (\ref{eq:secondStagePopulationLoss}) as follows:
\begin{align}
    &\hat{\gL}^{\text{2SR}, \epsilon}_m(\varphi) = \frac{1}{m} \|\mR^\epsilon - \mL \alpha\|^2 + \lambda_2 \alpha^T \mN \alpha
\end{align}
where $\mR^\epsilon = \mR + \rvepsilon$ with $\text{Var}(\epsilon_i) = \sigma^2$ and $\E[\epsilon_i] = 0$ for all $i$, and 
\begin{align*}
    \mL &= \begin{bmatrix}
        \mB^T \mK_{Z Z}\mB \odot \mK_{\Tilde{A}\Tilde{A}} & \big[{\mB}^T \mK_{Z Z} \bar{\mB} \odot \mK_{\Tilde{A}\Tilde{A}}\big] \frac{\vone}{m}
    \end{bmatrix} \\
    \mN &= \begin{bmatrix}
        \mB^T \mK_{Z Z} \mB \odot \mK_{\Tilde{A} \Tilde{A}} & [ \mB^T \mK_{Z Z} \Bar{\mB} \odot \mK_{\Tilde{A} \Tilde{A}} ] \frac{\vone}{m}\\
        (\frac{\vone}{m})^T [ \mB^T \mK_{Z Z} \Bar{\mB} \odot \mK_{\Tilde{A} \Tilde{A}} ]^T & (\frac{\vone}{m})^T \Big[ \Bar{\mB}^T \mK_{Z Z} \Bar{\mB} \odot \mK_{\Tilde{A} \Tilde{A}}\Big] \frac{\vone}{m}
    \end{bmatrix}.
\end{align*}
Also, consider the noiseless case
\begin{align*}
    &\hat{\gL}^{\text{2SR}}_m(\varphi) = \frac{1}{m} \|\mR - \mL \alpha\|^2 + \lambda_2 \alpha^T \mN \alpha
\end{align*}
The optimum of $\hat{\gL}^{\text{2SR}, \epsilon}(\varphi)$ can be written as 
\begin{align*}
    \alpha = \argmin_{\alpha} \hat{\gL}^{\text{2SR}, \epsilon}_m(\varphi) &= \big( \mL^T \mL + m \lambda_2 \mI\big)^{-1} \mL^T \mR^{\epsilon}\\
    &= \big( \mL^T \mL + m \lambda_2 \mI\big)^{-1} \mL^T (\mR + \rvepsilon).
\end{align*}
Now, consider the following expectation
\begin{align*}
    \E \Big[ \frac{1}{m} \|\mR^\epsilon - \mL \alpha\|^2  \Big] &= \E \Big[ \frac{1}{m} \|\mR - \mL \alpha\|^2  \Big] + \frac{2}{m} \E \Big[ \langle \mR - \mL \alpha, \rvepsilon\rangle \Big] + \sigma^2\\
    &=\E \Big[ \frac{1}{m} \|\mR - \mL \alpha\|^2  \Big] + \frac{2}{m} \E \Big[ \langle \mR - \mL \alpha, \rvepsilon\rangle \Big] + \sigma^2\\
    &=\E \Big[ \frac{1}{m} \|\mR - \mL \alpha\|^2  \Big] - \frac{2}{m} \E \Big[ \langle \mL \big( \mL^T \mL + m \lambda_2 \mI\big)^{-1} \mL^T (\mR + \rvepsilon), \rvepsilon\rangle \Big] + \sigma^2\\
    &=\E \Big[ \frac{1}{m} \|\mR - \mL \alpha\|^2  \Big] - \frac{2}{m} \E \Big[ \langle \mL \big( \mL^T \mL + m \lambda_2 \mI\big)^{-1} \mL^T \rvepsilon, \rvepsilon\rangle \Big] + \sigma^2\\
    &=\E \Big[ \frac{1}{m} \|\mR - \mL \alpha\|^2  \Big] - \frac{2 \sigma^2}{m} \text{Tr} \Big( \big( \mL^T \mL + m \lambda_2 \mI\big)^{-1} \mL^T \mL \Big) + \sigma^2.
\end{align*}
Hence,
\begin{align*}
    \E \Big[ \frac{1}{m} \|\mR - \mL \alpha\|^2  \Big] = \E \Big[ \frac{1}{m} \|\mR^\epsilon - \mL \alpha\|^2  \Big] + \frac{2 \sigma^2}{m} \text{Tr} \Big( \big( \mL^T \mL + m \lambda_2 \mI\big)^{-1} \mL^T \mL \Big) - \sigma^2.
\end{align*}
Note that the term $\E \Big[ \frac{1}{m} \|\mR - \mL \alpha\|^2  \Big]$ is the \emph{expected risk} in the noise free setup. Furthermore, the term $\frac{2 \sigma^2}{m} \text{Tr} \Big( \big( \mL^T \mL + m \lambda_2 \mI\big)^{-1} \mL^T \mL \Big)$ is referred as \emph{degrees of freedom} (or complexity) of the estimator, and it is a positive scalar. It penalizes the complex estimators. Having derived this, now consider the validation loss $\hat{\gL}^{\text{Val}}(\varphi)$ that we have previously derived. As we pointed out earlier, optimizing the regularization parameter $\lambda_2$ with respect to this validation loss can still lead to overfitting due to its variance. Hence, we propose to optimize the regularization parameter $\lambda_2$ with respect to the following surrogate cost that is both an upper bound on the validation error and penalizes the overly complex models:
\begin{align*}
    \hat{\gL}^{\text{Val}}(\varphi) \le \hat{\gL}^{\text{Val}}_{\sigma^2}(\varphi) \quad \forall \sigma \ge 0,
\end{align*}
where 
\begin{align}
   \hat{\gL}^{\text{Val}}_{\sigma^2}(\varphi)& =  \frac{1}{n} \begin{bmatrix}
        \alpha_{1:m} \\ \alpha_{m + 1}
    \end{bmatrix}
    \begin{bmatrix}
        \mB^T \mK_{Z Z}\mC \odot \mK_{\Tilde{A} {A}} \\ (\frac{\vone}{m})^T \big(\bar{\mB}^T \mK_{Z Z} {\mC} \odot \mK_{\Tilde{A}{A}}\big)
    \end{bmatrix} 
    \begin{bmatrix}
        \mC^T \mK_{Z Z}\mB \odot \mK_{{A}\Tilde{A}} & \big({\mC}^T \mK_{Z Z} \bar{\mB} \odot \mK_{{A}\Tilde{A}}\big) \frac{\vone}{m}
    \end{bmatrix} 
    \begin{bmatrix}
        \alpha_{1:m}\\ \alpha_{m+1}
    \end{bmatrix}\nonumber
    \\& -2\begin{bmatrix}
        \alpha_{1:m} \\ \alpha_{m + 1}
    \end{bmatrix}^\top 
    \begin{bmatrix}
        \big( \mB^\top \mK_{Z Z} \bar{\mC} \odot \mK_{\Tilde{A} A}\big) \frac{\vone}{n}\\
        (\frac{\vone}{m})^\top \big(\bar{\mB}^\top \mK_{Z Z} \bar{\mC} \odot \mK_{\Tilde{A} A}) \frac{\vone}{n}
    \end{bmatrix} + \frac{2 \sigma^2}{m} \text{Tr} \Big( \big( \mL^T \mL + m \lambda_2 \mI\big)^{-1} \mL^T \mL \Big)
    \label{eq:SecondStageTuningSurrogateLoss}
\end{align}
Hence, we can tune $\lambda_2$ over a grid $\Lambda_2 \subset \R$ that minimizes the surrogate loss in Equation (\ref{eq:SecondStageTuningSurrogateLoss}). One drawback of this approach is that we need to either estimate $\sigma^2$ or treat it as another hyperparameter. In our experiments, we opted to treat $\sigma^2$ as a hyperparameter. For the synthetic low-dimensional data and the legalized abortion and crime dataset, we set $\sigma^2 = 1$. In our dSprite and grade retention experiments, we used $\sigma^2 = 3$.

\subsubsection{Tuning \texorpdfstring{$\lambda_2$}{Lg} 
Regularization Parameter in ATT Algorithm}
\label{sec:TuningSecondStageRegularizationATT}
Recall that in the second-stage regression for ATT estimation, we minimize the following loss:
\begin{align*}
\hat{\mathcal{L}}^{\text{2SR}}_m(\varphi) &= \frac{1}{m} \sum_{i = 1}^m \langle \varphi, \hat{\mu}_{Z|W, A} (\Tilde{w}_i, \Tilde{a}_i) \otimes \phi_\gA(\Tilde{a}_i) \otimes \phi_\gA(a') \rangle^2 \nonumber\\ &-2 \frac{1}{m} \sum_{j = 1}^m \sum_{\substack{i = 1 \\ i \ne j}}^m \langle \varphi,  \theta_i \mu_{Z | W, A}(\Tilde{w}_i, \Tilde{a}_j) \otimes \phi_\gA(\Tilde{a}_j) \otimes \phi_\gA(a') \rangle + \lambda_2 \|\varphi\|_{\gH_\gZ \otimes \gH_\gA \otimes \gH_\gA}^2.
\end{align*}
We can compute the validation loss using the first-stage data $\{z_i, w_i, a_i\}_{i = 1}^n$ similar to ATE algorithm with the following expression
\begin{align}
    \mathcal{L}^{\text{Val}}(\varphi) &= \frac{1}{n} \sum_{i = 1}^n \langle \varphi, \hat{\mu}_{Z|W, A} ({w}_i, {a}_i) \otimes \phi_\gA({a}_i) \otimes \phi_\gA(a') \rangle^2 \nonumber\\ &-2 \frac{1}{n} \sum_{j = 1}^n \sum_{\substack{i = 1 \\ i \ne j}}^n \langle \varphi,  \theta^{(2)}_i \hat{\mu}_{Z | W, A}({w}_i, {a}_j) \otimes \phi_\gA({a}_j) \otimes \phi_\gA(a') \rangle
    \label{eq:ATTSecondStageValidationLoss}
\end{align}
where $\theta^{(2)}_i = [(\mK_{{A} {A}} + n \zeta^{(2)} \mI)^{-1} \mK_{{A} a'}]_i$. Furthermore, recall that the expression for $\varphi$ that we have is
\begin{align*}
    {\varphi = \sum_{i = 1}^m \alpha_i  \hat{\mu}_{Z|W, A} (\Tilde{w}_i, \Tilde{a}_i)\otimes \phi_\gA(\Tilde{a}_i) \otimes \phi_\gA(a') + \frac{\alpha_{m + 1}}{m} \sum_{j = 1}^m \sum_{\substack{l = 1 \\ l \ne j}}^m \theta_l \hat{\mu}_{Z|W, A} (\Tilde{w}_l, \Tilde{a}_j) \otimes \phi_\gA(\Tilde{a}_j) \otimes \phi_\gA(a')}.
\end{align*}
where the set $\{\alpha_i\}_{i = 1}^{m + 1}$ are the optimizer of the loss function in Equation (\ref{eq:AlternativeProxyMethodATTProofLossMatrixVectorProductForm}). Now, we need to compute this validation loss in terms of matrix-vector multiplications. First, consider the following inner product:
\begin{align*}
    &\Big\langle \varphi,  \hat{\mu}_{Z| W, A}({w}_i, {a}_i) \otimes \phi_\gA({a}_i) \otimes \phi_\gA(a') \Big\rangle\\
    &= \Bigg\langle \sum_{l = 1}^m \alpha_l  \hat{\mu}_{Z|W, A} (\Tilde{w}_l, \Tilde{x}_l, \Tilde{a}_l)\otimes \phi_\gA(\Tilde{a}_l) \otimes \phi_\gA(a') + \frac{\alpha_{m + 1}}{m}  \sum_{j = 1}^m \sum_{\substack{l = 1 \\ l \ne j}}^m \theta_l \hat{\mu}_{Z|W, A} (\Tilde{w}_l, \Tilde{a}_j) \otimes \phi_\gA(\Tilde{a}_j) \otimes \phi_\gA(a'),\\ 
    &\hat{\mu}_{Z| W, A}({w}_i, {a}_i) \otimes \phi_\gA({a}_i) \otimes \phi_\gA(a')\Bigg\rangle\\
    &=\sum_{l = 1}^m \alpha_l \Big\langle  \hat{\mu}_{Z|W, A} (\Tilde{w}_l, \Tilde{a}_l)\otimes \phi_\gA(\Tilde{a}_l) \otimes \phi_\gA(a'), \hat{\mu}_{Z| W, A}({w}_i, {a}_i) \otimes \phi_\gA({a}_i) \otimes \phi_\gA(a')\Big\rangle\\
    &+ \frac{\alpha_{m + 1}}{m} \sum_{j = 1}^m \sum_{\substack{l = 1 \\ l \ne j}}^m \Big\langle \theta_l \hat{\mu}_{Z|W, A} (\Tilde{w}_l, \Tilde{a}_j) \otimes \phi_\gA(\Tilde{a}_j) \otimes \phi_\gA(a'), 
    \hat{\mu}_{Z| W, A}({w}_i, {a}_i) \otimes \phi_\gA({a}_i) \otimes \phi_\gA(a') \Big\rangle\\
    &= \sum_{l = 1}^m \alpha_l \mBeta(\Tilde{w}_l, \Tilde{a}_l)^T \mK_{Z Z} \mBeta({w}_i, {a}_i) k_\gA(\Tilde{a}_l, {a}_i) k_\gA(a', a') \\&+ \frac{\alpha_{m + 1}}{m} \sum_{j = 1}^m \Big(\sum_{\substack{l = 1 \\ l \ne j}}^m \theta_l \mBeta(\Tilde{w}_l, \Tilde{a}_j)\Big)^T \mK_{Z Z} \mBeta({w}_i, {a}_i) k_\gA(\Tilde{a}_j, {a}_i) k_\gA(a', a')\\
    &=\Big[\big[\mC^T \mK_{Z Z}\mB \odot \mK_{{A}\Tilde{A}}\big] \alpha_{1:m}\Big]_i k_\gA(a', a') + \alpha_{m+1}\Big[\big[{\mC}^T \mK_{Z Z}\Tilde{\mB} \odot \mK_{{A}\Tilde{A}}\big] (\vone / m)\Big]_i k_\gA(a', a')\\
    &= \Bigg[\begin{bmatrix}
        \mC^T \mK_{Z Z} \mB \odot \mK_{{A}\Tilde{A}} & \big[{\mC}^T \mK_{Z Z} \Tilde{\mB} \odot \mK_{{A}\Tilde{A}}\big] (\vone / m)
    \end{bmatrix} \begin{bmatrix}
        \alpha_{1:m}\\ \alpha_{m+1}
    \end{bmatrix} \Bigg]_i k_\gA(a', a')
\end{align*}
where
\begin{align*}
    \mC = \Big( \mK_{W W} \odot \mK_{A A} + n \lambda_1 \mI\Big)^{-1} (\mK_{W W} \odot \mK_{A A}).
\end{align*}

As a result, the first component in Equation (\ref{eq:ATTSecondStageValidationLoss}) is given by
\begin{align}
    &\frac{1}{n} \sum_{i = 1}^n \Big\langle \varphi,  \hat{\mu}_{Z| W, A}({w}_i, {a}_i) \otimes \phi_\gA({a}_i) \otimes \phi_\gA(a') \Big\rangle^2 \nonumber\\
    &= \frac{k_\gA(a', a')^2}{n} \begin{bmatrix}
        \alpha_{1:m} \\ \alpha_{m+1}
    \end{bmatrix}^T 
    \begin{bmatrix}
        \big[\mC^T \mK_{Z Z}\mB \odot \mK_{{A}\Tilde{A}}\big]^T \\ (\frac{\vone}{m})^T \big[{\mC}^T \mK_{Z Z} \Tilde{\mB} \odot \mK_{{A}\Tilde{A}}\big]^T
    \end{bmatrix} 
    \begin{bmatrix}
        \mC^T \mK_{Z Z}\mB \odot \mK_{{A}\Tilde{A}} & \big[{\mC}^T \mK_{Z Z} \Tilde{\mB} \odot \mK_{{A}\Tilde{A}}\big] \frac{\vone}{m}
    \end{bmatrix} 
    \begin{bmatrix}
        \alpha_{1:m}\\ \alpha_{m+1}
    \end{bmatrix}
\label{eq:AlternativeProxyMethodsATTValidationLossFirstTerm}
\end{align}
Next, for the second component in Equation (\ref{eq:ATTSecondStageValidationLoss}), we note that
\begin{align}
    &\frac{1}{n} \sum_{i = 1}^n \sum_{\substack{j = 1 \\ j \ne i}}^n \Big\langle \varphi, \theta_j^{(2)} \hat{\mu}_{Z|W, A} ({w}_j, {a}_i) \otimes \phi_\gA({a}_i) \otimes \phi_\gA(a') \Big\rangle \nonumber\\
    &= \frac{1}{n} \sum_{i = 1}^n \sum_{\substack{j = 1 \\ j \ne i}}^n \sum_{l = 1}^m \alpha_l \Big\langle 
     \hat{\mu}_{Z|W, A} (\Tilde{w}_l, \Tilde{a}_l)\otimes \phi_\gA(\Tilde{a}_l) \otimes \phi_\gA(a'), \theta_j^{(2)} \hat{\mu}_{Z|W, A} ({w}_j, {a}_i) \otimes \phi_\gA({a}_i) \otimes \phi_\gA(a')
    \Big\rangle \nonumber\\
    &+ \frac{\alpha_{m + 1}}{m n} \sum_{i = 1}^n \sum_{\substack{j = 1 \\ j \ne i}}^n \sum_{r = 1}^m \sum_{\substack{s = 1 \\ s \ne r}}^m  \Big\langle \theta_s \hat{\mu}_{Z|W, A} (\Tilde{w}_s, \Tilde{a}_r) \otimes \phi_\gA(\Tilde{a}_r) \otimes \phi_\gA(a'), 
    \theta_j^{(2)} \hat{\mu}_{Z|W, A} ({w}_j, {a}_i) \otimes \phi_\gA({a}_i) \otimes \phi_\gA(a') \Big\rangle \nonumber\\
    &= \frac{1}{n} \sum_{i = 1}^n \sum_{\substack{j = 1 \\ j \ne i}}^n \sum_{l = 1}^m \alpha_l \mBeta(\Tilde{w}_l, \Tilde{a}_l)^T\mK_{Z Z} \theta_j^{(2)} \mBeta({w}_j, {a}_i) k_\gA(\Tilde{a}_l, {a}_i) k_\gA(a', a') \nonumber\\
    &+ \frac{\alpha_{m + 1}}{mn} \sum_{i = 1}^n \sum_{\substack{j = 1 \\ j \ne i}}^n \sum_{r = 1}^m \sum_{\substack{s = 1 \\ s \ne r}}^m \theta_s \mBeta(\Tilde{w}_s, \Tilde{a}_r)^T \mK_{Z Z} \theta_j^{(2)} \mBeta({w}_j, {a}_i) k_\gA(\Tilde{a}_r, {a}_i) k_\gA(a', a') \nonumber\\
    &= \frac{1}{n} \sum_{i = 1}^n \sum_{l = 1}^m \alpha_l \mBeta(\Tilde{w}_l, \Tilde{a}_l)^T\mK_{Z Z}\Big(\sum_{\substack{j = 1 \\ j \ne i}}^n  \theta_j^{(2)} \mBeta({w}_j, {a}_i)\Big) k_\gA(\Tilde{a}_l, {a}_i) k_\gA(a', a') \nonumber\\
    &+ \frac{\alpha_{m + 1}}{mn} \sum_{i = 1}^n  \sum_{r = 1}^m \Big( \sum_{\substack{s = 1 \\ s \ne r}}^m \theta_s \mBeta(\Tilde{w}_s, \Tilde{a}_r) \Big)^T \mK_{Z Z} \Big( \sum_{\substack{j = 1 \\ j \ne i}}^n \theta_j^{(2)} \mBeta({w}_j, {a}_i) \Big) k_\gA(\Tilde{a}_r, {a}_i) k_\gA(a', a') \nonumber\\
    &=\frac{1}{n} \alpha_{1:m}^T \Big[ \mB^T \mK_{Z Z} \Tilde{\mC} \odot \mK_{\Tilde{A}{A}}\Big] \vone k_\gA(a', a') + \alpha_{m+1} \frac{1}{mn} \vone^T \Big[ \Tilde{\mB}^T \mK_{Z Z} \Tilde{\mC} \odot \mK_{\Tilde{A}{A}} \Big] \vone k_\gA(a', a') \nonumber \\
&=  \begin{bmatrix}
    \alpha_{1:m } \\ \alpha_{m + 1}
\end{bmatrix}^T \begin{bmatrix}
 [ \mB^T \mK_{Z Z} \Tilde{\mC} \odot \mK_{\Tilde{A} {A}} ] \frac{\vone}{n}\\
 (\frac{\vone}{m})^T \Big[ \Tilde{\mB}^T \mK_{Z Z} \Tilde{\mC} \odot \mK_{\Tilde{A} \Tilde{A}}\Big] \frac{\vone}{n}
\end{bmatrix} k_\gA(a', a') \label{eq:AlternativeProxyMethodsATTValidationLossSecondTerm}
\end{align}
where
\begin{align*}
    \Tilde{\mC}_{:, j} &=(\mK_{W W} \odot \mK_{A A} + n \lambda_1 \mI)^{-1}\Big( \sum_{\substack{l = 1 \\ l \ne j}}^n \theta^{(2)}_l \mK_{W {w}_l} \odot \mK_{A {a}_j} \Big)\\
\end{align*}

Now, we are ready to combine our findings and write the loss function in terms of matrix-vector multiplications. Using Equations (\ref{eq:AlternativeProxyMethodsATTValidationLossFirstTerm}) and (\ref{eq:AlternativeProxyMethodsATTValidationLossSecondTerm}), the loss function can be expressed as 

\begin{align}
    &\hat{\mathcal{L}}^{\text{Val}}(\varphi) = \frac{1}{n} \sum_{i = 1}^n \langle \varphi, \hat{\mu}_{Z|W, A} ({w}_i, {a}_i) \otimes \phi_\gA({a}_i) \otimes \phi_\gA(a') \rangle^2 \nonumber\\ &-2 \frac{1}{n} \sum_{j = 1}^n \sum_{\substack{i = 1 \\ i \ne j}}^n \langle \varphi,  \theta^{(2)}_i \hat{\mu}_{Z |W, A}({w}_i, {a}_j) \otimes \phi_\gA({a}_j) \otimes \phi_\gA(a') \rangle \nonumber\\
    &= \frac{k_\gA(a', a')^2}{n} \begin{bmatrix}
        \alpha_{1:m} \\ \alpha_{m+1}
    \end{bmatrix}^T 
    \begin{bmatrix}
        \mC^T \mK_{Z Z}\mB \odot \mK_{{A}\Tilde{A}} \\ (\frac{\vone}{m})^T \big[{\mC}^T \mK_{Z Z} \Tilde{\mB} \odot \mK_{{A}\Tilde{A}}\big]^T
    \end{bmatrix} 
    \begin{bmatrix}
        \mC^T \mK_{Z Z}\mB \odot \mK_{{A}\Tilde{A}} & \big[{\mC}^T \mK_{Z Z} \Tilde{\mB} \odot \mK_{{A}\Tilde{A}}\big] \frac{\vone}{m}
    \end{bmatrix} 
    \begin{bmatrix}
        \alpha_{1:m}\\ \alpha_{m+1}
    \end{bmatrix} \nonumber\\
    &-2 \begin{bmatrix}
    \alpha_{1:m } \\ \alpha_{m + 1}
\end{bmatrix}^T \begin{bmatrix}
 [ \mB^T \mK_{Z Z} \Tilde{\mC} \odot \mK_{\Tilde{A} {A}} ] \frac{\vone}{n}\\
 (\frac{\vone}{m})^T \Big[ \Tilde{\mB}^T \mK_{Z Z} \Tilde{\mC} \odot \mK_{\Tilde{A} {A}}\Big] \frac{\vone}{n}
\end{bmatrix} k_\gA(a', a')
\label{eq:AlternativeProxyMethodATTValidationLossMatrixVectorProductForm}
\end{align}
Similar to the tuning procedure of $\lambda_2$ in dose-response curve estimation, we can augment this validation loss with the complexity loss that will upper bound the hold-out loss while penalizing the overly complex models. As a result, we can write the final surrogate loss to tune $\lambda_2$ as:
\begin{align*}
    &\hat{\mathcal{L}}^{\text{Val}}(\varphi)\\
    &\le \frac{k_\gA(a', a')^2}{n} \begin{bmatrix}
        \alpha_{1:m} \\ \alpha_{m+1}
    \end{bmatrix}^T 
    \begin{bmatrix}
        \mC^T \mK_{Z Z}\mB \odot \mK_{{A}\Tilde{A}} \\ (\frac{\vone}{m})^T \big[{\mC}^T \mK_{Z Z} \Tilde{\mB} \odot \mK_{{A}\Tilde{A}}\big]^T
    \end{bmatrix} 
    \begin{bmatrix}
        \mC^T \mK_{Z Z}\mB \odot \mK_{{A}\Tilde{A}} & \big[{\mC}^T \mK_{Z Z} \Tilde{\mB} \odot \mK_{{A}\Tilde{A}}\big] \frac{\vone}{m}
    \end{bmatrix} 
    \begin{bmatrix}
        \alpha_{1:m}\\ \alpha_{m+1}
    \end{bmatrix} \nonumber\\
    &-2 \begin{bmatrix}
    \alpha_{1:m } \\ \alpha_{m + 1}
\end{bmatrix}^T \begin{bmatrix}
 [ \mB^T \mK_{Z Z} \Tilde{\mC} \odot \mK_{\Tilde{A} {A}} ] \frac{\vone}{n}\\
 (\frac{\vone}{m})^T \Big[ \Tilde{\mB}^T \mK_{Z Z} \Tilde{\mC} \odot \mK_{\Tilde{A} {A}}\Big] \frac{\vone}{n}
\end{bmatrix} k_\gA(a', a') + \frac{2 \sigma^2}{m} \text{Tr} \Big( \big( \mL^T \mL + m \lambda_2 \mI\big)^{-1} \mL^T \mL \Big)
\end{align*}

where the matrix $\mL$ is defined as
\begin{align*}
    \mL = \begin{bmatrix}
        \mB^T \mK_{Z Z}\mB \odot \mK_{\Tilde{A}\Tilde{A}} &  \big[{\mB}^T \mK_{Z Z} \Tilde{\mB} \odot \mK_{\Tilde{A}\Tilde{A}}\big] (\frac{\vone}{m})
    \end{bmatrix} 
\end{align*}
In our ATT experiment that is illustrated in Figure (\ref{fig:alternativeProxyMethodATTComparison}), we used $\sigma^2 = 1$.
\subsubsection{Tuning \texorpdfstring{$\zeta$}{Lg} 
\label{sec:ZetaRegularizationTuning}Regularization Parameter in ATT Algorithm}
In Algorithm (\ref{algo:ATT_algorithm_with_confounders}), we estimate the conditional mean embedding ${\E}[\phi_\gW(W)| A = a']$ by 
\begin{align*}
    \hat{\E}[\phi_\gW(W)| A = a'] = \Phi_{\gW}(\mK_{\Tilde{A} \Tilde{A}} + m \zeta \mI)^{-1} \mK_{\Tilde{A} a'} = \sum_{i = 1}^m \theta_i \phi_\gW(\Tilde{w}_i) = \Phi_{\gW} \vtheta
\end{align*}
where $\vtheta = (\mK_{\Tilde{A} \Tilde{A}} + m \zeta \mI)^{-1} \mK_{\Tilde{A} a'}$. This is kernel ridge regression solution for measurements $\{\phi_\gA(\Tilde{a}_i)\}_{i = 1}^m$ and targets $\{\phi_\gW(\Tilde{w}_i)\}_{i = 1}^m$. Hence, one can use the LOOCV procedure presented in S.M. (\ref{section:LOOCV_KernelRidgeRegression}) to tune the regularization parameter in this estimator. Here, we will present another method that can be used for conditional mean embeddings that is provided in \citep{singh2023kernelmethodsunobservedconfounding}. Its proof can be found in the derivation of Algorithm 7 of \citep{singh2023kernelmethodsunobservedconfounding}.

\begin{theorem}[Algorithm 7 in \citep{singh2023kernelmethodsunobservedconfounding}]
 Consider the conditional mean embedding
 \begin{align*}
     {\E}[\phi_\gW(W)| A = a'] = \mu_{W | A}(a') = \int \phi_\gW(w) p(w | a') d w.
 \end{align*}

Sample based estimation using data $\{\Tilde{w}_i, \Tilde{a}_i\}_{i = 1}^m$ for this conditional mean embedding is given by
\begin{align*}
     \hat{\mu}_{W|A}(a') = \Phi_\gW (\mK_{\Tilde{\mA} \Tilde{\mA}} + m \zeta \mI)^{-1} \mK_{\Tilde{\mA} a'}
 \end{align*}
 where $\Phi_\gW = \begin{bmatrix}
     \phi_\gW(\Tilde{w}_1) & \phi_\gW(\Tilde{w}_2) & \ldots & \phi_\gW(\Tilde{w}_m) 
 \end{bmatrix}$, $[\mK_{\Tilde{\mA} \Tilde{\mA}}]_{ij} = k_\gA(\Tilde{a}_i, \Tilde{a}_j)$, $[\mK_{\Tilde{\mA} a'}]_i = k_\gA(\Tilde{a}_i, a')$ and $\mI \in \R^{m \times m}$ is the identity matrix. The LOOCV loss for the conditional mean embedding estimation is given by
 \begin{align}
     \text{LOOCV}_{\mu_{W|A}}(\zeta) = \frac{1}{m} \text{Tr}\Big(\mS \big(\mK_{\Tilde{W}\Tilde{W}} - 2\mK_{\Tilde{W}\Tilde{W}} \mR^T + \mR \mK_{\Tilde{W}\Tilde{W}} \mR^T \big)  \Big)
\label{eq:ConditionalMeanEmbeddingLOOCVEquation}
 \end{align}
 where
 \begin{align*}
 [\mK_{\Tilde{W}\Tilde{W}}]_{ij} &= k_\gW(\Tilde{w}_i, \Tilde{w}_j)\\
     \mR &= \mK_{\Tilde{\mA} \Tilde{\mA}} (\mK_{\Tilde{\mA} \Tilde{\mA}} + m \zeta \mI)^{-1} \in \R^{m \times m}\\
     \mS &\in \R^{m \times m} \enspace \text{s.t.} \enspace [\mS]_{ij} = \mathbf{1}[i = j]\Bigg( \frac{1}{1 - [\mR]_{ij}}  \Bigg)^2.
 \end{align*}
 \label{thm:LOOCV_conditional_mean_embedding}
\end{theorem}
In our numerical experiments, we utilized Theorem (\ref{thm:LOOCV_conditional_mean_embedding}) to tune regularization parameters $\zeta$ in ATT algorithm. 
In particular, we picked the regularization parameter that minimizes the LOOCV loss given in Equation (\ref{eq:ConditionalMeanEmbeddingLOOCVEquation}) over a grid $\zeta \in \mathrm{Z}$ that is generated with a \emph{logspace} with maximum and minimum values of $1.0$ and $10^{-7}$, respectively. We used $150$ grid points in our numerical experiments.

\subsection{Discussion on the Time Complexity of the Proposed Methods}

Similar to kernel ridge regression, the complexity of our methods is governed by the matrix inversion. For simplicity, we consider the dose-response curve estimation in Algorithm (\ref{algo:ATE_algorithm_with_confounders}). In the first-stage regression, the following matrix must be inverted:
\begin{align*}
    \mK_{W W} \odot \mK_{A A} + n \lambda \mI \in \R^{n \times n}.
\end{align*}
This inversion operation has complexity of $O(n^3)$, making the first-stage sample size the limiting factor. Furthermore, to tune the regularization parameter $\lambda_1$ with LOOCV procedure, as discussed in S.M. (Sec. \ref{section:LOOCV_KernelRidgeRegression}), this inversion must be performed for each grid point of $\lambda_1 \in \Lambda_1$. 

In the second-stage, the following matrix needs to be inverted to obtain the optimizer coefficients $\{\alpha_i\}_{i=1}^{m+1}$ of Equation (\ref{eq:AlternativeProxyMethodProofLossMatrixVectorProductFormWithConfounder}):
\begin{align*}
    \frac{1}{m} \mL^T \mL + \lambda_2 \mN \in \R^{(m+1) \times (m + 1)}
\end{align*}
where
\begin{align*}
\mL &= \begin{bmatrix}
        \mB^T \mK_{Z Z}\mB \odot \mK_{\Tilde{A}\Tilde{A}} \nonumber\\ (\frac{\vone}{m})^T \big[{\mB}^T \mK_{Z Z} \bar{\mB} \odot \mK_{\Tilde{A}\Tilde{A}}\big]^T
    \end{bmatrix}^T \in \R^{m \times (m+1)},\\
\mN &= \begin{bmatrix}
        \mB^T \mK_{Z Z} \mB \odot \mK_{\Tilde{A} \Tilde{A}} & [ \mB^T \mK_{Z Z} \Bar{\mB} \odot \mK_{\Tilde{A} \Tilde{A}} ] \frac{\vone}{m}\\
        (\frac{\vone}{m})^T [ \mB^T \mK_{Z Z} \Bar{\mB} \odot \mK_{\Tilde{A} \Tilde{A}} ]^T & (\frac{\vone}{m})^T \Big[ \Bar{\mB}^T \mK_{Z Z} \Bar{\mB} \odot \mK_{\Tilde{A} \Tilde{A}}\Big] \frac{\vone}{m}
    \end{bmatrix} \in \R^{(m+1) \times (m + 1)},
\end{align*}
as given in Algorithm (\ref{algo:ATE_algorithm_with_confounders}). This inversion has complexity of $O((m+1)^3)$. Additionally, tuning the regularization parameter $\lambda_2$, as discussed in S.M. (Sec. \ref{sec:SecondStageRegularizationTuning}), requires performing this inversion for each grid point for $\lambda_2 \in \Lambda_2$. 

Finally, we note that the third-stage regression is another kernel ridge regression, where the matrix to be inverted (if only first-stage data is used) is
\begin{align*}
    \mK_{A A} + n \lambda_3 \mI \in \R^{n \times n}.
\end{align*}
As previously mentioned, data from either first-stage, second-stage, or a combination of both can be used in the third-stage regression. If the combination of first and second-stage data is used, the matrix to be inverted will have dimensions $(n + m) \times (n + m)$. Denoting by $t$ the number of samples used for third-stage regression, the required inversion have complexity of $O(t^3)$. Therefore, to tune the regularization parameter $\lambda_3$ with LOOCV, as outlined in S.M. (Sec. \ref{section:LOOCV_KernelRidgeRegression}), this inversion must be performed for each grid point of $\lambda_3 \in \Lambda_3$. 

Overall, assuming $t = m + n$, the time complexity of our proposed method scales as $O(t^3)$. A similar analysis applies to our algorithm for conditional dose-response curve estimation. 

\subsection{Further Notes on Numerical Experiments}
\label{sec:FurtherNotesOnNumericalExperiments}
Our implementation code is available on GitHub link\footnote{\url{https://github.com/BariscanBozkurt/Density-Ratio-Based-Proxy-Causal-Learning-without-Density-Ratios}}, which includes the instructions for reproducing the experiments presented in this paper. 

For the tuning of the regularization parameters $(\lambda_1, \lambda_2, \lambda_3)$ (and $\zeta$ for ATT), we followed the procedures described in S.M.  (Sec. \ref{section:LOOCV_KernelRidgeRegression}), (Sec. \ref{sec:SecondStageRegularizationTuning}), (Sec. \ref{sec:TuningSecondStageRegularizationATT}) and (Sec. \ref{sec:ZetaRegularizationTuning}). In all experiments, the variables are normalized by subtracting mean and dividing by the standard deviation unless otherwise stated. The data were split uniformly into equal-sized first-stage and second-stage sets for the experiments. In the third-stage regression of our proposed methods, we used the combination of first and second-stage data.

\subsection{Additional Numerical Experiments and Ablation Studies}
\label{sec:AdditionalNumericalExperiments}
\subsubsection{Comparison of Our Approach and Outcome Bridge Function-Based Methods}
\label{sec:ComparisonOfTreatmentVsOutcomeBridge_Appendix}
A key question arising from our experiments in Section (\ref{sec:NumericalExperiments}) is whether our method outperforms outcome bridge-based methods under specific conditions. To investigate this, we conducted synthetic experiments where one proxy variable was highly informative of confounders while the other was noisier. We adapted the data generation process from \citep[][Appendix D.5]{pmlr-v238-tsai24b} to construct various scenarios. We considered six scenarios where $U$ follows different Beta distributions, and proxies $Z$ and $W$ vary in informativeness. The treatment and outcome variables were generated as follows:

\textbf{Setting 1:} $U \sim \text{Beta}(5, 4)$, $W = g(U) + \mathcal{U}[0, 1]$ where the function $g(x) = 0.8 \frac{\exp(x)}{1 + \exp(x)} + 0.1$ is applied elementwise, $Z = (1 - U) \times Z_1 + U \times Z_2 + \mathcal{U}[0, 100]$ where $Z_1 = \mathcal{N}(-1, 0.1)$ and $Z_2 = \mathcal{N}(1, 0.1)$, $A = 0.1 U + 0.1 Z + \mathcal{U}[0, 1]$, and $Y = (2 U - 1) + \cos(1.5 A)$.

\textbf{Setting 2:} $U \sim \text{Beta}(5, 4)$, $Z = g(U) + \mathcal{U}[0, 1]$, $W = (1 - U) \times W_1 + U \times W_2 + \mathcal{U}[0, 100]$ where $W_1 = \mathcal{N}(-1, 0.1)$ and $W_2 = \mathcal{N}(1, 0.1)$, $A = 0.1 U + 0.1 Z + \mathcal{U}[0, 1]$, and $Y = (2 U - 1) + \cos(1.5 A)$.

\textbf{Setting 3:} $U \sim \text{Beta}(8, 4)$, $W = U + \mathcal{U}[0, 1]$, $Z = g((1 - U) \times Z_1 + U \times Z_2) + \mathcal{U}[0, 100]$ where $Z_1 = \mathcal{N}(-1, 0.1)$ and $Z_2 = \mathcal{N}(1, 0.1)$, $A = 0.1 U + 0.1 Z + \mathcal{U}[0, 1]$, and $Y = (2 U - 1) + \cos(1.5 A)$.

\textbf{Setting 4:} $U \sim \text{Beta}(8, 4)$, $Z = U + \mathcal{U}[0, 1]$, $W = g((1 - U) \times W_1 + U \times W_2) + \mathcal{U}[0, 100]$ where $W_1 = \mathcal{N}(-1, 0.1)$ and $W_2 = \mathcal{N}(1, 0.1)$, $A = 0.1 U + 0.1 Z + \mathcal{U}[0, 1]$, and $Y = (2 U - 1) + \cos(1.5 A)$.

\textbf{Setting 5:} $U \sim \text{Beta}(3, 5)$, $W = - U^2 + \mathcal{U}[0, 1]$, $Z = g((1 - U) \times Z_1 + U \times Z_2) + \mathcal{U}[0, 100]$ where $Z_1 = \mathcal{N}(-1, 0.1)$ and $Z_2 = \mathcal{N}(1, 0.1)$, $A = 0.25 \sqrt{|U|} - 0.2 Z + \mathcal{U}[0, 1]$, and $Y = 3 W - 0.1 A - \cos(0.5 A + 5 U)$.

\textbf{Setting 6:} $U \sim \text{Beta}(3, 5)$, $Z = - U^2 + \mathcal{U}[0, 1]$, $W = g((1 - U) \times W_1 + U \times W_2 + \mathcal{U}[0, 100])$ where $W_1 = \mathcal{N}(-1, 0.1)$ and $W_2 = \mathcal{N}(1, 0.1)$, $A = 0.25 \sqrt{|U|} - 0.2 Z + \mathcal{U}[0, 1]$, and $Y = 3 W - 2 A - \cos(10 A + 5 U)$.

For each setting, we generated $1000$ samples and evaluated our method against outcome bridge-based approaches over five runs, approximating the ground truth dose-response via Monte Carlo. Table (\ref{tab:ATE_TreatmentProxy_vs_Outcome_Proxy_Table1}) reports mean squared error and standard deviation across five independent realizations of each setting. Additionally, we compare the PCL algorithms to the oracle method Kernel-ATE, which directly uses the confounding variable $U$, in Table (\ref{tab:ATE_TreatmentProxy_vs_Outcome_Proxy_Table1}). Our method outperformed in odd-numbered settings where $W$ was more informative, while outcome bridge-based methods excelled in even-numbered settings where $Z$ was more informative. When the link between $Z$ and $U$ is highly noisy (or incomplete as in the next section), Assumption (\ref{assum:ExistenceCompletenessAssumption})—which ensures the existence of our treatment bridge function—is likely violated. Conversely, outcome bridge-based methods rely on this assumption for causal function identifiability. Thus, we hypothesize that our method is more robust when existence is challenged rather than identifiability. Meanwhile, we use Assumption (\ref{assum:AlternativeProxyAssumptionCompleteness1}) for identifiability, while KPV and KNC use it for outcome bridge function existence. This assumption is likely violated when the link between $W$ and $U$ is highly noisy (or incomplete), as in Settings 2, 4, and 6, where KPV and KNC outperform our method. We conjecture that violating our method’s identifiability condition impacts performance more than violating bridge function existence.

Experimental results in S.M. (Sec. \ref{sec:Appendix_JobCorpsExperiments}) also validates this hypothesis with the Job Corps dataset, where high-dimensional proxies were synthetically generated. Results reinforce the complementary strengths of treatment and outcome bridge-based methods. Understanding these trade-offs warrants deeper analysis, which we leave for future work. Nonetheless, our findings highlight the importance of further exploring treatment bridge-based approaches. 


\begin{table}[ht]
    \centering
\caption{Mean squared error for ablation studies with synthetic settings in S.M. (Sec. \ref{sec:ComparisonOfTreatmentVsOutcomeBridge_Appendix}). We report mean $\pm$ standard deviation from $n = 5$ independent realizations of each setting.}
    \label{tab:ATE_TreatmentProxy_vs_Outcome_Proxy_Table1}
    \begin{tabular}{c | c l l l}
    \hline
    {\bf} & {\bf Kernel Alternative Proxy} & {\bf KNC} & {\bf  KPV} & {\bf  Kernel-ATE} \\
    \hline 
    {\bf Setting 1} & $\mathbf{0.00553\pm0.00069}$  &  $0.29752\pm0.08545$  & $0.04184\pm0.02661$ & $0.00026\pm0.00019$\\
    \hline 
    {\bf Setting 2} & $0.00932\pm0.00529$  &  $0.01086\pm0.00373$  & $\mathbf{0.00541\pm0.00217}$ &$0.00018\pm0.00021$\\
    \hline 
    {\bf Setting 3} & $\mathbf{0.00347\pm0.00104}$  &  $0.25505\pm0.09238$  & $0.05122\pm0.04610$ & $0.00014\pm0.00009$\\
    \hline 
    {\bf Setting 4} & $0.01532\pm0.00548$  &  $0.01333\pm0.00437$  & $\mathbf{0.00899\pm0.00491}$ & $0.00004\pm	0.00003$\\
    \hline 
    {\bf Setting 5} & $\mathbf{0.01129\pm0.00823}$  &  $0.03312\pm0.02424$  & $	0.01982\pm0.00885$ & $0.01105\pm0.00701$\\
    \hline 
    {\bf Setting 6} & $0.18436\pm0.04889$  &  $0.09568\pm0.02510$  & $\mathbf{0.05394\pm0.01516}$ & $0.00330\pm0.00262$\\
    \hline 
    \end{tabular}
\end{table}

\subsubsection{Dose-Response and Conditional Dose-Response Estimations in Job Corps Dataset}
\label{sec:Appendix_JobCorpsExperiments}
In this section, we conduct a new set of semi-synthetic experiments to compare our proposed method with the outcome bridge function-based algorithm based on the US Job Corps dataset \citep{Schochet_JobCorps, Flores_JobCorps}, adapted for the Proxy Causal Learning (PCL) setting. The US Job Corps Program is an educational intervention targeting disadvantaged youth. In this context, the continuous treatment variable $A$ represents the total hours spent in academic or vocational training, while the continuous outcome variable $Y$ corresponds to the proportion of weeks employed during the second year of training. Consistent with the setup in \citep{RahulKernelCausalFunctions}, the covariates $U \in \R^{65}$ include factors such as gender, ethnicity, age, language proficiency, education level, marital status, household size, and others. We obtained the dataset from the publicly available code of \citep{RahulKernelCausalFunctions} (see \url{https://github.com/liyuan9988/KernelCausalFunction/tree/master}). To adapt this dataset to the PCL framework, we synthetically generated two proxy variables, $Z$ and $W$, from $U$ using the following settings:

\textbf{Setting 1:} $W = U + \epsilon$, $Z = g\left(U^{(1:20)} / \max\left(U^{(1:20)}\right)\right) + \nu$ where the function $g(x) = 0.8 \frac{\exp(x)}{1 + \exp(x)} + 0.1$ is the elementwise truncated logistic link function, $\epsilon^{(i)} \sim \gN(0, 1) \enspace \forall i = 1, \ldots, 65$, $\nu^{(i)} \sim \gU[-1, 1] \enspace \forall i = 1, \ldots, 20$, and $U^{(1:20)}$ indicates taking the first $20$ components of the vector $U$. The division operation and the max function are executed elementwise. 

\textbf{Setting 2:} $Z = U + \epsilon$, $W = g\left(U^{(1:20)} / \max\left(U^{(1:20)}\right)\right) + \nu$ where 
$\epsilon^{(i)} \sim \gN(0, 1) \enspace \forall i$, $\nu^{(i)} \sim \gU[-1, 1] \enspace \forall i$.

\textbf{Setting 3:} $W = U + \epsilon$, $Z = g\left(U^{(20:40)} / \max\left(U^{(20:40)}\right)\right) + \nu$ where 
$\epsilon^{(i)} \sim \gN(0, 1) \enspace \forall i$, $\nu^{(i)} \sim \gU[-1, 1] \enspace \forall i$.

\textbf{Setting 4:} $Z = U + \epsilon$, $W = g\left(U^{(20:40)} / \max\left(U^{(20:40)}\right)\right) + \nu$ where 
$\epsilon^{(i)} \sim \gN(0, 1) \enspace \forall i$, $\nu^{(i)} \sim \gU[-1, 1] \enspace \forall i$.

\textbf{Setting 5:} $W = U + \epsilon$, $Z = g\left(U^{(40:60)} / \max\left(U^{(40:60)}\right)\right) + \nu$ where 
$\epsilon^{(i)} \sim \gN(0, 1) \enspace \forall i$, $\nu^{(i)} \sim \gU[-1, 1] \enspace \forall i$.

\textbf{Setting 6:} $Z = U + \epsilon$, $W = g\left(U^{(40:60)} / \max\left(U^{(40:60)}\right)\right) + \nu$ where 
$\epsilon^{(i)} \sim \gN(0, 1) \enspace \forall i$, $\nu^{(i)} \sim \gU[-1, 1] \enspace \forall i$.

\textbf{Setting 7:} $W = U + \epsilon$, $Z^{(1:20)} = g\left(U^{(1:20)} / \max\left(U^{(1:20)}\right)\right) + \nu$, $Z^{(21:65)} \sim \gN(0, I)$, where 
$\epsilon^{(i)} \sim \gN(0, 1) \enspace \forall i$, $\nu^{(i)} \sim \gU[-1, 1] \enspace \forall i$.

\textbf{Setting 8:} $Z = U + \epsilon$, $W^{(1:20)} = g\left(U^{(1:20)} / \max\left(U^{(1:20)}\right)\right) + \nu$, $W^{(21:65)} \sim \gN(0, I)$, where 
$\epsilon^{(i)} \sim \gN(0, 1) \enspace \forall i$, $\nu^{(i)} \sim \gU[-1, 1] \enspace \forall i$.

\textbf{Setting 9:} $W = U + \epsilon$, $Z^{(46:65)} = g\left(U^{(46:65)} / \max\left(U^{(46:65)}\right)\right) + \nu$, $Z^{(1:45)} \sim \gN(0, I)$, where 
$\epsilon^{(i)} \sim \gN(0, 1) \enspace \forall i$, $\nu^{(i)} \sim \gU[-1, 1] \enspace \forall i$.

\textbf{Setting 10:} $Z = U + \epsilon$, $W^{(46:65)} = g\left(U^{(46:65)} / \max\left(U^{(46:65)}\right)\right) + \nu$, $W_{1:45} \sim \gN(0, I)$, where 
$\epsilon^{(i)} \sim \gN(0, 1) \enspace \forall i$, $\nu^{(i)} \sim \gU[-1, 1] \enspace \forall i$.

\textbf{Setting 11:} $W = U + \epsilon$, $Z^{(21:39)} = g\left(U^{(21:39)} / \max\left(U^{(21:39)}\right)\right) + \nu$, $Z^{(1:20)} \sim \gN(0, I)$, $Z^{(40:65)} \sim \gN(0, I)$, where 
$\epsilon^{(i)} \sim \gN(0, 1) \enspace \forall i$, $\nu^{(i)} \sim \gU[-1, 1] \enspace \forall i$.

\textbf{Setting 12:}  $Z = U + \epsilon$, $W^{(21:39)} = g\left(U^{(21:39)} / \max\left(U^{(21:39)}\right)\right) + \nu$, $W^{(1:20)} \sim \gN(0, I)$, $W^{(40:65)} \sim \gN(0, I)$, where 
$\epsilon^{(i)} \sim \gN(0, 1) \enspace \forall i$, $\nu^{(i)} \sim \gU[-1, 1] \enspace \forall i$.

Settings 1, 3, and 5 feature an incomplete link between the treatment proxy $Z$ and the confounding variable $U$, combined with a nonlinearity function. These setups are likely to violate Assumption (\ref{assum:ExistenceCompletenessAssumption}), which is crucial for the identifiability of KPV and KNC, as well as for the existence of the treatment bridge function in our method. Conversely, Settings 2, 4, and 6 feature an incomplete link between the outcome proxy $W$ and the confounding variable $U$, along with a nonlinearity function. These setups are likely to violate Assumption (\ref{assum:AlternativeProxyAssumptionCompleteness1}), which is essential for the identifiability of our method and plays a role in establishing the existence of the outcome bridge function for KPV/KNC algorithms \citep{xu2021deep}.

Figures (\ref{fig:JobCorpsComparison1})-(\ref{fig:JobCorpsComparison6}) illustrate the simulation results, averaged over five runs with different two-stage splits, for Settings 1-6. We also compare our method against the Kernel-ATE algorithm \citep{RahulKernelCausalFunctions}, which uses $U$ directly and serves as an oracle benchmark. For all algorithms, we employed Gaussian kernels with median length scale heuristics for all input variables. Each dimension of the input variables $Z$, $W$, and $A$ (as well as $U$) is normalized by subtracting the mean and dividing by the standard deviation before being fed into our method, KPV, KNC, and Kernel-ATE. In settings where there is an incomplete link between $Z$ and $U$, our method outperforms KPV and KNC, yielding results closer to the oracle method. In contrast, in settings where there is an incomplete link between $W$ and $U$, KPV and/or KNC outperform our method. Consistent with our hypothesis in S.M. (Sec. \ref{sec:ComparisonOfTreatmentVsOutcomeBridge_Appendix}), our method demonstrates better robustness when the existence of the bridge function is violated rather than when causal function identifiability is compromised, as seen in Settings 1, 3, and 5. Conversely, outcome bridge-based methods show better robustness when the existence of the outcome bridge function is violated rather than when causal function identifiability is compromised. These experimental results highlight the complementary strengths of treatment and outcome bridge-based methods under different assumptions.

We also investigate scenarios where the confounding variable $U$ has a noisy link with one of the proxies, as detailed in Settings 7-12. Specifically, we generate a proxy variable of the same dimension as the confounding variable, but some of its dimensions consist entirely of noise. In Settings 7, 9, and 11, the proxy variable $Z$ has a highly noisy link to the confounder $U$, in addition to nonlinearity, which is likely to violate Assumption (\ref{assum:ExistenceCompletenessAssumption}). Similarly, in Settings 8, 10, and 12, the proxy variable $W$ has a highly noisy link to the confounder $U$, which is likely to violate Assumption (\ref{assum:AlternativeProxyAssumptionCompleteness1}). Figures (\ref{fig:JobCorpsComparison7})-(\ref{fig:JobCorpsComparison12}) present the estimation results for each setting, comparing treatment and outcome bridge-based methods against the oracle Kernel-ATE method. Our findings indicate that our method performs better in settings where $Z$ has a noisy link with $U$, producing results closer to the oracle method. Conversely, outcome bridge-based algorithms perform better in settings where $W$ has a noisy link with $U$, yielding estimates closer to the oracle method. In these noisy link experiments, we constructed the input kernel as a product of separate kernels for noisy and non-noisy dimensions. For instance, in Setting 7, we used the kernel $k_\gZ(z_i, z_j) = k_\gZ^{(1)}(z_i^{(1:20)}, z_j^{(1:20)}) \times k_\gZ^{(2)}(z_i^{(21:65)}, z_j^{(21:65)})$ where $k_\gZ^{(1)}$ and $k_\gZ^{(2)}$ are Gaussian kernels with length scales determined using the median heuristic. Here, $z_i^{(1:20)}$ denotes the first $20$ dimension of the $i$-th training variable $z_i$.
 
\begin{figure*}[ht!]
\centering
\subfloat[]
{\includegraphics[trim = {0cm 0cm 0cm 0.0cm},clip,width=0.225\textwidth]{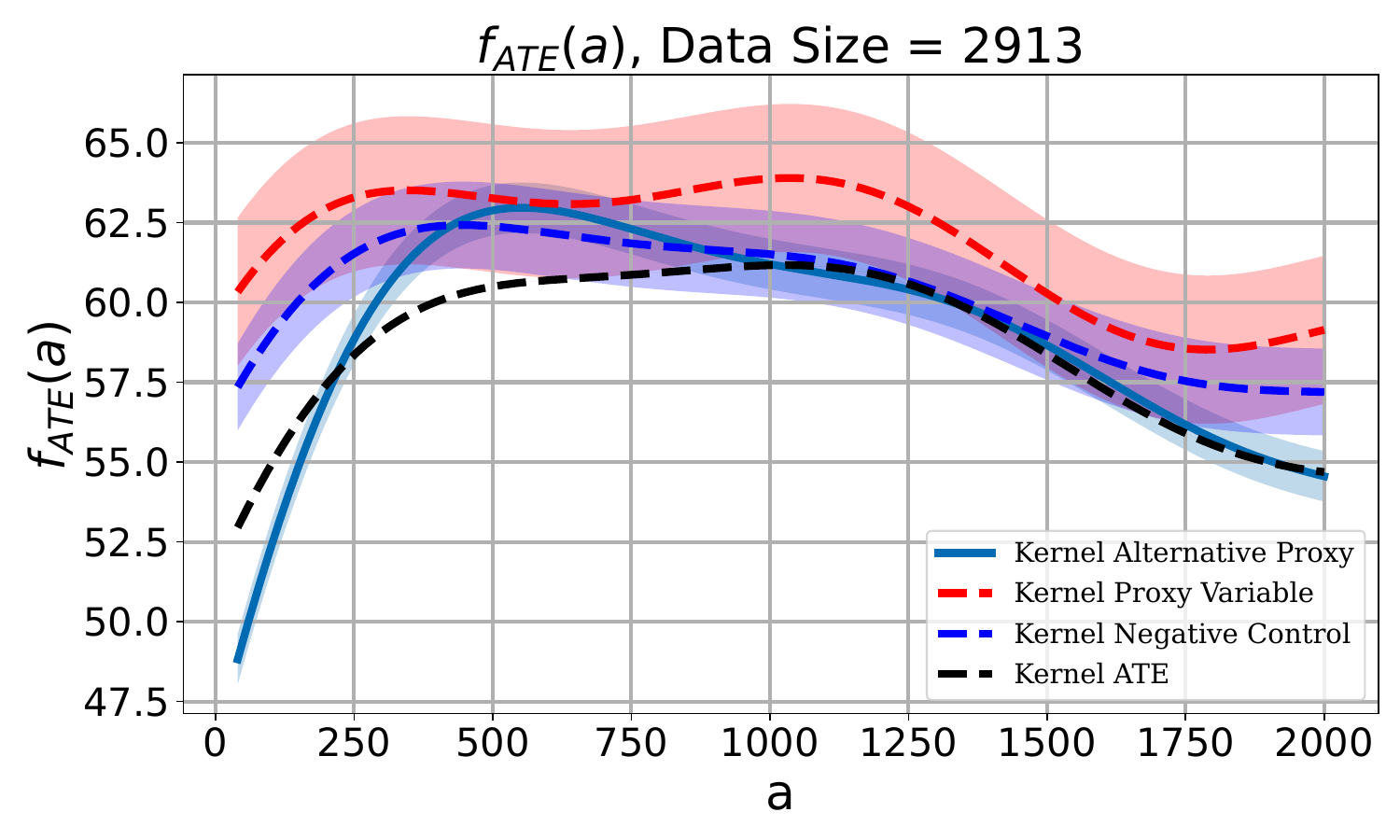}
\label{fig:JobCorpsComparison1}}
\subfloat[]
{\includegraphics[trim = {0cm 0cm 0cm 0.0cm},clip,width=0.225\textwidth]{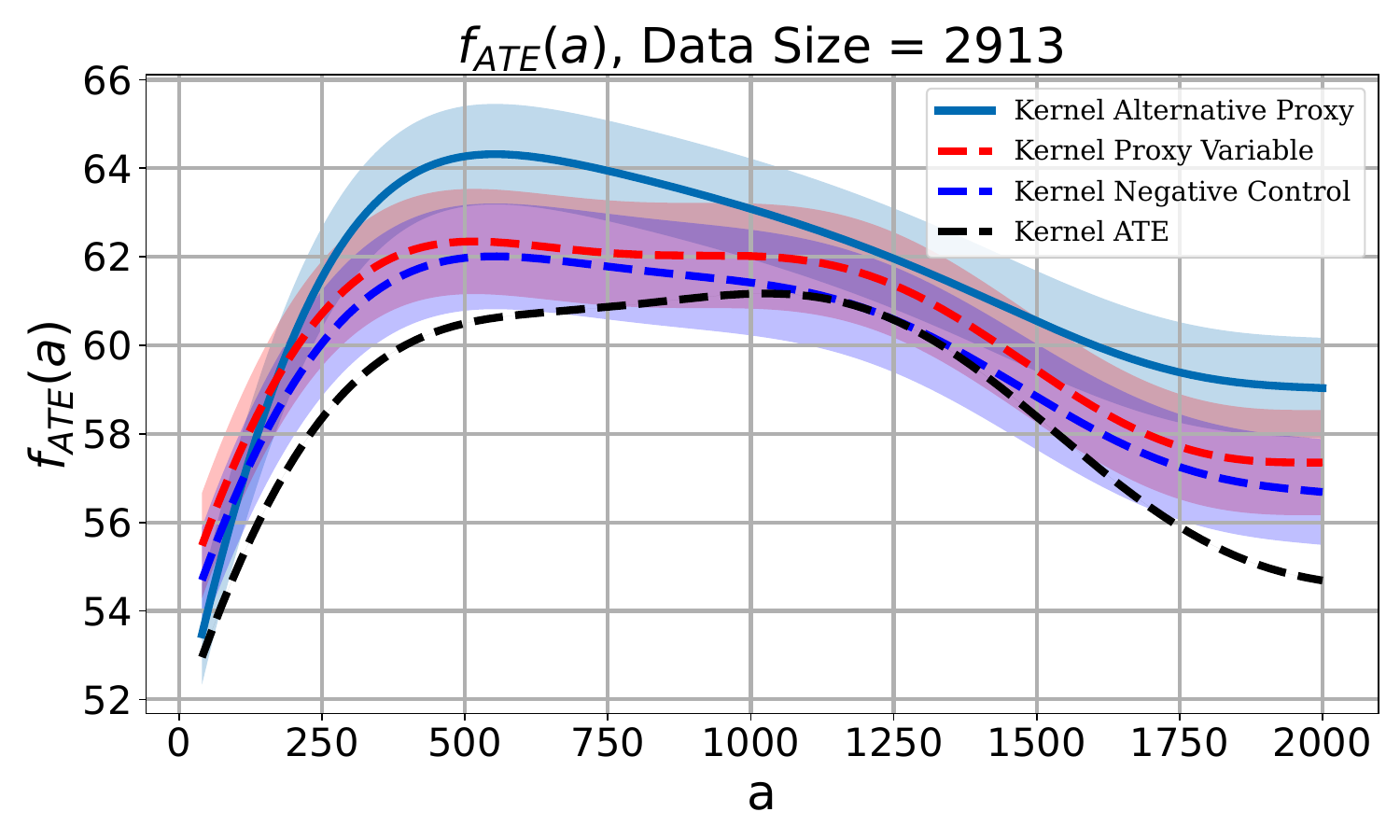}
\label{fig:JobCorpsComparison2}}
\subfloat[]
{\includegraphics[trim = {0cm 0cm 0cm 0.0cm},clip,width=0.225\textwidth]{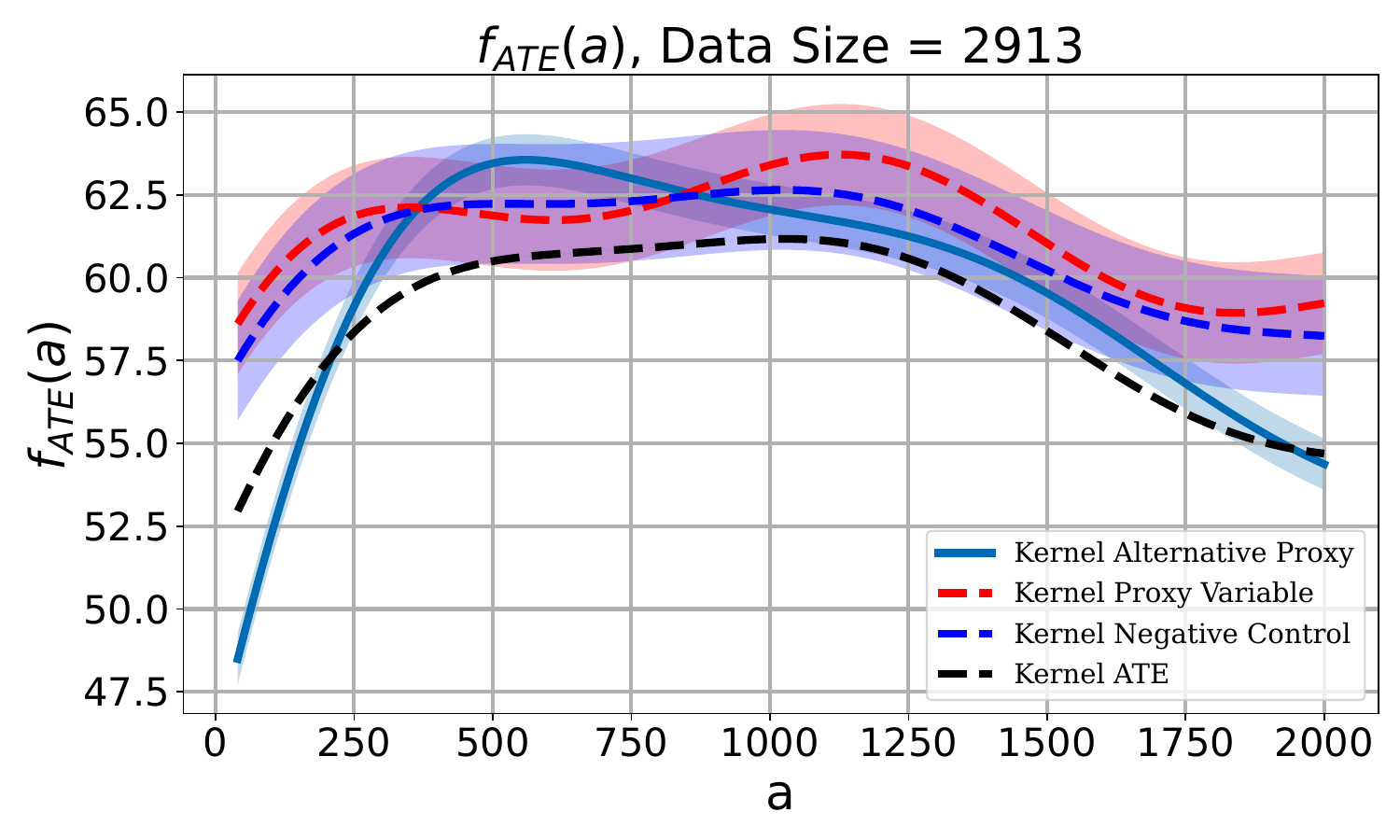}
\label{fig:JobCorpsComparison3}}
\subfloat[]
{\includegraphics[trim = {0cm 0cm 0cm 0.0cm},clip,width=0.225\textwidth]{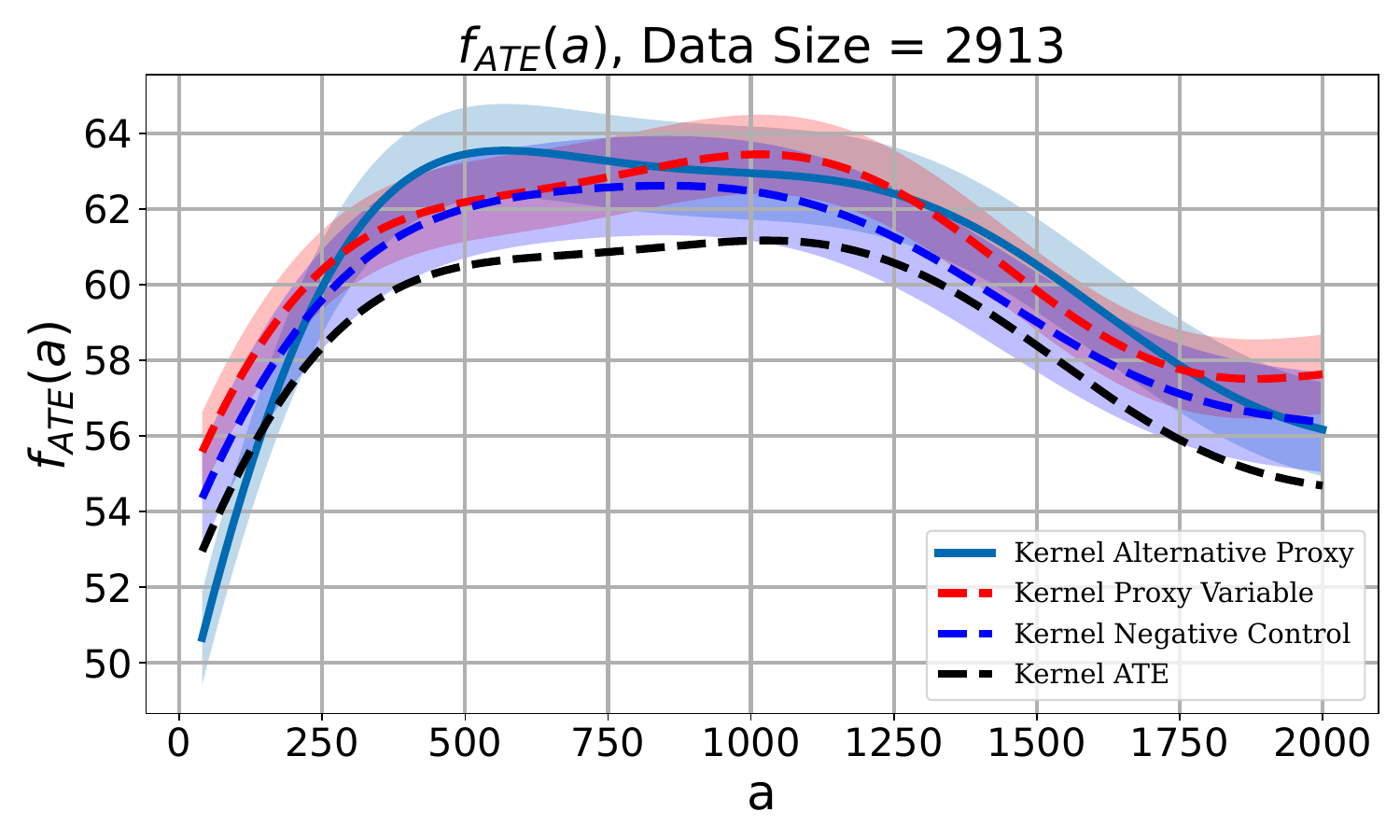}
\label{fig:JobCorpsComparison4}}

\centering
\subfloat[]
{\includegraphics[trim = {0cm 0cm 0cm 0.0cm},clip,width=0.225\textwidth]{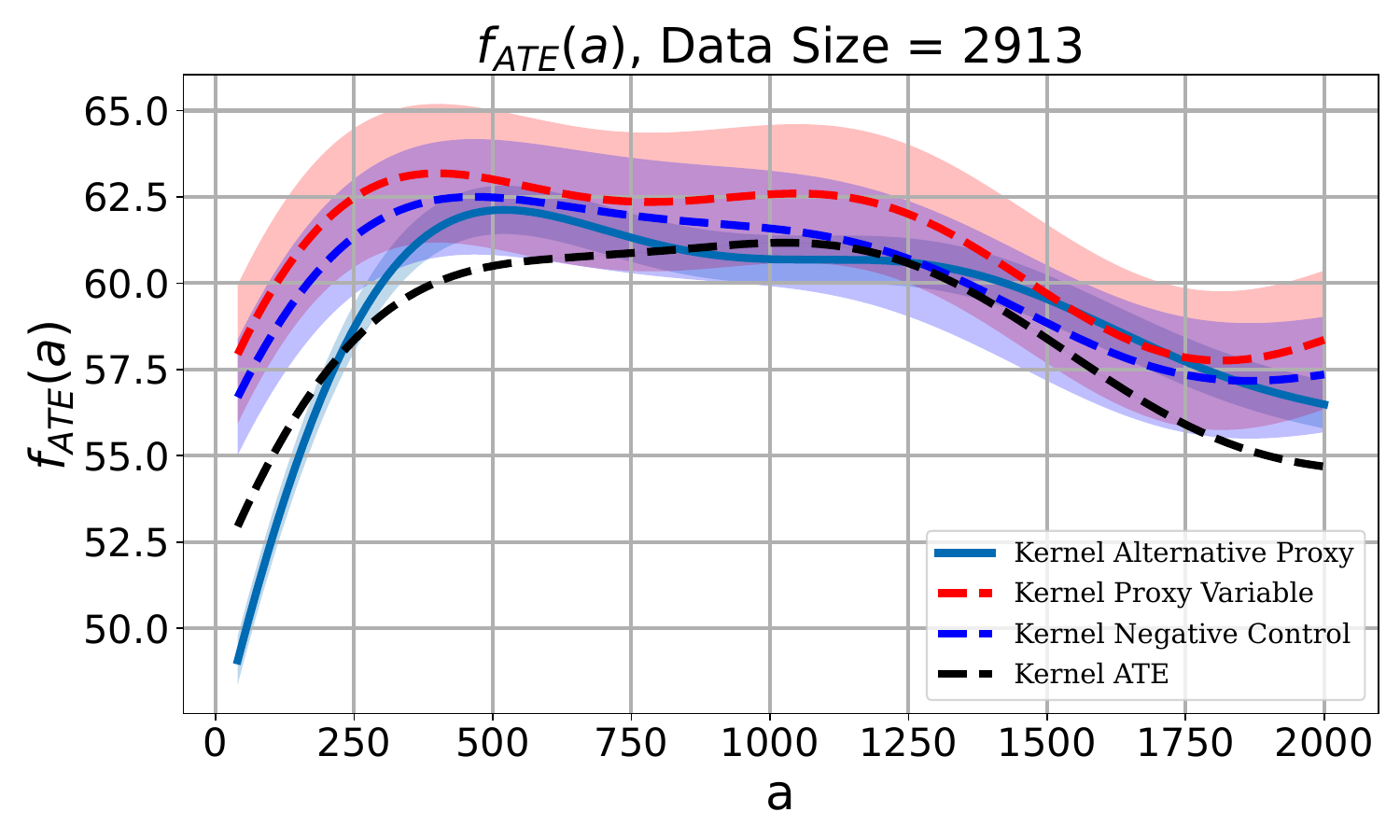}
\label{fig:JobCorpsComparison5}}
\subfloat[]
{\includegraphics[trim = {0cm 0cm 0cm 0.0cm},clip,width=0.225\textwidth]{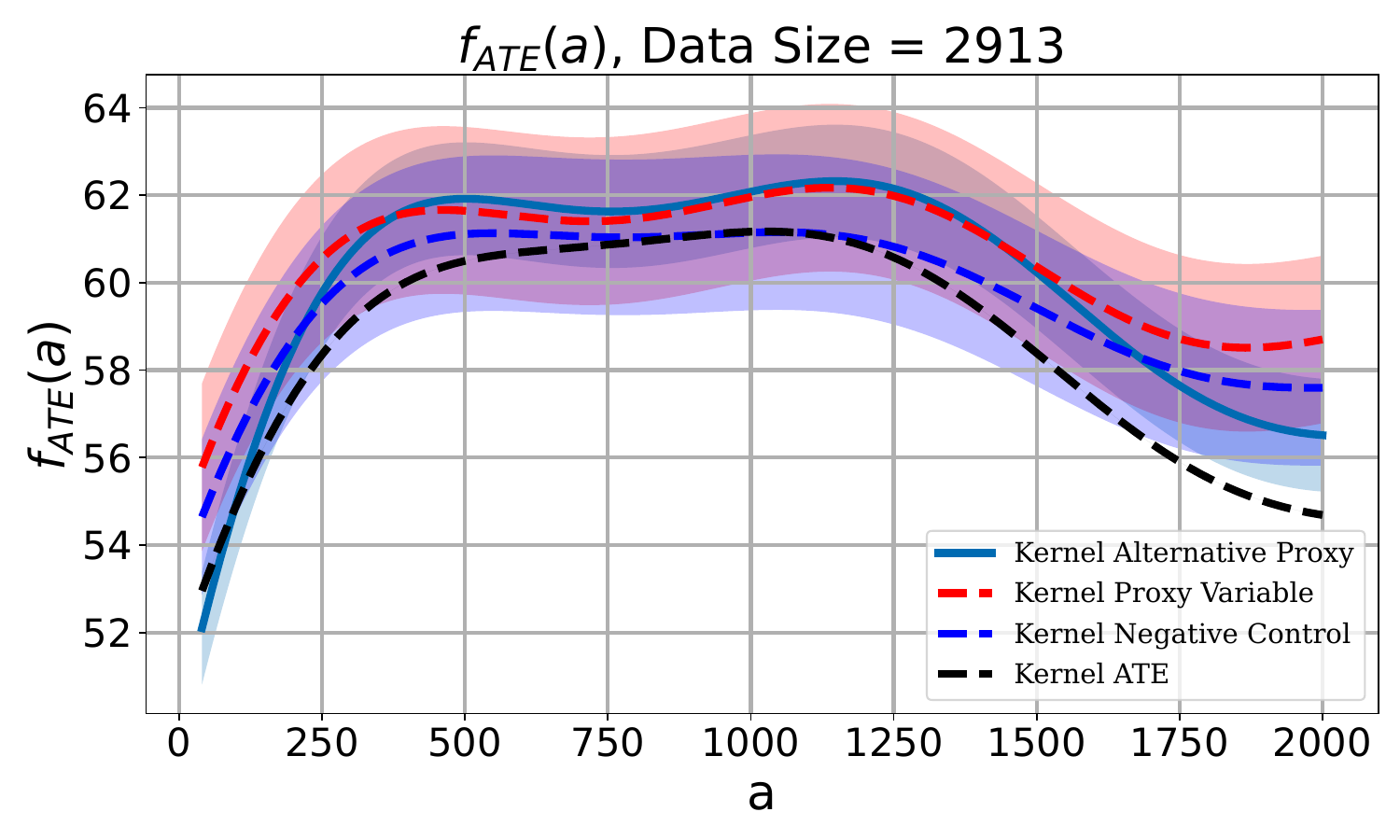}
\label{fig:JobCorpsComparison6}}
\subfloat[]
{\includegraphics[trim = {0cm 0cm 0cm 0.0cm},clip,width=0.225\textwidth]{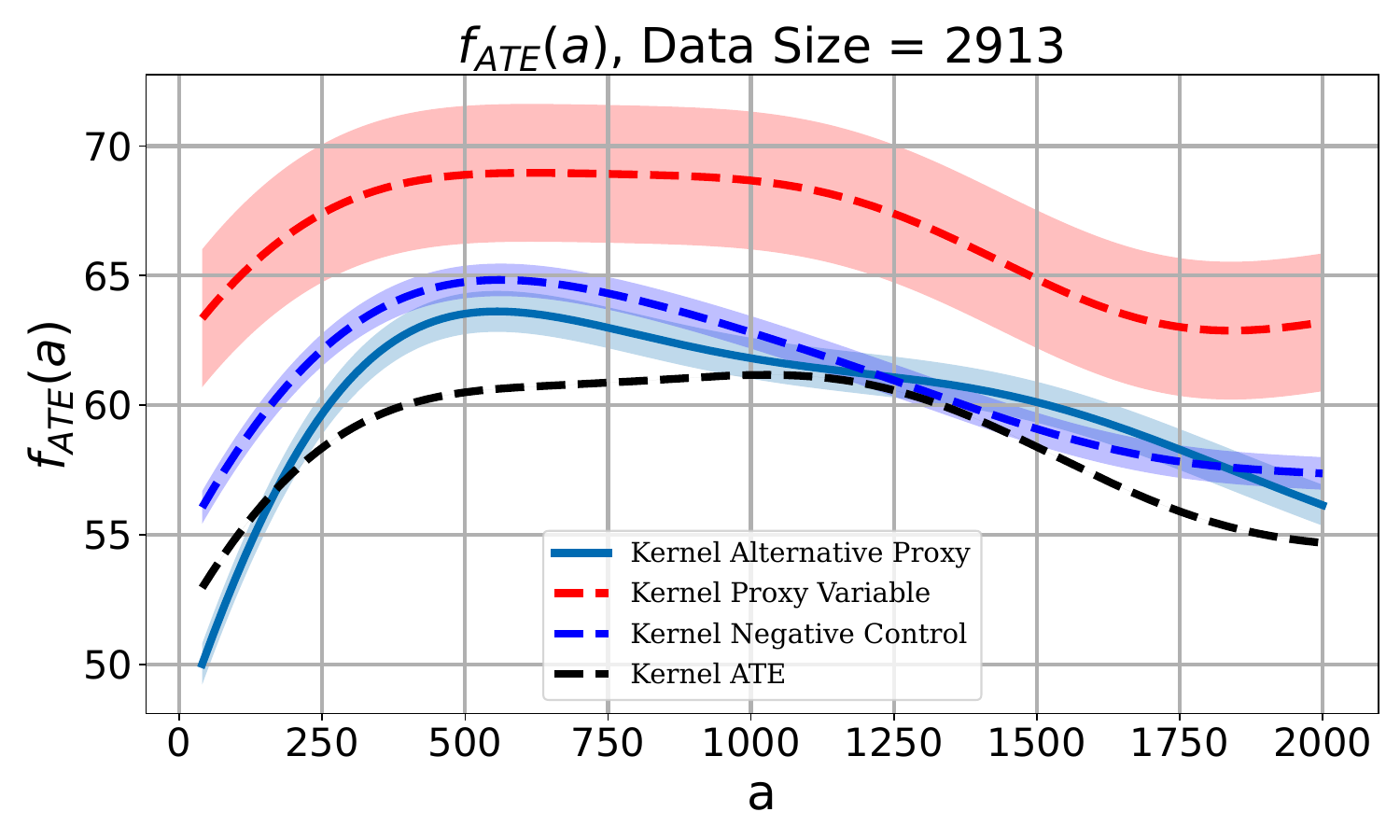}
\label{fig:JobCorpsComparison7}}
\subfloat[]
{\includegraphics[trim = {0cm 0cm 0cm 0.0cm},clip,width=0.225\textwidth]{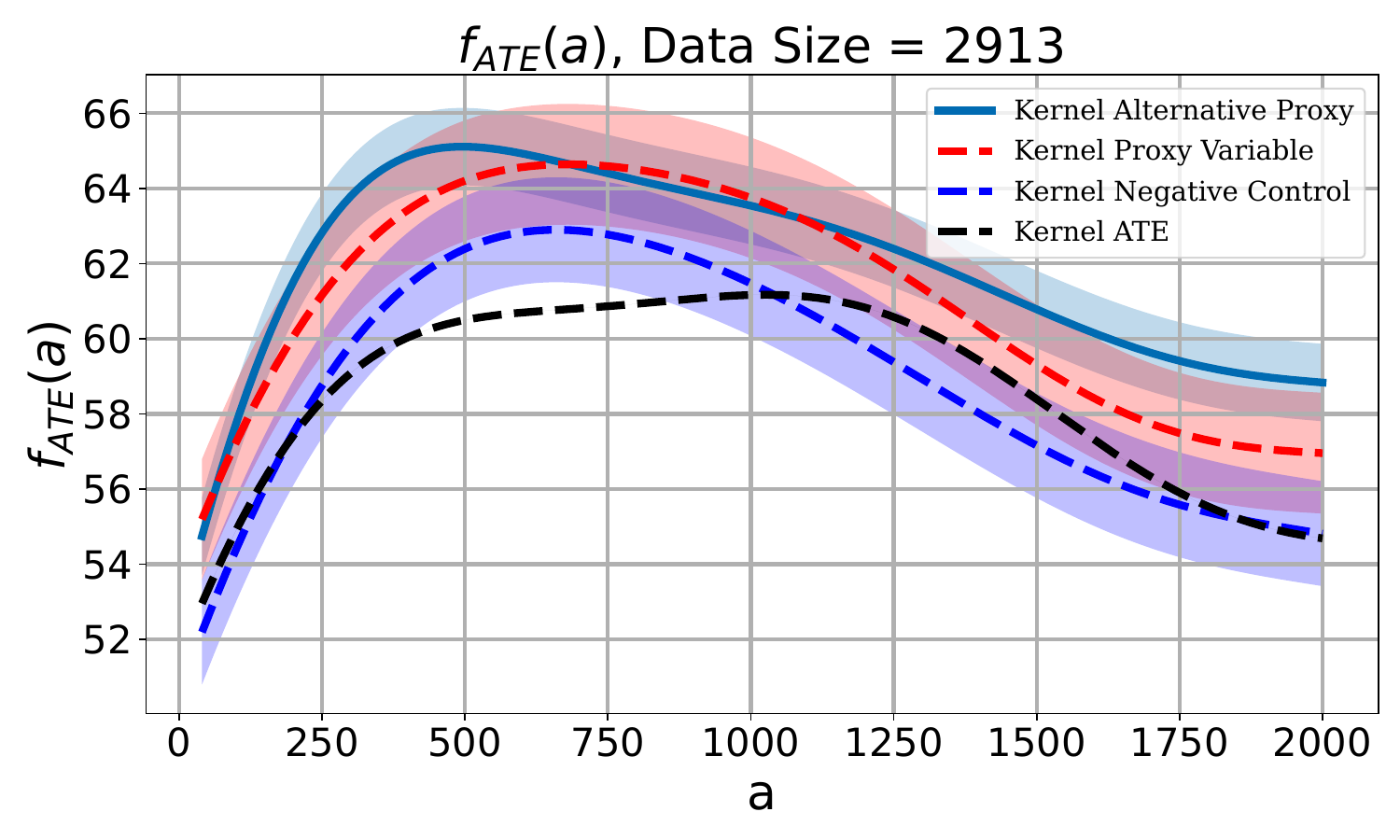}
\label{fig:JobCorpsComparison8}}

\centering
\hspace{0.25cm}
\subfloat[]
{\includegraphics[trim = {0cm 0cm 0cm 0.0cm},clip,width=0.225\textwidth]{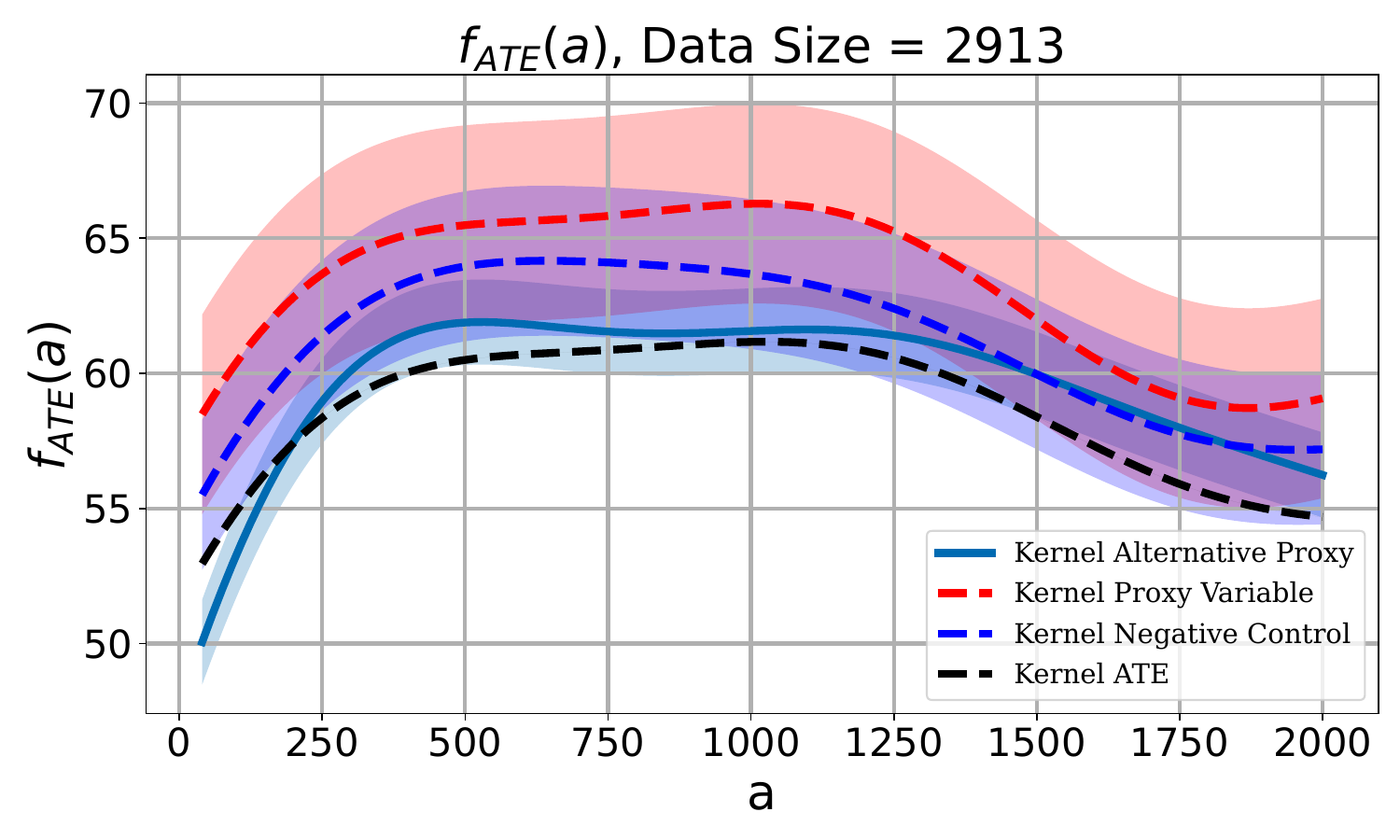}
\label{fig:JobCorpsComparison9}}
\subfloat[]
{\includegraphics[trim = {0cm 0cm 0cm 0.0cm},clip,width=0.225\textwidth]{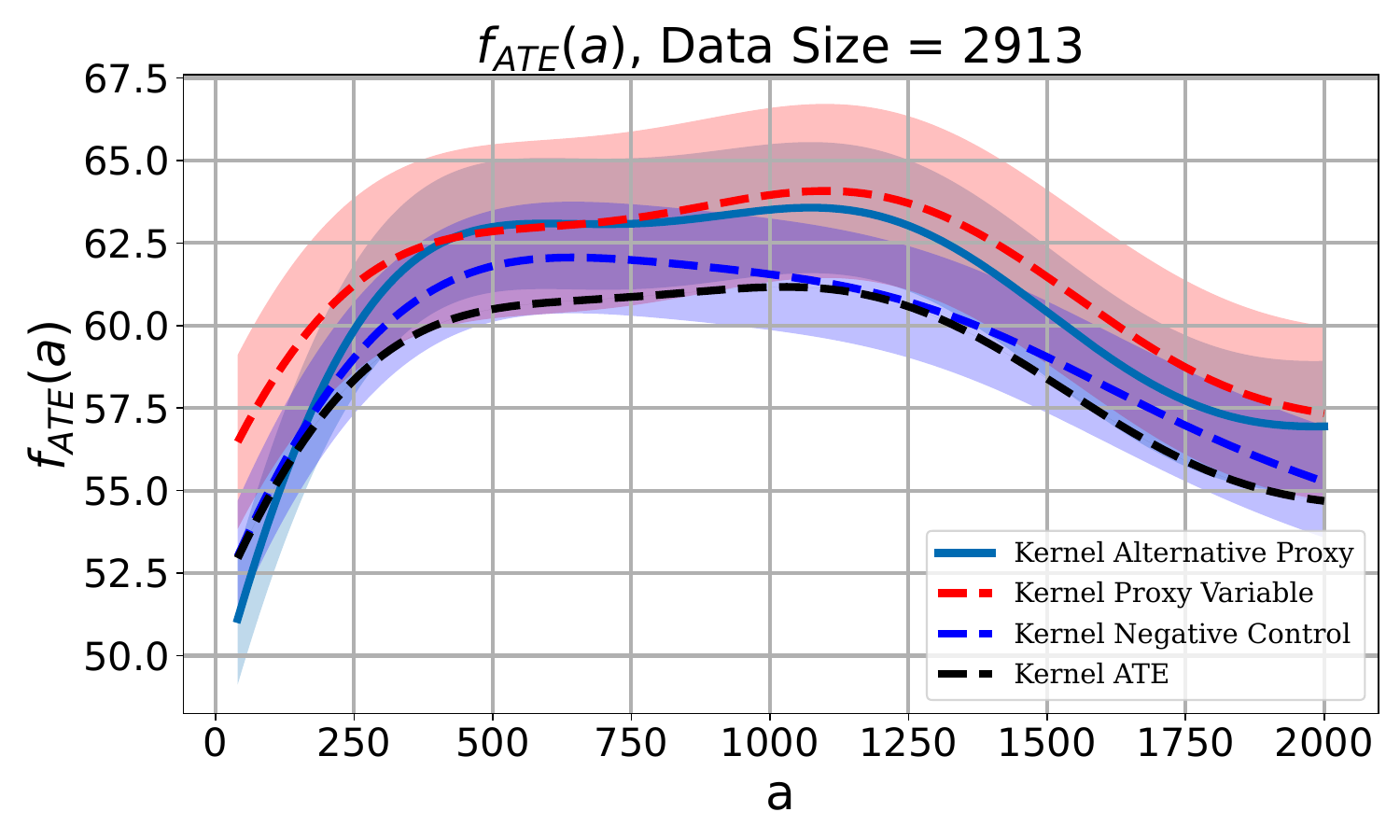}
\label{fig:JobCorpsComparison10}}
\subfloat[]
{\includegraphics[trim = {0cm 0cm 0cm 0.0cm},clip,width=0.225\textwidth]{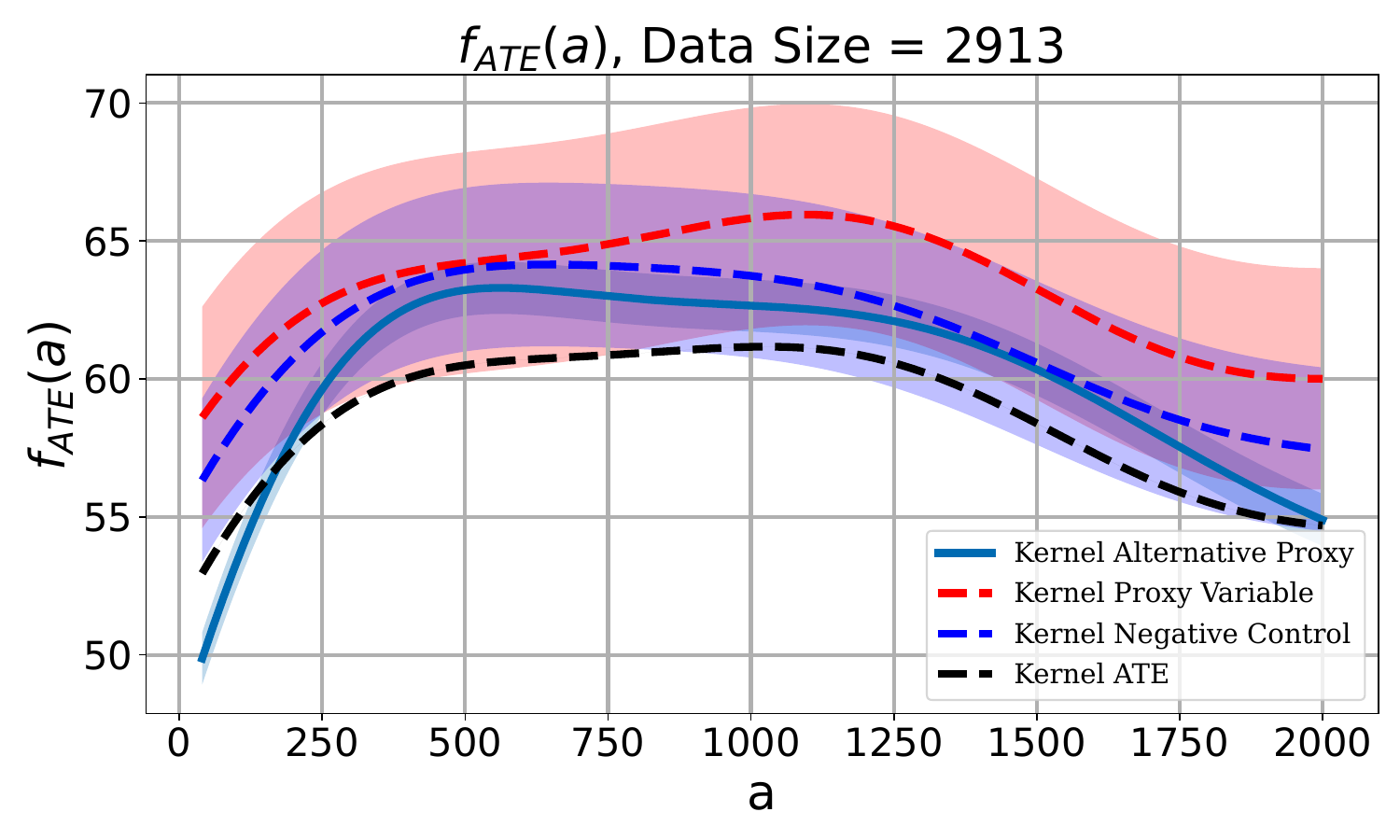}
\label{fig:JobCorpsComparison11}}
\subfloat[]
{\includegraphics[trim = {0cm 0cm 0cm 0.0cm},clip,width=0.225\textwidth]{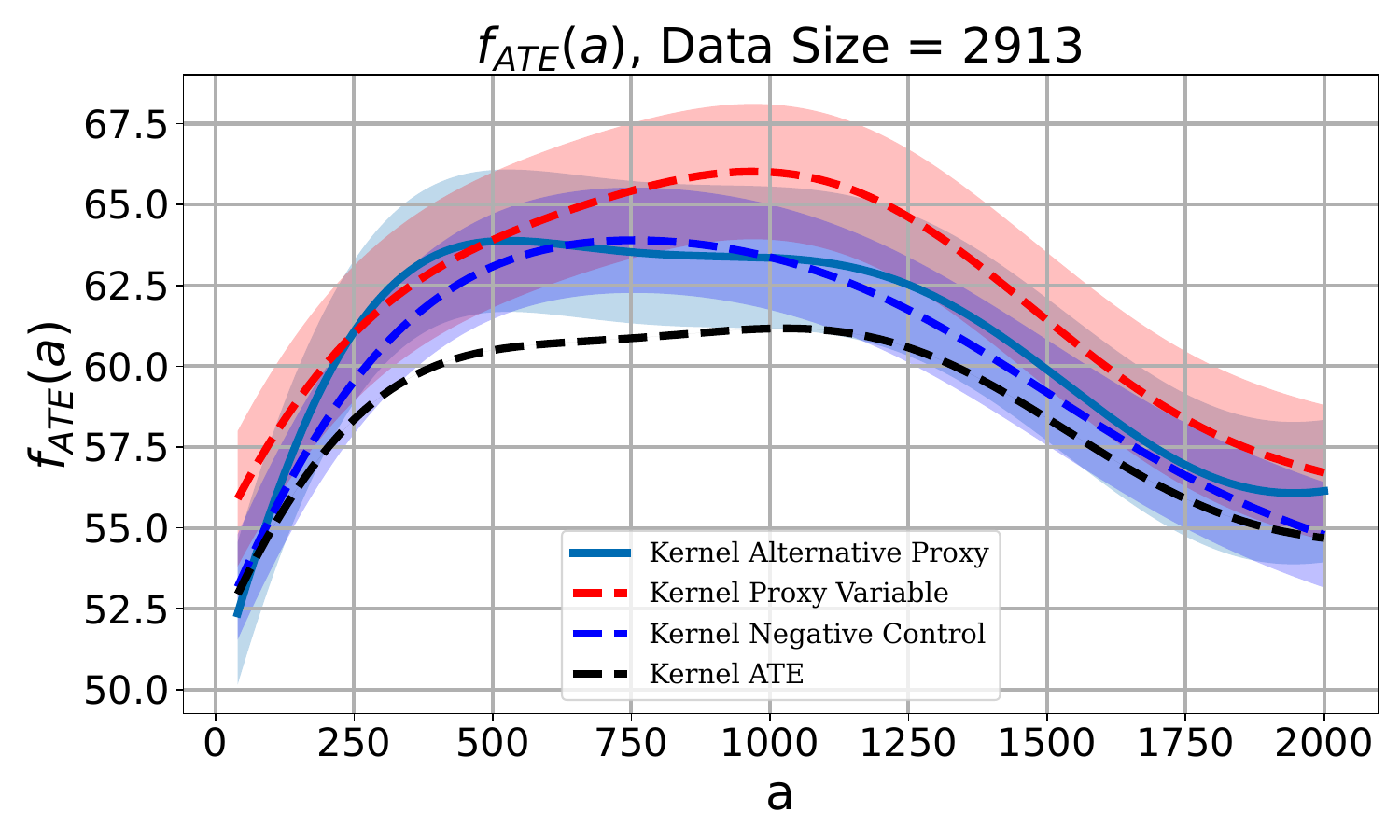}
\label{fig:JobCorpsComparison12}}
\newline\caption{Dose-response estimation curves for the Job Corps experimental settings that are introduced in S.M. (Sec. \ref{sec:Appendix_JobCorpsExperiments}). Panels (a)-(l) illustrate the estimation curves for our approach, KPV, KNC, and the oracle method Kernel-ATE across Settings 1-12, respectively.}
\label{fig:JobCorpsComparison}
\end{figure*}

To illustrate the conditional dose-response estimation capability of our proposed method in high dimensional settings, we also conduct experiments using the Job Corps dataset in Settings 1, 2, 5, and 6. Figures (\ref{fig:JobCorpsATTComparison1}) and (\ref{fig:JobCorpsATTComparison2}) show the ATT estimation results of our algorithm in comparison with the Kernel Negative Control method and the Kernel-ATT algorithm for $a' = 500$ and $a' = 1000$ in Setting 1, respectively. Figures (\ref{fig:JobCorpsATTComparison3}) and (\ref{fig:JobCorpsATTComparison4}) present the ATT estimation results for Setting 2. Figures (\ref{fig:JobCorpsATTComparison5}) and (\ref{fig:JobCorpsATTComparison6}) show the results for Setting 5, while Figures (\ref{fig:JobCorpsATTComparison7}) and (\ref{fig:JobCorpsATTComparison8}) provide results for Setting 6. 

\begin{figure*}[ht!]
\centering
\subfloat[]
{\includegraphics[trim = {0cm 0cm 0cm 0.0cm},clip,width=0.225\textwidth]{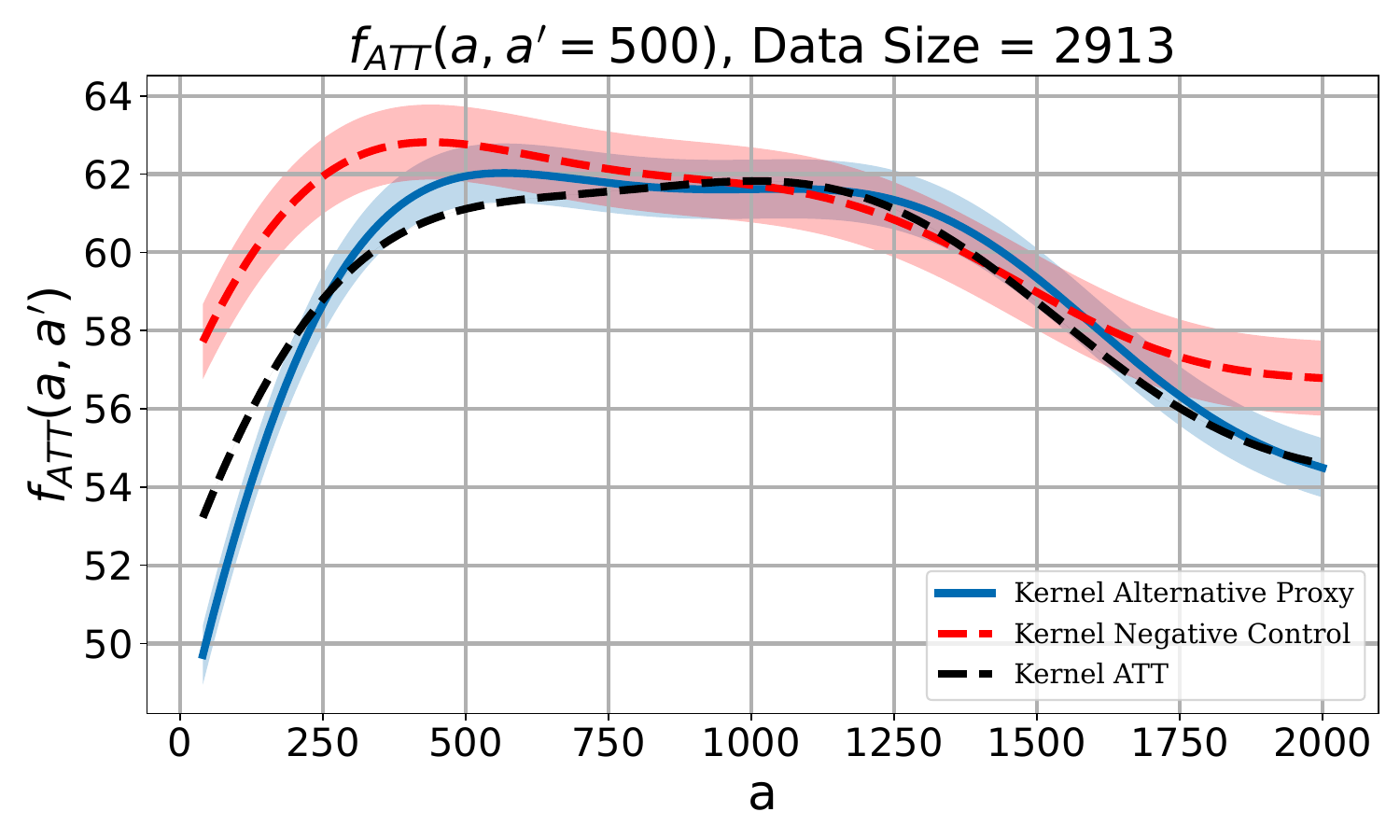}
\label{fig:JobCorpsATTComparison1}}
\subfloat[]
{\includegraphics[trim = {0cm 0cm 0cm 0.0cm},clip,width=0.225\textwidth]{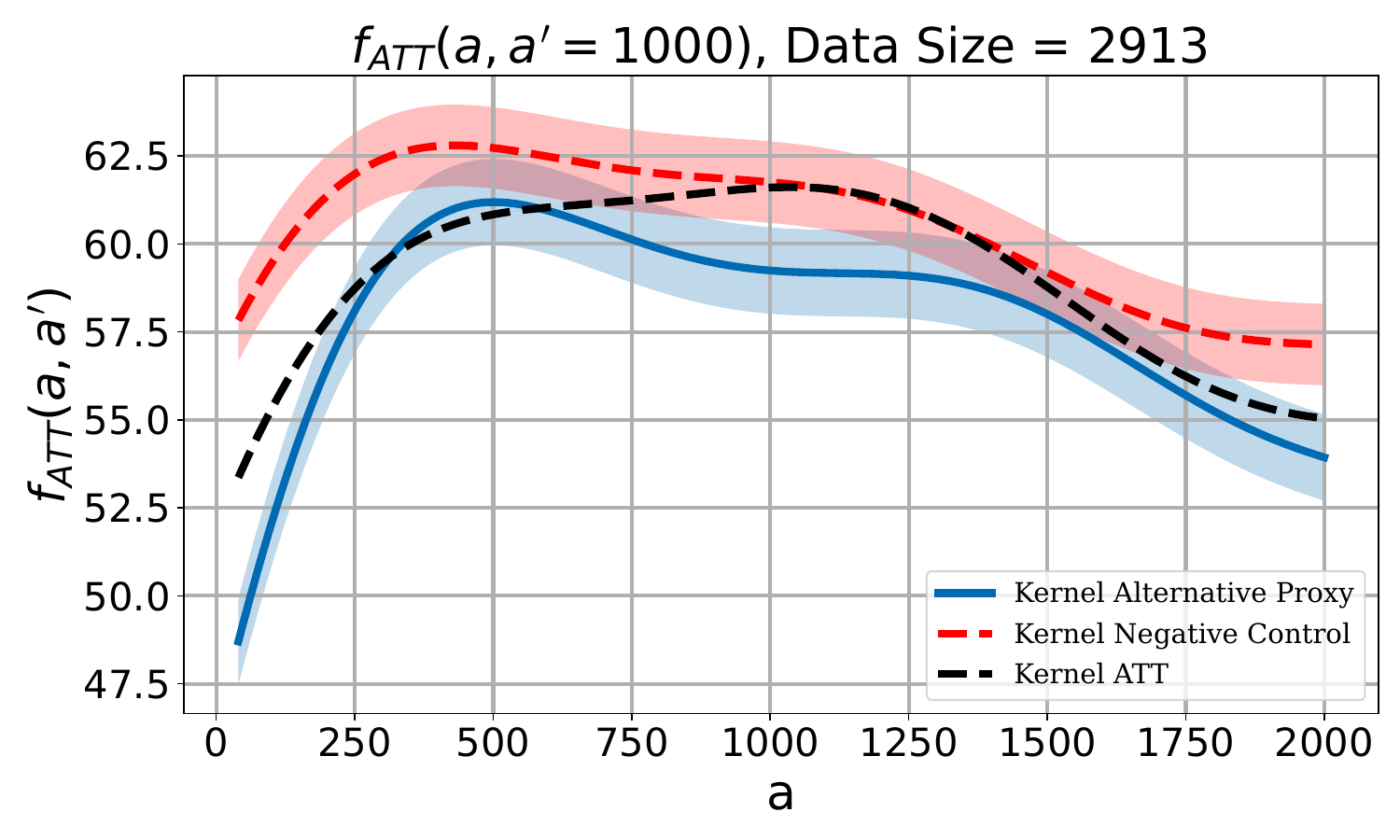}
\label{fig:JobCorpsATTComparison2}}
\subfloat[]
{\includegraphics[trim = {0cm 0cm 0cm 0.0cm},clip,width=0.225\textwidth]{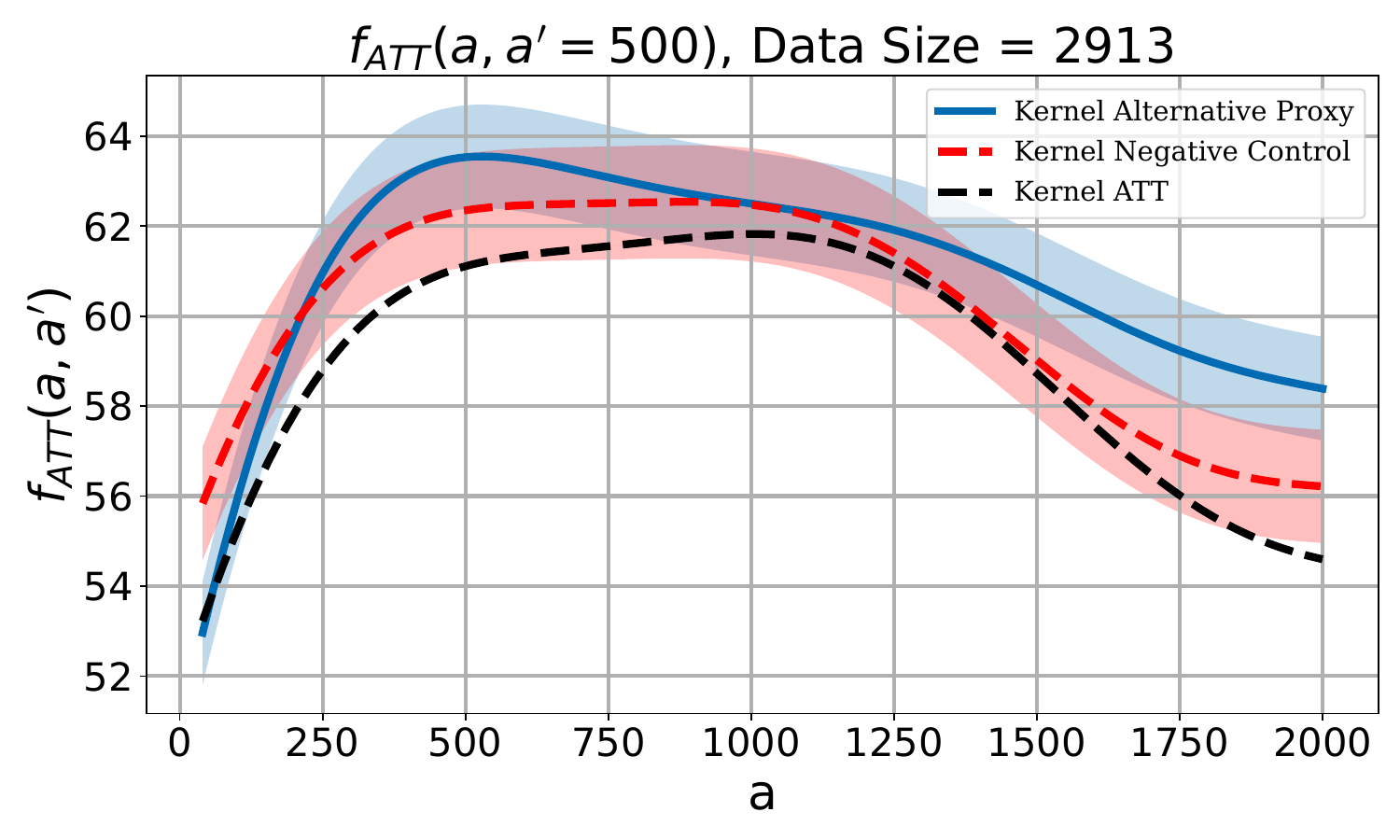}
\label{fig:JobCorpsATTComparison3}}
\subfloat[]
{\includegraphics[trim = {0cm 0cm 0cm 0.0cm},clip,width=0.225\textwidth]{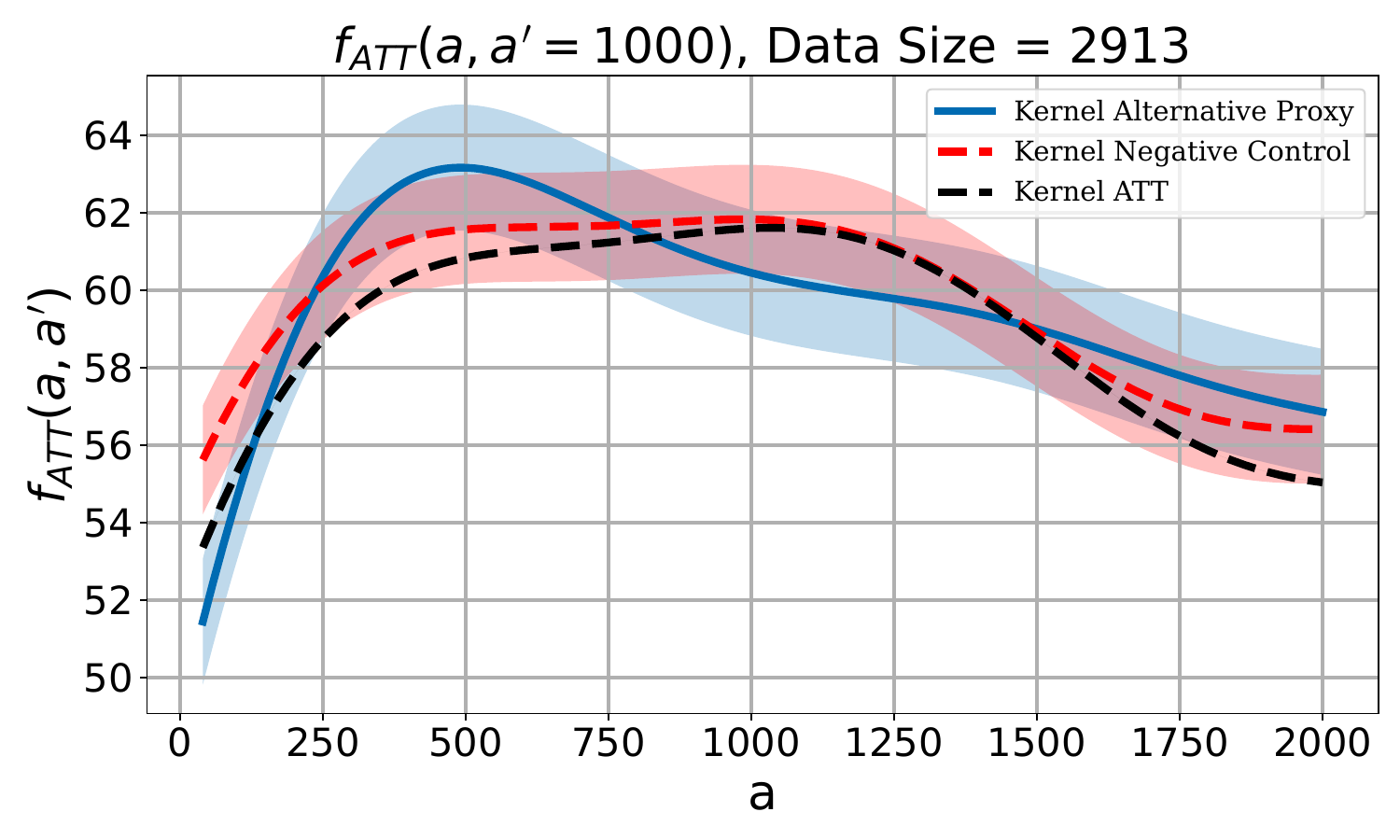}
\label{fig:JobCorpsATTComparison4}}

\centering
\hspace{0.2cm}
\subfloat[]
{\includegraphics[trim = {0cm 0cm 0cm 0.0cm},clip,width=0.225\textwidth]{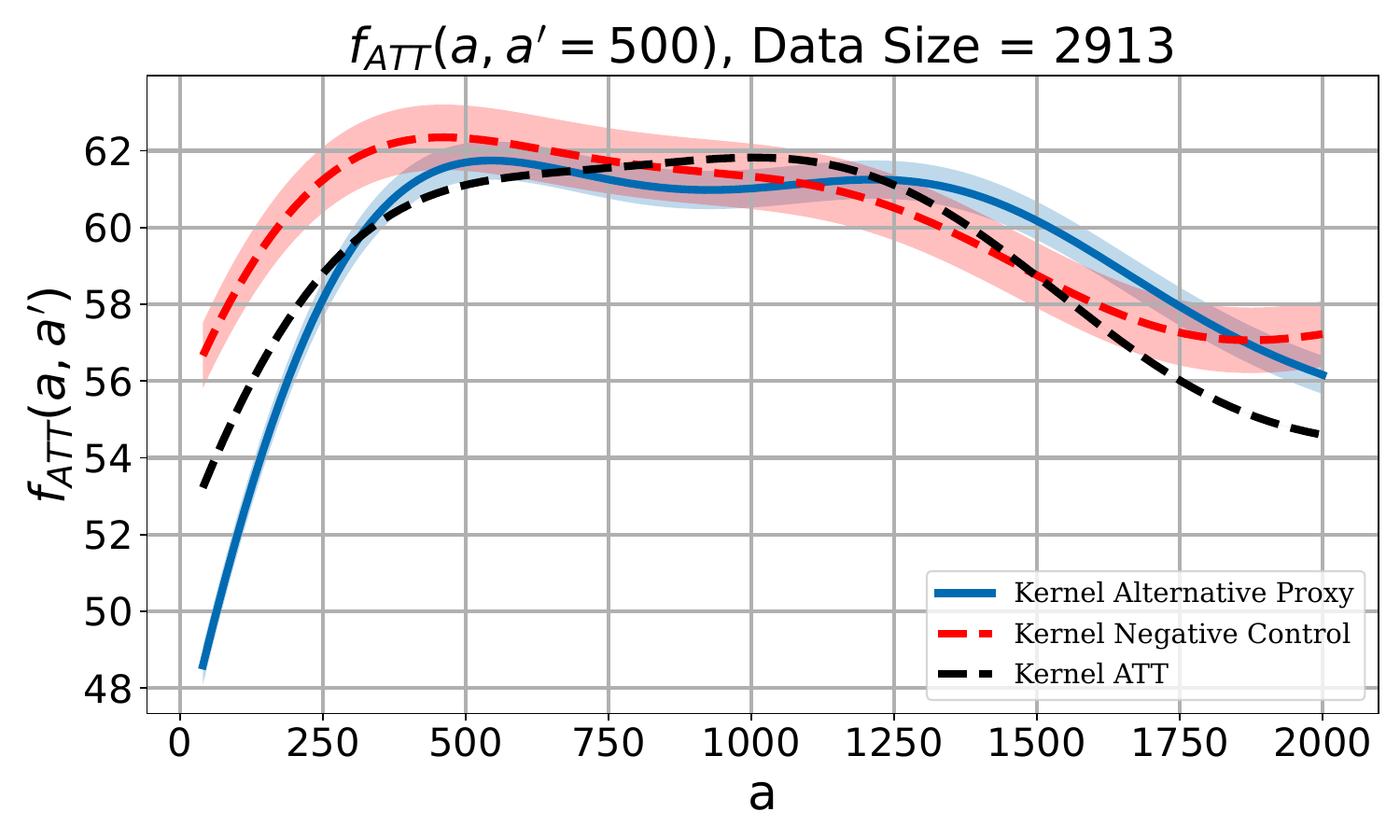}
\label{fig:JobCorpsATTComparison5}}
\subfloat[]
{\includegraphics[trim = {0cm 0cm 0cm 0.0cm},clip,width=0.225\textwidth]{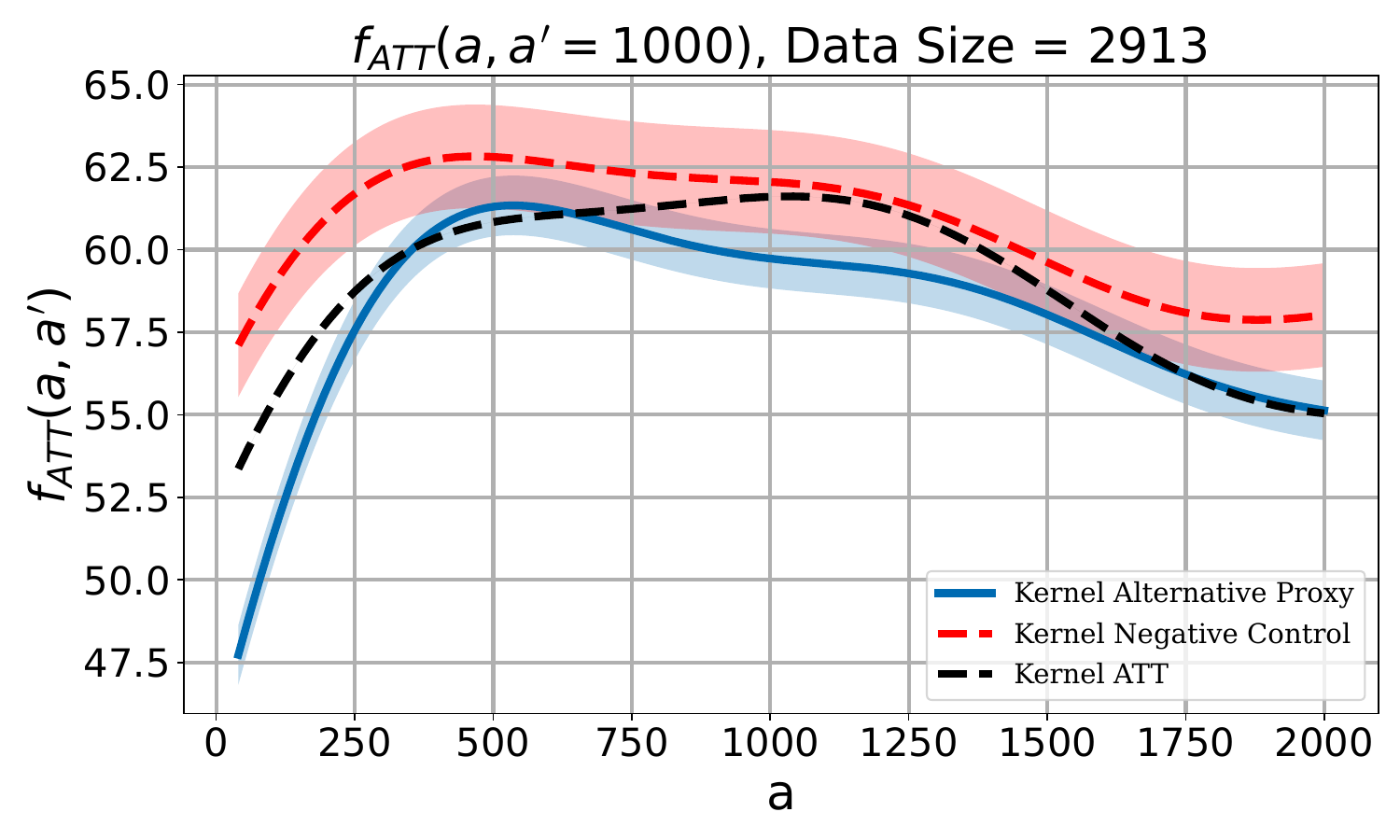}
\label{fig:JobCorpsATTComparison6}}
\subfloat[]
{\includegraphics[trim = {0cm 0cm 0cm 0.0cm},clip,width=0.225\textwidth]{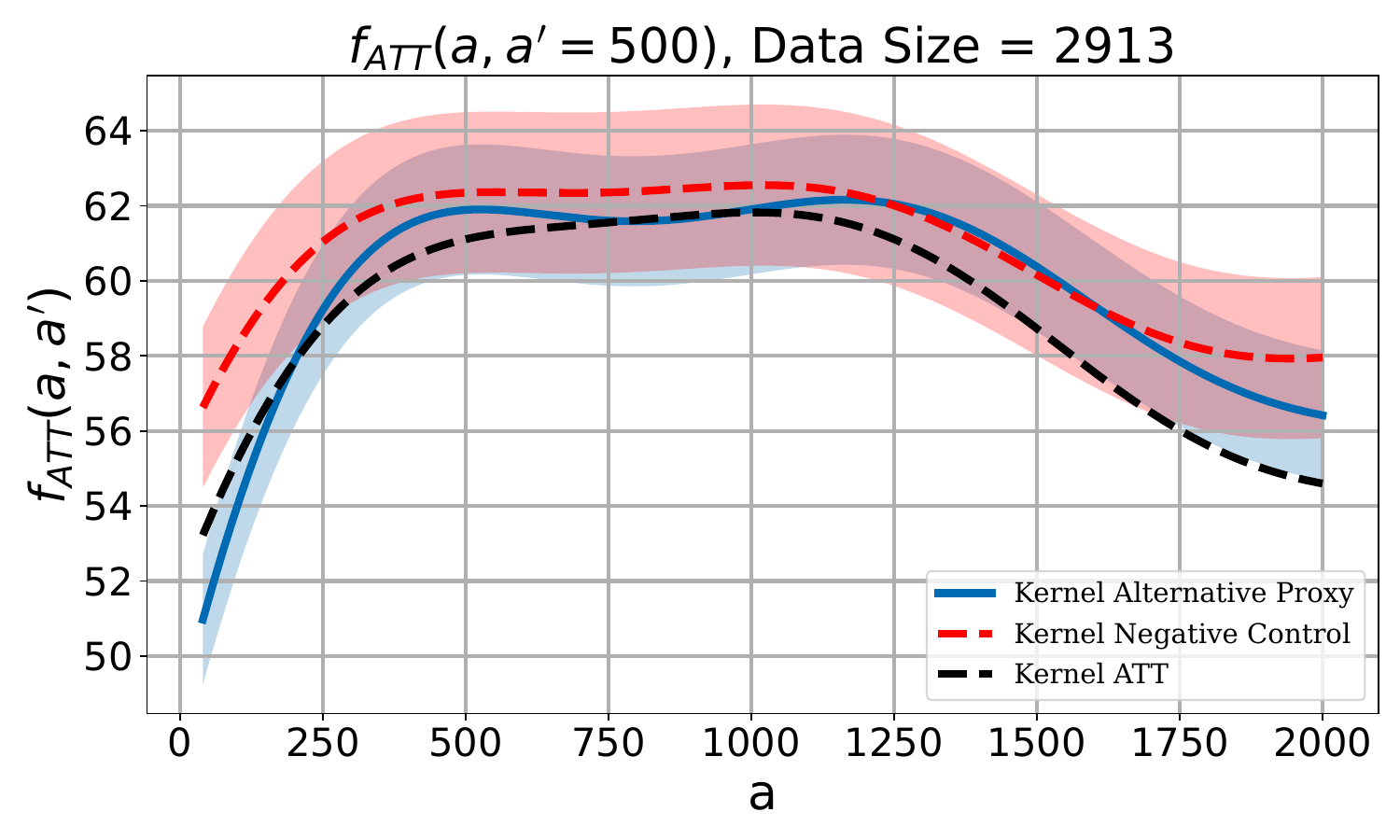}
\label{fig:JobCorpsATTComparison7}}
\subfloat[]
{\includegraphics[trim = {0cm 0cm 0cm 0.0cm},clip,width=0.225\textwidth]{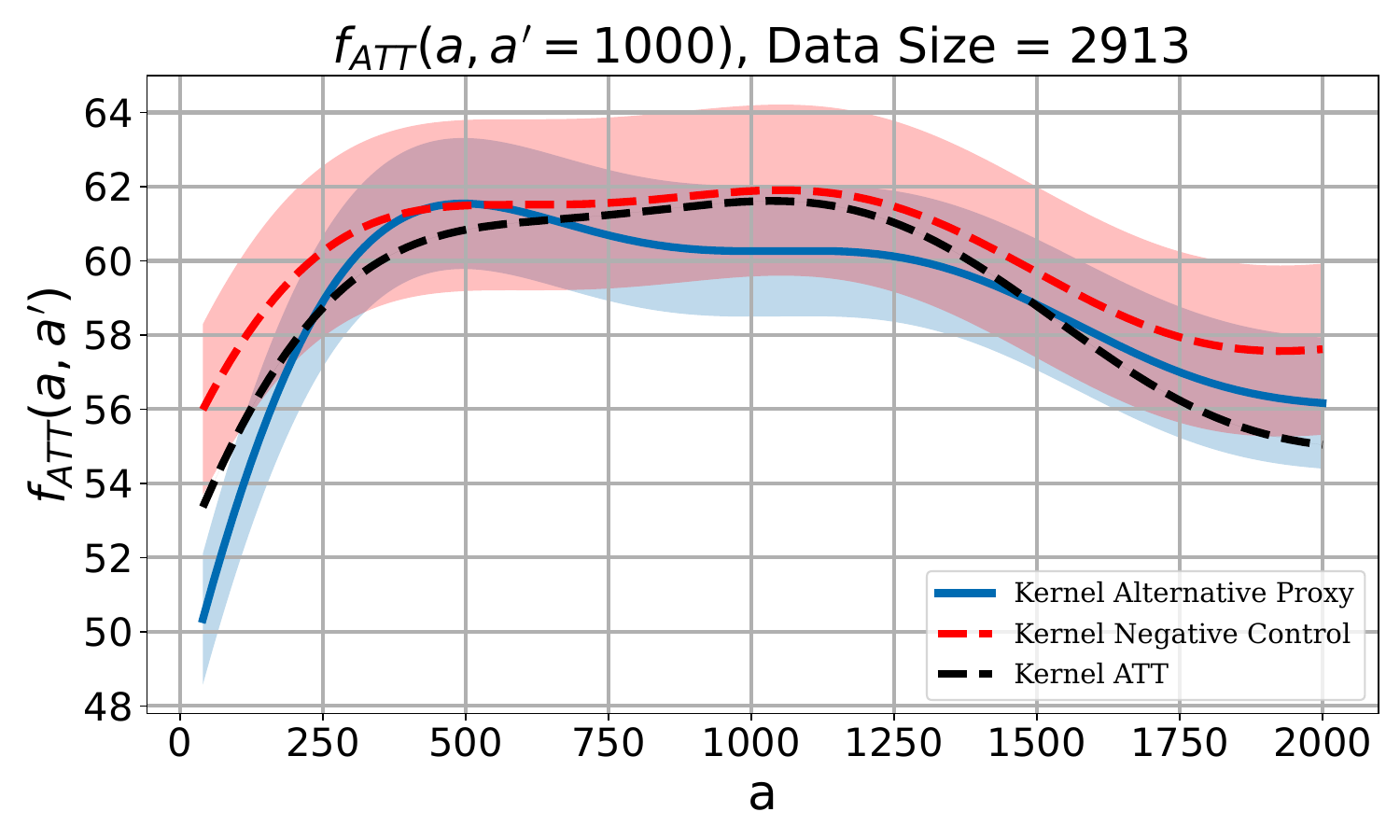}
\label{fig:JobCorpsATTComparison8}}
\newline\caption{Conditional dose-response estimation curves for Job Corps experimental settings $1$, $2$, $5$, and $6$ that are introduced in S.M. (Sec. \ref{sec:Appendix_JobCorpsExperiments}). Panels (a) and (b) show estimation curves for our approach, KNC, and the oracle method Kernel-ATT in Setting $1$ for $a' = 500$ and $a' = 1000$, respectively. Panels (c) and (d) display the corresponding curves for Setting $2$. Similarly, panels (e) and (f) illustrate the results for Setting $5$, while panels (g) and (h) present those for Setting $6$.}
\label{fig:JobCorpsATTComparison}
\end{figure*}

\subsubsection{Ablation Study on the Effect of Bandwidth Selection of Gaussian Kernel}

In this section, we present an ablation study to investigate the effect of kernel bandwidth selection on the performance of our dose-response curve estimation algorithm. Figure (\ref{fig:KernelBandwidthAblation}) illustrates the performance of our proposed method across different bandwidth selections for the kernels $k_\gA(.,.)$, $k_\gW(.,.)$, and $k_\gZ(.,.)$. Specifically, Figure (\ref{fig:KernelBandwidthAblation_kernelA}) shows the performance our dose-response curve estimation algorithm for various bandwidth values for $k_\gA(.,.)$ in the low-dimensional data generation setting (see Section (\ref{sec:NumericalExperiments})). The bandwidths of other kernels are set using median heuristic. Similarly, Figure (\ref{fig:KernelBandwidthAblation_kernelW}) and (\ref{fig:KernelBandwidthAblation_kernelZ}) depicts the performance our method across different bandwidth selection for the kernels $k_\gW(., .)$ and $k_\gZ(.,.)$, respectively.

We observe that while median heuristic does not always yield the best result, it generally produces robust or comparable results. Although one could perform a grid search on the kernel bandwidth to minimize the validation error in the second stage (see Equation (\ref{eq:secondstage-holdout_loss})), this procedure introduces additional search complexity. Therefore, for simplicity, we opted to use median heuristic in our experiments.  

\begin{figure*}[ht!]
\centering
\subfloat[]
{\includegraphics[trim = {0cm 0cm 0cm 0.0cm},clip,width=0.325\textwidth]{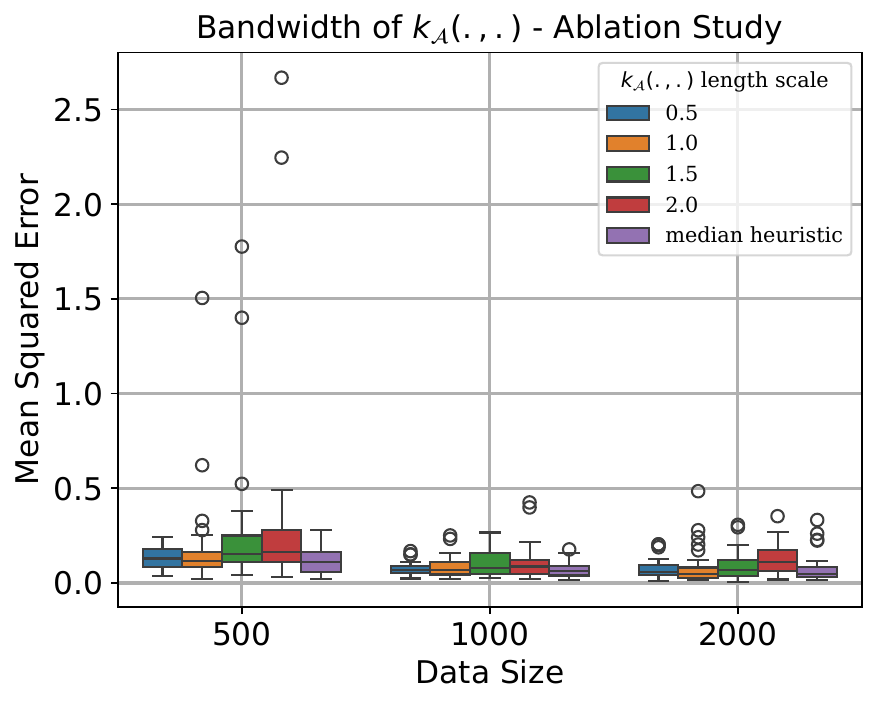}
\label{fig:KernelBandwidthAblation_kernelA}}
\subfloat[]
{\includegraphics[trim = {0cm 0cm 0cm 0.0cm},clip,width=0.325\textwidth]{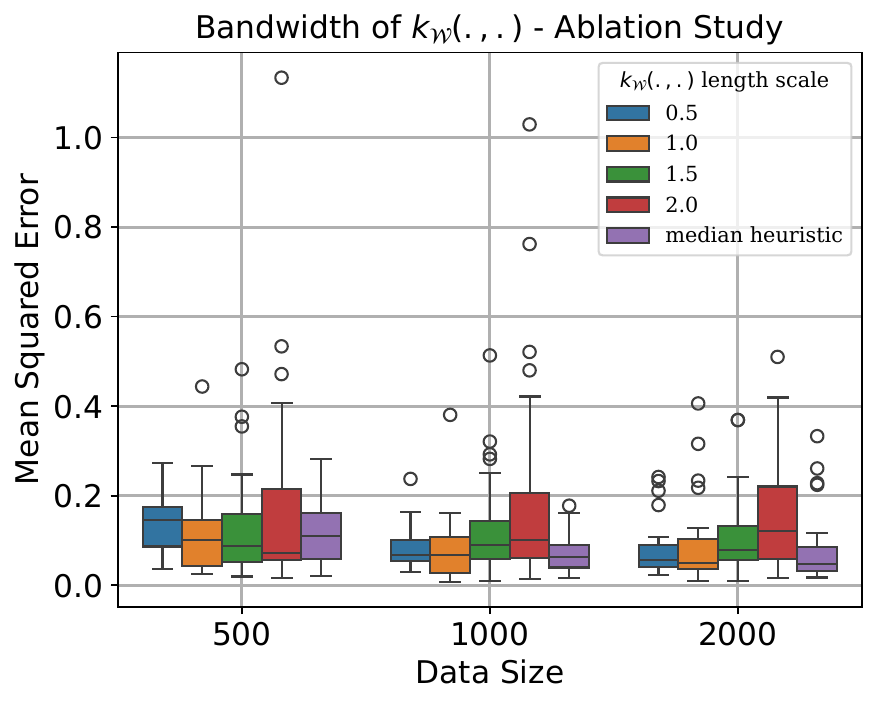}
\label{fig:KernelBandwidthAblation_kernelW}}
\subfloat[]
{\includegraphics[trim = {0cm 0cm 0cm 0.0cm},clip,width=0.325\textwidth]{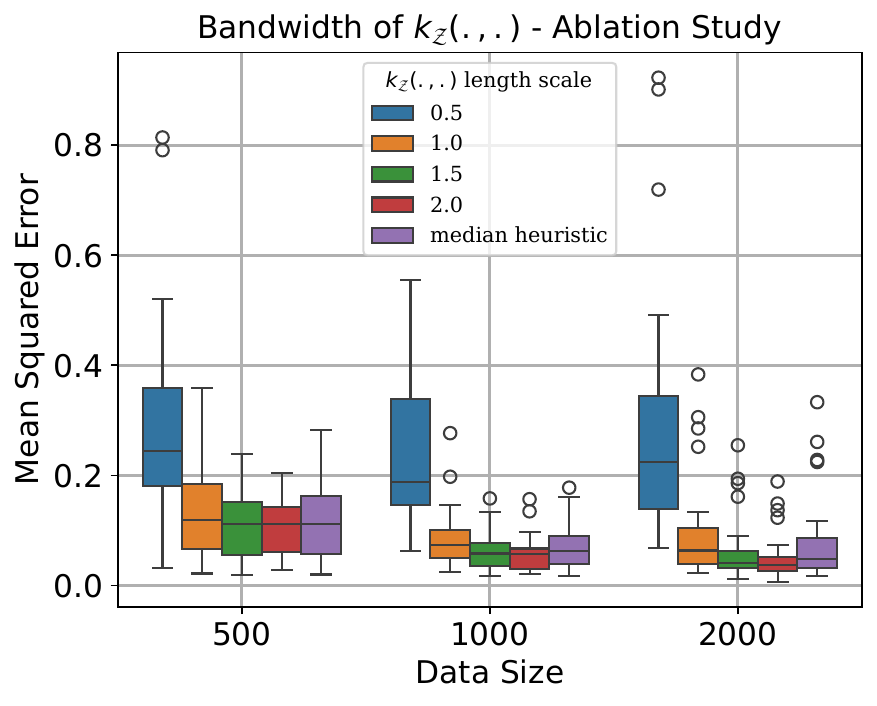}
\label{fig:KernelBandwidthAblation_kernelZ}}
\newline\caption{Ablation study on kernel bandwidth selection in the low-dimensional synthetic data experiment: (a) effect of the bandwidth of kernel $k_\gA(., .)$ on the performance, (b) effect of the bandwidth of kernel $k_\gW(., .)$ on the performance, (c) effect of the bandwidth of kernel $k_\gZ(., .)$ on the performance.}
\label{fig:KernelBandwidthAblation}
\end{figure*}

\section{IDENTIFIABILITY OF DOSE-RESPONSE IN DISCRETE CASE}
\label{sec:IdentifiabilityOfDoseResponseDiscreteCase}
Here, we additionally present the identification of dose-response curve in the discrete variable case. Assume that each space $\gF \in \{\gY, \gW, \gZ, \gU\}$ are discrete and $\gF = \{1, 2, \ldots, d_\gF\}$. Note that the dose-response can be written in terms of matrix-vector products in this case:
\begin{align}
    f_{\text{ATE}}(a) &= \E_U[\E[Y | U, A = a]] = \sum_{i = 1}^{d_\gU} \E[Y | U = i, A = a] p(U = i) \nonumber\\
    &= \sum_{i = 1}^{d_\gU} \sum_{j = 1}^{d_\gY} d_j p(Y = j | U = i) p(U = i) = \vd_\gY^T \mP(Y | U, A = \va) \mP(U) \label{eq:ATE_discereteCaseEquation}
\end{align}
where $\vd_\gY = \begin{bmatrix}
    1 & 2 & \ldots & d_\gY
\end{bmatrix}^T \in \R^{d_\gY}$, and $\mP(Y | U, A = a) \in \R^{d_\gY \times d_\gU}$, $\mP(U) \in \R^{d_\gU}$ are probability matrices with 
\begin{align*}
    [\mP(Y | U, A = \va)]_{ij} &= p(Y = i | U = j, A = \va)\\
    [\mP(U)]_i &= p(U = i).
\end{align*}
Equation (\ref{eq:ATE_discereteCaseEquation}) cannot be computed directly since it involves the distribution of the unobserved confounder $U$. However, we can determine $f_{\text{ATE}}$ by only the observable variables ${Y, W, Z}$. This approach is formalized in the following theorem. \looseness=-1

\begin{theorem}
Given the observed variables ${Y, W, Z}$ from their corresponding discrete sets  $\{\gY, \gW, \gZ\}$, the dose-response curve can be calculated by
\begin{align}
    f_{ATE}(a) &= \vd_\gY^T \mP(Y, Z | A = a) \mP^{-T}(Z | W, A = a) \frac{1}{\mP(A = a | W)^T} p(A = a)
    \label{eq:ATEFormulaDiscreteCaseIdentification}
\end{align}
where $\mP(Y, Z | A = a) \in \R^{d_\gY \times d_\gZ}$ and $\mP(Z | W, A = a)\in \R^{d_\gZ \times d_\gW}$ are the probability matrices defined as     
\begin{align*}
    [\mP(Y, Z | A = a)]_{ij} &= p(Y = i, Z = j | A = a),\\
    [\mP(Z | W, A = a)]_{ij} &= p(Z = i | W = j, A = a),
\end{align*} 
$\mP^{-T}(Z | W, A = a) $ is the transpose of the inverse of the probability matrix $\mP(Z | W, A = a)$, and 
\begin{align*}
    \frac{1}{\mP(A = a | W)^T} &= \begin{bmatrix}
        \frac{1}{p(A = a | W = 1)} & \ldots & \frac{1}{p(A = a | W = d_\gW)} 
    \end{bmatrix} \in \R^{1 \times d_\gW}.
\end{align*}
\label{thm:AlternativeProxyATEIdentificationDiscrete}
\end{theorem}
 Note that every component in the Equation (\ref{eq:ATEFormulaDiscreteCaseIdentification}) is in terms of the observed variables and they can be estimated from the data.

\begin{proof}




We will prove the theorem in four steps:

\textbf{{Step 1:}} We will first prove that 
\begin{align}
    \mP(Z| W, A = a) = \mP(Z | U, A = a) \mP(U| W, A = a).
\label{eq:AlternativeProxyDiscreteATEProofEquation1}
\end{align}
First consider the right-hand side of the Equation (\ref{eq:AlternativeProxyDiscreteATEProofEquation1}):
\begin{align*}
    [\mP(Z | U, A = a) \mP(U| W, A = a)]_{ij} &= \sum_{k = 1}^{d_\gU} p(Z = i | U = k, A = a) p(U = k| W = j, A = a)\\
    &= \sum_{k = 1}^{d_\gU} p(Z = i | U = k, A = a, {W = j}) p(U = k| W = j, A = a)\\
    &\text{(the above equality is due to Assumption (\ref{assum:ProxyCausalAssumptions1}))}\\
    &= \sum_{k = 1}^{d_\gU} p(Z = i, U = k | A = a, {W = j}) \quad \text{(by Baye's Rule)}\\
    &= p(Z = i | W = j, A = a) = [\mP(Z | W, A = a)]_{ij}
\end{align*}
and this verifies the Equation (\ref{eq:AlternativeProxyDiscreteATEProofEquation1}). It also implies that 
\begin{align}
    \underbrace{\mP(U| W, A = a)}_{\in \R^{d_\gU \times d_\gW}} = \underbrace{\mP^{-1}(Z | U, A = a)}_{\in \R^{d_\gU \times d_\gZ}} \underbrace{\mP(Z| W, A = a)}_{\in \R^{d_\gZ \times d_\gW}}
\label{eq:AlternativeProxyDiscreteATEProofEquation2}
\end{align}
where $\mP^{-1}(Z | U, A = a)$ is the (left) inverse of the probability matrix $\mP(Z | U, A = a)$.

\textbf{{Step 2:}} Secondly, we want to show that
\begin{align}
    \mP^{-T}(Z | U, A = a) \frac{p(A = a)}{\mP(A = a | U)^T} &= \mP^{-T}(Z| W, A = a) \frac{p(A = a)}{\mP(A = a | W)^T}
\label{eq:AlternativeProxyDiscreteATEProofEquation3}
\end{align}
Equivalently, we will show that 
\begin{align*}
    \frac{1}{\mP(A = a | U)} \mP^{-1}(Z | U, A = a) = \frac{1}{\mP(A = a | W)} \mP^{-1}(Z | W, A = a)
\end{align*}
which also implies that
\begin{align}
    \frac{1}{\mP(A = a | U)} \mP^{-1}(Z | U, A = a) \mP(Z | W, A = a) = \frac{1}{\mP(A = a | W)} .
\label{eq:AlternativeProxyDiscreteATEProofEquation4}
\end{align}
Next, consider the left-hand side of the equation (\ref{eq:AlternativeProxyDiscreteATEProofEquation4}):
\begin{align*}
    &\Big[\frac{1}{\mP(A = a | U)} \mP^{-1}(Z | U, A = a) \mP(Z | W, A = a)\Big]_{i} = \Big[\frac{1}{\mP(A = a | U)} \mP(U | W, A = a)\Big]_{i} \quad \text{(by Equation (\ref{eq:AlternativeProxyDiscreteATEProofEquation2}))}\\
    &= \sum_{k = 1}^{d_\gU} \frac{1}{p(A = a | U = u)} p(U = k| W = i, A = a)\\
    &= \sum_{k = 1}^{d_\gU} \frac{1}{p(A = a | U = u)} \frac{p(A = a| U = k, W = i) p(U = k| W = i)}{p(A = a| W = i)} \quad \text{(by Baye's Rule)}\\
    &= \sum_{k = 1}^{d_\gU} \frac{1}{p(A = a | U = u)} \frac{p(A = a| U = k) p(U = k| W = i)}{p(A = a| W = i)} \quad \text{(since $W \perp A | U$, Assumption (\ref{assum:ProxyCausalAssumptions1}))}\\
    &= \frac{1}{p(A = a | W = i)} \underbrace{\sum_{k = 1}^{d_\gU} p(U = k | W = i)}_{= 1} = \frac{1}{p(A = a | W = i)} \\
    &= \Bigg[\frac{1}{\mP(A = a | W)} \Bigg]_i,
\end{align*}
and that verifies the Equation (\ref{eq:AlternativeProxyDiscreteATEProofEquation4}). As a result, we proved that Equation (\ref{eq:AlternativeProxyDiscreteATEProofEquation3}) holds.

\textbf{{Step 3:}} We will further prove that 
\begin{align}
    \mP(Y| U, A = a) = \mP(Y, Z| A = a) \mP^{-T}(Z | U, A = a) \text{diag}\{\mP(U| A = a)\}^{-1}
\label{eq:AlternativeProxyDiscreteATEProofEquation5}
\end{align}

where 
\begin{align}
    \text{diag}\{\mP(U| A = a)\}^{-1} &= \begin{bmatrix}
        \frac{1}{p(U = 1| A = a)} & & &\text{\huge0}\\
         & \frac{1}{p(U = 2| A = a)} & &\\
         &   &   \ddots&            \\
        \text{\huge0} & & & \frac{1}{p(U = d_\gU| A = a)}
    \end{bmatrix} \in \R^{d_\gU \times d_\gU} \nonumber\\
    &= p(A = a)\begin{bmatrix}
        \frac{1}{p(A = a| U = 1) p(U = 1)} & & &\text{\huge0}\\
         & \frac{1}{p(A = a| U = 2) p(U = 2)} & &\\
         &   &   \ddots&            \\
        \text{\huge0} & & & \frac{1}{p(A = a| U = d_\gU) p(U = d_\gU)}
    \end{bmatrix}\nonumber\\
    &= p(A = a) \text{diag} \Big\{ \frac{1}{\mP(A = a | U)}\Big\}  \text{diag} \Big\{ \frac{1}{\mP(U)}\Big\}
\label{eq:AlternativeProxyDiscreteATEProofEquation6}\\
    &= p(A = a) \text{diag} \Big\{ {\mP(A = a | U)}\Big\}^{-1}  \text{diag} \Big\{ {\mP(U)}\Big\}^{-1}
\label{eq:AlternativeProxyDiscreteATEProofEquation7}
\end{align}

Next, consider 
\begin{align*}
    &\Big[ \text{diag}\{\mP(U| A = a)\} \mP^{T}(Z | U, A = a) \Big]_{ij} = \sum_{k = 1}^{d_\gU} \Big[ \text{diag}\{\mP(U| A = a)\}\Big]_{ik} p(Z = j| U = k, A = a)\\
    &= p(U = i| A = a)p(Z = j| U = i, A = a) \quad \text{(since $\Big[ \text{diag}\{\mP(U| A = a)\}\Big]_{ik}$ is nonzero iff $k  = i$)}\\
    &= p(Z = j, U = i| A = a) \quad \text{(by Baye's Rule)}
\end{align*}

Hence, this illustrates that 
\begin{align}
    {\text{diag}\{\mP(U| A = a)\} \mP^{T}(Z | U, A = a) = \mP^T(Z, U| A = a)} \label{eq:AlternativeProxyDiscreteATEProofEquation8}
\end{align}
where ${[\mP(Z, U| A = a)]}_{ij} = p(Z = i, U = j| A = a)$.

Furthermore, note that
\begin{align*}
    &\Big[ \mP(Y | U, A = a) \text{diag}\{\mP(U| A = a)\} \mP^{T}(Z | U, A = a)\Big]_{ij} = \Big[ \mP(Y | U, A = a) \mP^T(Z, U| A = a) \Big]_{ij} \quad \text{(by Eq. (\ref{eq:AlternativeProxyDiscreteATEProofEquation8}))}\\
    &= \sum_{k = 1}^{d_\gU} p(Y = i | U = k, A = a) p(Z = j, U = k| A = a)\\
    &= \sum_{k = 1}^{d_\gU} p(Y = i | U = k, A = a) p(Z = j|U = k, A = a) p(U = k| A = a)\\
    &= \sum_{k = 1}^{d_\gU} p(Y = i, Z = j | U = k, A = a)  p(U = k| A = a) \quad \text{(since $Y \perp Z | U, A$, Assumption (\ref{assum:ProxyCausalAssumptions1}))}\\
    &= \sum_{k = 1}^{d_\gU} p(Y = i, Z = j, U = k | A = a) = p(Y = i, Z = j, U = k | A = a) \\
    &= \Big[\mP(Y, Z| A = a)\Big]_{ij}.
\end{align*}

Thus, we showed that
\begin{align}
    \mP(Y | U, A = a) \text{diag}\{\mP(U| A = a)\} \mP^{T}(Z | U, A = a) = \mP(Y, Z| A = a)
\label{eq:AlternativeProxyDiscreteATEProofEquation9}
\end{align}
which also implies Equation (\ref{eq:AlternativeProxyDiscreteATEProofEquation5}). That is also equivalent to (due to Equation (\ref{eq:AlternativeProxyDiscreteATEProofEquation7}))
\begin{align}
    \mP(Y| U, A = a) = \mP(Y, Z| A = a) \mP^{-T}(Z | U, A = a) p(A = a)\text{diag}\Big\{\frac{1}{\mP(A = a |U)}\Big\} \text{diag}\Big\{\frac{1}{\mP(U)}\Big\}
\label{eq:AlternativeProxyDiscreteATEProofEquation10}
\end{align}

\textbf{{Step 4 (Combining all the steps above and finishing the proof):}} 

Finally, consider the dose-response
\begin{align*}
    f_{\text{ATE}}(a) &= \E_U[\E[Y | U, A = a]] = \vd_\gY^T \mP(Y | U, A = a) \mP(U)\\
    &= \vd_\gY^T \mP(Y, Z| A = a) \mP^{-T}(Z | U, A = a) p(A = a)\text{diag}\Big\{\frac{1}{\mP(A = a |U)}\Big\} \underbrace{\text{diag}\Big\{\frac{1}{\mP(U)}\Big\} \mP(U)}_{\vone}\\
    &\text{(the above equality is due to Equation (\ref{eq:AlternativeProxyDiscreteATEProofEquation10}))}\\
    &= \vd_\gY^T \mP(Y, Z| A = a) \mP^{-T}(Z | U, A = a) p(A = a)\text{diag}\Big\{\frac{1}{\mP(A = a |U)}\Big\} {\vone}\\
    &= \vd_\gY^T \mP(Y, Z| A = a) \underbrace{\mP^{-T}(Z | U, A = a) p(A = a)\frac{1}{\mP(A = a |U)^T}}_{\mP^{-T}(Z | W, A = a) p(A = a)\frac{1}{\mP(A = a |W)^T}}\\
    &= \vd_\gY^T \mP(Y, Z| A = a) {\mP^{-T}(Z | W, A = a) p(A = a)\frac{1}{\mP(A = a |W)^T}} \quad \text{(by Equation (\ref{eq:AlternativeProxyDiscreteATEProofEquation3}))}
\end{align*}
As a result, the ATE functions can be calculated by 
\begin{align*}
    {f_{\text{ATE}}(a) = \vd_\gY^T \mP(Y, Z| A = a) {\mP^{-T}(Z | W, A = a) p(A = a)\frac{1}{\mP(A = a | W)^T}}}
\end{align*}
where every component in the above equation is observed and the corresponding probability matrices can be estimated from the data.
\end{proof}

\section{KERNEL ALTERNATIVE PROXY METHOD WITH OBSERVABLE CONFOUNDERS}
\label{sec:KernelAlternativeProxyWithAdditionalCovariates}
In the section, we formulate the identifiability our proposed method when there exists observable confounding variables. In this setting, we consider the causal graph shown in Figure (\ref{fig:ProxyCausalDAGFigureWithX}). In addition to the variables $(A, Y, Z, W)$, we also assume that there exists observable confounding variables $X$. In this case, the structural functions of interest are defined as follows:

\begin{enumerate}
    \item[i-)] {Dose-response:} $f_{\text{ATE}}({a}) = \E[\E[Y | A = a, U, X]]$ 
    
    \item[ii-)] {Conditional dose-response:} $f_{\text{ATT}}(a, a') = \E[\E[Y | A = a, U, X] | A = a']$ 
\end{enumerate}

The conditional independence and completeness assumptions can be stated as follows:

\begin{assumption}
    We assume the following conditional independence statements: i-) $Y \perp Z  | U, X, A$ (Conditional Independence for $Y$), ii-) $ W \perp Z  | U, X, A = a$ and $W \perp A | U, X$ (Conditional Independence for $W$).
\label{assumption:proxy_with_X}
\end{assumption}



\begin{assumption}
    Let $\ell : \gU \rightarrow \R$ be any square integrable function. We assume that the following conditions hold for all $a \in \gA$, $x \in \gX$:
    \begin{itemize}
        \item $\E[\ell(U) | W = w, X = x, A = a] = 0 \quad \forall w \in \gW$ if and only if $\ell(U) = 0 \quad p(U) - $almost everywhere
        \item $\E[\ell(U) | Z = z, X = x, A = a] = 0 \quad \forall z \in \gZ$ if and only if $\ell(U) = 0 \quad p(U) - $almost everywhere
    \end{itemize}
    \label{assum:AlternativeProxyAssumptionCompleteness_with_X}
\end{assumption}

\begin{figure}[ht!]
\centering
\includegraphics[trim = {0cm 0cm 0cm 0.0cm},clip,width=0.55\textwidth]{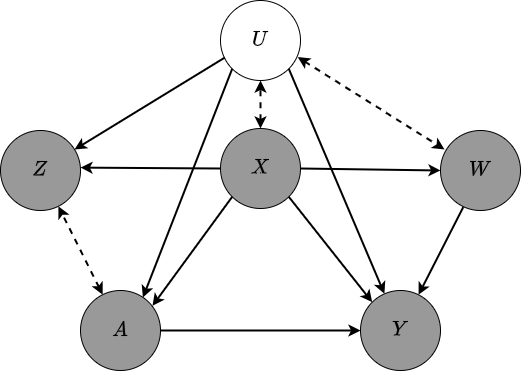}
\caption{An instance of a Directed Acyclic Graph (DAG) for the PCL setting, which satisfies the required Assumption (\ref{assumption:proxy_with_X}). In this graph, the gray circles denote the observed variables: $A$ denotes the treatment, $Y$ denotes the outcome, $X$ denotes the additional observable confounding variables, $Z$ denotes the treatment proxy, and $W$ denotes the outcome proxy. The white circle denotes the unobserved confounding variable $U$.
Bi-directional dotted arrows indicate that either direction in the DAG is possible, or that both variables may share a common ancestor.}
\hfill
\label{fig:ProxyCausalDAGFigureWithX}
\end{figure}

Next, we show the identifiability results of our proposed framework when there exist observable confounding variables $X \in \gX$ for both dose-response and conditional dose-response case.
\begin{theorem}
    Assume there exists a \emph{bridge function} $\varphi_0(z, x, a)$ that satisfies:
\begin{align*}
    \E[\varphi_0(Z, X, a) | W, X, A = a] = \frac{p(W | X) p(a)}{p(W, a | X)}.
\end{align*}
Given the Assumptions (\ref{assumption:proxy_with_X}) and (\ref{assum:AlternativeProxyAssumptionCompleteness_with_X}), the dose-response curve is given by
\begin{align*}
    f_{\text{ATE}}(a) = \E[Y \varphi_0(Z, X, a) | A = a]
\end{align*}
\end{theorem}

\begin{proof}
Suppose that 
\begin{align*}
    \E[\varphi_0(Z, X, a) | W, X, A = a] = \frac{p(W | X) p(a)}{p(W, a | X)}.
\end{align*}
Then, note the following, 

\begin{align}
\E[\varphi_0(Z, X, a) | W, X, A = a] &= \E_{U|W, X, A = a}[\E[\varphi_0(Z, X, a) | U, W, X, A = a]]  \quad \text{(by Law of Total Expectations)}\nonumber\\
&= \E_{U|W, X, A = a}[\E[\varphi_0(Z, X, a) | U, X, A = a]] \quad \text{(since $Z \perp W | U, X$, Assumption (\ref{assumption:proxy_with_X}))}\label{eq:ATEProofExpectationOfBridgeWRT_U_withConfounders}
\end{align}

Furthermore, note that
\begin{align*}
    p(w | x) &= \int p(w| u, x) p(u | x) d u\\
    &= \int p(w| a, u, x) p(u | x) d u \quad \text{(since $W \perp A | U, X$, Assumption (\ref{assumption:proxy_with_X}))}\\
    &= \int \frac{p(u | w, a, x) p(w | a, x)}{p(u | a, x)} p(u | x) d u \quad (\text{Baye's Rule})\\
    &= \int \frac{p(u | w, a, x) p(w , a | x)}{p(u , a | x)} p(u | x) d u\\
    &= p(w, a | x) \int \frac{p(u | x)}{p(u, a | x)} p(u | w, a, x) d u\\
    &= p(w, a | x) \E \Bigg[  \frac{p(U | x)}{p(U, a | x)} \Bigg| W = w, A = a, X = x \Bigg]
\end{align*}

As a result,
\begin{align*}
    \frac{p(w | X)}{p(w, a | X)} = \E \Bigg[  \frac{p(U | X)}{p(U, a | X)} \Bigg| W = w, A = a, X \Bigg].
\end{align*}
Hence,
\begin{align*}
    \frac{p(w | X) p(a)}{p(w, a | X)} = \E \Bigg[  \frac{p(U | X) p(a)}{p(U, a | X)} \Bigg| W = w, A = a, X \Bigg].
\end{align*}

Using the assumption of the Theorem and Equation (\ref{eq:ATEProofExpectationOfBridgeWRT_U_withConfounders}), we see that
\begin{align*}
   \E_{U|W = w, X, A = a}[\E[\varphi_0(Z, X, a) | U, X, A = a]] = \E_{U |W = w, X, A = a } \Bigg[  \frac{p(U | X) p(a)}{p(U, a | X)}\Bigg]
\end{align*}
Therefore, due to the Assumption (\ref{assum:AlternativeProxyAssumptionCompleteness_with_X}), we have
\begin{align}
    \E[\varphi_0(Z, X, a) | U, X, A = a ] = \frac{p(U | X) p(a)}{p(U, a | X)} \quad \text{almost surely.}
\label{eq:AlternativeProxyBridgeFuncProposition3}
\end{align}

Next, we observe that
\begin{align*}
    &\E[ \E[Y | A = a, U, X]] = \int \E[Y | A = a, u, x] p(u, x) d u d x\\
    &= \int \E[Y | A = a, u, x] \frac{p(u, x) p(a)}{p(u, a, x)} \frac{p(u, a, x)}{p(a)} d u dx\\
    &= \int \E[Y | A = a, u, x] \frac{p(u, x) p(a)}{p(u, a, x)} p(u, x| a) d u dx\\
    &= \E_{X, U | A = a} \Bigg[ \E[Y | A = a, U, X] \frac{p(U, X) p(a)}{p(U, a, X)} \Bigg]\\
    &= \E_{X, U | A = a} \Bigg[ \E[Y | A = a, U, X] \E[\varphi_0(Z, X, a) | U, X, A = a ] \Bigg] \quad \text{(by Equation \ref{eq:AlternativeProxyBridgeFuncProposition3})}\\
    &= \E_{X, U | A = a} \Bigg[ \int y p(y | A = a, U, X) d y \int \varphi_0(z, X, a) p(z| U, X, A = a) d z \Bigg]\\
    &= \E_{X, U | A = a} \Bigg[\int \int\varphi_0(z, X, a)  y \underbrace{p(y | A = a, U, X, z)   p(z| U, X, A = a)}_{p(y, z| A = a, U, X)} d y  d z \Bigg] \text{(since $Y \perp Z | A = a, U, X $, Assump. (\ref{assumption:proxy_with_X}))}\\
    &= \int \int \int \int \varphi_0(z, x, a)  y \underbrace{p(y, z| A = a, u, x) p(x, u | A = a)}_{p(u,x, y, z | A = a)} d y  d z du dx\\
    &= \int \int \int \varphi_0(z, x, a)  y \underbrace{\int p(u, x, y, z | A = a) d u}_{p( y, z, x | A = a)} d y  d z dx\\
    &= \int \int \int \varphi_0(z, x, a)  y p( y, z, x | A = a) d y  d z dx = \E[Y \varphi_0(Z, X, a) | A = a].
\end{align*}
As a result, we obtained
\begin{align*}
    \E[Y\varphi_0(Z, X, a) | A = a] = \E_{X, U}[ \E[Y | A = a, X, U]],
\end{align*}
which indicates that $f_{\text{ATE}}(a) = \E[Y \varphi_0(Z, X, a) | A = a]$ and finishes the proof.

\end{proof}

\begin{theorem}
    Assume there exists a \emph{bridge function} $\varphi_0(z, x, a, a')$ that satisfies:
\begin{align*}
    \E[\varphi_0(Z, X, a, a') | W, X, A = a] = \frac{p(W, a' | X) p(a)}{p(W, a | X) p(a')}.
\end{align*}
Given the Assumptions (\ref{assumption:proxy_with_X}) and (\ref{assum:AlternativeProxyAssumptionCompleteness_with_X}), the conditional dose-response curve is given by
\begin{align*}
    f_{\text{ATT}}(a, a') = \E[Y \varphi_0(Z, X, a, a') | A = a]
\end{align*}
\end{theorem}

\begin{proof}
First, observe the following
\begin{align}
    &\E[\varphi_0(Z, X, a, a') | W, X, A = a] = \E_{U| W, X, A = a}\big[\E[\varphi_0(Z, X, a, a') | U, W, X, A = a] \big] \quad \text{(by Law of Total Expectations)} \nonumber\\
    &= \E_{U| W, X, A = a}\big[\E[\varphi_0(Z, X, a, a') | U, X, A = a] \big] \quad \text{(since $Z \perp W | U, X, A = a$, Assumption (\ref{assumption:proxy_with_X}))}. \label{eq:AlternativeProxyATTPropositionImportantEquation4}
\end{align}

Furthermore, note that
\begin{align*}
    p(w, a' | x) &= \int p(w, a' | u, x) p(u | x) d u = \int p(w | u, x) p(a' | u, x) p(u | x) d u \quad \text{(since $W \perp A | U, X$, Assumption (\ref{assumption:proxy_with_X}))}\\
    &= \int p(w | a, u, x) p(a' | u, x) p(u | x) d u \quad \text{(again due to $W \perp A | U, X$, Assumption (\ref{assumption:proxy_with_X}))}\\
    &= \int\frac{ p(u | w, x, a) p(w | a, x)}{p(u | a, x)} p(a' | u, x) p(u | x) d u \quad \text{(Baye's Rule)}\\
    &= \int\frac{ p(u | w, x, a) p(w , a | x)}{p(u , a | x)} p(u, a' | x)  d u  = p(w , a | x) \int \frac{p(u, a' | x)}{p(u, a | x)} p(u | w, x, a) d u.
\end{align*}

As a result,
\begin{align*}
    \frac{p(w, a' | x)}{p(w , a | x)} = \int \frac{p(u, a' | x)}{p(u, a | x)} p(u | w, x, a) d u.
\end{align*}

Hence
\begin{align}
    \frac{p(w, a' | x) p(a)}{p(w , a | x) p(a')} = \E_{U | W = w, X = x, A = a} \Bigg[ \frac{p(U, a' | x) p(a)}{p(U, a | x) p(a')} \Bigg].
\label{eq:AlternativeProxyATTPropositionImportantEquation5}
\end{align}   

Recall that our assumption was
\begin{align*}
\E[\varphi_0(Z, X, a, a') | W, X, A = a] = \frac{p(W, a' | X) p(a)}{p(W, a | X) p(a)}
\end{align*}
Thus, combining Equation (\ref{eq:AlternativeProxyATTPropositionImportantEquation4}) and Equation (\ref{eq:AlternativeProxyATTPropositionImportantEquation5}) yields
\begin{align*}
    \E_{U | W = w, X, A = a} \Big[ \E[\varphi_0(Z, X, a, a') | U, X, A = a] \Big] = \E_{U | W = w, X, A = a} \Bigg[ \frac{p(U, a' | X) p(a)}{p(U, a | X) p(a')} \Bigg].
\end{align*}

Using the completeness Assumption (\ref{assum:AlternativeProxyAssumptionCompleteness_with_X}), we obtain
\begin{align}
    \E[\varphi_0(Z, X, a, a') | U, X, A = a] = \frac{p(U, a' | X) p(a)}{p(U, a | X) p(a')} \quad \text{almost surely} \label{eq:AlternativeProxyATTPropositionImportantEquation6}
\end{align}

Next, to obtain the ATT function, consider
\begin{align*}
    &f_{\text{ATT}}(a, a') = \E_{U, X | A = a'}[\E[Y | A = a, U, X ]] =\int \E[Y | A = a, u, x ] p(u, x | A = a') d u\\
    &=\int \E[Y | A = a, u, x ] \frac{p(u, x, A = a')}{p(a')} \frac{p(a)}{p(u, x, A = a)} \frac{p(u, x, A = a)}{p(a)} d u\\
    &=\int \E[Y | A = a, u, x ] \frac{p(u, a' | x) p(a)}{p(u, a | x) p(a')}p(u, x | A = a) d u\\
    &= \E_{U, X | A = a} \Big[ \E[Y | A = a, U, X] \E[\varphi_0(Z, X, a, a') | U, X, A = a] \Big] \quad \text{(by Equation (\ref{eq:AlternativeProxyATTPropositionImportantEquation6}))}\\
    &= \E_{U, X | A = a} \Bigg[ \int y p(y| A = a, U, X) d y \int \varphi_0(z, X, a, a') p(z | A = a, U, X) d z
 \Bigg]\\
  &= \E_{U, X | A = a} \Bigg[ \int \int y \varphi_0(z, X, a, a') \underbrace{p(y| A = a, U, X, z)  p(z | A = a, U, X)}_{p(y, z|A = a, U, X)} d y d z
 \Bigg]  \\
 &= \E_{U, X | A = a} \Bigg[ \int \int y \varphi_0(z, X, a, a') p(y, z|A = a, U, X) d y d z
 \Bigg] \quad \text{($Y \perp Z | U, X ,A$, Assumption (\ref{assumption:proxy_with_X}) )}\\
 &= \int \int \int \int y \varphi_0(z, x, a, a') \underbrace{p(y, z|A = a, U = u, X = x) p(u, x | A = a)}_{p(u, x, y, z | A = a)}  d y d z d u d x\\
 &= \int \int \int  y \varphi_0(z, x, a, a') \underbrace{\int p(u, x, y, z | A = a) d u }_{p(y, z, x|A = a)} dx d y d z \\
 &=\int \int y \varphi_0(z, x, a, a') p(y, z, x|A = a) dx d y d z = \E [Y \varphi_0(Z, X, a, a') | A = a].
\end{align*}
Hence, we have shown that $f_{\text{ATT}}(a, a') = \E_{U, X | A = a'}[\E[Y | A = a, U, X ]] = \E [Y \varphi_0(Z, X, a, a') | A = a]$.
\end{proof}

\end{document}

%% file: math_commands.tex
\usepackage{amsmath,amsfonts,bm}

\def\eqref#1{equation~\ref{#1}}

\def\1{\bm{1}}

\def\rvepsilon{{\bm{\epsilon}}}

\def\vone{{\bm{1}}}

\def\vtheta{{\bm{\theta}}}
\def\va{{\bm{a}}}

\def\vd{{\bm{d}}}

\def\mA{{\bm{A}}}
\def\mB{{\bm{B}}}
\def\mC{{\bm{C}}}
\def\mD{{\bm{D}}}
\def\mE{{\bm{E}}}

\def\mH{{\bm{H}}}
\def\mI{{\bm{I}}}

\def\mK{{\bm{K}}}
\def\mL{{\bm{L}}}
\def\mM{{\bm{M}}}
\def\mN{{\bm{N}}}

\def\mP{{\bm{P}}}

\def\mR{{\bm{R}}}
\def\mS{{\bm{S}}}

\def\mW{{\bm{W}}}
\def\mX{{\bm{X}}}
\def\mY{{\bm{Y}}}
\def\mZ{{\bm{Z}}}
\def\mBeta{{\bm{\beta}}}
\def\malpha{{\bm{\alpha}}}

\DeclareMathAlphabet{\mathsfit}{\encodingdefault}{\sfdefault}{m}{sl}
\SetMathAlphabet{\mathsfit}{bold}{\encodingdefault}{\sfdefault}{bx}{n}

\def\gA{{\mathcal{A}}}

\def\gF{{\mathcal{F}}}
\def\gG{{\mathcal{G}}}
\def\gH{{\mathcal{H}}}

\def\gL{{\mathcal{L}}}

\def\gN{{\mathcal{N}}}

\def\gU{{\mathcal{U}}}

\def\gW{{\mathcal{W}}}
\def\gX{{\mathcal{X}}}
\def\gY{{\mathcal{Y}}}
\def\gZ{{\mathcal{Z}}}

\newcommand{\E}{\mathbb{E}}

\newcommand{\R}{\mathbb{R}}

\newcommand{\bboz}[1]{{\color{blue} #1}}

\DeclareMathOperator*{\argmin}{arg\,min}

%% file: main_kernel_alternative_proxy.bbl
\begin{thebibliography}{}

\bibitem[Alabdulmohsin et~al., 2023]{pmlr-v206-alabdulmohsin23a}
Alabdulmohsin, I., Chiou, N., D'Amour, A., Gretton, A., Koyejo, S., Kusner, M.~J., Pfohl, S.~R., Salaudeen, O., Schrouff, J., and Tsai, K. (2023).
\newblock Adapting to latent subgroup shifts via concepts and proxies.
\newblock In Ruiz, F., Dy, J., and van~de Meent, J.-W., editors, {\em Proceedings of The 26th International Conference on Artificial Intelligence and Statistics}, volume 206 of {\em Proceedings of Machine Learning Research}, pages 9637--9661. PMLR.

\bibitem[Aubin, 2011]{aubin2011applied}
Aubin, J.-P. (2011).
\newblock {\em Applied functional analysis}.
\newblock John Wiley \& Sons.

\bibitem[Bang and Robins, 2005]{BanRob05}
Bang, H. and Robins, J. (2005).
\newblock Doubly robust estimation in missing data and causal inference models.
\newblock {\em Biometrics}, 61:962–--972.

\bibitem[Ben-Israel and Greville, 2006]{ben2006generalized}
Ben-Israel, A. and Greville, T.~N. (2006).
\newblock {\em Generalized inverses: theory and applications}.
\newblock Springer Science \& Business Media.

\bibitem[Blanchard and M{\"u}cke, 2018]{blanchard2018optimal}
Blanchard, G. and M{\"u}cke, N. (2018).
\newblock Optimal rates for regularization of statistical inverse learning problems.
\newblock {\em Foundations of Computational Mathematics}, 18(4):971--1013.

\bibitem[Blundell et~al., 2012]{Blundell_measuringPriceResponsiveness}
Blundell, R., Horowitz, J., and Parey, M. (2012).
\newblock Measuring the price responsiveness of gasoline demand: Economic shape restrictions and nonparametric demand estimation.
\newblock {\em Quantitative Economics}, 3:29--51.

\bibitem[Buldygin and Kozachenko, 2000]{buldygin2000metric}
Buldygin, V.~V. and Kozachenko, I.~V. (2000).
\newblock {\em Metric characterization of random variables and random processes}, volume 188.
\newblock American Mathematical Soc.

\bibitem[Caponnetto and De~Vito, 2007]{caponnetto2007optimal}
Caponnetto, A. and De~Vito, E. (2007).
\newblock Optimal rates for the regularized least-squares algorithm.
\newblock {\em Foundations of Computational Mathematics}, 7:331--368.

\bibitem[Choi et~al., 2002]{CHOI20021173}
Choi, H.~K., Hernán, M.~A., Seeger, J.~D., Robins, J.~M., and Wolfe, F. (2002).
\newblock Methotrexate and mortality in patients with rheumatoid arthritis: a prospective study.
\newblock {\em The Lancet}, 359(9313):1173--1177.

\bibitem[Connors et~al., 1996]{The_effectiveness_of_right_heart_catheterization}
Connors, A., Speroff, T., Dawson, N., Thomas, C., Harrell, F., Wagner, D., Desbiens, N., Goldman, L., Wu, A., Califf, R., Fulkerson, W., Vidaillet, H., Broste, S., Bellamy, P., Lynn, J., and Knaus, W. (1996).
\newblock The effectiveness of right heart catheterization in the initial care of critically ill patients.
\newblock {\em Journal of the American Medical Association}, 276(11):889--897.

\bibitem[Cui et~al., 2024]{semiparametricProximalCausalInference}
Cui, Y., Pu, H., Shi, X., Miao, W., and Tchetgen~Tchetgen, E. (2024).
\newblock Semiparametric proximal causal inference.
\newblock {\em Journal of the American Statistical Association}, 119(546):1348--1359.

\bibitem[Deaner, 2023]{deaner2023proxycontrolspaneldata}
Deaner, B. (2023).
\newblock Proxy controls and panel data.

\bibitem[Donohue and Levitt, 2001]{LegalizedAbortionData}
Donohue, John~J., I. and Levitt, S.~D. (2001).
\newblock {The Impact of Legalized Abortion on Crime*}.
\newblock {\em The Quarterly Journal of Economics}, 116(2):379--420.

\bibitem[Engl et~al., 1996]{engl1996regularization}
Engl, H.~W., Hanke, M., and Neubauer, A. (1996).
\newblock {\em Regularization of inverse problems}, volume 375.
\newblock Springer Science \& Business Media.

\bibitem[Fischer and Steinwart, 2020]{fischer2020sobolev}
Fischer, S. and Steinwart, I. (2020).
\newblock Sobolev norm learning rates for regularized least-squares algorithms.
\newblock {\em Journal of Machine Learning Research}, 21(205):1--38.

\bibitem[Flores et~al., 2012]{Flores_JobCorps}
Flores, C.~A., Flores-Lagunes, A., Gonzalez, A., and Neumann, T.~C. (2012).
\newblock Estimating the effects of length of exposure to instruction in a training program: The case of job corps.
\newblock {\em The Review of Economics and Statistics}, 94(1):153--171.

\bibitem[Fruehwirth et~al., 2016]{Fruehwirth_GradeRetentions}
Fruehwirth, J.~C., Navarro, S., and Takahashi, Y. (2016).
\newblock {How the Timing of Grade Retention Affects Outcomes: Identification and Estimation of Time-Varying Treatment Effects}.
\newblock {\em Journal of Labor Economics}, 34(4):979--1021.

\bibitem[Gretton, 2013]{gretton2013introduction}
Gretton, A. (2013).
\newblock Introduction to {RKHS}, and some simple kernel algorithms.
\newblock Advanced Topics in Machine Learning lecture, University College London.

\bibitem[Gr{\"u}new{\"a}lder et~al., 2012a]{grunewalder2012conditional}
Gr{\"u}new{\"a}lder, S., Lever, G., Baldassarre, L., Patterson, S., Gretton, A., and Pontil, M. (2012a).
\newblock Conditional mean embeddings as regressors.
\newblock In {\em International Conference on Machine Learningg}.

\bibitem[Gr{\"u}new{\"a}lder et~al., 2012b]{GruLevBalPonetal12}
Gr{\"u}new{\"a}lder, S., Lever, G., Baldassarre, L., Pontil, M., and Gretton, A. (2012b).
\newblock Modelling transition dynamics in mdps with rkhs embeddings.
\newblock In {\em International Conference on Machine Learning}.

\bibitem[Higgins et~al., 2017]{higgins2017betavae}
Higgins, I., Matthey, L., Pal, A., Burgess, C., Glorot, X., Botvinick, M., Mohamed, S., and Lerchner, A. (2017).
\newblock beta-{VAE}: Learning basic visual concepts with a constrained variational framework.
\newblock In {\em International Conference on Learning Representations}.

\bibitem[Hill, 2011]{Hill_BayesianNonparametric}
Hill, J.~L. (2011).
\newblock Bayesian nonparametric modeling for causal inference.
\newblock {\em Journal of Computational and Graphical Statistics}, 20:217--240.

\bibitem[Johansson et~al., 2016]{pmlr_v48_johansson16}
Johansson, F., Shalit, U., and Sontag, D. (2016).
\newblock Learning representations for counterfactual inference.
\newblock In {\em International Conference on Machine Learning}.

\bibitem[Kallus et~al., 2021]{Kallus2021Causal}
Kallus, N., Mao, X., and Uehara, M. (2021).
\newblock Causal inference under unmeasured confounding with negative controls: A minimax learning approach.

\bibitem[Kanamori et~al., 2009]{Kanamori_LeastSquareImportance}
Kanamori, T., Hido, S., and Sugiyama, M. (2009).
\newblock A least-squares approach to direct importance estimation.
\newblock {\em J. Mach. Learn. Res.}, 10:1391–1445.

\bibitem[Klebanov et~al., 2020]{klebanov2020rigorous}
Klebanov, I., Schuster, I., and Sullivan, T.~J. (2020).
\newblock A rigorous theory of conditional mean embeddings.
\newblock {\em SIAM Journal on Mathematics of Data Science}, 2(3):583--606.

\bibitem[Kompa et~al., 2022]{kompa2022deep}
Kompa, B., Bellamy, D., Kolokotrones, T., Beam, A., et~al. (2022).
\newblock Deep learning methods for proximal inference via maximum moment restriction.
\newblock {\em Advances in Neural Information Processing Systems}.

\bibitem[Kress, 2013]{linearIntegralEquationskress2013}
Kress, R. (2013).
\newblock {\em Linear Integral Equations}.
\newblock Applied Mathematical Sciences. Springer New York.

\bibitem[Kuroki and Pearl, 2014]{Kuroki2014Mesurement}
Kuroki, M. and Pearl, J. (2014).
\newblock {Measurement bias and effect restoration in causal inference}.
\newblock {\em Biometrika}, 101(2):423--437.

\bibitem[Li et~al., 2022]{li2022optimal}
Li, Z., Meunier, D., Mollenhauer, M., and Gretton, A. (2022).
\newblock Optimal rates for regularized conditional mean embedding learning.
\newblock {\em Advances in Neural Information Processing Systems}.

\bibitem[Li et~al., 2024]{li2024towards}
Li, Z., Meunier, D., Mollenhauer, M., and Gretton, A. (2024).
\newblock Towards optimal sobolev norm rates for the vector-valued regularized least-squares algorithm.
\newblock {\em Journal of Machine Learning Research}, 25(181):1--51.

\bibitem[Mastouri et~al., 2021]{Mastouri2021ProximalCL}
Mastouri, A., Zhu, Y., Gultchin, L., Korba, A., Silva, R., Kusner, M.~J., Gretton, A., and Muandet, K. (2021).
\newblock Proximal causal learning with kernels: Two-stage estimation and moment restriction.
\newblock In {\em International Conference on Machine Learning}.

\bibitem[Matthey et~al., 2017]{dsprites17}
Matthey, L., Higgins, I., Hassabis, D., and Lerchner, A. (2017).
\newblock dsprites: Disentanglement testing sprites dataset.
\newblock https://github.com/deepmind/dsprites-dataset/.

\bibitem[Meanti et~al., 2022]{pmlr-v151-meanti22a}
Meanti, G., Carratino, L., De~Vito, E., and Rosasco, L. (2022).
\newblock Efficient hyperparameter tuning for large scale kernel ridge regression.
\newblock In Camps-Valls, G., Ruiz, F. J.~R., and Valera, I., editors, {\em Proceedings of The 25th International Conference on Artificial Intelligence and Statistics}, volume 151 of {\em Proceedings of Machine Learning Research}, pages 6554--6572. PMLR.

\bibitem[Meunier et~al., 2024]{meunier2024nonparametric}
Meunier, D., Li, Z., Christensen, T., and Gretton, A. (2024).
\newblock Nonparametric instrumental regression via kernel methods is minimax optimal.
\newblock {\em arXiv preprint arXiv:2411.19653}.

\bibitem[Meunier et~al., 2023]{meunier2023nonlinear}
Meunier, D., Li, Z., Gretton, A., and Kpotufe, S. (2023).
\newblock Nonlinear meta-learning can guarantee faster rates.
\newblock {\em arXiv preprint arXiv:2307.10870}.

\bibitem[Meunier et~al., 2025]{meunier2025optimal}
Meunier, D., Shen, Z., Mollenhauer, M., Gretton, A., and Li, Z. (2025).
\newblock Optimal rates for vector-valued spectral regularization learning algorithms.
\newblock {\em Advances in Neural Information Processing Systems}, 37:82514--82559.

\bibitem[Miao et~al., 2018]{Miao2018Identifying}
Miao, W., Geng, Z., and {Tchetgen Tchetgen}, E. (2018).
\newblock Identifying causal effects with proxy variables of an unmeasured confounder.
\newblock {\em Biometrika}, 105(4):987—993.

\bibitem[Mollenhauer and Koltai, 2020]{mollenhauer2020nonparametric}
Mollenhauer, M. and Koltai, P. (2020).
\newblock Nonparametric approximation of conditional expectation operators.
\newblock {\em arXiv preprint arXiv:2012.12917}.

\bibitem[Mollenhauer et~al., 2022]{mollenhauer2022learning}
Mollenhauer, M., M{\"u}cke, N., and Sullivan, T. (2022).
\newblock Learning linear operators: Infinite-dimensional regression as a well-behaved non-compact inverse problem.
\newblock {\em arXiv preprint arXiv:2211.08875}.

\bibitem[Muandet et~al., 2020]{Muandet2020KernelCM}
Muandet, K., Jitkrittum, W., and K{\"u}bler, J.~M. (2020).
\newblock Kernel conditional moment test via maximum moment restriction.
\newblock In {\em Conference on Uncertainty in Artificial Intelligence}.

\bibitem[Park and Muandet, 2020]{park2020measure}
Park, J. and Muandet, K. (2020).
\newblock A measure-theoretic approach to kernel conditional mean embeddings.
\newblock {\em Advances in Neural Information Processing Systems}.

\bibitem[Pearl, 2009]{Pearl_2009}
Pearl, J. (2009).
\newblock {\em Causality}.
\newblock Cambridge University Press, 2 edition.

\bibitem[Pearl and Robins, 1995]{peaRob95}
Pearl, J. and Robins, J. (1995).
\newblock Probabilistic evaluation of sequential plans from causal models with hidden variables.
\newblock In {\em Uncertainty in Artificial Intelligence: Proceedings of the Eleventh Conference on Artificial Intelligence}, pages 444--453.

\bibitem[Pinelis, 1994]{pinelis1994optimum}
Pinelis, I. (1994).
\newblock Optimum bounds for the distributions of martingales in banach spaces.
\newblock {\em The Annals of Probability}, pages 1679--1706.

\bibitem[Rosenbaum and Rubin, 1983]{rosRub83}
Rosenbaum, P. and Rubin, D. (1983).
\newblock The central role of the propensity score in observational studies for causal effects.
\newblock {\em Biometrika}, 70:41--55.

\bibitem[Schochet et~al., 2008]{Schochet_JobCorps}
Schochet, P.~Z., Burghardt, J., and McConnell, S. (2008).
\newblock Does job corps work? impact findings from the national job corps study.
\newblock {\em American Economic Review}, 98(5):1864–86.

\bibitem[Singh, 2023]{singh2023kernelmethodsunobservedconfounding}
Singh, R. (2023).
\newblock Kernel methods for unobserved confounding: Negative controls, proxies, and instruments.

\bibitem[Singh et~al., 2023]{RahulKernelCausalFunctions}
Singh, R., Xu, L., and Gretton, A. (2023).
\newblock {Kernel methods for causal functions: dose, heterogeneous and incremental response curves}.
\newblock {\em Biometrika}, 111(2):497--516.

\bibitem[Smola et~al., 2007]{HilbertSpaceEmbeddingforDistributions}
Smola, A., Gretton, A., Song, L., and Sch{\"o}lkopf, B. (2007).
\newblock A hilbert space embedding for distributions.
\newblock In Hutter, M., Servedio, R.~A., and Takimoto, E., editors, {\em Algorithmic Learning Theory}, pages 13--31, Berlin, Heidelberg. Springer Berlin Heidelberg.

\bibitem[Song et~al., 2009]{HilbertSpaceEmbeddingforConditionalDistributions}
Song, L., Huang, J., Smola, A., and Fukumizu, K. (2009).
\newblock Hilbert space embeddings of conditional distributions with applications to dynamical systems.
\newblock In {\em International Conference on Machine Learning}.

\bibitem[Steinwart and Christmann, 2008]{Ingo_support_vector_machines}
Steinwart, I. and Christmann, A. (2008).
\newblock {\em Support Vector Machines}.
\newblock Springer Publishing Company, Incorporated, 1st edition.

\bibitem[Tsai et~al., 2024]{pmlr-v238-tsai24b}
Tsai, K., R~Pfohl, S., Salaudeen, O., Chiou, N., Kusner, M., D'Amour, A., Koyejo, S., and Gretton, A. (2024).
\newblock Proxy methods for domain adaptation.
\newblock In Dasgupta, S., Mandt, S., and Li, Y., editors, {\em Proceedings of The 27th International Conference on Artificial Intelligence and Statistics}, volume 238 of {\em Proceedings of Machine Learning Research}, pages 3961--3969. PMLR.

\bibitem[Wang~Miao and Tchetgen, 2024]{Miao01102024}
Wang~Miao, Xu~Shi, Y.~L. and Tchetgen, E. J.~T. (2024).
\newblock A confounding bridge approach for double negative control inference on causal effects.
\newblock {\em Statistical Theory and Related Fields}, 8(4):262--273.

\bibitem[Wasserman, 2006]{Wasserman06}
Wasserman, L. (2006).
\newblock {\em All of Nonparametric Statistics}.
\newblock Springer.

\bibitem[Woody et~al., 2020]{woody2020estimatingheterogeneouseffectscontinuous}
Woody, S., Carvalho, C.~M., Hahn, P.~R., and Murray, J.~S. (2020).
\newblock Estimating heterogeneous effects of continuous exposures using bayesian tree ensembles: revisiting the impact of abortion rates on crime.

\bibitem[Wu et~al., 2024]{wu2024doubly}
Wu, Y., Fu, Y., Wang, S., and Sun, X. (2024).
\newblock Doubly robust proximal causal learning for continuous treatments.
\newblock In {\em International Conference on Learning Representations}.

\bibitem[Xu and Gretton, 2023]{xu_dSprite}
Xu, L. and Gretton, A. (2023).
\newblock Causal benchmark based on disentangled image dataset.

\bibitem[Xu and Gretton, 2024]{xu2024kernelsingleproxycontrol}
Xu, L. and Gretton, A. (2024).
\newblock Kernel single proxy control for deterministic confounding.

\bibitem[Xu et~al., 2021]{xu2021deep}
Xu, L., Kanagawa, H., and Gretton, A. (2021).
\newblock Deep proxy causal learning and its application to confounded bandit policy evaluation.
\newblock In {\em Advances in Neural Information Processing Systems}.

\bibitem[Yao et~al., 2018]{NEURIPS2018_a50abba8}
Yao, L., Li, S., Li, Y., Huai, M., Gao, J., and Zhang, A. (2018).
\newblock Representation learning for treatment effect estimation from observational data.
\newblock In {\em Advances in Neural Information Processing Systems}.

\end{thebibliography}
